\documentclass[11pt,hyphens]{article}
\usepackage[margin=1in]{geometry}
\usepackage[utf8]{inputenc}
\usepackage[T1]{fontenc}
\usepackage{hyperref}
\usepackage{booktabs}
\usepackage{amsfonts,amsthm,amsmath,amssymb}
\usepackage{nicefrac}
\usepackage{microtype}
\usepackage{xcolor}
\usepackage{cleveref}
\usepackage{csquotes}
\usepackage{graphicx}
\usepackage{subcaption}
\usepackage{enumitem}
\usepackage[most]{tcolorbox}
\usepackage{fancyhdr}
\usepackage{listings}

\usepackage{algorithm}
\usepackage{algpseudocode}
\usepackage{caption}
\usepackage{tabularx}
\usepackage{multicol}
\usepackage{systeme}
\usepackage{tikz}
\tcbuselibrary{breakable, skins}
\usetikzlibrary{decorations.pathreplacing, positioning, graphs, graphs.standard, shapes, arrows, arrows.meta, fit, matrix, patterns, shapes.misc, tikzmark}
\usepackage{pgfplots}
\pgfplotsset{compat=1.5}
\usepackage{framed}
\usepackage{thmtools}
\usepackage{thm-restate}
\usepackage{authblk}

\usepackage{pdfpages}
\usepackage{placeins}

\definecolor{light-blue-bg}{RGB}{235, 245, 255}
\definecolor{light-gray-bg}{RGB}{242, 242, 242}
\definecolor{mahogany}{rgb}{0.75, 0.25, 0.0}
\definecolor{darkblue}{rgb}{0.0, 0.0, 0.55}
\definecolor{darkpastelgreen}{rgb}{0.01, 0.75, 0.24}
\definecolor{forestgreen}{rgb}{0.0, 0.27, 0.13}
\definecolor{crimson}{rgb}{0.86, 0.08, 0.24}
\definecolor{airforceblue}{rgb}{0.36, 0.54, 0.66}
\definecolor{mygreen}{rgb}{0.0, 0.55, 0.0}
\definecolor{blue-violet}{rgb}{0.54, 0.17, 0.89}
\definecolor{sand}{RGB}{255, 248, 238}
\definecolor{chatbg}{RGB}{255, 248, 238}
\definecolor{userbg}{RGB}{238,244,255}
\definecolor{assistantbg}{RGB}{246,240,255}
\definecolor{boxblue}{RGB}{0, 0, 150}
\definecolor{boxback}{RGB}{245, 245, 255}

\hypersetup{
    colorlinks   = true,
    citecolor    = mahogany,
    linkcolor    = blue,
    urlcolor     = darkblue
}

\newtheorem{theorem}{Theorem}[section]
\newtheorem{lemma}[theorem]{Lemma}
\newtheorem{corollary}[theorem]{Corollary}
\newtheorem{conjecture}[theorem]{Conjecture}
\newtheorem{definition}[theorem]{Definition}
\newtheorem{remark}[theorem]{Remark}
\newtheorem{problem}[theorem]{Problem}

\newtheorem{fact}[theorem]{Fact}

\newtheorem{assumption}[theorem]{Assumption}

\newtcolorbox{userbox}{
    colback=userbg,
    colframe=userbg,
    arc=3mm,
    width=0.9\textwidth,
    halign=left,
}

\newtcolorbox{assistantbox}{
    colback=assistantbg,
    colframe=assistantbg,
    arc=3mm,
    width=0.9\textwidth,
    halign=left,
}

\newenvironment{chatbox}
{%
    \begin{tcolorbox}[
        colframe=boxblue,
        boxrule=0.8pt,
        arc=6mm,
        enhanced,
        width=\textwidth,
        left=2mm, right=2mm,
        top=2mm, bottom=2mm,
        interior style={fill=chatbg}, 
        before skip=10pt plus 2pt,
        after skip=10pt plus 2pt,
        breakable
    ]
}
{%
    \end{tcolorbox}
}

\newenvironment{userturn}
{%
    \small
    \begin{tcolorbox}[
        breakable,
        enhanced,
        flush left,
        width=0.95\textwidth,
        colback=userbg,
        colframe=gray!120,
        boxrule=1.3pt,                  
    ]
}
{%
    \end{tcolorbox}
}

\newenvironment{assistantturn}
{%
    \small
    \begin{tcolorbox}[
        breakable,
        enhanced,
        flush right,
        width=0.95\textwidth,
        colback=assistantbg,
        colframe=gray!120,
        boxrule=1.3pt,                  
    ]
}
{%
    \end{tcolorbox}
}

\usepackage{mdframed}
\usepackage{xcolor}
\newenvironment{aigenerated}
{%
    \begin{tcolorbox}[
        breakable,                     
        enhanced,                       
        colback=assistantbg,                
        colframe=gray!120,               
        boxrule=1.3pt,                  
        arc=5pt,                        
        top=10pt, bottom=10pt,          
        left=6pt, right=6pt             
    ]
}
{%
    \end{tcolorbox}
}

\newcommand{\E}{\mathbb{E}}
\newcommand{\R}{\mathbb{R}}

\newcommand{\sgn}{\text{sgn}}
\newcommand{\inner}[2]{\langle #1, #2 \rangle}
\newcommand{\MG}{\mathrm{MG}}
\newcommand{\SMT}{\mathsf{SMT}}
\newcommand{\Len}{\mathsf{Length}}
\newcommand{\Star}{\mathsf{Star}}
\newcommand{\NP}{\mathsf{NP}}

\newcommand{\ZPP}{\mathsf{ZPP}}
\newcommand{\TFNP}{\mathsf{TFNP}}

\newcommand{\Poly}{\mathsf{P}}
\newcommand{\SsubTwoP}{\mathsf{S}_2^\Poly}
\newcommand{\CH}{\mathsf{CH}}
\newcommand{\OO}{\mathcal{O}}
\newcommand{\tO}[1]{\tilde{\mathcal{O}}\left(#1\right)}
\newcommand{\eps}{\varepsilon}
\newcommand{\poly}{\mathrm{poly}}
\newcommand{\polylog}{\mathrm{polylog}}
\newcommand{\PPr}[1]{\mathbf{Pr}\left[#1\right]}
\newcommand{\OPT}{\mathsf{OPT}}
\newcommand{\Ent}{\mathrm{Ent}}

\newcommand{\Perm}{\mathrm{perm}}
\newcommand{\PermB}{\mathrm{perm}_{\mathrm{B}}}
\newcommand{\thetaMat}{\boldsymbol{\theta}}
\DeclareMathOperator*{\Recourse}{Recourse}

\newtheorem*{utheorem}{Theorem}
\newtheorem{proposition}[theorem]{Proposition}

\makeatletter
\def\@fnsymbol#1{%
  \ifcase#1 %
    \or * 
    \or $\dagger$ 
    \or $\ddagger$ 
    \or $\mathsection$ 
    \or $\mathparagraph$ 
    \or || 
    \or $\dagger\dagger$ 
    \or $\ddagger\ddagger$ 
    \or $\#$ 
    \or $\triangledown$ 
    \or $\ast$ 
    \or $\bigstar$ 
    \or $\lozenge$ 
    \or $\Diamond$ 
    \or $\clubsuit$ 
    \or $\spadesuit$ 
    \or $\heartsuit$ 
    \or $\diamondsuit$ 
    \or $\circ$ 
    \or $\bullet$ 
  \else ?\fi%
}
\makeatother

\title{Accelerating Scientific Research with Gemini:\\Case Studies and Common Techniques}

\author[1,2]{David P. Woodruff\textsuperscript{*, \dag, \ddag}}
\author[1]{Vincent Cohen-Addad\textsuperscript{\dag, \ddag}}
\author[1]{Lalit Jain\textsuperscript{\ddag}}
\author[1]{Jieming Mao\textsuperscript{\ddag}}
\author[1]{Song Zuo\textsuperscript{\dag, \ddag}}

\author[1]{MohammadHossein Bateni\textsuperscript{\dag}}
\author[3,1]{Simina Br\^anzei\textsuperscript{\dag}}
\author[1,5]{Michael P. Brenner\textsuperscript{\dag}}
\author[1]{Lin Chen\textsuperscript{\dag}}
\author[6]{Ying Feng\textsuperscript{\dag}}
\author[7]{Lance Fortnow\textsuperscript{\dag}}
\author[1]{Gang Fu\textsuperscript{\dag}}
\author[13]{Ziyi Guan\textsuperscript{\dag}}
\author[10]{Zahra Hadizadeh\textsuperscript{\dag}}
\author[1,14]{Mohammad T. Hajiaghayi\textsuperscript{\dag}}
\author[14]{Mahdi JafariRaviz\textsuperscript{\dag}}
\author[1,4]{Adel Javanmard\textsuperscript{\dag}}
\author[8]{Karthik C. S.\textsuperscript{\dag}}
\author[12]{Ken-ichi Kawarabayashi\textsuperscript{\dag}}
\author[1]{Ravi Kumar\textsuperscript{\dag}}
\author[1]{Silvio Lattanzi\textsuperscript{\dag}}
\author[9]{Euiwoong Lee\textsuperscript{\dag}}
\author[15]{Yi Li\textsuperscript{\dag}}
\author[10]{Ioannis Panageas\textsuperscript{\dag}}
\author[1]{Dimitris Paparas\textsuperscript{\dag}}
\author[2]{Benjamin Przybocki\textsuperscript{\dag}}
\author[2]{Bernardo Subercaseaux\textsuperscript{\dag}}
\author[13]{Ola Svensson\textsuperscript{\dag}}
\author[10]{Shayan Taherijam\textsuperscript{\dag}}
\author[15]{Xuan Wu\textsuperscript{\dag}}
\author[16]{Eylon Yogev \textsuperscript{\dag}}
\author[1]{Morteza Zadimoghaddam\textsuperscript{\dag}}
\author[11]{Samson Zhou\textsuperscript{\dag}}
\author[1]{Yossi Matias}
\author[1]{James Manyika}
 
\author[1]{Vahab Mirrokni\textsuperscript{*, \dag, \ddag}}

\affil[1]{Google Research}
\affil[2]{Carnegie Mellon University}
\affil[3]{Purdue University}
\affil[4]{University of Southern California}
\affil[5]{Harvard University}
\affil[6]{MIT}
\affil[7]{Illinois Institute of Technology}
\affil[8]{Rutgers University}
\affil[9]{University of Michigan}
\affil[10]{University of California, Irvine}
\affil[11]{Texas A\&M University}
\affil[12]{National Institute of Informatics, Tokyo and The University of Tokyo}
\affil[13]{EPFL}
\affil[14]{University of Maryland, College Park}
\affil[15]{Nanyang Technological University}
\affil[16]{Bar-Ilan University}

\date{}

\newcommand\blfootnote[1]{%
  \begingroup
  \renewcommand\thefootnote{}\footnote{#1}%
  \addtocounter{footnote}{-1}%
  \endgroup
}

\begin{document}
\maketitle

\blfootnote{\textsuperscript{*} Corresponding authors: \texttt{woodruffd@google.com}, \texttt{mirrokni@google.com}}
\blfootnote{\textsuperscript{\dag} Individual section contributor. Authors of individual sections are only responsible for correctness of their section.}
\blfootnote{\textsuperscript{\ddag} Contributed to advanced variants of Google models used in this paper.}

\vspace{-15mm}

\begin{abstract}
Recent advances in large language models (LLMs) have opened new avenues for accelerating scientific research. While models are increasingly capable of assisting with routine tasks, their ability to contribute to novel, expert-level mathematical discovery is less understood. We present a collection of case studies demonstrating how researchers have successfully collaborated with advanced AI models, specifically Google's Gemini-based models (in particular Gemini Deep Think and its advanced variants), to solve open problems, refute conjectures, and generate new proofs across diverse areas in theoretical computer science, as well as other areas such as economics, optimization, and physics. 
Based on these experiences, we extract common techniques for effective human-AI collaboration in theoretical research, such as iterative refinement, problem decomposition, and cross-disciplinary knowledge transfer. While the majority of our results stem from this interactive, conversational methodology, we also highlight specific instances that push beyond standard chat interfaces. These include deploying the model as a rigorous adversarial reviewer to detect subtle flaws in existing proofs, and embedding it within a "neuro-symbolic" loop that autonomously writes and executes code to verify complex derivations. Together, these examples highlight the potential of AI not just as a tool for automation, but as a versatile, genuine partner in the creative process of scientific discovery.
\end{abstract}

\tableofcontents
\newpage 
\section{Introduction}
The integration of artificial intelligence into the scientific workflow has traditionally focused on data analysis, simulation, and routine automation. However, the emergence of frontier large language models (LLMs) with enhanced reasoning capabilities suggests a shift towards AI systems that can act as an effective research collaborator, and  contribute to the core intellectual tasks of research: formulating hypotheses, designing algorithms, developing novel techniques to tackle an open problem, and proving theorems. 

This paper documents a series of independent experiments where researchers utilized advanced AI models to tackle specific, often long-standing, open problems in their respective fields. The results range from resolving conjectures in information theory and submodular maximization to deriving exact analytical spectra for cosmic strings and improving bounds for graph algorithms.

\textbf{Key Contributions:} Given the breadth of the fields and case studies covered, we highlight three primary contributions of this paper:
\begin{enumerate}
    \item \textbf{Demonstrating the power of hybrid models:} We show the potential of a hybrid human-AI model in tackling real, highly complex open problems across a diverse set of scientific and mathematical domains.
    \item \textbf{Advancing state-of-the-art science:} We make concrete, notable contributions to real problems in science, yielding novel results such as resolving conjectures, improving algorithmic bounds, and identifying critical flaws in literature.
    \item \textbf{Showcasing the future of scientific discovery:} We illustrate what the future of research looks like in practice. By detailing novel workflows, human-AI collaboration techniques, and interactive problem-solving strategies, we highlight the real implications for how science is likely to be conducted in the coming years.
\end{enumerate}

Crucially, these are not hypothetical scenarios but real-world instances where AI played a pivotal role in advancing state-of-the-art research. By analyzing these diverse case studies, we identify recurring patterns and effective strategies for leveraging AI in such theoretical research. The significance of these results is the generality of these techniques and existence of common recipes that can be applied by probing publicly available Gemini models. These include:

\begin{itemize}
      \item \textbf{Agentic Execution Loops:} Moving beyond manual chat interfaces, models can be embedded in automated ``neuro-symbolic" pipelines. In these setups, the AI proposes a mathematical solution, writes code to numerically verify it, and automatically ingests execution errors (e.g., Python tracebacks) to self-correct and autonomously prune invalid mathematical branches (see Section \ref{sec:cosmic_strings}).
       \item \textbf{Deep Technical Review and Bug Detection:} Beyond constructive tasks, AI models can act as adversarial reviewers. We present a case where an LLM, guided by an iterative self-correction protocol, identified a fatal flaw in a recent cryptography preprint claiming a major breakthrough (SNARGs from LWE) \cite{YogevGuan2025}, a subtle inconsistency between a definition and a construction that had escaped initial human review.
    \item \textbf{Deep Literature Synthesis and Connection:} AI models can identify obscure connections between disparate fields (e.g., linking Steiner trees to the Kirszbraun Extension Theorem) that human experts might overlook.
    \item \textbf{Counterexample Generation:} Models are adept at constructing counterexamples to refute plausible conjectures, saving researchers from pursuing dead ends.
    \item \textbf{Algorithmic Insight and Optimization:} In algorithmic research, AI can propose novel data structures or analysis techniques (e.g., adapting quadtrees for different norms) to improve time complexity bounds.
    \item \textbf{Automated Proof Generation and Verification:} For well-defined subproblems, AI can generate rigorous proofs, sometimes requiring minimal human intervention, or verify complex manual derivations.
    \item \textbf{Interactive Refinement:} A recurring theme is the iterative "conversation" where the researcher guides the model, correcting errors and refining the problem statement, which often leads to the final solution.
    \item \textbf{Theoretical Justification of Heuristics:} AI models can bridge the gap between empirical success and theory by deriving rigorous justifications for heuristic methods, such as characterizing the implicit regularization induced by specific architectural choices like the Self-regularized Gumbel Sigmoid.
\end{itemize}

\subsection{Model Description}
Unless otherwise specified in the testimonials, the model used in this work is a Google-internal advanced version of Gemini Deep Think \cite{GeminiDeepThinkIMO} – an enhanced reasoning model for complex problems that incorporates some of our latest research techniques, including parallel thinking, i.e., exploring multiple proof branches simultaneously. Similar models were also used recently in the International Mathematics Olympiad, achieving gold-medal standard performance \cite{GeminiDeepThinkIMO}. This setup enables the model to simultaneously explore and combine multiple possible solutions before giving a final answer, rather than pursuing a single, linear chain of thought. This approach aligns with prior research demonstrating the efficacy of tree-based search methods in mathematical reasoning \cite{wu2023tree}. 

To make the most of the reasoning capabilities of Deep Think, we additionally trained this version of Gemini on novel reinforcement learning techniques that can leverage more multi-step reasoning, problem-solving and theorem-proving data. We also provided Gemini with access to a curated corpus of high-quality solutions to mathematics problems \cite{GeminiDeepThinkIMO}. We also added a long linear chain of interactive verification calls at the output to increase reasoning depth and provide a lot of iterative verification, which was augmented by human expert verification.

\begin{figure}[htbp]
    \centering
    \includegraphics[width=0.7\textwidth]{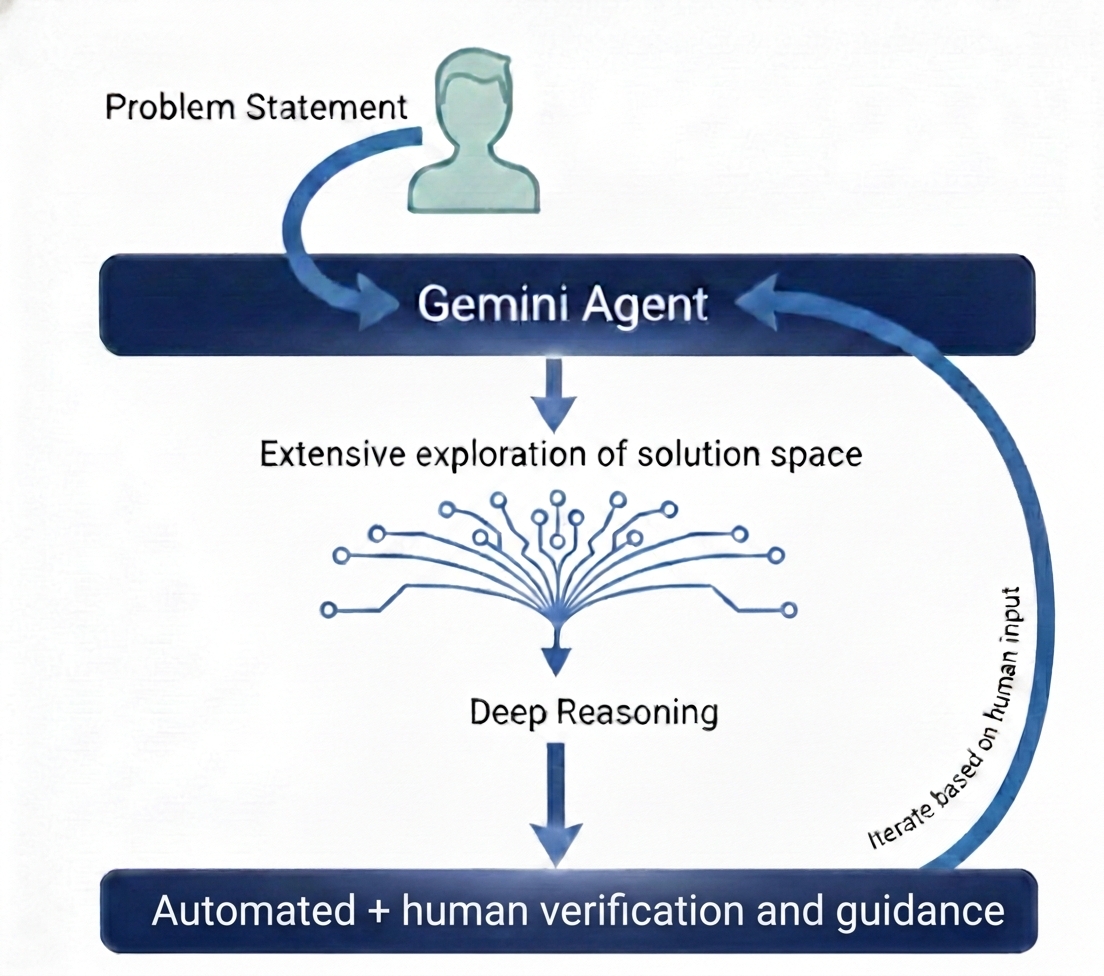}
    \caption{Overview of the reasoning architecture used in many testimonials: an extensive exploration of the solution space combined with deep reasoning and a long tail of automated and human verification and in several cases, guidance and iterative feedback.}
    \label{fig:model_arch}
\end{figure}


\subsection{Related Work}
Recent advances in Large Language Models (LLMs) have spurred significant interest in their potential to accelerate scientific discovery. Several concurrent works explore the capabilities of AI in assisting mathematical and theoretical research.
Here, we situate our contributions in the context of two particularly relevant studies: OpenAI's investigation into GPT-5's scientific capabilities \cite{OpenAI2025} and the work of Nagda et al. on using AlphaEvolve for hardness of approximation \cite{Nagda2025}.

\paragraph{Comparison with OpenAI's GPT-5 Experiments.}
Concurrent work by OpenAI \cite{OpenAI2025} documents a series of case studies demonstrating GPT-5's ability to contribute to research in fields ranging from biology and physics to mathematics. While both our work and theirs highlight the potential of frontier models as research assistants, there are differences in scope and focus.
Our work places a heavier emphasis on \emph{Theoretical Computer Science} (TCS), covering a wide array of subfields including algorithmic game theory, approximation algorithms, complexity theory, cryptography, graph theory, information theory, optimization, online algorithms, streaming algorithms, and query complexity.

Furthermore, recent activity on the Erd\H{o}s Problem database has showcased the growing role of AI in resolving open mathematical conjectures. Notably, Erd\H{o}s Problem \#728, concerning factorial divisibility, was solved autonomously by an AI system (Aristotle), with GPT-5 playing a crucial role in generating the proof strategy and subsequent expositions \cite{OpenAI2025}. This achievement, verified by Terence Tao and the broader community, underscores the capacity of these models not only to solve problems but also to rapidly iterate on and improve mathematical writing. Other problems, such as \#729, \#401, \#205, \#333, \#281, and \#460 have also seen AI-assisted progress or full resolutions, further validating the trend of AI-accelerated mathematical discovery. Our work complements these findings by providing a detailed methodology and a broader set of examples across TCS, demonstrating that these capabilities are not isolated incidents but part of a larger shift in research practices.

\paragraph{Comparison with AlphaEvolve for Hardness of Approximation.}
Nagda, Raghavan, and Thakurta \cite{Nagda2025} present significant results in hardness of approximation for MAX-CUT, MAX-k-CUT, and TSP. Their primary methodology relies on \emph{AlphaEvolve} \cite{NVE25}, a specialized evolutionary algorithm that uses an LLM as a mutation operator to search for combinatorial structures (gadgets) that optimize a specific objective function (inapproximability ratios).

In contrast, our work focuses on the utility of \emph{general-purpose} LLMs (specifically Gemini and its reasoning-enhanced variants like Deep Think) as conversational research partners.
While AlphaEvolve is a powerful tool for search problems where a clear objective function exists (e.g., finding a gadget with specific properties), our case studies demonstrate that general-purpose models can be effective across a broader range of intellectual tasks where the "objective" is less well-defined, such as generating a proof idea, identifying a relevant theorem from a different field, or interactively refining a definition.

For instance, our derivation of the analytical spectrum for cosmic strings (Section \ref{sec:cosmic_strings}) involved a tree-search over \emph{derivation strategies} and mathematical concepts, rather than a direct optimization of a numerical value.
Similarly, the resolution of the "Simplex is the Best for Graph Embeddings" conjecture (Section \ref{sec:steiner}) relied on conceptual bridging (linking Steiner trees to Lipschitz extensions) rather than combinatorial search.
Thus, our work complements \cite{Nagda2025} by showing that standard frontier models, without specialized evolutionary wrappers, can drive progress in theoretical research through dialogue and reasoning.

\paragraph{Concurrent Work.}
A recent work \cite{georgiev2025mathematical} does mathematical discovery at scale using AI-assisted methodologies. Our work aligns with it in demonstrating the potential of AI as a collaborative partner in theoretical research, but we offer a distinct set of case studies and a focus on specific TCS domains. Additionally, we note independent works \cite{sellke2025learning, sothanaphan2026resolution} that further contribute to the growing body of literature on AI in mathematics and science.

In concurrent and independent work of Feng et al. \cite{F26Erdos, F26}, they study similar Gemini-based models using similar techniques, though their focus is on pure math problems whereas ours is mostly on other disciplines, such as many areas of theoretical computer science, as well as economics and physics.

\paragraph{Roadmap}
The remainder of this paper is organized as follows. Section \ref{sec:techniques} synthesizes the common techniques and methodologies derived from these case studies. The sections after that give detailed and extensive testimonials from each collaboration, roughly categorized by the main role of the AI in the collaboration. 
Finally, Section \ref{sec:conclusion} discusses the broader implications for the future of theoretical research.

\section{Techniques for AI-Assisted Research}\label{sec:techniques}

Across the various successful collaborations documented here, several common techniques emerged. These strategies represent a "playbook" for researchers looking to integrate AI into their theoretical work.

\subsection{Iterative Prompting and Refinement}
Rarely does a model solve a deep open problem in a single shot. Success often comes from an iterative dialogue.
\begin{itemize}
    \item \textbf{Initial Broad Query:} Start by asking the model to digest a relevant paper or problem statement to gauge its understanding.
    \item \textbf{Specific Sub-tasks:} Break down the main problem into smaller, verifiable lemmas or calculations.
    \item \textbf{Error Correction:} When the model makes a mistake (e.g., a wrong constant or invalid assumption), pointing it out specifically often leads to a correct and sometimes more elegant solution in the next turn.
    \item \textbf{Scaffolding:} Providing the model with a high-level proof strategy or "scaffold" allows it to fill in the technical details effectively.
        \item \textbf{Adversarial Self-Correction for Review:} When tasked with reviewing complex proofs, standard prompts often yield superficial results. A rigorous protocol instructing the model to (1) generate an initial review, (2) critique its own findings for hallucinations, and (3) iteratively refine the logic, enables deep technical critique. This was critical in identifying the flaw in the SNARGs paper (Section \ref{sec:crypto_bug}), where the model distinguished between a definition of \textit{perfect consistency} and a construction of \textit{statistical consistency}.
\end{itemize}

\subsection{Cross-Pollination of Ideas}
Models have ingested vast amounts of literature across all fields. They excel at:
\begin{itemize}
    \item \textbf{Finding Analogies:} Identifying similar problems in different domains (e.g., applying techniques from computational geometry to graph theory).
    \item \textbf{Retrieving obscure theorems:} Bringing relevant but less-known theorems to the researcher's attention (e.g., Stone-Weierstrass or Kirszbraun Extension Theorem) to bridge gaps in a proof.
\end{itemize}

\begin{figure}[htbp]
    \centering
    \begin{subfigure}[b]{0.45\textwidth}
        \centering
        \includegraphics[width=\textwidth]{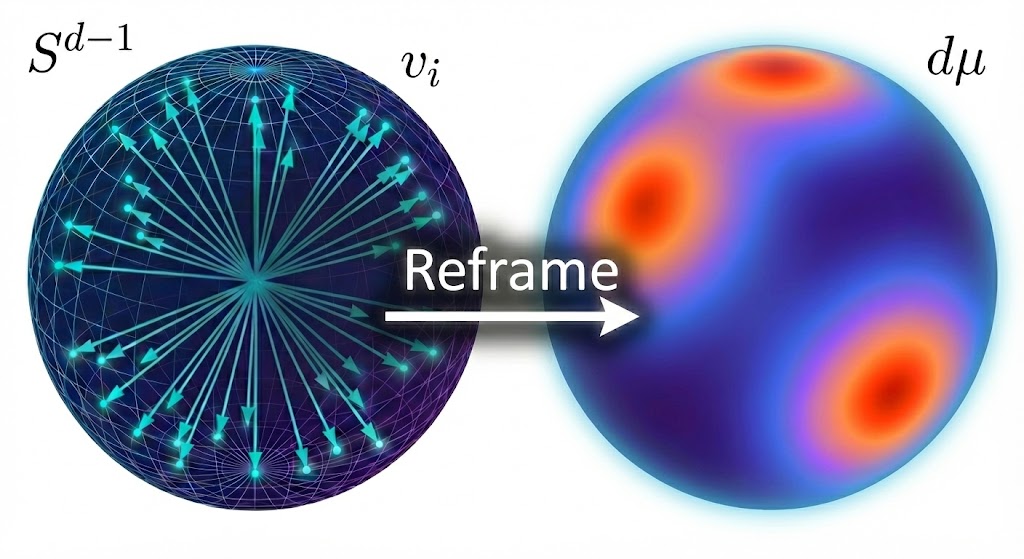}
        \caption{\textbf{From Discrete Combinatorics to Continuous Measure Theory:} To resolve an open question about bounded-rank SDP solutions for Max-Cut, the AI reframed a discrete combinatorial problem involving unit vectors into an energy minimization problem over continuous probability measures on the unit sphere ($S^{d-1}$, illustrated above). This cross-pollination allowed the AI to bypass traditional approximation tool-kits and instead apply the Stone-Weierstrass Theorem from geometric functional analysis to establish the necessary variance bounds. See Section \ref{sec:max_cut}.}
        \label{fig:cross_sphere}
    \end{subfigure}
    \hfill
    \begin{subfigure}[b]{0.45\textwidth}
        \centering
        \includegraphics[width=\textwidth]{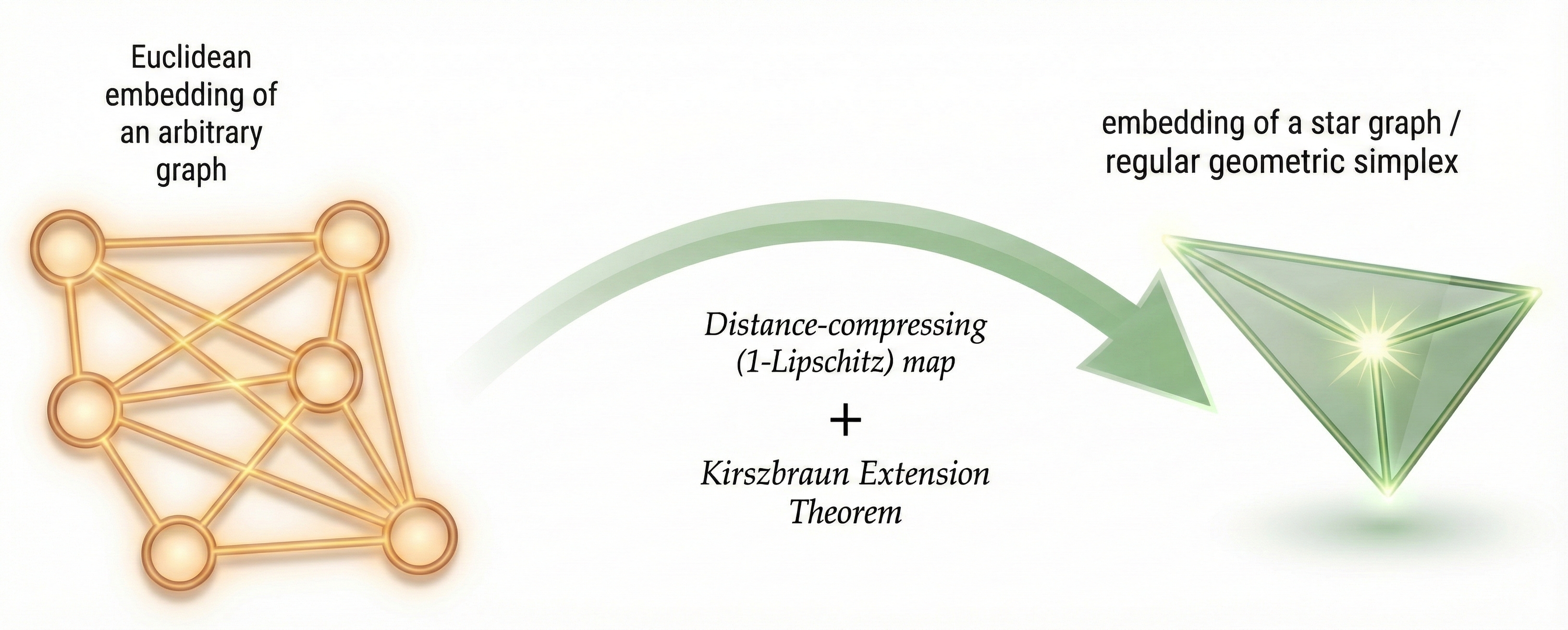}
        \caption{\textbf{From Graph Embeddings to Hilbert Space Geometry:} To resolve the "Simplex is Best" conjecture for Euclidean Steiner Trees, the AI constructed a mapping from the Euclidean embedding of an arbitrary graph (orange network, left) to the embedding of a star graph, which forms a regular geometric simplex (green tetrahedron, right). By recognizing this as a distance-compressing (1-Lipschitz) map, the AI realized it could apply the obscure Kirszbraun Extension Theorem to formally guarantee that transforming the graph into a simplex never increases the Steiner tree cost. See Section \ref{sec:steiner}.}
        \label{fig:cross_simplex}
    \end{subfigure}
    \vspace{0.2cm}
    \caption{Visualizing Cross-Pollination: The AI model successfully resolved open problems by reframing them in entirely different mathematical domains, bringing advanced topological and geometric theorems to bear on algorithmic and combinatorial problems.}
    \label{fig:cross_pollination}
\end{figure}

\subsection{Simulation and Counterexample Search}
For conjectures, models can be tasked to:
\begin{itemize}
    \item \textbf{Construct Counterexamples:} Generating specific instances (graphs, matrices, set systems) that violate a proposed conjecture.
    \item \textbf{Verify Small Cases:} Writing code to computationally verify a conjecture for small $n$, providing empirical evidence before attempting a general proof.
\end{itemize}

\subsection{Formalization and Rigor Checks}
While models can hallucinate, they are increasingly capable of formal reasoning when prompted correctly.
\begin{itemize}
    \item \textbf{Proof Sketch to Formal Proof:} Asking the model to expand a high-level sketch into a rigorous \LaTeX\ proof.
    \item \textbf{Sanity Checking:} Using the model to check consistent usage of notation or to verify that all conditions of a theorem are met.
    \item  \textbf{Mathematical Derivation:} Researchers can offload the mechanical heavy lifting of complex derivations to the model, such as simplifying expressions, computing limits, or solving integrals, allowing them to focus on the high-level logic.
\end{itemize}

\subsection{Interactive Proof Construction with External Validation}
A powerful technique involves using the model to identify necessary external theorems and then validating those theorems with external sources.
\begin{itemize}
    \item \textbf{Identifying Dependencies:} Asking the model to list all external theorems required for a proof.
    \item \textbf{External Verification:} The researcher finds the formal statements of these theorems (e.g., via Google Search or textbooks) and feeds them back to the model.
    \item \textbf{Self-Contained Proof Generation:} The model then incorporates these verified statements to generate a rigorous, self-contained proof.
\end{itemize}

\subsection{Agentic Tool-Use and Automated Feedback}
While most of the collaborations documented here rely on manual, iterative dialogue, frontier models can also be deployed as agents within automated programmatic loops. For problems requiring heavy algebraic manipulation or where symbolic math must be rigorously grounded in numerical reality (e.g., the physics case study in Section 6.4), researchers can construct a ``neuro-symbolic" pipeline consisting of the following steps:
\begin{itemize}
\item {\bf Symbolic Proposal:} The LLM generates a mathematical hypothesis or intermediate expression.
\item {\bf Code Generation:} The LLM autonomously writes an executable script (e.g., in Python) to evaluate its proposed math against a known numerical baseline.
\item {\bf Automated Feedback:} The system executes the code. If the code fails, hits a runtime error, or reveals numerical instability (such as catastrophic cancellation), the automated harness captures the exact execution traceback and injects it back into the LLM's context window.
\end{itemize}
This programmatic loop allows the AI to systematically explore a tree of mathematical solutions, autonomously pruning dead-ends and self-correcting its algebra without requiring a human-in-the-loop for every intermediate step.

\subsection{Human-AI Collaboration Dynamics}
It is important to note that AI models like Gemini function best as powerful collaborators rather than autonomous researchers. In the successful case studies presented here, the partnership between the model and the human expert was key to the results.

\begin{itemize}
    \item \textbf{Selection and Refinement:} Models are capable of generating a high volume of diverse mathematical statements. Human expertise is valuable for filtering these outputs and identifying the most promising directions for further investigation.
    \item \textbf{Iterative Guidance:} While models can solve some problems in a single shot, tackling deep open problems is often most successful through an iterative process. The researcher guides the model, refining the problem statement and narrowing the focus to achieve the desired result.
    \item \textbf{Standard Verification:} As with any research collaboration, the AI can make mistakes, and AI-generated proofs and counterexamples benefit from rigorous verification. The model serves as an excellent accelerator for ideation and drafting, while the researcher validates the mathematical correctness.
    \item \textbf{Optimizing Context:} Performance is often optimized by providing clear, self-contained definitions, particularly when using highly specialized notation that may deviate from standard literature.
    \item \textbf{Leveraging Literature:} We found that incorporating relevant papers directly into the context significantly enhanced the model's ability to construct correct proofs for specialized domains.
    \item \textbf{Context De-Identification:} The model sometimes avoids non-trivial machinery (for example, the Kirszbraun extension theorem as in Section~\ref{sec:steiner}), treating such proofs as non-elementary, or it may do so because the prompt steers it toward conservatism to avoid hallucinations, causing it to abandon an otherwise viable approach. Separately, on occasion, when shown the paper as context in the prompt, it recognizes the statement to prove as a conjecture in the paper and refuses to attempt it on the grounds that it is an open problem. One way to bypass both issues is via context de-identification (remove the paper and provide only the problem statement and definitions), after which the model typically engages (and, in cases like Section~\ref{sec:steiner}, may ultimately draw on deeper results to resolve a conjecture).
\end{itemize}

We view the AI as a tireless, knowledgeable, and creative bright junior collaborator. Its value lies in its ability to synthesize vast amounts of information and generate novel hypotheses that human researchers can then validate and build upon.

\subsection{Summary: The AI-Assisted Research Playbook}\label{sec:conclusion}

Taken together, the techniques outlined above represent a fundamental shift in how theoretical research can be conducted. The LLM is no longer acting merely as a search engine or a syntax formatter; it is functioning as a combinatorial reasoning engine and a sounding board for abstract ideation. 

However, the most successful collaborations documented in the following case studies all share a common denominator: strong human orchestration. Although several of our successes came from a single ``zero-shot" prompt, many required scaffolded reasoning, i.e., breaking down deep open problems into verifiable parts, testing hypotheses through adversarial prompting, and actively steering the model. Informally this interactive workflow has been called ``vibe-proving". 

By mastering the techniques outlined above—particularly iterative refinement and adversarial self-correction—researchers can effectively elevate the AI from a passive tool into an active, high-leverage research partner. The following sections document these dynamics in practice across a wide array of disciplines.
\\\\
{\bf Color-coding:} In the remainder of the paper, chat interactions are presented using a light, color-coded design. 
When present, the overall chat box uses a sand-colored background to group messages within a conversation. 
In cases where only a single message appears—such as an isolated user message or an isolated AI response—the surrounding chat box may be omitted, and only the message styling is shown. 
Messages authored by the user are displayed on a light blue background
and AI-generated messages are shown on a light red background. 

\section{Deep Technical Review and Counterexamples}
In this section, we see how AI can act as a critical adversary, finding subtle bugs in human proofs, refuting plausible hypotheses, and hunting for counterexamples.

\subsection{Online Algorithms: Submodular Welfare}\label{sec:submodular}
\textit{Written by Morteza Zadimoghaddam and Vahab Mirrokni.}

\subsubsection*{Problem Context}
In online submodular welfare maximization, the Greedy algorithm achieves a competitive ratio of $0.5$. A conjecture by Korula et al. (2015) suggested that the gain from "copying" an item to the end of the stream is bounded by the gain of "moving" it. Proving this would imply a competitive ratio of $0.567$.
\\\\
\textbf{AI Contribution}

The researcher provided the research paper that included the formal definitions as well as a conjecture statement, and tasked the model zero-shot with just identifying any open question in the paper and solving it. Throughout this paper we add a prompt that asks the model to be rigorous in its mathematical proofs and to thoroughly check its answers - we do not specify the details of this prompt. 

\begin{figure}[!htb]
\begin{userturn}
Please try to improve the paper by identifying and solving an open question from it. 
\\\\
<Math rigor prompt>
\end{userturn}
\end{figure}
 Remarkably, the model chose the following counterexample and refuted it autonomously. The subsequent verification calculations were also generated  by the model without extra interaction or supervision. This demonstrates its ability to construct valid, non-trivial instances that violate human-proposed conjectures. The human researchers defined the initial search parameters and independently verified the model's arithmetic.

\begin{itemize}
    \item \textbf{Refutation Strategy:} The model independently selected the minimal non-trivial dimensions ($n=3$ items, $m=2$ agents).
    \item \textbf{Autonomous Construction:} In a single output, the model successfully defined the specific, valid submodular valuation functions (presented in Table 1) and the $\varepsilon$-perturbation required to strictly break the bound.
    \item \textbf{Automated Verification:} Without human intervention, the model correctly performed the tedious calculations of expected marginal gains across all $3! = 6$ permutations to formally verify the violation.
\end{itemize}

\subsubsection*{Technical Details}

In the problem of Online Submodular Welfare Maximization (Online SWM), items arrive one by one in a random order and must be irrevocably assigned to agents with submodular valuation functions. The goal is to maximize the total social welfare. It is well-established that the simple Greedy algorithm achieves a competitive ratio of $1/2$ in the adversarial setting. In the random order model, Korula et al. proved that Greedy achieves a ratio of at least $0.5052$, strictly beating $1/2$ \cite{korula2015online}.

To obtain a stronger bound, the authors proposed a natural conjecture (Conjecture 15) relating the expected marginal gain of the Greedy algorithm when an item is moved to the end of the sequence versus when it is duplicated at the end. If true, this conjecture would imply a competitive ratio of at least $0.567$.

In this note, we refute this conjecture. We construct a counter-example involving a specific submodular valuation function and a set of items where the conjectured inequality fails to hold.

\subsubsection*{The Open Question}

We focus on the definitions and conjecture provided in Section 5 of \cite{korula2015online}. Let $\pi = (\pi_1, \pi_2, \pi_3, \dots, \pi_n)$ be a permutation of $n$ items.

\begin{definition}[Permutation Variants \cite{korula2015online}]
Fix a permutation $\pi$.
\begin{itemize}
    \item Let $\pi^{Move, i}$ be the permutation achieved by moving the item $\pi_i$ to the end of the sequence:
    $$ \pi^{Move, i} = (\pi_1, \dots, \pi_{i-1}, \pi_{i+1}, \dots, \pi_n, \pi_i) $$
    \item Let $\pi^{Copy, i}$ be the sequence of $n+1$ items achieved by copying $\pi_i$ to the end without removing the original $\pi_i$:
    $$ \pi^{Copy, i} = (\pi_1, \dots, \pi_{i-1}, \pi_i, \pi_{i+1}, \dots, \pi_n, \pi_i) $$
\end{itemize}
\end{definition}

Let $\MG(k, \sigma)$ denote the marginal gain that the Greedy algorithm obtains by allocating the $k$-th arriving item in a sequence $\sigma$. Let $\mathbb{S}_n$ be the set of all $n!$ permutations. The conjecture proposed by Korula et al. is as follows:

\begin{conjecture}[Conjecture 15 of \cite{korula2015online}] \label{conj:main}
For any instance of the online submodular welfare maximization problem,
\begin{equation}
    \mathbb{E}_{\pi \sim \mathbb{S}_n} \left[ \sum_{i=1}^{n} \MG(n+1, \pi^{Copy, i}) \right] \le \mathbb{E}_{\pi \sim \mathbb{S}_n} \left[ \sum_{i=1}^{n} \MG(n, \pi^{Move, i}) \right]
\end{equation}
\end{conjecture}

The Right Hand Side (RHS) of the inequality is equivalent to $n \cdot \mathbb{E}_{\pi \sim \mathbb{S}_n}[\MG(n, \pi)]$, as $\pi^{Move, i}$ and $\pi$ share the same distribution. The conjecture essentially posits that the gain from a duplicated item at the end of the stream is, in expectation, bounded by the gain of the last item in a standard random permutation.

\subsubsection*{Refutation of the Conjecture}

We answer the open question negatively. We show that there exist submodular valuation functions and instances where the expected marginal gain from the copied item exceeds the bound proposed in Conjecture \ref{conj:main}.

\begin{theorem}
Conjecture \ref{conj:main} is false. There exists an instance of Online SWM such that:
$$ \mathbb{E}_{\pi \sim \mathbb{S}_n} \left[ \sum_{i=1}^{n} \MG(n+1, \pi^{Copy, i}) \right] > \mathbb{E}_{\pi \sim \mathbb{S}_n} \left[ \sum_{i=1}^{n} \MG(n, \pi^{Move, i}) \right] $$
\end{theorem}

\subsubsection*{Counterexample Construction}

We construct an instance with $n=3$ items $N=\{e, x_1, x_2\}$ and $m=2$ agents $\{A_1, A_2\}$. We define the valuation functions $v_1$ and $v_2$. Let $\varepsilon=0.1$. We define $v_1(S) = v_1^0(S) + \varepsilon|S|$, where $v_1^0(S)$ and $v_2(S)$ are defined in Table \ref{tab:valuations}.

\begin{table}[!htb]
    \centering
    \caption{Valuation Functions}
    \label{tab:valuations}
    \renewcommand{\arraystretch}{1.2}
    \begin{tabular}{c c c c}
        \toprule
        $S$ & $v_1^0(S)$ & $v_2(S)$ & $v_1(S)$ \\
        \midrule
        $\emptyset$ & 0 & 0 & 0 \\
        $\{e\}$ & 10 & 10 & 10.1 \\
        $\{x_1\}$ & 10 & 7 & 10.1 \\
        $\{x_2\}$ & 10 & 11 & 10.1 \\
        $\{e, x_1\}$ & 15 & 17 & 15.2 \\
        $\{e, x_2\}$ & 20 & 16 & 20.2 \\
        $\{x_1, x_2\}$ & 20 & 16 & 20.2 \\
        $\{e, x_1, x_2\}$ & 25 & 21 & 25.3 \\
        \bottomrule
    \end{tabular}
\end{table}

\textbf{Verification of Validity}
We must verify that $v_1$ and $v_2$ are monotone and submodular.

\begin{enumerate}
    \item \textbf{Monotonicity:} It is clear from Table \ref{tab:valuations} that both $v_1$ and $v_2$ are monotone (i.e., $v(S) \leq v(T)$ if $S \subseteq T$).

    \item \textbf{Submodularity of $v_2$:} We check the marginal gains (MG).
    \begin{itemize}
        \item $\MG_2(e)$: $\emptyset(10)$, $\{x_1\}(10)$, $\{x_2\}(5)$, $\{x_1, x_2\}(5)$.
        \item $\MG_2(x_1)$: $\emptyset(7)$, $\{e\}(7)$, $\{x_2\}(5)$, $\{e, x_2\}(5)$.
        \item $\MG_2(x_2)$: $\emptyset(11)$, $\{e\}(6)$, $\{x_1\}(9)$, $\{e, x_1\}(4)$.
    \end{itemize}
    In all cases, the marginal gains are non-increasing as the conditioning set grows. Thus, $v_2$ is submodular.

    \item \textbf{Submodularity of $v_1$:} We first check $v_1^0$.
    \begin{itemize}
        \item $\MG_1^0(e)$: $\emptyset(10)$, $\{x_1\}(5)$, $\{x_2\}(10)$, $\{x_1, x_2\}(5)$.
        \item $\MG_1^0(x_1)$: $\emptyset(10)$, $\{e\}(5)$, $\{x_2\}(10)$, $\{e, x_2\}(5)$.
        \item $\MG_1^0(x_2)$: $\emptyset(10)$, $\{e\}(10)$, $\{x_1\}(10)$, $\{e, x_1\}(10)$.
    \end{itemize}
    $v_1^0$ is submodular. Since $v_1(S) = v_1^0(S) + \varepsilon|S|$ is the sum of a submodular function and a modular function, $v_1(S)$ is submodular.
\end{enumerate}

\subsubsection*{Calculation of the RHS}

The RHS is the sum of the expected gains of each item when it arrives last. Let $W(j)$ denote this expectation for item $j$.

\begin{enumerate}
    \item \textbf{$W(e)$ ($e$ arrives last).} We consider the permutations of $\{x_1, x_2\}$.
    \begin{itemize}
        \item $(x_1, x_2)$. $x_1$: $A_1(10.1)$ vs $A_2(7)$. $x_1 \to A_1$. \\
        $x_2$: $\MG_1(x_2|x_1)=10.1$. $\MG_2(x_2)=11$. $x_2 \to A_2$. \\
        Gain of $e$: $\MG_1(e|x_1)=5.1$. $\MG_2(e|x_2)=5$. Gain is 5.1.
        \item $(x_2, x_1)$. $x_2$: $A_1(10.1)$ vs $A_2(11)$. $x_2 \to A_2$. \\
        $x_1$: $\MG_1(x_1)=10.1$. $\MG_2(x_1|x_2)=5$. $x_1 \to A_1$. \\
        Gain of $e$: 5.1.
    \end{itemize}
    Result: $W(e)=5.1$.

    \item \textbf{$W(x_1)$ ($x_1$ arrives last).} We consider the permutations of $\{e, x_2\}$.
    \begin{itemize}
        \item $(e, x_2)$. $e$: $A_1(10.1)$ vs $A_2(10)$. $e \to A_1$. \\
        $x_2$: $\MG_1(x_2|e)=10.1$. $\MG_2(x_2)=11$. $x_2 \to A_2$. \\
        Gain of $x_1$: $\MG_1(x_1|e)=5.1$. $\MG_2(x_1|x_2)=5$. Gain is 5.1.
        \item $(x_2, e)$. $x_2 \to A_2$. $e \to A_1$. Gain of $x_1$ is 5.1.
    \end{itemize}
    Result: $W(x_1)=5.1$.

    \item \textbf{$W(x_2)$ ($x_2$ arrives last).} We consider the permutations of $\{e, x_1\}$.
    \begin{itemize}
        \item $(e, x_1)$. $e \to A_1$. \\
        $x_1$: $\MG_1(x_1|e)=5.1$. $\MG_2(x_1)=7$. $x_1 \to A_2$. \\
        Gain of $x_2$: $\MG_1(x_2|e)=10.1$. $\MG_2(x_2|x_1)=9$. Gain is 10.1.
        \item $(x_1, e)$. $x_1 \to A_1$. \\
        $e$: $\MG_1(e|x_1)=5.1$. $\MG_2(e)=10$. $e \to A_2$. \\
        Gain of $x_2$: $\MG_1(x_2|x_1)=10.1$. $\MG_2(x_2|e)=6$. Gain is 10.1.
    \end{itemize}
    Result: $W(x_2)=10.1$.
\end{enumerate}

\[ RHS = W(e) + W(x_1) + W(x_2) = 5.1 + 5.1 + 10.1 = 20.3. \]

\subsubsection*{Calculation of the LHS}

The LHS is the expected total residual welfare. We calculate the total residual $R(\pi)$ for all $3!=6$ permutations. The residual of item $j$, $R_j(\pi)$, is the maximum marginal gain of $j$ for any agent given the final allocation by Greedy on $\pi$. (The agent who received $j$ has 0 marginal gain for a copy).

\begin{enumerate}
    \item $\pi_1=(e, x_1, x_2)$. \\
    $e \to A_1$. $x_1 \to A_2$ (7 vs 5.1). $x_2 \to A_1$ (10.1 vs 9). \\
    Allocation: $A_1=\{e, x_2\}, A_2=\{x_1\}$. \\
    $R_e=\MG_2(e|x_1)=10$. $R_{x_1}=\MG_1(x_1|e, x_2)=5.1$. $R_{x_2}=\MG_2(x_2|x_1)=9$. \\
    $R(\pi_1) = 24.1$.

    \item $\pi_2=(e, x_2, x_1)$. \\
    $e \to A_1$. $x_2 \to A_2$ (11 vs 10.1). $x_1 \to A_1$ (5.1 vs 5). \\
    Allocation: $A_1=\{e, x_1\}, A_2=\{x_2\}$. \\
    $R_e=\MG_2(e|x_2)=5$. $R_{x_1}=\MG_2(x_1|x_2)=5$. $R_{x_2}=\MG_1(x_2|e, x_1)=10.1$. \\
    $R(\pi_2) = 20.1$.

    \item $\pi_3=(x_1, e, x_2)$. \\
    $x_1 \to A_1$. $e \to A_2$ (10 vs 5.1). $x_2 \to A_1$ (10.1 vs 6). \\
    Allocation: $A_1=\{x_1, x_2\}, A_2=\{e\}$. \\
    $R_e=\MG_1(e|x_1, x_2)=5.1$. $R_{x_1}=\MG_2(x_1|e)=7$. $R_{x_2}=\MG_2(x_2|e)=6$. \\
    $R(\pi_3) = 18.1$.

    \item $\pi_4=(x_2, e, x_1)$. \\
    $x_2 \to A_2$. $e \to A_1$. $x_1 \to A_1$. \\
    Allocation: $A_1=\{e, x_1\}, A_2=\{x_2\}$. \\
    $R(\pi_4) = R(\pi_2) = 20.1$.

    \item $\pi_5=(x_1, x_2, e)$. \\
    $x_1 \to A_1$. $x_2 \to A_2$. $e \to A_1$. \\
    Allocation: $A_1=\{x_1, e\}, A_2=\{x_2\}$. \\
    $R(\pi_5) = 20.1$.

    \item $\pi_6=(x_2, x_1, e)$. \\
    $x_2 \to A_2$. $x_1 \to A_1$. $e \to A_1$. \\
    Allocation: $A_1=\{x_1, e\}, A_2=\{x_2\}$. \\
    $R(\pi_6) = 20.1$.
\end{enumerate}

The sum of the total residuals is $24.1 + 20.1 + 18.1 + 20.1 + 20.1 + 20.1 = 122.6$.
\[ LHS = E[R(\pi)] = \frac{122.6}{6}. \]

\subsubsection*{Conclusion}

We compare the LHS and the RHS.
\begin{align*}
    RHS &= 20.3 = \frac{121.8}{6}. \\
    LHS &= \frac{122.6}{6}.
\end{align*}
Since $\frac{122.6}{6} > \frac{121.8}{6}$, we have $LHS > RHS$. This violates the inequality stated in Conjecture 15. Therefore, Conjecture 15 is false.

The failure of Conjecture 15 implies that the specific path to proving a $0.567$ competitive ratio proposed by \cite{korula2015online} is invalid. While the Greedy algorithm is known to beat $1/2$ in the random order model, establishing the tightest possible bound remains an open challenge. Our result indicates that the interaction between duplicated items and previous greedy allocations is more complex than the ``Move'' vs ``Copy'' intuition suggests.

\newpage

\subsection{Cryptography: AI-Assisted Bug Detection in SNARGs}\label{sec:crypto_bug}
\textit{Written by Ziyi Guan, David P. Woodruff, and Eylon Yogev}

\subsubsection*{Problem Context}
Succinct non-interactive arguments (SNARGs) are among the most important and actively studied primitives in modern cryptography. A SNARG enables a prover to convince a verifier of some computational statement (such as knowledge of a solution to a complex equation or the validity of a transaction according to network rules) using a proof that is extremely short and efficient to verify. Remarkably, the proof size and verification time can be exponentially smaller than the time required to check the statement directly. The succinct proof and verification is critical for real-world applications. For instance:
\begin{itemize}
    \item \textit{Verifiable Outsourced Computation:} A weak client (like a smartphone) can offload a heavy computation to a cloud server and receive a short proof that the result is correct, without trusting the server.
    \item \textit{Blockchain Scalability:} Blockchains can use SNARGs (often in the form of zk-SNARKs) to compress the verification of thousands of transactions into a single tiny proof, enabling massive throughput improvements (e.g., ZK-Rollups).
\end{itemize}

Constructions of SNARGs are notoriously difficult despite the active research. For decades, researchers have sought to build them based on \textbf{standard cryptographic assumptions}---assumptions that are widely believed to be true and have withstood years of cryptanalytic attacks. The ``holy grail'' in this area is to build a SNARG from the \textbf{learning with errors (LWE)} assumption. LWE is the foundation of modern lattice-based cryptography; it is believed to be secure even against quantum computers (post-quantum security) and allows for powerful functionalities like fully homomorphic encryptions.

Existing SNARG constructions typically fall short of this ideal in the following ways:
\begin{enumerate}
    \item \textit{Idealized Models:} The most efficient SNARGs are usually constructed in the random oracle model (ROM), where hash functions are modeled as perfect random functions. However, random oracles cannot be instantiated in the real world and thus only provide heuristic security.
    \item \textit{Non-Standard Assumptions:} Another common way to build SNARGs (with knowledge) relies on assumptions like ``Knowledge of Exponent,'' which are non-falsifiable and sometimes imply conclusions too strong to be true. 
    \item \textit{Indistinguishability Obfuscation (iO):} While iO can build almost anything, including SNARGs, current constructions of iO are extremely complex, inefficient, and rely on very heavy mathematical machinery.
\end{enumerate}

The paper ``SNARGs for NP from LWE'' by Ziyi Guan and Eylon Yogev \cite{YogevGuan2025} claimed to achieve the breakthrough: a SNARG for all of NP based solely on the sub-exponential hardness of LWE. Their proposed construction involved two novel components: ``PCP shadows'' (a way to compress probabilistically checkable proofs) and ``function vector commitments'' (a lattice-based vector commitment scheme).

\subsubsection*{AI Contribution via Advanced Prompting}

As part of an experiment to evaluate the capability of AI models in theoretical computer science research, we analyzed the preprint of the paper using a Google-internal LLM. Crucially, we did not simply ask the model to ``check the proof,'' as standard prompts often lead to superficial reviews or hallucinations when dealing with complex mathematical logic.

Instead, we employed a rigorous \textbf{iterative self-correction prompt}. The prompt was designed to force the model to act as a meticulous and adversarial reviewer. 
The explicit instructions are in \Cref{fig:snarg:prompt}. 

\begin{figure}[!htb]
\begin{userturn}
\begin{enumerate}
    \item \textbf{Generate an initial review} that was strictly objective, focusing only on identifying errors and suggesting improvements.
    \item \textbf{Self-correct its first review} by rigorously critiquing its own findings. The model was told to verify every derivation, check for hallucinations, and ensure that any claim of an error was substantive.
    \item \textbf{Generate a revised review} incorporating these corrections.
    \item \textbf{Perform a second round of self-correction} to further refine the logic and ensure comprehensive coverage of the paper, including appendices.
    \item \textbf{Produce a final, verified review} adhering to strict mathematical standards.
\end{enumerate}
For your proof, adhere to: $\langle\text{rigor text}\rangle$.
\end{userturn}
\caption{Iterative self-correction prompt.}
\label{fig:snarg:prompt}
\end{figure}

Furthermore, the prompt included a specific protocol for mathematical rigor. It required the model to distinguish between a ``Complete Proof'' (only if every step was verified) and ``Structured Partial Progress'' (if any gaps were found). It also required the model to explicitly flag gaps or unproven assumptions using a specific tag. This approach was used by Google recently for providing preliminary feedback to authors of STOC submissions, see \cite{CW25}. 

\subsubsection*{The Finding and Verification}
Using this method, the model's final review flagged a critical, fatal flaw in the paper's construction of ``PCP with shadow soundness'' (Section 4).

The issue lay in the discrepancy between the definition of the scheme and its actual construction:
\begin{itemize}
    \item \textbf{The Definition:} Definition 4.1 (Property 3) of the paper required \textit{perfect consistency}. This meant that for any two valid proofs that shared a local view, the generated ``shadows'' (succinct digests of the proof) must be identical for \textit{every single choice} of randomness used in the shadow generation process.
    \item \textbf{The Construction:} The construction presented in Section 4.3, however, only achieved \textit{imperfect consistency}. This meant that the shadows would match with high probability over the choice of randomness, but there could exist bad randomness values where they differed.
\end{itemize}

The model argued that this was not a minor technicality. The security proof for the SNARG (specifically the soundness argument) relied heavily on the strong, perfect consistency definition. It assumed that once a commitment was opened, the underlying shadow was unique and well-defined. By only achieving statistical consistency, the construction opened the door to an adversary who could find specific bad randomness values to break the binding between the proof and its shadow. This would allow an adversary to forge a valid argument for a false statement, completely undermining the security of the SNARG.

\subsubsection*{Expert Verification and Outcome}
To validate the model's finding, we shared the detailed critique with cryptography experts Aayush Jain and Zhengzhong Jin. They performed an independent review of the paper and the model's output.

Their assessment was definitive: the model was correct. They confirmed that the gap between the definition and construction of perfect consistency was a fatal flaw in the logic. They noted that while the model produced some noise (flagging other, less relevant issues), this specific insight was accurate and hit the core of the problem. Without perfect consistency, the reduction used in the security proof fails, and the main theorem does not hold.

We subsequently communicated these findings to the paper's authors, Ziyi Guan and Eylon Yogev. They acknowledged the validity of the issue raised by the model. They agreed that the gap between the definition and construction of consistency was a significant error. Following this, they updated their paper on the Cryptology ePrint Archive. The updated manuscript includes a note in red text explicitly stating that a gap in the proof of the main theorem had been found.

This episode serves as a powerful testament to the evolving capabilities of large language models in scientific research. While LLMs are often criticized for their inability to reason or their tendency to hallucinate, this case shows that with the right prompting strategy---specifically one that enforces iterative self-correction and rigorous verification---they can identify subtle, deep technical flaws in state-of-the-art research.



\newpage

\section{Cross-Pollination of Ideas}
In this section we see AI acting as an interdisciplinary bridge, retrieving obscure theorems from entirely different mathematical domains to overcome roadblocks.

\subsection{Approximation Algorithms: Max-Cut}\label{sec:max_cut}
\textit{Written by Euiwoong Lee.}

\subsubsection*{Problem Context}
The Goemans-Williamson (GW) algorithm uses Semidefinite Programming (SDP) to achieve an optimal 0.878-approximation for Max-Cut. A long-standing open question asks whether rounding SDP solutions of {\em bounded rank} $d$ can yield a better approximation ratio. While known for $d=2,3$, the case for general $d$ remained open.

I realized that this could be resolved if the following simple and natural geometric property held: Does the variance of a sum of random variables $X = \sum \text{sgn}(\langle g, v_i \rangle)$ (where $v_i$ are unit vectors with bounded pairwise inner products) have a lower bound $C(d) n^2$?

\subsubsection*{AI Contribution}
The model solved this geometric variance question. 
\begin{itemize}
    \item \textbf{Synthesizing Advanced Math:} The model recognized this as a problem solvable via \emph{geometric functional analysis}, a field outside the typical approximation algorithms toolkit.
    \item \textbf{Proof via Measure Theory:} It constructed a proof using the compactness of the space of probability measures on the sphere.
    \item \textbf{Iterative Rigor:} When we questioned the validity of using ``relaxed" measure spaces, the model justified its steps by citing and applying the Stone-Weierstrass Theorem, the Riesz Representation Theorem, and properties of spherical harmonics.
\end{itemize}

\subsubsection*{Max-Cut on Bounded-dimension SDP Solutions}

Max-Cut is one of the most fundamental problems in combinatorial optimization and approximation algorithms. 
Given an undirected graph $G = (V, E)$, the goal is to partition $V$ into $V^+$ and $V^-$ to maximize the number of edges between $V^+$ and $V^-$; equivalently, compute $x : V \to \{ \pm 1 \}$ to maximize $\sum_{(u, v) \in E} \frac{1 - x(u)x(v)}{2}$. 

While combinatorial or linear programming (LP)-based methods only yielded a $1/2$-approximation (e.g., the returned solution's value is at least half of the optimal value) in polynomial time, the celebrated result of Goemans and Williamson~\cite{goemans1995improved} used a semidefinite programming (SDP) relaxation to obtain an approximation ratio of $\alpha_{GW} := \min_{a \in [-1, +1]} \frac{\arccos(a)/\pi}{(1-a)/2} \approx 0.878$, which was later proved to be optimal for all polynomial-time algorithms assuming the Unique Games Conjecture~\cite{khot2007optimal}. 
This tight relationship between algorithms, complexity, and SDPs was later extended to numerous other combinatorial optimization tasks, most notably all Constraint Satisfaction Problems (CSPs)~\cite{raghavendra2008optimal}. 

Given the importance of this result, it is natural to ask which additional assumptions allow polynomial-time algorithms to achieve an approximation ratio $\alpha_{GW} + C$ for some constant $C > 0$. Structural properties of the input graph $G = (V, E)$ have been the most popular assumptions, which include when $G$ is dense (i.e., $|E| = \Omega(n^2)$~\cite{arora1995polynomial}), $G$ is bounded-degree~\cite{feige2002improved, hsieh2023approximating}, or $G$ has a certain expansion property~\cite{barak2011rounding}.

In order to motivate assumptions depending on the structure of SDP solutions, let us briefly recall Goemans-Williamson's SDP relaxation. Given $G = (\{1, \dots, n \}, E)$, their SDP relaxation, with a symmetric matrix $M \in \R^{n \times n}$ as the variable, is as follows: 
\begin{align*}
\mbox{maximize } \quad
& \sum_{(i, j) \in E} \frac{1 - M_{i,j}}{2} \\
\mbox{subject to } \quad & M_{i,i} = 1 \qquad \qquad \forall i \in [n]. \\
& M \succeq 0.
\end{align*}
The fact that this SDP is a relaxation for Max-Cut (i.e., the SDP optimal value is at least the Max-Cut optimal value) is based on the fact that for any solution $x : V \to \{ \pm 1 \}$ for Max-Cut, $M = xx^T$ is a feasible solution for the SDP. 
For the $\alpha_{GW}$-approximation, the algorithm is the following: 
\begin{enumerate}
    \item Compute the optimal SDP solution $M$.
    \item Compute its Gram decomposition $M = VV^T$ for some $V \in \R^{n \times d}$.  Let $v_i$ be the $i$th row of $V$.
    \item Sample a random Gaussian vector $g \sim N(0, I_d)$.
    \item For each $i \in [n]$, let $x_i \leftarrow \sgn(\langle v_i, g \rangle)$. (I.e., $x_i \leftarrow +1$ if $\langle v_i, g \rangle \geq 0$ and $x_i \leftarrow -1$ otherwise.)
\end{enumerate}
The $\alpha_{GW}$-approximation can be proved by a simple edge-by-edge analysis where each edge $(i, j)$ contributes $\frac{1 - \langle v_i, v_j \rangle}{2}$ to the SDP objective, and $\arccos(\langle v_i, v_j \rangle)/\pi$ to the Max-Cut objective in expectation (because it is the probability that $\sgn(v_i, g) \neq \sgn(v_j, g)$).

So, how about the structural properties of the optimal SDP solution $M$? Perhaps one of the most natural assumptions on $M$ is its rank; the gist of the Max-Cut problem itself is to force the rank of $M$ to be $1$. 
Then the question is, if $M$ is a feasible solution to the SDP with rank $d$, can we {\em round} $M$ to a Max-Cut solution of value at least $(\alpha_{GW} + C(d))$ times the SDP value of $M$, where $C(d) > 0$ is a constant depending only on $d$? 

This is a natural question asked since the Goemans-Williamson algorithm, and Goemans gave a positive answer when $d = 2$. Avidor and Zwick~\cite{avidor2005rounding} refined it and gave a positive answer for $d = 3$ as well, but to the best of my knowledge, even the case $d = 4$ is open. 
It is surprising, especially given that analogous questions have already been answered for more general {\em Grothendieck problems}~\cite{briet2010positive, briet2014grothendieck}; however, they are not directly comparable, as the optimal ratios for Grothendieck problems are worse than $\alpha_{GW}$.

While studying the recent result of Hsieh and Kothari~\cite{hsieh2023approximating}, which simplifies and extends the result of Feige, Karpinski, and Langberg~\cite{feige2002improved} for Max-Cut on bounded-degree graphs, I realized that their algorithmic framework could be applied to the bounded-dimension case as well, modulo one very natural geometric question, which is stated below. 

The rest of the section details how the model ``solved it.'' 
Now, there exists a polynomial-time algorithm that, given an SDP solution $M$ of rank $d$, produces a Max-Cut solution whose value is at least $(\alpha_{GW} + C(d))$ times the value of $M$, for some $C(d) > 0$.  
But, before detailing the answer, let us mention that the answer is not completely satisfactory in the sense that it does not produce any explicit expression for $C$ as a function of $d$; in contrast, when $G$'s degree is at most $d$, the best-known approximation ratio is $(\alpha_{GW} + \widetilde{\Omega}(\frac{1}{d^2 \log d}))$~\cite{hsieh2023approximating}.  

On the other hand, the proof is based on results from geometric analysis, including the compactness of a certain space of probability measures, which have not been used much in the design of approximation algorithms. 
Will it be among the first of many examples where such analytic tools are used in approximation algorithms (and inspire more ``constructive proofs'' with ``more explicit bounds'')?
In any case, this experience reinforces common beliefs that large language models can help researchers in areas where they are not experts.

\subsubsection*{Geometric Question}
\label{subsec:geoq}
The original question I asked the model is depicted in \Cref{fig:max:cut:original:prompt}. 

\begin{figure}[!htb]
\begin{userturn}
\begin{problem}
Let $v_{1}, \dots, v_{n}$ be unit vectors in $\mathbb{R}^{d}$ such that $\inner{v_{i}}{v_{j}} \ge -\beta$ for every $i \ne j$, where $\beta=0.9$. Let $g \sim N(0,I_{d})$, $X_{i} = \sgn(\inner{g}{v_{i}})$, and $X = \sum_{i=1}^{n}X_{i}$. Is it always true that $|X| \ge \varepsilon n$ with probability at least $\delta$, for some $\varepsilon > 0$ and $\delta > 0$ depending only on the dimension $d$?
\end{problem}
\end{userturn}
\caption{Original prompt for model}
\label{fig:max:cut:original:prompt}
\end{figure}

Using the standard second-moment methods and the fact that $\E[X_i] = 0$, the model reduced it to the following equivalent question about the variance of $X$, which is indeed the version one needs for the Max-Cut application.

\begin{problem}
The variance of $X$ is
$
\E[X^{2}] = \sum_{i=1}^{n}\sum_{j=1}^{n} h(\inner{v_{i}}{v_{j}})
$,
where $h(t) = \frac{2}{\pi}\arcsin(t)$. 
Is $\E[X^{2}] \ge C(d,\beta)n^{2}$ for some constant $C(d,\beta) > 0$?
\label{problem:var}
\end{problem}

Due to space constraints, let me not fully elaborate on how a positive answer to \Cref{problem:var} leads to an improved algorithm for Max-Cut in the bounded-dimension case. It is simple to check formally, because one can just replace Lemma 7 of~\cite{hsieh2023approximating} with the answer to \Cref{problem:var}, and the rest of their (elegant) seven-page paper works as is. 
Intuitively, their algorithm, which is originally suggested by~\cite{feige2002improved}, performs the standard rounding algorithm mentioned above, followed by the natural {\em local improvement} step; if there is a way to flip the assignment to a vertex that strictly improves the value, execute it. (The actual algorithm implements it in a more conservative way.) 
In this context, 
\Cref{problem:var} seems natural, because from the perspective of a vertex $v$, the more variance its neighbors have, the bigger chance there is that the standard rounding algorithm (which already guarantees an $\alpha_{GW}$-approximation) made a wrong decision for $v$ and the local improvement step further improves upon it by flipping $v$. In particular, the answer to \Cref{problem:var} is applied to the neighborhood of every vertex. 

\subsubsection*{First Response}
{\em The response in this section is fully generated by the model, except for removing standard applications of second-order methods.}
\begin{aigenerated}
Let $\mu = \frac{1}{n}\sum_{i=1}^{n}\delta_{v_{i}}$ be the empirical probability measure on the unit sphere $S^{d-1}$. The normalized variance is:
\[
Q(\mu) = \frac{1}{n^{2}}\E[X^{2}] = \iint_{S^{d-1}\times S^{d-1}} h(\inner{x}{y}) d\mu(x)d\mu(y).
\]
The constraint on the vectors implies that the support $S_{\mu}$ of $\mu$ satisfies $\inner{x}{y} \ge -\beta$ for all $x, y \in S_{\mu}$. Let $\mathcal{M}_{\beta}(d)$ be the set of all probability measures on $S^{d-1}$ satisfying this support constraint. We define:
\[
C(d,\beta) = \inf_{\mu \in \mathcal{M}_{\beta}(d)} Q(\mu).
\]

\begin{lemma}
For finite $d \ge 1$ and $0 \le \beta < 1$, $C(d,\beta) > 0$.
\end{lemma}

\begin{proof}
We proceed in three steps:

\paragraph{1. Compactness and Continuity:} Since $d$ is finite, $S^{d-1}$ is compact. The space of probability measures $\mathcal{P}(S^{d-1})$ is compact in the weak topology. We show that $\mathcal{M}_{\beta}(d)$ is a closed subset. Let $\mu_{k} \in \mathcal{M}_{\beta}(d)$ converge weakly to $\mu$.

Let $S_{k}$ and $S$ be their respective supports. Let $x, y \in S$. For any $\varepsilon > 0$, the balls $B(x,\varepsilon/2)$ and $B(y,\varepsilon/2)$ have positive measure under $\mu$. By the Portmanteau theorem, for sufficiently large $k$, $\mu_{k}(B(x,\varepsilon/2)) > 0$. Thus, there exist $x_{k} \in S_{k}$ and $y_{k} \in S_{k}$ such that $\|x-x_{k}\| < \varepsilon/2$ and $\|y-y_{k}\| < \varepsilon/2$.

We have $\inner{x_{k}}{y_{k}} \ge -\beta$. Also,
\[
|\inner{x}{y} - \inner{x_{k}}{y_{k}}| \le \|x-x_{k}\|\|y\| + \|x_{k}\|\|y-y_{k}\| < \varepsilon/2 + \varepsilon/2 = \varepsilon.
\]
Thus, $\inner{x}{y} \ge \inner{x_{k}}{y_{k}} - \varepsilon \ge -\beta - \varepsilon$. Since $\varepsilon$ is arbitrary, $\inner{x}{y} \ge -\beta$. Therefore, $\mu \in \mathcal{M}_{\beta}(d)$. Thus, $\mathcal{M}_{\beta}(d)$ is compact.

Since $h(\inner{x}{y})$ is continuous, $Q(\mu)$ is continuous in the weak topology. Therefore, the infimum $C(d,\beta)$ is attained by some measure $\mu^{*} \in \mathcal{M}_{\beta}(d)$.

\paragraph{2. Moment Expansion:} The Taylor series of $h(t)$ is $h(t) = \sum_{m \text{ odd}, m \ge 1} a_{m}t^{m}$, where $a_{m} > 0$ for all odd $m$. Let $T_{m}(\mu) = \int x^{\otimes m} d\mu(x)$ be the $m$-th moment tensor. We can write:
\[
Q(\mu) = \sum_{m \text{ odd}} a_{m} \iint \inner{x}{y}^{m} d\mu(x)d\mu(y) = \sum_{m \text{ odd}} a_{m} \|T_{m}(\mu)\|^{2}.
\]

\paragraph{3. Symmetry and Contradiction:} Suppose $C(d,\beta) = 0$. Then $Q(\mu^{*}) = 0$. Since $a_{m} > 0$, this implies $T_{m}(\mu^{*}) = 0$ for all odd $m$. This means $\int P(x) d\mu^{*}(x) = 0$ for any odd polynomial $P(x)$. By the Stone-Weierstrass theorem, odd polynomials are dense in the space of continuous odd functions on $S^{d-1}$ (in the uniform norm).

Thus, $\int g(x) d\mu^{*}(x) = 0$ for any continuous odd function $g$. This implies that $\mu^{*}$ is a symmetric measure (i.e., $\mu^{*}(A) = \mu^{*}(-A)$).

Let $S^{*}$ be the support of $\mu^{*}$. If $\mu^{*}$ is symmetric, $S^{*}$ must be symmetric. (If $x \in S^{*}$ but $-x \notin S^{*}$, there is an open ball $B(-x,r)$ such that $\mu^{*}(B(-x,r)) = 0$. By symmetry, $\mu^{*}(B(x,r)) = \mu^{*}(-B(-x,r)) = 0$, contradicting $x \in S^{*}$).

Since $S^{*}$ is non-empty and symmetric, there exists $x \in S^{*}$ such that $-x \in S^{*}$. The support constraint requires $\inner{x}{-x} \ge -\beta$. But $\inner{x}{-x} = -\|x\|^{2} = -1$. So $-1 \ge -\beta$ or $\beta \ge 1$. This contradicts the assumption $\beta = 0.9 < 1$.

Therefore, $C(d,\beta) > 0$.
\end{proof}

This establishes that $\E[X^{2}] \ge C(d,\beta)n^{2}$.
\end{aigenerated}

\subsubsection*{Interactions}

The first response, though it looked like a natural and sound approach using geometric functional analysis, had some gaps in the proof and was not fully verifiable, partially due to my inexpertise on the topic. Two major hurdles were (1) the foundations of measure theory to establish the compactness of $\mathcal{M}_{\beta}(d)$ and (2) working knowledge of spherical harmonics to see how $T_m(\mu^*) = 0$ for all odd $m$ implies $\int P(x) d\mu^*(x) = 0$ for any odd polynomial $P(x)$. Note that (2) was somewhat expected as the previous results for $d=2,3$~\cite{avidor2005rounding} and Grothendieck problems~\cite{briet2010positive, briet2014grothendieck} used spherical harmonics and related Gegenbauer polynomials, but the use of (1), via {\em relaxing} $n$-point discrete measures to $\mathcal{M}_{\beta}(d)$, is a type of an argument rarely used in the approximation algorithms literature. 
In order to formally understand it, one needs to review the foundations of measure theory. For example, what is the weak topology? How is the Stone-Weierstrass theorem formally applied here? 
How does $\int g(x) d\mu^*(x) = 0$ for every odd $g$ imply $\mu^*$ is a symmetric measure?

After several iterations of asking for the reference behind each nontrivial step and manually searching each reference and feeding it back to the model, it produced the output in the next section. Now it is understandable to me, whose only formal exposure to analysis was an undergraduate course almost twenty years ago.

\subsubsection*{Second Response}

{\em This response is fully generated by the model, except for (1) removing standard applications of second-order methods, (2) adding two footnotes further justifying some steps, and (3) minor cosmetic changes.}

\begin{aigenerated}
\paragraph{Foundational Theorems}

We rely on the following established theorems, sourced externally (e.g., via Google Search, Wikipedia, or standard textbooks).

\paragraph{External Result 1 (Stone-Weierstrass Theorem (Real Version)).}
Suppose $X$ is a compact Hausdorff space and $A$ is a subalgebra of $C(X,\mathbb{R})$ which contains a non-zero constant function. Then $A$ is dense in $C(X,\mathbb{R})$ (under the uniform norm) if and only if it separates points.

\paragraph{External Result 2 (Riesz Representation Theorem).}
Let $\Omega$ be a compact Hausdorff space. The dual space $C(\Omega)^{*}$ is isomorphic to the space of finite regular signed measures on $\Omega$.

\paragraph{External Result 3 (Banach-Alaoglu Theorem).}
The closed unit ball of the dual space of a normed vector space is compact in the weak* topology.

\paragraph{External Result 4 (Definition and Properties of Spherical Harmonics).}
\begin{enumerate}
    \item \textbf{(Definition):} Spherical harmonics of degree $k$ on $S^{d-1}$ are the restrictions to the sphere of homogeneous harmonic polynomials of degree $k$ in $\mathbb{R}^{d}$. Let $H_{k}$ denote this space.
    \item \textbf{(Completeness):} The spaces $H_{k}$ are mutually orthogonal in $L^{2}(S^{d-1})$. The algebraic sum $\bigoplus_{k=0}^{\infty}H_{k}$ is dense in $C(S^{d-1})$ and $L^{2}(S^{d-1})$.
    \item \textbf{(Parity):} A spherical harmonic $Y_{k} \in H_{k}$ has parity $(-1)^{k}$. That is, $Y_{k}(-x)=(-1)^{k}Y_{k}(x)$ for $x \in S^{d-1}$.
\end{enumerate}

\paragraph{External Result 5 (Addition Theorem for Spherical Harmonics).}
Let $d \ge 2$ and $\lambda=(d-2)/2$. The Gegenbauer polynomial $C_{k}^{\lambda}(\inner{x}{y})$ is a positive semi-definite kernel. The energy integral 
\[
I_{k}(\mu) = \iint C_{k}^{\lambda}(\inner{x}{y}) d\mu(x)d\mu(y)
\]
is non-negative. $I_{k}(\mu)=0$ if and only if $\mu$ is orthogonal to $H_{k}$.

\paragraph{Energy Minimization Framework}

If $d=1$, the constraint forces all $v_{i}$ to be identical, so $|X|=n$. Assume $d \ge 2$. Let $\mu_{n} = \frac{1}{n}\sum \delta_{v_{i}}$. The normalized variance is:
\[
F(\mu_{n}) = \frac{\E[X^{2}]}{n^{2}} = \iint K(\inner{x}{y}) d\mu_{n}(x)d\mu_{n}(y),
\]
where $K(t) = \frac{2}{\pi}\arcsin(t)$. 
Let $\mathcal{M}_{\beta}^{\prime}(S^{d-1})$ be the set of probability measures $\mu$ on $S^{d-1}$ such that for any distinct $x, y$ in the support $S_{\mu}$, $\inner{x}{y} \ge -\beta$. We want to show $C(d) = \inf_{\mu \in \mathcal{M}_{\beta}^{\prime}} F(\mu) > 0$.

\paragraph{Compactness and Continuity}

We consider $\mathcal{M}(S^{d-1})$ with the weak-* topology. $S^{d-1}$ is a compact metric space.

\begin{lemma}[Compactness of $\mathcal{M}(S^{d-1})$]
\label{lem:compactness}
$\mathcal{M}(S^{d-1})$ is compact in the weak-* topology.
\end{lemma}

\begin{proof}
By External Results 2 and 3, $\mathcal{M}(S^{d-1})$ is identified as a weak-* closed subset of the unit ball in $C(S^{d-1})^{*}$, which is compact by Banach-Alaoglu.
\end{proof}

\begin{lemma}[Portmanteau Lemma (Partial)]
\label{lem:portmanteau}
If $\mu_{k} \to \mu$ weakly in a metric space $X$, then for every open set $G \subset X$, $\liminf_{k} \mu_{k}(G) \ge \mu(G)$.
\end{lemma}

\begin{proof}
Let $G$ be open, $F=G^{c}$. Define the continuous function $f_{m}(x) = \min(1, m \cdot \text{dist}(x,F))$. As $m \to \infty$, $f_{m}(x) \nearrow \mathbf{1}_{G}(x)$.
\[
\liminf_{k} \mu_{k}(G) \ge \liminf_{k} \int f_{m} d\mu_{k} = \int f_{m} d\mu.
\]
By the Monotone Convergence Theorem, as $m \to \infty$, the RHS converges to $\mu(G)$.
\end{proof}

\begin{lemma}
$\mathcal{M}_{\beta}^{\prime}(S^{d-1})$ is a closed subset of $\mathcal{M}(S^{d-1})$ and thus compact.
\end{lemma}

\begin{proof}
Let $\mu_{k} \in \mathcal{M}_{\beta}^{\prime} \to \mu$. Suppose $\mu \notin \mathcal{M}_{\beta}^{\prime}$. There exist distinct $x, y \in S_{\mu}$ such that $\inner{x}{y} < -\beta$. We find disjoint open balls $B_{x}, B_{y}$ where the inner product remains $< -\beta$. Since $x, y \in S_{\mu}$, $\mu(B_{x}) > 0$. By \Cref{lem:portmanteau}, $\liminf_{k} \mu_{k}(B_{x}) > 0$. For large $k$, there exist $x_{k}, y_{k}$ in the support of $\mu_{k}$ within $B_{x}$ and $B_y$.
Thus $x_{k} \ne y_{k}$. As $\mu_{k} \in \mathcal{M}_{\beta}^{\prime}$, $\inner{x_{k}}{y_{k}} \ge -\beta$. Contradiction.
\end{proof}

\begin{lemma}[Product Convergence]
\label{lem:product_convergence}
If $\mu_{k} \to \mu$ weakly on a compact space $X$, then $\mu_{k} \times \mu_{k} \to \mu \times \mu$ weakly on $X \times X$.
\end{lemma}

\begin{proof}
Let $H \in C(X \times X)$. By External Result 1, separable functions $H_{\text{approx}}(x,y) = \sum f_{i}(x)g_{i}(y)$ are dense. For $H_{\text{approx}}$, the convergence holds by definition of $\mu_{k} \to \mu$. The result extends to $H$ by density.
\end{proof}

\begin{lemma}
$F(\mu)$ is continuous on $\mathcal{M}(S^{d-1})$.
\end{lemma}

\begin{proof}
$H(x,y) = K(\inner{x}{y})$ is continuous. If $\mu_{k} \to \mu$, by \Cref{lem:product_convergence}, $\mu_{k} \times \mu_{k} \to \mu \times \mu$. By definition of weak convergence, $F(\mu_{k}) \to F(\mu)$.
\end{proof}

Since $F$ is continuous and $\mathcal{M}_{\beta}^{\prime}$ is compact, the infimum $C(d)$ is attained.

\paragraph{Strict Positivity via Harmonic Analysis}

We show $C(d) > 0$. Let $\lambda = (d-2)/2$. We write $K(t) = \sum_{k=0}^{\infty} b_{k}C_{k}^{\lambda}(t)$. The energy is $F(\mu) = \sum_{k=0}^{\infty} b_{k}I_{k}(\mu)$. By External Result 5, $I_{k}(\mu) \ge 0$.

\begin{lemma}[Monomial Expansion Positivity]
\label{lem:monomial_positivity}
Let $d \ge 2$ (so $\lambda \ge 0$). The expansion of $t^n$ in the basis of Gegenbauer polynomials $C_{k}^{\lambda}(t)$ has strictly positive coefficients.
\end{lemma}

\begin{proof}
Case $\lambda > 0$ ($d > 2$). The inversion formula is known to be:\footnote{{\em The section author's note: It can be proved by induction using the recurrence relation given in
\url{https://en.wikipedia.org/wiki/Gegenbauer_polynomials}.}
}
\[
t^{n} = \frac{n!}{2^{n}} \sum_{k=0}^{\lfloor n/2 \rfloor} \frac{(n-2k+\lambda)}{\lambda \cdot k! \cdot (\lambda+1)_{n-k}} C_{n-2k}^{\lambda}(t)
\]
where $(a)_{m}$ is the rising factorial. Since $\lambda > 0$, $n-2k \ge 0$, and the factorials and Pochhammer symbols are positive, the coefficients are strictly positive.

Case $\lambda = 0$ ($d=2$). The basis are Chebyshev polynomials $T_{k}(t)$. The inversion formula is:
\[
t^{n} = \frac{1}{2^{n-1}} \sum_{k=0}^{\lfloor n/2 \rfloor} \binom{n}{k} \varepsilon_{n-2k} T_{n-2k}(t)
\]
where $\varepsilon_{j}$ is $1/2$ if $j=0$ and $1$ if $j>0$. The coefficients are positive.
\end{proof}

We analyze the coefficients $b_{k}$ for $K(t) = \frac{2}{\pi}\arcsin(t)$.
\begin{enumerate}
    \item \textbf{Parity:} $K(t)$ is odd. $C_{k}^{\lambda}(t)$ has the parity of $k$. Thus, $b_{k}=0$ for even $k$.
    \item \textbf{Positivity of $b_{k}$:} $\arcsin(t) = \sum c_{m}t^{2m+1}$ with $c_{m} > 0$. By \Cref{lem:monomial_positivity}, $t^{2m+1}$ is a sum of $C_{2j+1}^{\lambda}(t)$ with strictly positive coefficients. Therefore, $b_{k} > 0$ for all odd $k$.
\end{enumerate}

Suppose $F(\mu) = 0$. Since $b_{k} > 0$ (odd $k$) and $I_{k}(\mu) \ge 0$, we must have $I_{k}(\mu)=0$ for all odd $k$. By External Result 5, $\mu$ is orthogonal to $H_{k}$ for all odd $k$.

\begin{lemma}[Parity of Polynomials and Spherical Harmonics]
\label{lem:parity_polynomials}
A polynomial $P(x)$ restricted to $S^{d-1}$ is an odd function if and only if its expansion in the basis of spherical harmonics consists solely of spherical harmonics of odd degree.
\end{lemma}

\begin{proof}
By External Result 4 (Completeness and Definition), the space of polynomials restricted to $S^{d-1}$ is the algebraic direct sum $\bigoplus H_{k}$. Let $P(x) = \sum_{k=0}^{N} Y_{k}(x)$, where $Y_{k} \in H_{k}$. By External Result 4 (Parity), $Y_{k}(-x) = (-1)^{k}Y_{k}(x)$.
So $P(-x) = \sum_{k=0}^{N} (-1)^{k}Y_{k}(x)$. If $P$ is odd, $P(-x) = -P(x)$. Then $\sum_{k=0}^{N} ((-1)^{k}+1)Y_{k}(x) = 0$. Since the spaces $H_{k}$ are orthogonal, they are linearly independent.
Thus $((-1)^{k}+1)Y_{k}(x) = 0$ for all $k$. If $k$ is even, $2Y_{k}(x) = 0$, so $Y_{k} = 0$. Thus, $P(x)$ consists only of odd degree harmonics. The converse is immediate.
\end{proof}

We show $F(\mu)=0$ implies $\mu$ is symmetric (i.e., $\int f d\mu = 0$ for any continuous odd function $f$).
Let $f$ be continuous and odd. By External Result 1, let polynomials $P_{n} \to f$ uniformly. The odd part $Q_{n}(x) = (P_{n}(x)-P_{n}(-x))/2$ is an odd polynomial, and $Q_{n} \to f$ uniformly.
By \Cref{lem:parity_polynomials}, $Q_{n}$ is composed solely of spherical harmonics of odd degree. Since $\mu$ is orthogonal to these, $\int Q_{n}(x) d\mu(x) = 0$. By uniform convergence, $\int f(x) d\mu(x) = 0$. Thus, $\mu$ is symmetric.\footnote{{\em The section author's note: For the formal justification of this step, note that the pushforward $A_* \mu$ for the antipodal map $A(x) = -x$ satisfies $\int f d\mu = \int f dA_*\mu$ for every continuous $f$. As the 
Riesz Representation Theorem shows that $\mu$ is fully characterized by $(\int f d \mu)_{f \in C(S^{d-1})}$, we can conclude that $\mu = A_* \mu$.}}

Now we use the constraint $\mu \in \mathcal{M}_{\beta}^{\prime}$. Suppose $C(d)=0$. The minimizer $\mu^{*}$ has $F(\mu^{*})=0$, so $\mu^{*}$ is symmetric. Let $S$ be the support of $\mu^{*}$. Let $x \in S$. Since $\mu^{*}$ is symmetric, $-x \in S$. Since $\|x\|=1$, $x \ne -x$. By the definition of $\mathcal{M}_{\beta}^{\prime}$, $\inner{x}{-x} \ge -\beta$. This means $-1 \ge -\beta$, or $\beta \ge 1$.

This contradicts $\beta=0.9$. Therefore, $C(d) > 0$.
\end{aigenerated}

\newpage

\subsection{Computational Geometry: Steiner Trees}\label{sec:steiner}
\textit{Written by Karthik C. S.}

\subsubsection*{Problem Context}
The "Simplex is the Best for Graph Embeddings" conjecture states that among all graphs with $m$ edges embedded into Euclidean space   as characteristic vectors, the star graph configuration yields the minimum Steiner tree cost. Proving this is a key step towards establishing APX-hardness for high-dimensional Steiner Tree problems.

\subsubsection*{AI Contribution}
The model's involvement was decisive.
\begin{itemize}
    \item \textbf{The Crucial Link:} The model suggested an approach based on the \emph{Kirszbraun Extension Theorem}, which guarantees the existence of Lipschitz-preserving extensions of maps between Hilbert spaces.
    \item \textbf{Overcoming "Elementary" Bias:} Initially, the model dismissed this approach as "non-elementary". However, when the researcher clarified that an elementary proof was not required, the model successfully formalized the argument.
    \item \textbf{Formal Proof:} The model constructed a mapping from any graph embedding to the star graph embedding and used Kirszbraun's theorem to show that a Steiner tree for the former can be transformed into a valid tree for the latter without increasing cost.
\end{itemize}

\subsubsection*{Technical Details: Steiner Trees and Graph Embeddings}

In the Euclidean Steiner Tree problem, we are given as input a set of points (called \emph{terminals}) in the $\ell_2$-metric space, and the goal is to find the minimum-cost tree connecting them. Additional points from the space (called \emph{Steiner points}) can be introduced as nodes in the solution. 

The seminal works of Arora \cite{Arora_1998} and Mitchell \cite{M99} provide a Polynomial Time Approximation Scheme (PTAS) for solving the Euclidean Steiner Tree problem in fixed dimensions. However, the problem remains poorly understood in higher dimensions (such as when the dimension is logarithmic in the number of terminals). Ruling out a PTAS for the problem in high dimensions is a notoriously long-standing open problem (for example, see Trevisan~\cite{Trevisan00}).

In \cite{Fleischmann2025}, the authors proposed various conjectures related to Steiner trees in order to make progress on understanding the inapproximability of the high-dimensional Euclidean Steiner Tree problem. Assuming these, they could prove APX-hardness of the Euclidean Steiner Tree problem in dimensions polynomial in the input size.

The simplest such conjecture, called the \emph{Simplex is the Best for Graph Embeddings Conjecture} in \cite{Fleischmann2025}, was a special case of the widely open generalized Gilbert-Pollak conjecture \cite{gilpol,Ivanov_Tuzhilin_2012,Du_Smith_1996}. In this report, we detail the resolution of this conjecture by Gemini, including the process.

First, we state the conjecture and provide the necessary context. Next, we describe our interaction with Gemini that led to its resolution. Finally, we present the proof produced by Gemini.

\subsubsection*{Research Context}\label{sec:conj}
A natural approach to prove the APX-hardness of the Euclidean Steiner Tree problem is via a reduction from the Vertex Cover problem on bounded degree graphs (for example, see \cite{Fleischmann2023}). A candidate reduction proposed in \cite{Fleischmann2025} is to embed a vertex cover instance graph $G=(V,E)$ into $\mathbb{R}^{|V|}$ by mapping each edge $\{u,v\}\in E$ to $\mathbf{e}_u + \mathbf{e}_v$, where $\mathbf{e}_u$ is the standard basis vector with $1$ in the coordinate indexed by $u$ and $0$ elsewhere. Thus, each edge is embedded as its characteristic vector. The embedding of the set of edges incident to a single node forms the vertices of a regular simplex of side length $\sqrt{2}$. The point configuration as a whole consists of the vertices of a regular simplicial complex (where we take the union of the simplices associated with each node).

Proving that the above reduction is gap-preserving, by analyzing the completeness and soundness cases, requires proving the following conjecture:

\begin{conjecture}[Euclidean Steiner Tree for Regular Simplicial Complexes~\cite{Fleischmann2025}] \label{con: simplicial complex Steiner trees}
For all constants $r \in (0,1)$ and $ \alpha \in (0, 1/r - 1)$, there exist constants $s, \beta > 0$ and $M \in \mathbb{Z}^+$ sufficiently large so that, for all $m \geq M$, given a regular, unit, simplicial complex on $m$ vertices:
\begin{enumerate}
    \item \textbf{Completeness:} If the vertices can be partitioned into the vertices of at most $rm$ unit regular simplices, then the point configuration of the $m$ vertices admits a Euclidean Steiner tree of cost at most $sm$. 
    \item \textbf{Soundness:} If the vertices cannot be partitioned into the vertices of fewer than $(1 + \alpha)rm$ unit regular simplices, then the point configuration of the $m$ vertices does not admit a Euclidean Steiner tree of cost less than $(1 + \beta)sm$.
\end{enumerate}
\end{conjecture}

As a first step toward proving the above conjecture, \cite{Fleischmann2025} proposed the following much simpler conjecture about Steiner ratios of specialized point-sets. 

\begin{conjecture}[Simplex is the Best for Graph Embeddings~\cite{Fleischmann2025}] \label{conj: simplex is the best graph embed}
Over all simple graphs with $m$ edges, the embedding\footnote{Here we allude to embedding each edge by its characteristic vector, as detailed in the aforementioned reduction from the Vertex Cover problem.} of the star graph on $m$ edges has the minimum cost Steiner tree.
\end{conjecture}

Observe that restricted to connected graphs with the same number of edges, the minimum spanning tree costs are identical. Consequently, Conjecture~\ref{conj: simplex is the best graph embed} over connected graphs is equivalent to finding the point configuration with the minimum Steiner ratio. In \cite{Fleischmann2025}, the authors verified the conjecture computationally up to $m = 10$ using the exact algorithm of Smith \cite{smith} and established the following structural property.

\begin{lemma}[Fleischmann et al.\ \cite{Fleischmann2025}] \label{lem: diam 2}
For any fixed $m \geq 1$, the graph of size $m$ whose embedding (as above) has the minimum cost Steiner tree has diameter at most $2$.
\end{lemma}

\subsubsection*{Methodology}\label{sec:method}

I requested David Woodruff to formulate a prompt to solve Conjecture~\ref{conj: simplex is the best graph embed} using the Google internal model, providing only the paper \cite{Fleischmann2025} as context.

The model proposed two approaches but did not offer a conclusive answer. The second approach involved starting from an arbitrary graph $G$ and applying local transformations that do not increase the optimal Steiner tree cost, eventually converging to the star graph. As we had already attempted this strategy, it offered no new insights.

The first approach, however, was based on the Kirszbraun Extension Theorem \cite{Kirszbraun1934,Valentine1945}. I must confess that while I was aware of various extension theorems from my work on the computation of fixed-points (see, for example, \cite{GanorSP21}), I had not immediately seen the connection to Steiner trees. To my frustration, instead of exploring this direction further, the model noted that the proof of Kirszbraun Extension Theorem is non-elementary, and concluded that this approach would not yield an elementary proof of Conjecture~\ref{conj: simplex is the best graph embed}.

As I was not seeking an elementary proof and indeed, had not requested one, this comment left me slightly frustrated, yet intrigued. Consequently, I decided to run my own experiments using Gemini Ultra (utilizing the standard ``thinking'' mode rather than ``deep think'').

Initially, I provided the exact prompt used for the internal model, but Gemini consistently declined to attempt the problem, insisting that it was an open problem that could not be solved. 

Next, I removed the paper from the context and provided only the problem statement along with the relevant definitions. At first, the model suggested a hand-waving argument; however, when I pointed out the parts that required formalization, it successfully provided them. It first derived the same lemma suggested by the internal model, though this time with a proof. As the proof remained somewhat informal, I provided Gemini with the summary of the internal model's output and with its explanation of why it had abandoned the approach. At that point, Gemini successfully formalized the lemma using the Kirszbraun Extension Theorem (the tool it had been missing earlier).

Through this process, I have learned about the power of the Kirszbraun Extension Theorem for Steiner tree computation and analysis. To the best of my knowledge, this is a new connection (yet one that feels very natural!). This discovery makes me more optimistic about the prospects of establishing the APX-hardness of computing the Euclidean Steiner Tree in high dimensions.

\subsubsection*{Gemini's Proof of Conjecture~\ref{conj: simplex is the best graph embed}}\label{sec:proof}	

\begin{aigenerated}
\begin{theorem}
Let $f$ be an embedding of a graph into Euclidean space such that for any edge $\{u, v\} \in E(G)$, $f(\{u, v\}) = e_u + e_v$, where $\{e_u: u\in V(G)\}$ are standard basis vectors. Let $\Star_m$ denote the star graph with $m$ edges, and let $G$ be any arbitrary graph with $m$ edges. Let $\SMT(P)$ denote the length of the Euclidean Steiner Minimal Tree of a point set $P$. Then:
\[
\SMT(f(\Star_m)) \le \SMT(f(G)).
\]
\end{theorem}

\begin{proof}
Let $G=(V_G, E_G)$ be a graph with $|E_G|=m$ edges and $|V_G|=k$ vertices. The embedding $f$ maps edges to points in $\mathbb{R}^{k}$. Let $P_G = \{x_1, \dots, x_m\} \subset \mathbb{R}^{k}$ be the set of $m$ points generated by the edges of $G$.
Similarly, let $\Star_m$ have vertex set $V_{\Star}$ with $|V_{\Star}| = m+1$. Its embedding lies in $\mathbb{R}^{m+1}$. Let $P_{\text{Star}} = \{y_1, \dots, y_m\} \subset \mathbb{R}^{m+1}$ be the corresponding point set generated by the edges of the star graph.

The star graph edges share a center. For any distinct pair $y_i, y_j \in P_{\Star}$, the distance is:
\[
\|y_i - y_j\| = \sqrt{2}.
\]
Thus, $P_{\Star}$ forms the vertices of a regular simplex in $\mathbb{R}^{m+1}$.

For $P_G \subset \mathbb{R}^k$, the distance between points depends on edge adjacency, where we have for adjacent edges that $\|x_i - x_j\| = \sqrt{2}$ and for disjoint edges we have $\|x_i - x_j\| = 2$.
Thus, $\forall x_i, x_j \in P_G, i \neq j: \|x_i - x_j\| \ge \sqrt{2}$.

We construct a bijection $\phi: P_G \to P_{\Star}$. Let $\phi(x_i) = y_i$ for $i = 1, \dots, m$.
We verify the Lipschitz condition for $\phi$ on the domain $P_G$. Consider any $x_i, x_j \in P_G$:
\[
\|\phi(x_i) - \phi(x_j)\| = \|y_i - y_j\| = \sqrt{2} \le \|x_i - x_j\|.
\]
The map $\phi$ is 1-Lipschitz on the set of terminal points $P_G$.

Let $T_G$ be the Euclidean Steiner Minimal Tree for $P_G$. $T_G$ is a subset of the ambient space $\mathbb{R}^k$. We must map this tree to the ambient space of the star graph, $\mathbb{R}^{m+1}$. We now use the Kirszbraun Extension Theorem \cite{Kirszbraun1934,Valentine1945} in the following formulation. 

\begin{theorem}[Theorem 1.31 in \cite{Schwartz1969NonlinearFA}]
Let $H_1$ and $H_2$ be Hilbert spaces. If $U \subset H_1$ and $\phi: U \to H_2$ is a 1-Lipschitz map, there exists an extension $\Phi: H_1 \to H_2$ such that $\Phi|_U = \phi$ and $\Phi$ is 1-Lipschitz on all of $H_1$.\label{thm:Kirz}
\end{theorem}

We apply this theorem with $H_1 = \mathbb{R}^k$ and $H_2 = \mathbb{R}^{m+1}$. The extension $\Phi: \mathbb{R}^k \to \mathbb{R}^{m+1}$ satisfies:
\[
\|\Phi(u) - \Phi(v)\| \le \|u - v\| \quad \forall u, v \in \mathbb{R}^k.
\]

Let $T_G$ denote the Euclidean Steiner Minimal Tree for the point set $P_G$. We treat $T_G$ as a graph defined by a set of vertices $V_T$ and a set of edges $E_T$:
\begin{itemize}
    \item $V_T = P_G \cup S_G$, where $P_G$ are the terminal points and $S_G$ are the Steiner points introduced to minimize the total length.
    \item The total cost is the sum of the lengths of the edges:
    \[
    \SMT(f(G)) = \Len(T_G) = \sum_{\{u, v\} \in E_T} \|u - v\|.
    \]
\end{itemize}

We construct a candidate tree $T'$ for the star graph configuration $P_{\Star}$ by applying the extended map $\Phi$ to every vertex in $V_T$. Let $V'_{T} = \{ \Phi(z) \mid z \in V_T \}$. Since $\Phi$ is an extension of $\phi$, we have $\Phi(P_G) = P_{\Star}$. Thus, $V'_{T}$ contains the required terminals $P_{\Star}$. For every edge $\{u, v\} \in E_T$, we form a corresponding edge $\{\Phi(u), \Phi(v)\}$ in $T'$. Since $T_G$ is connected, the image graph $T'$ is also connected and spans $P_{\Star}$. Thus, $T'$ is a valid Steiner tree for $P_{\Star}$.

We now compare the lengths using the 1-Lipschitz property of $\Phi$:
\[
\Len(T') = \sum_{\{u, v\} \in E_T} \|\Phi(u) - \Phi(v)\|.
\]
Since $\|\Phi(u) - \Phi(v)\| \le \|u - v\|$ for every pair $u, v$:
\[
\Len(T') \le \sum_{\{u, v\} \in E_T} \|u - v\| = \Len(T_G).
\]

Finally, since $\SMT(P_{\Star})$ is the infimum of lengths of all valid trees connecting $P_{\Star}$, it must be less than or equal to the length of our candidate tree $T'$:
\[
\SMT(P_{\Star}) \le \Len(T').
\]

Combining these inequalities:
\[\SMT(f(\Star_m)) = \SMT(P_{\Star}) \le \Len(T') \le \Len(T_G) = \SMT(f(G)).\]
\end{proof}
\end{aigenerated}

\newpage

\subsection{Graph Theory: Computing Perfect Matchings in Regular Bipartite Graphs}\label{sec:matching}
\textit{Written by Ken-ichi Kawarabayashi}

The main contribution of the AI here is to provide a solution to a problem in Graph Theory by pulling in tools from:
\begin{enumerate} 
    \item  Statistical Physics (Bethe Approximation), 
    \item  Number Theory (Coprime integers and integrality gaps)
    \item Spectral Analysis (Ihara-Bass identity)
\end{enumerate}
 
\paragraph{Remark after the first version is posted}
After we posted the first version on ArXiv,  Leonid Gurvits made a few remarks, including (i) the AI and I missed several vital citations, including improvement over Schrijver's bound below, and (ii) the AI's suggestion based on spectral graph theory is, indeed, an important factor for the solution of Kadison-Singer problem by Adam Marcus, Daniel Spielman and Nikhil Srivastava, which is one of the most fascinating and, perhaps one of the most important breakthroughs in mathematics of 21-st century. 


 Below, we add these remarks (as mentioned in ``Post Remark''). We sincerely thank Leonid Gurvits for these suggestions, which make this section even more interesting.

\subsubsection*{Problem Context}
The goal was to improve bounds on the number of perfect matchings (permanents) in regular bipartite graphs. In 1998, Schrijver provided a lower bound with a superior base~\cite{Schrijver1998}.

\begin{theorem}[Schrijver, 1998]
\label{thm:schrijver}
Any $k$-regular bipartite graph with $2n$ vertices has at least
$$ B_S(k, n) = \left(\frac{(k-1)^{k-1}}{k^{k-2}}\right)^n $$
perfect matchings.
\end{theorem}

The base in Schrijver's bound is best possible, as shown in \cite{Schrijver1980}, and the equality for all $k$ is shown by Schrijver \cite{Schrijver1998}. 
The goal is to improve not the base, but a multiplicative factor. 
 
\paragraph{Post Remark.} Both the AI and I missed the following important results \cite{GU1, gurvits2}. Gurvits has shown that 
the number of perfect matchings in $k$-regular bipartite graphs with $2n$ vertices is at least 
$$ k^n ((\frac{k-1}{k})^{k-1})^{n-k} \frac{k!}{k^k} > B_S(k, n). $$
Let us point out that the AI proved a slightly stronger form (in terms of the Bethe approximation, and this form may be of independent interest, slightly different from the improvement of Schrijver's bound). 

\subsubsection*{AI Contributions}
We fed the Schrijver's paper \cite{Schrijver1998} to the model and asked it to improve the bound. 
\begin{itemize}
    \item \textbf{Iterative Proof Discovery via Integrality Gaps:} The AI's initial proof attempt had errors (misunderstanding of the proof in \cite{Schrijver1998} and citing some wrong papers/theorems), 
    but through an iterative prompting process, the AI successfully connected Schrijver’s bound to the Bethe permanent \cite{Gurvits2011}. By identifying a number-theoretic integrality gap, the AI established a strictly stronger lower bound (a marginal improvement, but it shows a stronger form).  
        
        I should mention that the AI demonstrated a strong capacity for self-correction. Acting as its own adversarial peer reviewer (like ``STOC'' review), it critiqued its own intermediate proofs, identifying and correcting not only typographical errors but substantive logical flaws. 
     \item \textbf{Generalization via Exact Bounds:} 
     The AI analyzed cubic bipartite graphs to extract exact boundary constants. By recovering Voorhoeve’s exact differential formula, it successfully extracted a strict asymptotic improvement factor. 
     
    \item \textbf{A New Perspective:} Finally, the AI suggested a strategy based on spectral graph theory to improve further. This insight was not obvious to the researchers as it was a consequence of the ``worst'' case analysis (i.e., by considering Ramanujan/expander graphs, the AI brought the Alon-Boppana theorem \cite{AB} and Kesten-McKay laws \cite{Kes,mc})). Although the AI's improvement is marginal (though it shows a slightly stronger form), this part is perhaps interesting. 
    
    In addition, this collaboration could serve as a prime example of AI-driven cross-pollination: to make progress on a pure graph-theoretic question, the AI autonomously synthesized tools from statistical physics, number theory, and spectral analysis.

     \item \textbf{Post Remark:} 
     It turns out that the AI's suggestion for a strategy based on spectral graph theory to improve further was an essential factor for the solution of the Kadison-Singer problem by Marcus, Spielman, and Srivastava \cite{MSS2}. 
      
      Thus, this suggestion appears highly promising, given the prior success. Indeed, it is surprising that, although their work is not directly connected to counting perfect matchings in regular bipartite graphs, the AI has already identified one of the most successful approaches in 21st-century mathematics. 
\end{itemize}

\subsubsection*{Technical Details} 
For an $n \times n$ matrix $\mathbf{A}$, let $a_{i,j}$ be the entry at the $i$th row and $j$th column. The permanent of $\mathbf{A}$ is defined as
\begin{equation}
\text{perm}(\mathbf{A}) = \sum_{\sigma \in \mathcal{P}} \prod_{i=1}^{n} a_{i,\sigma(i)}
\end{equation}
where $\mathcal{P}$ is the set of all permutations of $\{1, 2, \ldots, n\}$. Computing the permanent is known to be \#P-complete~\cite{Valiant1979}.

We are interested in the case when $\mathbf{A}$ is the bi-adjacency matrix of a $k$-regular bipartite graph $G$ with $2n$ vertices. In this case, $\text{perm}(\mathbf{A})$ equals the number of perfect matchings in $G$ (such a bipartite graph always has a perfect matching by Hall's theorem (see \cite{lovasz}).

A famous result concerning the lower bound is the resolution of the Van der Waerden conjecture, proved by Falikman~\cite{Falikman1981} and Egorychev~\cite{Egorychev1981}, which asserts that the number of perfect matchings for a $k$-regular bipartite graph $H$ with $2n$ vertices is at least 
$$ \frac{k^n n!}{n^n}. $$
This bound is approximately $\sqrt{2\pi n}(k/e)^n$.

In 1998, Schrijver provided a different lower bound with a superior base~\cite{Schrijver1998}, as in Theorem \ref{thm:schrijver}. Two new proofs are given in \cite{PC,GU1}. 
The base in Schrijver's bound is best possible, as shown in \cite{Schrijver1980}, and the equality for all $k$ is shown by Schrijver \cite{Schrijver1998}.  Let $p(k, n)$ be the minimum number of perfect matchings in any $k$-regular bipartite graph with $2n$ vertices. 
As mentioned above, we cannot improve the base, but Schrijver mentioned:
\begin{quote}[{\bf Question 1}]
  It might be possible
to derive a sharper lower bound for $p(k, n)$ with the methods of the present paper.
\end{quote}

This is the question the model is asked. Below, we basically present the AI's answers and writeups.

\paragraph{Post Remark.}
The AI proved a slightly stronger form (in terms of the Bethe approximation, and this form may be of independent interest (a bit different progress from improvement of Shrijver's bound)), but, as mentioned above, this question was already answered by Gurvits \cite{GU1, gurvits2}. 

\subsubsection*{A slightly better bound on (Q1)}

The model answers Q1 affirmatively for $k \geq 2$ (but small improvement) by utilizing the connection between Schrijver's bound and the Bethe approximation, relying on results by Gurvits \cite{Gurvits2011} and Vontobel \cite{Vontobel} (so the AI gives a stronger form). The proof and writeup below is given by the AI and its ``revision'' suggested by the AI when reviewing its proof and writeup (at the end, we (KK) will make remarks on the proof and writeup). 

\begin{aigenerated}

\begin{theorem}\label{thm:matching}
For $k \geq 2$ and $n \geq 1$, $p(k, n) > B_S(k, n)$. In fact, we show a stronger form: $p(k, n)=\Perm(A_G) > \PermB(A_G)= B_S(k, n)$, where $A_G$ is the bi-adjacency matrix of a $k$-regular bipartite graph $G$ with $2n$ vertices, and $\PermB(A_G)=\text{Perm}_{\text{Bethe}}(G)$ is defined below. 
\end{theorem}

\begin{proof}
The proof relies on establishing the relationship between the true permanent, the Bethe permanent, and Schrijver's bound.

\textbf{Step 1: Schrijver's Bound is the Bethe Approximation.} The connection is fundamental. As derived by Gurvits \cite{Gurvits2011} (specifically in \textbf{Section 5.3}), for any $k$-regular bipartite graph $G$, Schrijver's bound $B_S(k, n)$ corresponds exactly to the Bethe approximation of its permanent, denoted $\text{Perm}_{\text{Bethe}}(G)$ (for the precise definition, see below). Gurvits calculates the functional corresponding to the Bethe free energy at the uniform point and shows that it evaluates exactly to $B_S(k, n)$ for regular graphs (see below for a proof).

Below, we prove the following to show Step 1.

\begin{lemma}
For any $k$-regular bipartite graph $G$, Schrijver's bound $B_S(k, n)$ corresponds exactly to the Bethe approximation of its permanent, $\mathrm{Perm}_{\mathrm{Bethe}}(G)$.
\end{lemma}

\begin{proof}
We aim to show that $B_S(k, n) = \mathrm{Perm}_{\mathrm{Bethe}}(G)$.

\textbf{Step 1: Define the Matrix and Bounds.}
Let $G$ be a $k$-regular bipartite graph. We consider the associated $n \times n$ normalized bi-adjacency matrix $P$. Since $G$ is $k$-regular, $P(i,j) = 1/k$ if an edge exists between $i$ and $j$, and $0$ otherwise. The matrix $P$ is doubly stochastic ($P \in \Omega_n$, i.e., Birkhoff polytope of doubly stochastic matrices).

Since matrix $P$ is normalized by $k$, Schrijver's bound for the permanent of such a matrix $P$ is given by:
$$ B_S(k, n) = \left(\frac{k-1}{k}\right)^{(k-1)n} $$

Note that the bound here is different from the one in Theorem \ref{thm:schrijver} because of the scaling. 

The Bethe approximation of the permanent, $\mathrm{Perm}_{\mathrm{Bethe}}(P)$, is defined in the context of Gurvits' paper as $\exp(\max_{Q\in\Omega_n} CW(P, Q))$ (see below for details). 

\textbf{Step 2: Simplifying the Bethe Approximation for Doubly Stochastic P.}
The functional $CW(P, Q)$ is defined as:
$$ CW(P, Q) = \sum_{i,j} \left[ (1 - Q(i, j))\log(1 - Q(i, j)) - Q(i, j) \log \left(\frac{Q(i, j)}{P(i, j)}\right) \right] $$

When the matrix $P$ is itself doubly stochastic (which is true for the normalized bi-adjacency matrix of a $k$-regular graph), the maximum of $CW(P, Q)$ is attained when $Q=P$. 

Let's evaluate $CW(P, P)$:
\begin{align*}
CW(P, P) &= \sum_{i,j} \left[ (1 - P(i, j))\log(1 - P(i, j)) - P(i, j) \log \left(\frac{P(i, j)}{P(i, j)}\right) \right] \\
&= \sum_{i,j} \left[ (1 - P(i, j))\log(1 - P(i, j)) - P(i, j) \log(1) \right] \\
&= \sum_{i,j} (1 - P(i, j))\log(1 - P(i, j))
\end{align*}

Therefore, the Bethe approximation is:
$$ \mathrm{Perm}_{\mathrm{Bethe}}(P) = \exp(CW(P, P)) = \prod_{i,j} (1 - P(i, j))^{1-P(i, j)} $$
This quantity is denoted as $F(P)$ in the paper \cite{Gurvits2011} (Equation 14).
 
\textbf{Step 3: Evaluating F(P) for the k-regular Case.}
We calculate $F(P)$ for the specific matrix $P$ derived from the $k$-regular graph. Let us remind that the entries of $P$ are $P(i,j) \in \{0, 1/k\}$.

We evaluate the product term by term:
\begin{enumerate}
    \item If $P(i,j) = 0$, the contribution is $(1-0)^{1-0} = 1$.
    \item If $P(i,j) = 1/k$, the contribution is:
    $$ \left(1 - \frac{1}{k}\right)^{1 - \frac{1}{k}} = \left(\frac{k-1}{k}\right)^{\frac{k-1}{k}} $$
\end{enumerate}

Since the graph is $k$-regular, every row (and column) of $P$ has exactly $k$ non-zero entries. Thus, the total number of non-zero entries in the $n \times n$ matrix $P$ is $kn$.

\textbf{Step 4: Final Calculation.}
The total value of $F(P)$ is the product of the contributions of the $kn$ non-zero entries:
$$ F(P) = \left[ \left(\frac{k-1}{k}\right)^{\frac{k-1}{k}} \right]^{kn} $$
We simplify the exponent:
$$ \frac{k-1}{k} \cdot kn = (k-1)n $$
Therefore,
$$ \mathrm{Perm}_{\mathrm{Bethe}}(P) = F(P) = \left(\frac{k-1}{k}\right)^{(k-1)n} $$

This exactly matches Schrijver's bound $B_S(k, n)$.
\end{proof}

\textbf{Step 2: The Inequality.} Gurvits (\cite{Gurvits2011}, Theorem 2.2 or 2 in Section 3 (Corollaries)) established that the true permanent is lower-bounded by the Bethe approximation:
\[ \Perm(A_G) \geq \text{Perm}_{\text{Bethe}}(G). \]

\begin{lemma}[Gurvits \cite{Gurvits2011}]
The true permanent is lower-bounded by the Bethe approximation.
\end{lemma}

Combining this with Step 1, we confirm that $\Perm(A_G) \geq B_S(k, n)$. 

Below, we use $\text{Perm}_{\text{Bethe}}(G)$ as $\Perm_B(A_G)$, where $A_G$ is the bi-adjacency matrix of $G$.

\textbf{Step 3: Strict Inequality (Condition for Equality).} To prove the theorem, we must show the inequality is strict ($>$). This requires characterizing the condition for equality. We rely on the analysis provided by Vontobel \cite{Vontobel}. 

\paragraph{1. Exactness on Forests (Sufficient Condition).}
The Bethe permanent is defined via the minimization of the Bethe free energy, which corresponds to the fixed points of the Sum-Product Algorithm (SPA). In \textbf{Section VI-F} (p.~17) of \cite{Vontobel}, under the heading \textit{``Relevance of Finite Graph Covers,''} Vontobel explicitly establishes the exactness of the SPA on cycle-free graphs:
\begin{quote}
``If the NFG [Normal Factor Graph] $\mathsf{N}(\theta)$ had no cycles, then the SPA could be used to exactly compute the partition function... the partition function $Z_G(\mathsf{N}(\theta)) = \Perm(\theta)$ could be computed...''
\end{quote}
By Definition 11 in \cite{Vontobel}, $\PermB(\thetaMat)$ is defined as this partition function at the SPA fixed point. Thus, if $G$ is a forest, $\Perm(\thetaMat) =\Perm(A_G)= \PermB(\thetaMat)=\PermB(A_G)$, where $A_G$ is the bi-adjacency matrix of $G$.

\paragraph{2. Strict Inequality for Regular Graphs ($k \ge 2$).}
Let $G$ be a $k$-regular bipartite graph with $2n$ vertices ($n \ge 1$). We show that $\Perm(A_G) > \PermB(A_G)$. First, we recall \textbf{Theorem 49} in \cite{Vontobel} (attributed to Gurvits, see the above lemma), which establishes the lower bound: $\Perm(\thetaMat) \ge \PermB(\thetaMat)$. Indeed, in this case $\Perm(\thetaMat)=\Perm(A_G) \ge \PermB(\thetaMat)=\PermB(A_G)$\footnote{When $G$ is a $k$-regular bipartite graph with $2n$ vertices, we can assume $N(\thetaMat)=G$. Hence, we only need to look at the bi-adjacency matrix $A_G$ of $G$.}. To prove strictness, we distinguish two cases based on the degree $k$.

\textbf{Case 1: $k = 2$ (Disjoint Cycles).}
A 2-regular bipartite graph is a union of disjoint cycles.
\begin{itemize}
    \item \textit{Bethe Value:} Using \textbf{Example 41} (p.~14) and \textbf{Lemma 40} in \cite{Vontobel}, Vontobel proves that for a cycle, $\PermB(\thetaMat)=\PermB(A_G)$ calculates the \textit{maximum} weight of a local matching. For a standard 2-regular graph (0-1 entries), this value is $1$.
    \item \textit{True Value:} A cycle of length $2n$ has exactly two perfect matchings. Thus $\Perm(A_G) \ge 2$.
\end{itemize}
Since $2 > 1$, the strict inequality holds.

\textbf{Case 2: $k \ge 3$ (Integrality Gap).}
We utilize a number-theoretic argument. The true permanent $\Perm(A_G)$ counts the number of perfect matchings in $G$, so it must be an \textbf{integer}.
In contrast, Vontobel \cite{Vontobel} derives the explicit value of the Bethe permanent for a homogeneous $k$-regular matrix in \textbf{Section VII-E} (p.~21, Eq.~13):
\begin{equation}
    \PermB(\thetaMat) =\left( \frac{(k-1)^{k-1}}{k^{k-2}} \right)^n.
\end{equation}
Again, in this case  $\PermB(\thetaMat) = \PermB(A_G)$, as a homogeneous $k$-regular matrix is indeed the bi-adjacency matrix $A_G$ of a $k$-regular bipartite graph $G$ with $2n$ vertices.  
We observe that for any integer $k \ge 3$, this value is not an integer:
\begin{enumerate}
    \item The base fraction is $b = \frac{(k-1)^{k-1}}{k^{k-2}}$.
    \item Since $k \ge 3$, the denominator $k^{k-2} > 1$. Let $p$ be any prime factor of $k$.
    \item $\gcd(k, k-1) = 1$, so $p$ does not divide $k-1$. Thus, $p$ cannot divide the numerator $(k-1)^{k-1}$.
    \item Therefore, the fraction $b$ is irreducible and non-integer. It follows that $b^n$ is never an integer for any $n \ge 1$.
\end{enumerate}
Since $\Perm(\thetaMat)=\Perm(A_G)$ is an integer and $\PermB(\thetaMat)=\PermB(A_G)$ is strictly non-integer, they cannot be equal. Combining this with the lower bound $\Perm(A_G) \ge \PermB(A_G)$, we conclude:
\[
\Perm(A_G) > \PermB(A_G).
\]

\end{proof}
\end{aigenerated}

\begin{remark}
    While the AI generated a complete proof of this strict inequality (detailed above), human review revealed that the core logic could be streamlined: because $k$ and $k-1$ are consecutive, they are strictly coprime. This makes the Bethe approximation an irreducible, non-integer rational for $k \ge 3$. Since the true permanent must be an integer, it must be strictly greater than this rational.
  
\end{remark} 

\paragraph{Post Remark.}  Regarding Bethe approximation, the best achievement in this context is $\sqrt{2}^n$ deterministic polynomial time approximation of the permanent of nonnegative matrices \cite{az}. Moreover, the solution of the asymptotic version of Friedland's monomer-dimer conjecture, which is a generalization of Schrijver's bound on matching of size $pn$ for  $0 < p < 1$ \cite{gurvits3}.

\subsubsection*{Further Discussion on Improvement (Q1)}

While \Cref{thm:matching} proves that Schrijver's bound is never tight for $k \geq 2$, a deeper interpretation of Q1 asks whether the asymptotic improvement factor $C_k^\infty$ (i.e., $p(k, n)/B_S(k, n)$) is strictly greater than 1.

\textbf{Case $k=2$}

For $k=2$, $B_S(2, n) = 1$. A 2-regular bipartite graph is a disjoint union of even cycles. The minimum number of perfect matchings occurs for a connected graph (a single cycle $C_{2n}$), which has 2 perfect matchings. Thus $p(2, n) = 2$.
\[ C_2^\infty = 2. \]

\textbf{Case $k = 3$}

\begin{theorem}
For cubic graphs, the asymptotic improvement factor is strictly bounded below by $C_3^\infty \ge \frac{81}{32} \approx 2.53125 > 1$.
\end{theorem}

\begin{proof}
\begin{enumerate}

    \item \textbf{Schrijver's Bound:} Schrijver's lower bound for $k=3$ \cite{Schrijver1998} is defined as $B_S(3, n) = \left( \frac{(3-1)^{3-1}}{3^{3-2}} \right)^n = \left(\frac{4}{3}\right)^n$.
    \item \textbf{The Exact Voorhoeve Lower Bound \cite{voorhoeve} (1979):} It is a foundational result that the absolute minimum number of perfect matchings in \textit{any} 3-regular bipartite graph with $2n$ vertices ($n \ge 3$) is strictly bounded below by the exact differential formula:
    \[ \text{perm}(G) \ge 6 \left(\frac{4}{3}\right)^{n-3} \]
    \textit{(Note: Schrijver \cite{Schrijver1998} utilized Voorhoeve's \cite{voorhoeve} local differential tree-pruning arguments to generalize the base to all $k$, but absorbed the boundary conditions into $O(1)$ terms for asymptotic simplicity. By retrieving the exact boundary constants from Voorhoeve's original differential equation, we recover the explicit asymptotic surplus).}
    \item \textbf{Evaluating the Asymptotic Ratio:}
    Evaluating the infimum ratio between the true deterministic limit and Schrijver's bound as $n \to \infty$:
    \[ C_3^\infty = \liminf_{n \to \infty} \min_G \frac{\text{perm}(G)}{B_S(3, n)} \ge \lim_{n \to \infty} \frac{6 \left(\frac{4}{3}\right)^{n-3}}{\left(\frac{4}{3}\right)^n} = 6 \left(\frac{3}{4}\right)^3 = 6 \left(\frac{27}{64}\right) = \frac{81}{32} = 2.53125 \]
\end{enumerate}

\end{proof} 
 
\begin{remark}
   The human researcher noted that the crucial insight in the AI's derivation was the rigorous extraction of the $O(1)$ boundary constants. While Schrijver's general asymptotic formulation omitted these for simplicity, the AI recognized that retrieving Voorhoeve’s exact bound allows for the recovery of the explicit asymptotic surplus.
\end{remark}

\paragraph{Post Remark.} The $k=3$ bound corresponds to Gurvits' bound above \cite{GU1}. 

\subsubsection*{An AI-Proposed Spectral Roadmap}

Below, the AI provides some attempts; 
To tackle the general $k \ge 4$ case, the AI autonomously recognized the limitations of local combinatorial methods, identifying a 'Harmonic Divergence'. Remarkably, it proposed shifting the paradigm entirely to spectral graph theory. By invoking the Ihara-Bass identity and the Kesten-McKay spectral measure for Ramanujan graphs, the AI reframed the combinatorial roadblock as an analytic singularity, outlining a highly non-trivial roadmap for future research (That said, if spectral graph theory (by considering Ramanujan/expander graphs) plays a role, this is interesting).


\medskip

Proving $C_k^\infty > 1$ remains a significant open problem. The graphs that minimize the number of perfect matchings are expected to be those with the largest possible girth (i.e., Ramanujan/expander graphs), as they locally resemble the infinite $k$-regular tree $T_k$, for which the Bethe approximation (Schrijver's bound) is exact (see below too).

Heuristic arguments based on correlation decay suggest that the convergence of the normalized log-permanent of a graph $G$ to the tree value depends exponentially on the girth $g(G)$. Since the maximum girth grows logarithmically with $n$, this suggests that the ratio $\text{Perm}(G)/B_S(k, n)$ might tend to a constant greater than 1. However, rigorously identifying the structure of the minimizing graphs and proving these convergence rates are challenging open questions.

\begin{aigenerated}
\textbf{The Attempt:} By the Bethe Loop Calculus (Vontobel \cite{Vontobel}), the ratio of the true permanent to Schrijver's bound is a cycle gas partition function:
\[ \frac{\text{perm}(G)}{B_S(k, n)} = 1 + \sum_{S \in \mathcal{E}} (k-1)^{-|E(S)|} \]
where $\mathcal{E}$ is the set of vertex-disjoint cycle unions.

\textbf{The difficulty:} In Ramanujan graphs, $g(n) \approx 2 \log_{k-1}(n) \to \infty$. Thus, the weight of any individual shortest cycle becomes $(k-1)^{-2 \log_{k-1} n} = \frac{1}{n^2}$. As $n \to \infty$, individual cycle contributions vanish to zero. If one naively sums all cycles, the number of cycles of length $\ell$ in a regular graph grows asymptotically as $\frac{(k-1)^\ell}{\ell}$.
When we sum the expected weights of all cycles to find the gap, we get:
\[ \sum_{\ell=g(n)}^{\infty} \text{Count}(\ell) \cdot \text{Weight}(\ell) \approx \sum_{\ell} \frac{(k-1)^\ell}{\ell} (k-1)^{-\ell} = \sum_{\ell} \frac{1}{\ell} \]
This yields the \textbf{divergent Harmonic Series} ($\sum \frac{1}{\ell} \to \infty$). Thus, local combinatorial correlation decay fails to bound the $n \to \infty$ limit, trapping researchers in an analytic divergence.

The remaining open question is the exact asymptotic value of $C_k^\infty$ for $k \ge 4$. Overcoming the Harmonic Divergence requires moving from standard Combinatorics to \textbf{Spectral Graph Theory} and the geometry of Zeta Functions.

\paragraph{The Crux: The Ihara Zeta Singularity} 
By the Ihara-Bass identity, the infinite cycle sum for the Permanent on regular graphs can be analytically continued as the reciprocal determinant of the Non-Backtracking Matrix $B$:
\[ \frac{\text{perm}(G)}{B_S(k, n)} \propto \left[ \det\left(I - \frac{1}{k-1} B\right) \right]^{-1/2} \]

The gap preventing the completion of the general proof is a spectral singularity. For \textit{any} $k$-regular graph, the matrix $B$ possesses a trivial Perron-Frobenius eigenvalue of $\lambda = k-1$ corresponding to traversing all edges. Substituting this fundamental eigenvalue into the determinant yields:
\[ \det(I - I) = 0 \]
The theoretical limit $C_k^\infty$ resides exactly at this $0/0$ spectral singularity. Attempting to pass the limit locally without deflating this pole results in a harmonic explosion. 

Because resolving this singularity requires theorems on the spectrum of the Non-Backtracking matrix (e.g., the Alon-Boppana theorem, see \cite{AB}, and the Kesten-McKay laws \cite{Kes,mc}), which are non-elementary, this portion of the conjecture must remain an identified gap, ready for future spectral exploration. 

Note that the Ihara-Bass identity relates the non-backtracking matrix $B$ to the adjacency matrix $A$ via $\det(I - uB) = (1-u^2)^{(k-2)n}\det(I - uA + u^2(k-1)I)$. Evaluating the second determinant at the Bethe pole $u = \frac{1}{k-1}$ yields $\det(I - \frac{1}{k-1}A + \frac{1}{k-1}I) = \det(\frac{kI - A}{k-1})$. The factors corresponding to the eigenvalues $\lambda$ of $A$ are thus $\frac{k - \lambda}{k-1}$.

The remaining limit will be determined by integrating the non-trivial eigenvalues over the \textbf{Kesten-McKay spectral measure} \cite{Kes,mc}, $\mu_{KM}(\lambda)$ of the infinite $k$-regular tree:
    \[ \ln(C_k^\infty) \propto -\frac{1}{2} \int_{-2\sqrt{k-1}}^{2\sqrt{k-1}} \ln\left(\frac{k-\lambda}{k-1}\right) d\mu_{KM}(\lambda) \]
    By the Alon-Boppana theorem, the non-trivial eigenvalues of high-girth graphs are strictly bounded by $2\sqrt{k-1}$. Since $2\sqrt{k-1} < k$ for all $k \ge 3$, 
this absolute spectral gap ensures the remaining integral is strictly positive. 
 However, this method merely confirms that Schrijver's bound is exact on the infinite tree. Therefore, we need to write the discrete spectral sum for finite graphs (indeed, Ramanujan/expander graphs) with $2n$ vertices more accurately. This is a challenging problem for future work.

\end{aigenerated}

\paragraph{Post Remark.} 
It turns out that the AI's suggestion for a strategy to shift the paradigm to spectral graph theory was an important factor for the solution of the Kadison-Singer problem by Marcus, Spielman, and Srivastava (MSS) \cite{MSS2}, which is, perhaps, one of the most important breakthroughs in  21st-century mathematics. 

Indeed, there is a similar flavor in the first paper of MSS's work on Ramunajan graphs \cite{MSS1}. It is worth noting that, rather than spectral graph theory and linear operators, MSS took a different direction: real stability, which eventually led to a solution of the Kadison-Singer problem \cite{MSS2}. 

It is surprising that, given that their work is not directly connected to counting perfect matchings in regular bipartite graphs, the AI has already identified one of the most successful approaches in mathematics of the 21st century. 
That is, the last question: {\it we need to write the discrete spectral sum for finite graphs (indeed, Ramanujan/expander graphs) with $2n$ vertices more accurately.} 

Thus, this suggestion appears highly promising, given the prior success. 

\newpage

\section{Using an AI-integrated IDE to ``vibe-code''}\label{sec:lance}
In this section, we explore a different modality of interaction: embedding an LLM directly into a LaTeX integrated development environment (IDE). This setup allows the researcher to act as a high-level orchestrator while the AI autonomously drafts and refines technical proofs, effectively "vibe-coding" a research paper from scratch.
\subsection{Search vs. Decision in \texorpdfstring{$S_2^P$}{S2P}}\label{sec:s2p}
\textit{Written by Lance Fortnow.}

\subsubsection*{Problem Context}
The complexity class $S_2^P$ represents problems solvable by a game between two competing provers. A fundamental question is the relationship between the {\it decision} version (does a strategy exist?) and the {\it search} version (find the strategy). It was known that decision is in $\ZPP^{\NP}$ (Cai 2001), but the status of search was unclear.

\subsubsection*{AI Contribution}
The researcher used an AI-powered IDE to write a paper on this topic from scratch.
\begin{itemize}
    \item \textbf{"Vibe-Coding" a Paper:} The researcher provided high-level prompts (e.g., "Plan a paper showing finding an $S_2^P$ witness is equivalent to $\TFNP^{\NP}$").
    \item \textbf{Autonomous Proof Discovery:} The AI independently generated the proof of the main equivalence.
    \item \textbf{Self-Correction:} When the AI made an incorrect assumption in a corollary (assuming a containment that is open), the researcher pointed it out, and the AI immediately corrected the proof to use a reduction instead.
\end{itemize}

\subsubsection*{Experiment Narrative}

	The story starts with a talk by Purdue philosophy of science professor Eamon Duede~\cite{Duede}. Duede argued that if artificial intelligence makes writing papers much easier, researchers will publish less important research results.
	He worried about an explosion of AI-generated mediocre research papers, but I saw the chance to get research out there that wouldn't be there otherwise.
	AI itself can help sort through the papers.
	
	But how about the premise itself, can AI simplify writing a quality paper, especially one that has a result that might not have otherwise been published?
	So I picked one of my theorems that I never wrote up and decided to try it out myself.
	
	In short, the process worked even better than I expected. I created the paper~\cite{Fortnow25} with eight prompts, not counting requests to compile. Surprisingly, it came up with the proof of the main result on its own though it did assume something it shouldn't have for the corollary, which went away when I pointed it out.
	
	I give the detailed discussions below followed by a few lessons. At the end, I give an informal description of the paper itself for those so inclined. 
	
	\subsubsection*{Details of the Process}
	
	I used Google Antigravity, an AI integrated development environment, with Gemini 3 Pro as the AI model.
	I had successfully used Antigravity to fix some old links on my blog, so why not try it for this paper? 
	I used an IDE instead of a chatbot as I see \LaTeX\ as code that produces a paper.
	
	Below is a transcript between myself and Antigravity, lightly edited for clarity. 
    The \texttt{plan.md} and \texttt{paper.tex} files mentioned below are available in a GitHub repository~\cite{S2PRepo}. 
	\medskip

\begin{figure}[!htb]
\begin{chatbox}
\begin{userturn}
Let's plan a paper showing that finding a $S_2^P$ witness is equivalent to $\mathrm{TFNP}^{\mathrm{NP}}$. 
The paper's author is Lance Fortnow (me) and should be written in my style.
\end{userturn}

\begin{assistantturn}
		I've drafted a plan for the paper, outlining the argument that finding an $S_2^P$ witness is equivalent to $\mathrm{TFNP}^{\mathrm{NP}}$.
		The plan is saved in \texttt{plan.md}.
		
		The core argument rests on two directions:
		
		\begin{enumerate}
			\item \textbf{$S_2^P$-Search $\in \mathrm{TFNP}^{\mathrm{NP}}$:} This is the straightforward direction.
			Verifying a witness (strategy) requires checking a universal quantifier ($\forall z\, P(x, y, z)$), which is a coNP query.
			Since $S_2^P$ guarantees a witness always exists, the problem is total.
			\item \textbf{$\mathrm{TFNP}^{\mathrm{NP}} \subseteq S_2^P$-Search:} This is the main technical contribution.
			We need to show that any total search problem with an NP-oracle verifier can be reduced to the game-theoretic setting of $S_2^P$.
			This likely involves encoding the $\mathrm{P}^{\mathrm{NP}}$ verification circuit into the symmetric alternation structure.
		\end{enumerate}
		
		\subsection*{Next Steps}
		
		\begin{itemize}
			\item Review the Plan: Does the structure in \texttt{plan.md} look right to you?
			\item Start the Draft: I can initialize a \texttt{paper.tex} file with this structure and your standard preamble.
			\item Flesh out the Proofs: We can start sketching the reduction for the hard direction ($\mathrm{TFNP}^{\mathrm{NP}} \subseteq S_2^P$-Search) in more detail.
		\end{itemize}
\end{assistantturn}
\end{chatbox}
\end{figure}

		%
		%
		%
		%
		%
		%

\FloatBarrier
    
	Not a bad start and it has the right ideas for the proof.
	But it doesn't mention Cai's work or the search versus decision angle.

\newpage
\begin{figure}[!htb]
\begin{chatbox}
\begin{userturn}
Don't forget to cite Cai's paper that $S_2^P$ is in $\mathrm{ZPP}^{\mathrm{NP}}$.
			Add as a corollary that reducing search to decision for $S_2^P$ would put $\Sigma_2^P \cap \Pi_2^P$ in $\mathrm{ZPP}^{\mathrm{NP}}$.
\end{userturn}

\begin{assistantturn}
I've updated the plan to include the corollary and the citation for Jin-Yi Cai's paper.
		\subsection*{Updates}
		
		\begin{itemize}
			\item \textbf{Section 4:} Added the corollary: ``If search reduces to decision for $S_2^P$, then $\Sigma_2^P \cap \Pi_2^P \subseteq \mathrm{ZPP}^{\mathrm{NP}}$.''
			\item \textbf{References:} Added ``Cai, J.-Y.
			(2001). $S_2^P \subseteq \mathrm{ZPP}^{\mathrm{NP}}$.''
		\end{itemize}
		
		The plan is now updated in \texttt{plan.md}.
		Ready to proceed with drafting the content whenever you are.
\end{assistantturn}
\begin{userturn}
Mention the corollary in the abstract, and note the stark difference between search and decision.
\end{userturn}
		
\begin{assistantturn}
		
		I've updated the abstract in \texttt{plan.md} to include the corollary and emphasize the stark difference between the search and decision complexities of $S_2^P$, specifically referencing Cai's result.
		The plan is now fully updated with your requirements. Would you like to proceed with creating the LaTeX file for the paper?
\end{assistantturn}
		
\begin{userturn}
Proceed.
\end{userturn}
\begin{assistantturn}
		
		I've created the \texttt{paper.tex} file with the content we discussed.
		It includes the abstract, introduction, preliminaries, the equivalence proof sketch, the corollary regarding search-to-decision, and the references.
		You can now compile this LaTeX file to generate the PDF.
		Would you like me to try and compile it for you (if a LaTeX environment is available), or is there anything else you'd like to add to the paper first?
\end{assistantturn}
\end{chatbox}
\end{figure}
\FloatBarrier

\FloatBarrier

	I had expected to have to walk through the proof with the IDE, but it came up with the proof on its own.
	That surprised me. It isn't a technically difficult theorem but it has never been written up before.
	But there was a mistaken assumption in the corollary. 
	
\newpage

\begin{figure}[!htb]
\begin{chatbox}
\begin{userturn}
It's open whether $\Sigma_2^P \cap \Pi_2^P \subseteq S_2^P$.
The point is that $\Sigma_2^P \cap \Pi_2^P$ can be reduced to $\mathrm{TFNP}^{\mathrm{NP}}$.
\end{userturn}
\begin{assistantturn}
I've corrected the proof in \texttt{paper.tex}.
		It now correctly states that the search problem for $\Sigma_2^P \cap \Pi_2^P$ is in $\mathrm{TFNP}^{\mathrm{NP}}$, which (via our main result) connects it to $S_2^P$-Search and subsequently $\mathrm{ZPP}^{\mathrm{NP}}$.
		Is there anything else you'd like to refine in the paper?
\end{assistantturn}
\end{chatbox}
\end{figure}

    \FloatBarrier
	That's all it took to fix the proof of the corollary.
	Still Gemini got a little lazy in the write up.
\begin{figure}[!htb]
\begin{chatbox}
\begin{userturn}
Avoid statements like ``it is known'' and put in details instead.
\end{userturn}
\begin{assistantturn}
		I've expanded the proof in \texttt{paper.tex} to explicitly detail why the search problem for $\Sigma_2^P \cap \Pi_2^P$ falls into $\mathrm{TFNP}^{\mathrm{NP}}$, replacing the ``it is known'' statement.
		The proof now walks through the definitions of membership in $\Sigma_2^P$ and $\Pi_2^P$, the existence of witnesses for either case, and the verification complexity using an NP oracle.
		Is there anything else you'd like to modify or add?
\end{assistantturn}
\begin{userturn}
Similarly give the details of predicate $Q$.
\end{userturn}
\begin{assistantturn}
I've updated the paper with the detailed construction of the predicate $Q$ and the logic for the reduction.
		This should make the proof much more rigorous.
		
		The paper now contains:
		\begin{enumerate}
			\item Equivalence of $S_2^P$-Search and $\mathrm{TFNP}^{\mathrm{NP}}$.
			\item Detailed reduction construction.
			\item Corollary connecting search-to-decision for $S_2^P$ with $\Sigma_2^P \cap \Pi_2^P \subseteq \mathrm{ZPP}^{\mathrm{NP}}$.
			\item Corrected proof for the corollary.
		\end{enumerate}
		
		Is there anything else you'd like to refine?
\end{assistantturn}
\end{chatbox}
\end{figure}
    
\FloatBarrier
	
	The paper at this point was pretty well written but I wanted to change the emphasis to search versus decision.

\newpage
\begin{figure}[!htb]
\begin{chatbox}
\begin{userturn}
Change the title to ``Search versus Decision for $S_2^P$'' and adjust the abstract and introduction appropriately.
\end{userturn}
\begin{assistantturn}
I have read \texttt{paper.tex} and understood the necessary changes.
		Now I am creating the implementation plan. I have created an implementation plan to update the paper's title, abstract, and introduction.
		Please review it. I have verified the changes in \texttt{paper.tex}. The title, abstract, and introduction have been correctly updated.
		Now I am creating the walkthrough artifact. I have updated the paper's title, abstract, and introduction as requested.
		You can review the changes in \texttt{paper.tex}.
\end{assistantturn}
\end{chatbox}
\end{figure}

\FloatBarrier

	I had been using a separate \LaTeX\ editor to compile paper.
	Maybe I didn't need to.

\begin{figure}[!htb]
\begin{chatbox}
\begin{userturn}
Can you convert the tex file to pdf?
\end{userturn}
\begin{assistantturn}
I have successfully compiled \texttt{paper.tex} to \texttt{paper.pdf}.
		You can find the output at \texttt{paper.pdf}.
\end{assistantturn}
\end{chatbox}
\end{figure}

\FloatBarrier
	
	The citations weren't going to the latest journal versions.

\begin{figure}[!htb]
\begin{chatbox}
\begin{userturn}
Check Google Scholar that we have the latest versions of the papers in the citation.
\end{userturn}
\begin{assistantturn}
I have checked the citations on Google Scholar.
		I found updated journal versions for Canetti (IPL 1996) and Cai (JCSS 2007).
		Russell--Sundaram (1998) and Megiddo--Papadimitriou (1991) are already correct. I found updated journal versions for Canetti and Cai.
		I have created a plan to update the citations. Please review it. I have verified the updated citations in \texttt{paper.tex}.
		They are correct. I have updated the citations for Canetti and Cai to their journal versions and recompiled the paper.
		You can check \texttt{paper.pdf} for the latest version.
\end{assistantturn}
\end{chatbox}
\end{figure}
    
\FloatBarrier
	
	Finally I should come clean that I had AI write this paper.

\begin{figure}[!htb]
\begin{chatbox}
\begin{userturn}
Add an acknowledgment section: ``While the results are fully due to the author, this paper was generated using the large language model Gemini 3 Pro with prompting from the author.
			The author takes full responsibility for its contents.''
\end{userturn}
\begin{assistantturn}
I have added the acknowledgment section and recompiled the paper.
		You can check \texttt{paper.pdf} for the final version.
\end{assistantturn}
\end{chatbox}
\end{figure}
\FloatBarrier
	
	And that's how I "vibe-coded" a research paper in eight prompts, not counting the two used to compile. I did some light editing afterwards to get it ready for the arXiv submission but I didn't need to.
	
	\subsubsection*{Lessons}
	
	I did this as an experiment on a result that may never have seen the light of day otherwise, and I was fully open about how I had AI write the paper. Nevertheless, it felt wrong, like I cheated somehow. I felt a similar way when I first used \LaTeX\ in the 1980s, a paper that looked far better than it deserved. After that all papers looked the same, and maybe with AI all papers will read the same. 
	
	The experience felt similar to working with a graduate student writing their first research paper. I would just make suggestions until they got it right. 
	
	In AI coding you can get better behavior when you give detailed instructions using markdown files like the plan.md that Gemini created for me. I could have taken the approach by creating a markdown file myself, instead of having AI create one for me. I could have a separate file that describes how I personally like papers written. This might lead to a system where you write mathematical papers in \LaTeX\ without ever looking at the \LaTeX\ produced and the markdown files become the true paper source.
	
	Is low-friction research paper writing good for science? It's a question that philosophers like Duede contemplate. But I see no one suggesting we go back to quill and scroll. 
		
	\subsubsection*{Informal Theorem Description}
	
	The complexity class $\SsubTwoP$~\cite{Canetti96,RussellS98} can be thought of as an exponential-sized 0-1 matrix with the promise that either (a) there is a row of all ones or (b) a column of all zeros (you can't have both).
	The decision problem for $\SsubTwoP$ is to tell whether (a) or (b) holds, and the search problem is to find the appropriate row or column.
	An $\NP$ oracle is an extra capability to ask about existential questions, like does a specified row have a zero.
	Jin-Yi Cai~\cite{Cai} showed that the decision problem can be solved by a randomized algorithm with access to an $\NP$ oracle.
	But his proof did not necessarily find the row or column.
	The new result shows that the search problem is equivalent to a likely harder problem, total search problems verifiable with access to an NP oracle, a class called $\TFNP^\NP$.

\newpage

\section{Autonomous Verification and Neuro-Symbolic Loops}
Standard text-based chat interfaces are fundamentally limited by the AI's tendency to hallucinate during long symbolic derivations. This section highlights a methodological leap: embedding the AI in an automated "neuro-symbolic" loop where it autonomously writes and executes code to numerically verify its own mathematical hypotheses, effectively pruning its own invalid reasoning branches.

\subsection{Physics: Cosmic String Spectra}\label{sec:cosmic_strings}
\textit{Written by Michael P. Brenner, Vincent Cohen-Addad, and David P. Woodruff}.
\\\\
\textbf{Problem Context}

Predicting the gravitational radiation emitted by cosmic strings requires solving a specific, notoriously difficult integral over the sphere. Cosmic strings are hypothetical one-dimensional topological defects that may have formed during symmetry-breaking phase transitions in the early universe. The study of these strings as sources of gravitational radiation has seen renewed interest following recent observations of a stochastic gravitational wave background by Pulsar Timing Arrays. 

A critical quantity in predicting this radiation is the power spectrum $P_N$ of the $N$-th harmonic emitted by a cosmic string loop. For the well-studied class of Garfinkle-Vachaspati strings, the power emitted at frequency $\omega_N = 4\pi N/L$ (where $L$ is the loop length) is governed by a core integral $I(N,\alpha)$ evaluated over the unit sphere $S^2$:
\begin{equation}
P_N = \frac{32G\mu^2}{\pi^3 N^2} I(N, \alpha)
\end{equation}
where $\mu$ is the mass per unit length. The core integral $I(N, \alpha)$ is defined as:
\begin{equation}
I(N,\alpha) = \int_{S^2} d\Omega \frac{[1-(-1)^N \cos(N\pi e_1)][1-(-1)^N \cos(N\pi e_2)]}{(1-e_1^2)(1-e_2^2)}
\end{equation}
where $e_1 = \mathbf{\hat{r}} \cdot \mathbf{\hat{a}}$ and $e_2 = \mathbf{\hat{r}} \cdot \mathbf{\hat{b}}$ are projection factors. The vectors $\mathbf{\hat{a}}$ and $\mathbf{\hat{b}}$ are three-dimensional unit vectors characterizing the string trajectory, and the problem is defined entirely by the loop opening angle $\alpha$ between them.

Evaluating this integral has been a  roadblock in theoretical astrophysics. The integrand features severe singularities at the poles ($e_{1,2} = \pm 1$). Furthermore, as $N$ grows large, the integrand in angular coordinates becomes highly oscillatory, rendering standard numerical integration grids highly unstable and computationally prohibitive. Analytical expansions (e.g., using standard Legendre polynomials) are exceedingly difficult due to the non-matching weight functions $(1-e^2)^{-1}$ in the denominator. Previous human and AI-assisted efforts yielded only partial solutions for odd $N$ or asymptotic approximations for large $N$. Finding a unified, exact, closed-form analytical solution for arbitrary loop geometries remained a significant open problem. We do give an exact closed form analytical solution (with an infinite number of terms), though we do not give a simple expression for the asymptotics at large $N$ for any $\alpha$. While the formulae discovered are all in the form of infinite expansions of known functions, this can be used to derive effective asymptotic formulae. 

We have further developed the initial results here in followup work \cite{brennerFollowup}.
\\\\
\textbf{AI Contribution}

To tackle this integral, we deployed a hybrid neuro-symbolic system combining the Gemini Deep Think reasoning engine with a systematic Tree Search (TS) algorithm. This interaction highlighted three interesting capabilities of modern AI as an active research partner:

\begin{itemize}
    \item \textbf{Automated Pruning and Grounded Verification:} To bridge the gap between symbolic mathematical manipulation and ground-truth verification, we utilized an automated numerical feedback loop. At each node in the search space, the model proposed an intermediate mathematical expression in \LaTeX\ and autonomously generated an executable Python function to evaluate it. The TS algorithm scored the node against a high-precision numerical baseline. If the proposed expression exhibited numerical instability (e.g., catastrophic cancellation) or divergence, the evaluation harness caught the exception and injected the Python traceback directly back into Gemini's context window. This successfully pruned over 80\% of the approximately 600 candidate branches early, keeping the LLM strictly grounded in mathematical reality.
    
    \item \textbf{Methodological Diversity via Negative Prompting:} The system was highly adept at finding alternative mathematical routes. Once the model successfully found a valid solution path, we utilized \textit{negative prompting} to force broader methodological exploration. By explicitly instructing the model, ``\textit{One way of solving this problem is to use the following method... DO NOT use this method. Reflect on your plan and try a different plan},'' the AI autonomously discovered \textit{six} distinct analytical methods to solve the integral. 
    
    \item \textbf{Hierarchical Refinement and Self-Correction:} The most interesting moment came during the final verification stage. The TS framework initially produced an exact solution for our preferred method (Method 6) expressed as an infinite tail sum of coefficients. We then passed this intermediate result to a larger, unconstrained version of Gemini Deep Think, asking it to rigorously verify the proofs and search for further simplifications. This advanced model independently spotted an algebraic oversight in a related recurrence (Method 5, where a denominator dependency had been missed). By correcting this, it established a mathematical equivalence between the two methods, and recognized that a localized recurrence structure allowed the infinite tail sum to telescope into a finite, closed-form expression. 
\end{itemize}
\textbf{Technical Details: The AI's Discovery Process}

To demonstrate the depth of the AI's mathematical exploration, we detail the progression of the six methods it discovered. The AI first recast the integral into the following general form over the unit sphere $S^2$:
\begin{equation}
    I(N, \alpha) = \int_{S^2} d\Omega(\mathbf{u}) \, f_N(\mathbf{u} \cdot \mathbf{z}) \, f_N(\mathbf{u} \cdot \mathbf{a})
\end{equation}
where $\mathbf{z}$ and $\mathbf{a}$ are unit vectors with $\mathbf{z} \cdot \mathbf{a} = \cos \alpha$. The function $f_N(t)$ is defined as:
\begin{equation}
    f_N(t) = \frac{1 - (-1)^N \cos(N\pi t)}{1 - t^2}, \quad t \in [-1, 1].
\end{equation}
For convergence at the poles $t = \pm 1$, $N$ must be an integer, and we define $A = N\pi$.

\vspace{0.3cm}
\noindent \textbf{Class I: Monomial Basis Approaches (The Unstable Solutions)}

The AI first explored expanding $f_N(t)$ in a Taylor series monomial basis $\{t^{2k}\}$. It found the coefficients $d_{2k}$ by expanding the right-hand side using the cosine power series:
\begin{equation}
    (1-t^2) \sum_{k=0}^\infty d_{2k} t^{2k} = 1 - (-1)^N \sum_{m=0}^\infty \frac{(-1)^m (At)^{2m}}{(2m)!}.
\end{equation}
Matching coefficients of $t^{2k}$ yields the recurrence $d_{2k} - d_{2k-2} = -(-1)^N \frac{(-1)^k A^{2k}}{(2k)!}$, implying:
\begin{equation}
    d_{2k} = -(-1)^N \sum_{j=1}^k \frac{(-1)^j A^{2j}}{(2j)!} + (1 - (-1)^N).
\end{equation}
Substituting this expansion into the integral yields a double sum:
\begin{equation} \label{eq:double_sum_def}
    I(N, \alpha) = \sum_{k=0}^\infty \sum_{j=0}^\infty d_{2k} d_{2j} J_{2k, 2j}(\alpha) \quad \text{where} \quad J_{k,l}(\alpha) = \int_{S^2} (\mathbf{u}\cdot\mathbf{z})^k (\mathbf{u}\cdot\mathbf{a})^l d\Omega.
\end{equation}
To evaluate the angular moments $J_{k,l}(\alpha)$, the AI discovered three separate sub-methods:

\textbf{Method 1: Generating Function Approach.} The AI defined a generating function $G(\lambda, \mu) = \int_{S^2} e^{\lambda \mathbf{u}\cdot\mathbf{z} + \mu \mathbf{u}\cdot\mathbf{a}} d\Omega$. Letting $\mathbf{K} = \lambda \mathbf{z} + \mu \mathbf{a}$ and aligning the polar axis with $\mathbf{K}$, the AI integrated to find:
\begin{equation}
    G(\lambda, \mu) = 2\pi \frac{e^K - e^{-K}}{K} = 4\pi \frac{\sinh K}{K},
\end{equation}
where $K^2 = \lambda^2 + \mu^2 + 2\lambda\mu \cos \alpha$. The AI then expanded $\sinh K / K$:
\begin{equation}
    \frac{\sinh K}{K} = \sum_{s=0}^\infty \frac{K^{2s}}{(2s+1)!} = \sum_{s=0}^\infty \frac{(\lambda^2 + \mu^2 + 2\lambda\mu \cos \alpha)^s}{(2s+1)!}.
\end{equation}
This allowed the AI to compute $J_{2k, 2j}$ via differentiation $\left[ \partial_\lambda^{2k} \partial_\mu^{2j} G \right]_{\lambda=\mu=0}$, resulting in an explicit sum involving factorials and powers of $\cos \alpha$.

\textbf{Method 2: Gaussian Integral Lifting.} Alternatively, the AI lifted the basic equation into $\mathbb{R}^3$ by introducing a Gaussian weight:
\begin{equation}
    M = \int_{\mathbb{R}^3} e^{-r^2} (\mathbf{r} \cdot \mathbf{z})^{2k} (\mathbf{r} \cdot \mathbf{a})^{2j} d^3\mathbf{r}.
\end{equation}
By switching to spherical coordinates $\mathbf{r} = r \mathbf{u}$, the radial integral separates and evaluates to $\frac{1}{2} \Gamma(k+j+3/2)$, implying $M = \frac{1}{2} \Gamma(k+j+3/2) J_{2k, 2j}$. Simultaneously, applying the differential operator identity $(\mathbf{r} \cdot \mathbf{z})^{2k} = \partial_\lambda^{2k} |_{\lambda=0} e^{\lambda \mathbf{r} \cdot \mathbf{z}}$ directly to the standard Gaussian integral allowed the AI to evaluate $M$ analytically:
\begin{equation}
    M = \pi^{3/2} \partial_\lambda^{2k} \partial_\mu^{2j} \bigg|_{\lambda=\mu=0} \exp\left( \frac{\lambda^2 + \mu^2 + 2\lambda\mu \cos \alpha}{4} \right).
\end{equation}
Equating the two representations of $M$ isolates $J_{2k, 2j}$.

\textbf{Method 3: Hybrid Coordinate Transformation.} The AI's third approach projected the power series onto a Legendre basis $P_{2m}(t)$, expanding $t^{2k} = \sum_{m=0}^k \mathcal{T}_{k,m} P_{2m}(t)$, where $\mathcal{T}_{k,m}$ are known analytical coefficients. Substituting this back into the Taylor expansion yielded $f_N(t) = \sum_{m=0}^\infty \left( \sum_{k=m}^\infty d_{2k} \mathcal{T}_{k,m} \right) P_{2m}(t)$.

\textit{AI Self-Correction:} While mathematically correct, the AI's automated Python verification routines revealed that all three Monomial methods were highly unstable. The calculation of the Taylor coefficients $d_{2k}$ involves alternating sums of very large numbers. For large $N$, this triggers large $O(e^{N\pi})$ cancellation, rendering the methods computationally intractable without arbitrary-precision libraries. Recognizing this, the AI autonomously pivoted to Spectral methods.

\vspace{0.3cm}
\noindent \textbf{Class II: Spectral Basis Approaches (The Pivot to Stability)}

The AI recognized that because $I(N, \alpha)$ is a spherical self-convolution, expanding $f_N(t) = \sum C_{2j} P_{2j}(t)$ (where $P$ are standard Legendre polynomials) allows the direct application of the Funk-Hecke Convolution Theorem to diagonalize the integral:
\begin{equation} \label{eq:spectral_sol}
    I(N, \alpha) = 4\pi \sum_{j=0}^\infty \frac{C_{2j}^2}{4j+1} P_{2j}(\cos \alpha).
\end{equation}
The problem now reduced to finding the Legendre coefficients $C_{2j}$ efficiently without relying on the unstable Taylor series. The AI found two stable $O(N)$ methods to do this:

\textbf{Method 4: Spectral Galerkin (Matrix Method).} The AI formulated a linear system by defining $g(t) = (1-t^2)f_N(t) = 1 - (-1)^N \cos(At)$. Substituting $f_N = \sum_j C_{2j} P_{2j}$ and projecting onto a test function $P_{2i}$ yielded:
\begin{equation}
    \sum_j C_{2j} \int_{-1}^1 (1-t^2) P_{2i}(t) P_{2j}(t) dt = \int_{-1}^1 P_{2i}(t) g(t) dt.
\end{equation}
This forms a matrix equation $\mathbf{G}\mathbf{C} = \mathbf{b}$. Using the recurrence identity $t^2 P_l = A_l P_{l+2} + B_l P_l + C_l P_{l-2}$, the AI deduced an explicit formula for the matrix elements $G_{ij}$ and proved that $\mathbf{G}$ is a symmetric positive-definite tridiagonal matrix. The RHS vector $\mathbf{b}$ was evaluated using the Bauer plane wave expansion to yield $b_i = 2\delta_{i0} - 2(-1)^{N+i} j_{2i}(A)$, allowing the coefficients to be solved with high stability.

\textbf{Method 5: Spectral Volterra (Recurrence Method).} Pushing further, the AI derived a forward recurrence for $C_{2j}$. Starting from $C_l = \frac{2l+1}{2} \int_{-1}^1 f_N(t) P_l(t) dt = \frac{2l+1}{2}\gamma_l$, the AI multiplied the Legendre differential equation by $f_N(t)$ and integrated by parts:
\begin{equation}
    l(l+1)\gamma_l = \int_{-1}^1 f_N'(t) (1-t^2) P_l'(t) \, dt.
\end{equation}
Differentiating $(1 - t^2)f_N(t) = 1 - (-1)^N \cos(At)$ gives $(1 - t^2)f_N'(t) = (-1)^N A\sin(At) + 2tf_N(t)$. Substituting this back into the integral brilliantly split the result into two tractable terms, $T_1(l)$ and $T_2(l)$:
\begin{equation}
    l(l+1)\gamma_l = \underbrace{\int_{-1}^1 (-1)^N A\sin(At) P_l'(t) dt}_{T_1(l)} + \underbrace{\int_{-1}^1 2t f_N(t) P_l'(t) dt}_{T_2(l)}.
\end{equation}
The AI evaluated $T_1(2j)$ via spherical Bessel functions as $-2A^2 (-1)^{N+j} j_{2j}(A)$. It expanded $T_2(2j)$ into a telescoping sum structure, yielding a clean $O(N)$ forward recurrence:
\begin{equation}
    \frac{4(2j^2 - j)}{4j+1} C_{2j} = T_1(2j) + S_R(j),
\end{equation}
where $S_R(j) = \sum_{m=0}^{j-1} (8m+2)\gamma_{2m}$ is a running sum of lower-order coefficients. 

\vspace{0.3cm}
\noindent \textbf{Class III: The Exact Analytic Solution (Method 6)}

While Methods 4 and 5 provided stable algorithmic resolutions, the crowning achievement was finding an exact analytic closed-form solution via the \textbf{Gegenbauer Method}.

The AI insightfully chose to expand $f_N(t)$ not in standard Legendre polynomials, but in the basis of Gegenbauer polynomials $C_{2m}^{(3/2)}(t)$:
\begin{equation}
    f_N(t) = \sum_{m=0}^{\infty} b_{2m} C_{2m}^{(3/2)}(t).
\end{equation}
It realized that Gegenbauer polynomials are orthogonal with respect to the specific weight $w(t) = 1-t^2$. This specific weight miraculously cancels the problematic singular denominator in $f_N(t)$ when determining the coefficients by orthogonality:
\begin{equation}
    b_{2m} = \frac{1}{h_{2m}} \int_{-1}^{1} \underbrace{\left[ \frac{1-(-1)^N \cos(N\pi t)}{1-t^2} \right]}_{f_N(t)} C_{2m}^{(3/2)}(t) \underbrace{(1-t^2)}_{\text{weight}} dt,
\end{equation}
where $h_{2m}$ is the standard normalization constant. By using the identity $C_k^{(3/2)}(t) = P_{k+1}'(t)$ and integrating by parts, the boundary terms vanish (because the numerator $1 - (-1)^N \cos(N\pi t)$ evaluates to 0 at $t = \pm 1$), and the integral elegantly reduces to the Fourier transform of Legendre polynomials:
\begin{equation}
    b_{2m} = -A (-1)^{N+m} j_{2m+1}(A) \frac{4m+3}{(2m+1)(2m+2)}.
\end{equation}

Using the expansion $C_{2m}^{(3/2)}(t) = \sum_{j=0}^{m} (4j+1) P_{2j}(t)$, the AI then relates the Legendre coefficients $C_{2j}$ (from Equation \ref{eq:spectral_sol}) to these Gegenbauer coefficients  via a tail sum:
\begin{equation}
    C_{2j} = (4j+1) \sum_{m=j}^{\infty} b_{2m},
\end{equation}
 noting a critical property: when $j=0$, the sum of all Gegenbauer coefficients simply equals $C_0$:
\begin{equation}
    C_0 = \sum_{m=0}^{\infty} b_{2m}.
\end{equation}

Finally, during the hierarchical refinement phase, the advanced Gemini model exploited a partial fraction decomposition of $\frac{1}{1-t^2} = \frac{1}{2}\left(\frac{1}{1-t} + \frac{1}{1+t}\right)$ to telescope the series and find an exact analytical expression for $C_0$:
\begin{equation}
    C_0 = \frac{1}{2} \int_{-1}^1 \frac{1 - (-1)^N \cos(N\pi t)}{1-t^2} \, dt = \frac{1}{2} \int_{0}^{2N\pi} \frac{1 - \cos(x)}{x} \, dx.
\end{equation}

Recognizing this final integral as the standard definition of the generalized cosine integral function $\text{Cin}(z) \equiv \int_0^z \frac{1 - \cos(t)}{t} dt$, the AI produced the final, exact closed-form solution:
\begin{equation} \label{eq:C0_exact}
    C_0 = \frac{1}{2} \text{Cin}(2N\pi).
\end{equation}
Combining, we obtain
\begin{equation}
    C_{2j} = (4j+1) \left(C_0 + A(-1)^N \sum_{m=0}^{j-1} (-1)^m j_{2m+1}(A) \frac{4m + 3}{(2m + 1)(2m + 2)} \right).
\end{equation}
\\\\
\textbf{Summary of Discovered Methods}

The table below summarizes the AI's discoveries and their computational tradeoffs. This AI-driven work provides a stable, closed-form analytic solution for the spectral coefficients. 

\begin{table}[h]
\centering
\begin{tabular}{@{}llll@{}}
\toprule
\textbf{Method} & \textbf{Core Technique} & \textbf{Complexity} & \textbf{Numerical Stability} \\ \midrule
1, 2, 3 & Monomial Expansions & $O(N^2)$ & \textbf{Unstable} (Catastrophic Cancellation) \\
4 & Spectral Galerkin Matrix & $O(N)$ & \textbf{Stable} (Tridiagonal SPD Matrix) \\
5 & Spectral Volterra Recurrence & $O(N)$ & \textbf{Stable} (Forward Step Recurrence) \\
6 & Gegenbauer Expansion & $O(1)$ & \textbf{Stable} (Exact Analytic Closed-Form) \\ \bottomrule
\end{tabular}
\vspace{0.2cm}
\caption{Comparison of AI-discovered methods for evaluating the cosmic string radiation integral $I(N, \alpha)$.}
\end{table}

We refer the reader to \cite{brennerFollowup} for our more recent developments. 

\newpage

\section{Algorithms and Bounds}
In this section we see the AI acting as an algorithmic optimizer—tightening mathematical bounds, removing logarithmic factors, and discovering exact optimal constants.

\subsection{Graph Theory: Biclique Partitions}\label{sec:biclique}
\textit{Written by Benjamin Przybocki and Bernardo Subercaseaux.}

\subsubsection*{Problem Context}
Biclique partitions are a well-studied topic in graph theory with applications to secret sharing, circuit design, and graph compression. Together with Andrew Krapivin and Nicol{\'a}s Sanhueza-Matamala, we recently proved that every $n$-vertex graph admits an {\it integral} biclique partition of weight at most $(\frac{1}{2}+o(1))n^{2}/\lg(n)$, matching an information-theoretic lower bound \cite{krapivin2025}. However, in the context of cryptographic secret sharing, \emph{fractional} solutions to the biclique partition problem are also studied. Previous work by Csirmaz, Ligeti, and Tardos \cite{csirmaz2015} had shown a fractional upper bound with a constant of $\frac{1}{2}$, and a fractional lower bound of $0.265$. Since we were able to match their fractional upper bound with integral solutions, a natural question was whether we could use our techniques to push the fractional upper bound strictly below $\frac{1}{2}$.

\subsubsection*{AI Contribution}
We had our paper fed to the model and asked it to resolve the open problem. 
\begin{figure}[!htb]
\begin{userturn}
Please try to improve the paper by identifying and solving the open problem from it. 
\\\\
<Math rigor prompt>
\end{userturn}
\end{figure}
While its initial proof attempt had errors, it ultimately led us to solve the problem by providing a crucial insight on how one of our lemmas could be repurposed in a way we had not seen.

\begin{itemize}
    \item \textbf{Proof Strategy:} The AI suggested a strategy based on partitioning vertices into ``good'' (degree close to $n/2$) and ``bad''. We had considered this before, so we had some understanding of the technical hurdle the AI would need to overcome to make it work.
    \item \textbf{Overcoming the Hurdle:} In the AI's proof attempt, it pointed out a crucial idea that allowed us to get better control on the proportion of bad vertices. This insight was not obvious to us as it was a consequence of a construction we had used for a different purpose.
    \item \textbf{From Insight to Proof:} The AI's initial proof attempt had errors (incorrect constants and Taylor approximations), but the core idea—that ``bad'' vertices actually help reduce the partition weight—was correct.
\end{itemize}

\subsubsection*{Technical Details: Fractional and Integral Biclique Partitions}

\begin{definition}
    Given a graph $G$, a \emph{fractional biclique partition} is a function $\mathcal{B} : \mathfrak{B} \to [0,1]$, where $\mathfrak{B}$ is the set of all complete bipartite subgraphs of $G$, such that for all $e \in E(G)$, we have 
    \[
    \sum_{B \in \mathfrak{B} } \mathbf{1}_{[e \in B]} \cdot \mathcal{B}(B) = 1.
    \]The \emph{weight} $w(\mathcal{B})$ of such a partition is $w(\mathcal{B}) := \sum_{B \in \mathfrak{B}} \mathcal{B}(B) \cdot |V(B)|$. A fractional biclique partition $\mathcal{B}$ is \emph{integral} if $\mathcal{B}(B) \in \{0,1\}$ for all $B \in \mathfrak{B}$.
\end{definition}
\noindent
The following notation will be useful for discussing the problem at hand.

\begin{definition}
    Given a graph $G$, let $\textsf{BP}(G)$ (resp., $\textsf{BP}^\ast(G)$) be the minimum weight of an integral (resp., fractional) biclique partition of $G$. Then, given $n \geq 1$, let
    \[
    \textsf{BP}(n) := \max_{G, \; |V(G)|=n} \textsf{BP}(G) \quad \text{and} \quad  \textsf{BP}^\ast(n) := \max_{G, \; |V(G)|=n} \textsf{BP}^\ast(G).
    \]
\end{definition}
\noindent
We had recently proven the following theorem regarding optimal integral biclique partitions:
\begin{theorem}[Krapivin, Przybocki, Subercaseaux, and Sanhueza-Matamala~\cite{krapivin2025}]\label{thm:bp12}
    $BP(n) \sim \frac{1}{2} \cdot \frac{n^2}{\lg n}$.
\end{theorem}
\noindent
In turn, the best bounds for fractional biclique partitions were:
\begin{theorem}[Csirmaz, Ligeti, and Tardos~\cite{csirmaz2015}] \label{thm-fractional}
    $(0.265 - o(1)) \frac{n^2}{\lg n} \le BP^\ast(n) \le (\frac{1}{2} + o(1)) \cdot \frac{n^2}{\lg n}$.
\end{theorem}

Importantly, Csirmaz, Ligeti, and Tardos also proved that on graphs of large minimum degree, a better upper bound can be obtained:

\begin{theorem}[C.L.T.~\cite{csirmaz2015}]\label{thm:clt-2}
Let $p \in [0,1]$ be a fixed constant, and let $G$ be a graph on $n$ vertices such that $\deg(v) \geq p \cdot n$ for every $v \in V(G)$. Then,
\[
\textsf{BP}^\ast(G) \leq (0.725\cdot (1-p) + o(1)) \frac{n^2}{\lg n}.
\]
\end{theorem}
In particular, for any graph $G$ with minimum degree at least $0.32 n$ (such as a $G(n, \frac{1}{2})$ graph, w.h.p.), this gives an upper bound strictly better than \Cref{thm-fractional}. On the other hand, for graphs whose edge density is far from $0.5$, we also have an improved upper bound, even for integral biclique partitions:
\begin{theorem}[K.P.S.S.-M.~\cite{krapivin2025}]\label{thm:density-aware}
    Let $\gamma \in (0, 1)$ be such that $\max\{\gamma^{-1}, (1-\gamma)^{-1}\} = n^{o(1)}$. Then, given a graph $G$ of density $\gamma$, we have
    \(
        BP(G) \le (\frac{1}{2} + o(1)) \cdot h_2(\gamma) \frac{n^2}{\lg n},
    \)
    where $h_2(x) := -x \lg x -(1-x)\lg (1-x)$ is the binary entropy function.
\end{theorem}

Moreover, we also proved a similar result that applies when the degree of every vertex is bounded away from $0.5$, even if the global density is close to $0.5$:
\begin{theorem}[K.P.S.S-M.~\cite{krapivin2025}]\label{thm:density-aware-local}
    Let $\dot{\gamma} \in (0, 1)$ be such that $\dot{\gamma}^{-1} = n^{o(1)}$. Then, given a graph $G$ such that $\min\{d(v)/n, 1 - d(v)/n\} \le \dot{\gamma}$, we have
    \(
        BP(G) \le (\frac{1}{2} + o(1)) \cdot h_2(\dot{\gamma}) \frac{n^2}{\lg n}.
    \)
\end{theorem}

Therefore, in light of \Cref{thm:clt-2,thm:density-aware,thm:density-aware-local}, if it is not possible to improve the $\frac{1}{2}$ constant from \Cref{thm-fractional}, then the problematic graphs must have a particular structure: (i) a constant fraction of vertices of degree at most $0.32n$, (ii) edge density $0.5$, and (iii) a constant fraction of vertices whose degree is approximately $\frac{n}{2}$.

Naturally, we thought of decomposing problematic graphs into different pieces that could be separately handled by the three aforementioned results. Roughly, the idea was to partition the vertices of a given graph into those whose degree is close to $\frac{n}{2}$ (call these $V_{\text{good}}$) and those whose degree is far from $\frac{n}{2}$ (call these $V_{\text{bad}}$). Our hope was to use Csirmaz, Ligeti, and Tardos's result to construct a biclique partition for $G[V_{\text{good}}]$ and then try to construct a biclique partition for the remaining edges using a different method. But, in order for this to work, we needed $V_{\text{bad}}$ to be a small proportion of the vertices, and it did not seem possible to ensure this with our strategy. We therefore left this as an open problem.

The AI model was fed our paper, and at our suggestion, was asked to resolve this open problem.\footnote{Initially, the AI model tried a different open problem from our paper, but its response did not seem very promising to us, and we were more interested in this problem anyway.}
The proof attempt generated by the AI was similar to the strategy outlined in the previous paragraph, which was impressive given that we did not share our ideas with the AI. But there was one crucial difference. The AI noticed that the construction we used for \Cref{thm:density-aware-local} actually implies that we can improve the $\frac{1}{2}$ constant from \Cref{thm-fractional} whenever the average of $|\deg(v)/n-\frac{1}{2}|$ is bounded away from 0. We neglected to consider this ourselves, since \Cref{thm:density-aware-local}, as stated here, is in fact a corollary of a stronger theorem in our paper bounding the maximum number of bicliques any given vertex belongs to; for this stronger result, it was necessary to assume that $|\deg(v)/n-\frac{1}{2}|$ is bounded away from 0 for \emph{every} vertex rather than merely on average.

However, despite the promising high-level strategy, the AI's proof attempt contained several mistakes, which did not seem to be fixable by local modifications of the proof attempt at least as written. For example, the AI asserted that some expression is $\approx 0.31966$ when it is actually $\approx 0.13933$, and it had a nonsensical asymptotic expression $Ax+x^{3}/2+O(x^{2})$, where $x$ is a constant. Nevertheless, the AI's observation in the previous paragraph turned out to be exactly what we needed to overcome the difficulties we had encountered in our first attempt, and we proved that every $n$-vertex graph admits a fractional biclique partition of weight at most $(0.4999+o(1))n^{2}/\lg(n)$. We made no attempt to optimize the constant $0.4999$; in fact, it seems to us that where the AI went wrong is trying too aggressively to optimize the constant in its proof, which resulted in overly complicated algebraic calculations that the AI struggled to accurately manipulate.

\paragraph{Final Proof}
For technical reasons, \Cref{thm:density-aware,thm:density-aware-local} require $\gamma$ or $\dot{\gamma}$ to not be too close to 0 or 1. For similar reasons, it is convenient to define a ``clipped'' version of the binary entropy function, which allows us to state \Cref{prop:fractional-cover}, a lemma corresponding to the AI's central insight mentioned above. The AI originally stated the lemma using the non-clipped binary entropy function, although the lemma is incorrect when stated that way. Interestingly, when the AI was asked the provide a rigorous proof of the lemma, it realized it could not be true as stated, although it did not discover that the clipped binary entropy function provides a simple way to salvage it that suffices for its application.

\begin{definition}
    Let $h_2(x) := -x \lg x -(1-x) \lg (1-x)$ be the binary entropy function, and let $h_2^\dagger \colon [0,1] \to [0, 1]$ be defined by
    \[
        h_2^\dagger(x) = \begin{cases}
            h_2(x) & \text {if } x \in [0.01, 0.99]\\
            h_2(0.01) & \text{if } x \not\in [0.01, 0.99].
        \end{cases}
    \]
\end{definition}

\begin{lemma}\label{prop:fractional-cover}
Let $G$ be an $n$-vertex graph, and let $\bar{h} = \bar{h}(G) := \frac{1}{n} \sum_{v \in V} h_2^\dagger(\deg(v)/n)$. Then,
    \[
    \textsf{BP}^\ast(G) \leq \left(\frac{\bar{h}}{2} + o(1)\right)\frac{n^2}{\lg n}.
   \]
\end{lemma}

The proof of~\Cref{prop:fractional-cover} is a relatively straightforward modification of the proof for~\Cref{thm:density-aware-local}, which corresponds to~\cite[Theorem~26]{krapivin2025}. With it, the proof of the final result is surprisingly simple.

\begin{theorem}
    $\textsf{BP}^\ast(n) \le (0.4999 + o(1)) \cdot \frac{n^2}{\lg n}$.
\end{theorem}
\begin{proof}
    Let $G$ be an arbitrary $n$-vertex graph. Let $V_{\text{good}} = \{v \in V(G) \mid d(v) \ge 0.37n\}$ and $V_{\text{bad}} = V(G) \setminus V_{\text{good}}$. 
   Then, note that
    \begin{align*}
        \bar{h}(G) = \frac{1}{n} \sum_{v \in V} h_2^\dagger(\deg(v)/n) &= \frac{1}{n} \left(\sum_{v \in V_{\text{bad}}} h_2^\dagger(\deg(v)/n) + \sum_{v \in V_{\text{good}}} h_2^\dagger(\deg(v)/n)  \right) \\
        &\le  \frac{1}{n}\left(|V_{\text{bad}}| \cdot h_2(0.37) + |V_{\text{good}}| \cdot 1\right).
    \end{align*}
    Thus, if $|V_{\text{bad}}| \ge n/100$, then 
    by \Cref{prop:fractional-cover},
    \[
        \textsf{BP}^\ast(G) \le \left(\frac{0.02 \cdot h_2(0.37) + 0.98}{2} + o(1)\right)\frac{n^2}{\lg n} \le (0.4999 + o(1)) \frac{n^2}{\lg n}.
    \]
    
    Therefore, we may assume that $|V_{\text{bad}}| < n/100$. Then, every vertex of $G[V_{\text{good}}]$ has degree at least $0.36 |V_{\text{good}}|$, whence \Cref{thm:clt-2} implies
    \[
        \textsf{BP}^\ast(G[V_{\text{good}}]) \leq (0.725\cdot (1-0.36) + o(1)) \frac{n^2}{\lg n} = (0.464 + o(1)) \frac{n^2}{\lg n}.
    \]
    Let $H$ be the graph with $V(H) = V(G)$ and $E(H) = \{e \in E(G) \mid e \cap V_{\text{bad}} \neq \emptyset\}$. It remains to construct a fractional biclique partition for $H$. Note that $|E(H)| \le |V_{\text{bad}}| \cdot 0.37n \le 0.0037n^2$, from where
    the edge density of $H$ is
    at most $0.0074 + o(1)$. Hence, by \Cref{thm:density-aware},
    \[
        \textsf{BP}^\ast(H) \le \left(\frac{h_2(0.0074)}{2} + o(1)\right) \frac{n^2}{\lg n} \le (0.0316 + o(1)) \frac{n^2}{\lg n}.
    \]
    Since, $E(G) := E(H) \sqcup E(G[V_{\text{good}}])$, we have
    \[
        \textsf{BP}^\ast(G) \le \textsf{BP}^\ast(G[V_{\text{good}}]) + \textsf{BP}^\ast(H) \le (0.496 + o(1)) \frac{n^2}{\lg n}. \qedhere
    \]
\end{proof}

\newpage

\subsection{Query Complexity: Local Search on General Graphs}\label{sec:local_search}
\textit{Written by Simina Br\^anzei, Ioannis Panageas, Dimitris Paparas \cite{BPP26}.}

\subsubsection*{Problem Context}

We analyze the query complexity of the abstract problem of finding a local minimum of a function defined on a general graph using $t$ rounds of interaction with the oracle. This theoretical model captures optimization tasks such as training neural networks, where the goal is to minimize a loss function, each ``query'' is an expensive loss evaluation, and batching queries is crucial for efficiency.

Prior work \cite{branzei2022query} resolved the query complexity of local search in rounds for the $d$-dimensional grid. However, the problem remained open for general graphs, which are essential for modeling non-Euclidean geometries such as manifold discretizations.

\subsubsection*{AI Contribution}

The AI was instrumental in obtaining the following results on the query complexity of local search on arbitrary graphs:
\begin{itemize} 
    \item A deterministic upper bound as a function of the number of vertices $n$, the number of rounds $t$, and the separation number $s$ of the graph.
    \item A randomized lower bound dependent on $n$ and $t$. 
\end{itemize}

These results constitute a research paper  \cite{BPP26}, developed entirely by iterating with the AI via a ``scaffolded reasoning'' process. 
Rather than asking the model to generate the paper from scratch, we treated it as a junior research assistant: we defined the lemmas and tasked the AI with writing the proofs, which we then rigorously verified and refined. 
We directed the model to expand unclear sections and flagged issues where applicable. This approach can be characterized as a form of ``vibe proving''.

\begin{itemize}
    \item \textbf{Algorithm Design:} We tasked the AI with deriving an upper bound for two-round local search parameterized by the graph's separation number. The model synthesized a two-round algorithm with $O(\sqrt{ns\Delta})$ queries and stated a supporting ``Shattering Lemma'' about separator decompositions.
    
    \item \textbf{Iterative Improvements:} For all the upper and lower  bounds, we guided the model to gradually generalize (e.g., obtain a lower bound for a tree in $2$ rounds and then generalize to any graph in $2$ rounds). We gave several relevant papers as context in all iterations. 
    \item \textbf{Adversarial Hypothesis Testing:} When we hypothesized a linear lower bound for local search in two rounds on constant-degree expanders, the AI refuted it by proposing a counter-algorithm (Randomized Parallel Steepest Descent) and demonstrating an upper bound of $O(n/\log{n})$ for it.
\end{itemize}

\textbf{\em Problem choice.} We chose this problem for several reasons. One of the authors has expertise on the topic of  local search (with a prior paper that focuses on grids \cite{branzei2022query}). Investigating the query complexity  for general geometries was written in an NSF grant by this author. Moreover, one of the journal reviewers of the paper on local search in rounds on grids \cite{branzei2022query} independently raised this question.  Pedagogically, the problem seemed ideal for a  collaboration with a student, as it allowed for a natural progression from simple cases---such as two-round search on trees---to general results.

\medskip 

\textbf{\em Strengths.} The AI model demonstrated distinct strengths during our collaboration. One of the key contributions was the model's tailoring of the  classical staircase construction to the round setting. While the staircase technique is well-known, the model's instantiation ensured that in the hard distribution, every vertex is the (only) local minimum with equal probability. This created a recursive structure where, at the end of each round, the set of viable candidates forms a smaller instance of the original problem, enabling a clean inductive argument to go through. The elegance of this structure was a very nice surprise.

\medskip 

\textbf{\em Weaknesses.} The model also exhibited weaknesses, such as sometimes producing incorrect proofs or conveniently mis-interpreting the question so that it could answer it.
For instance, when we first asked for a lower bound, it would often switch to proving an existential one. Specifically, when queried about a lower bound for local search on graphs with $n$ vertices and separation number $s$, it created a ``hard graph'' for these parameters (such as a line where each vertex is connected to a clique of size $s$).
We also encountered occasional hallucinations of non-existent results. Finally, when we prompted without suggesting  graph features, the AI model gave upper bounds that were less interpretable.

\subsubsection*{Model and Results}

Let $G = (V,E)$ be a connected undirected graph with vertex set $V = [n] = \{1, \ldots, n\}$. Let $f : V \to \mathbb{R}$. 
We call a vertex $v \in V$ a \emph{local minimum} if $f(v) \leq f(u)$ for every neighbor $u$ of $v$. 
In the local search problem, we are given $G$ and oracle access to the values of $f$. The goal is to locate a local minimum using a minimum number of queries, where a query at vertex $v$ reveals $f(v)$.

Suppose the interaction with the oracle has $t$ rounds. In each round $j$, the algorithm selects a set of queries based on the outcomes of rounds $1, \dots, j-1$, submits them, and receives the answers. 
The algorithm terminates after round $t$ with a proposed solution.

\paragraph{Query Complexity.} 
The \emph{deterministic query complexity} is the minimum number of queries required for an optimal deterministic algorithm to find a local minimum on a worst case instance.

The \emph{randomized query complexity} is the minimum worst-case number of queries required by a randomized algorithm to compute the function with probability at least $9/10$ for every input\footnote{The constant $9/10$ can be replaced by any other constant strictly more than $1/2$.}.

\paragraph{Graph Features.}
Let $\Delta$ be the maximum degree of $G$, and let $dist(u,v)$ denote the shortest-path distance between $u,v \in V$.
Given a parameter $1/2 \le \alpha < 1$ and an integer $s \in \mathbb{N}$, a subset $S \subseteq V$ is an \emph{$(s, \alpha)$-separator} of $G$ if there is a partition of $V$ into disjoint sets $A, B, S$ satisfying:
(a)  $V = A \cup B \cup S$;
(b) $|S| \le s$ and $|A|, |B| \le \alpha|V|$;
(c) No edges connect $A$ and $B$ (i.e., $E(A,B) = \emptyset$).

The \emph{separation number} $s(G)$ is the minimum $s$ such that every subgraph of $G$ admits an $(s, 2/3)$-separator. The separation number  $s(G)$ is within a constant factor of the graph's treewidth.

\subsubsection*{Experiment Narrative}

We started the investigation with trying to understand the query complexity for two rounds. 
One of the prompts used in the early stages of writing the paper is in \Cref{fig:local:search:prompt}.
\newpage
\begin{figure}[!htb]
\begin{chatbox}
\begin{userturn}
\emph{``Given a connected undirected graph $G = (V,E)$, there is an unknown function $f : V \to \mathbb{R}$, which assigns a value to every vertex $v$ (say the values are distinct). We have oracle access to the function (i.e. an algorithm can learn the value of the function at a vertex by probing that vertex). A vertex $v$ is a local minimum if $f(v) \le f(u)$ for all neighbors $u$ of $v$. What is the query complexity of finding a local minimum in $t$ rounds of interaction with the oracle?''} \\ 
        \medskip 
        
        Suppose we focus on two rounds. Can you design algorithms as a function of graph features/properties like minimum vertex cover, chromatic number, treewidth?
\end{userturn}
\end{chatbox}
\caption{Early prompt for trying to understand the query complexity for two rounds.}
\label{fig:local:search:prompt}
\end{figure}

%
        %

\FloatBarrier

The full response to this prompt is included below.

\begin{aigenerated}
\paragraph{Response:}
This is an excellent problem that explores the trade-offs between the number of interaction rounds and the total query complexity for finding a local minimum in a graph. When restricted to two rounds, algorithms must strategically balance the exploration in the first round with an adaptive, focused search in the second round. We can indeed design effective algorithms by leveraging graph structural properties.

\medskip 

\paragraph{General Framework: Decomposition-Based Approach.}

The core strategy for a two-round algorithm involves partitioning the vertices $V$ into a separator set $S$ and the remainder $V \setminus S$. The first round queries $S$. The key is that $G[V \setminus S]$ must be decomposed into small connected components, allowing the second round to adaptively explore the neighborhood of the most promising candidate.

Let $K$ be the maximum size of any connected component in $G[V \setminus S]$.

\medskip 

\textbf{Generalized Decomposition Algorithm:}
\begin{enumerate}
    \item \emph{Decomposition:} Identify a separator set $S$.
    \item \emph{Round 1:} Query all vertices in $S$ ($Q_1 = S$).
    \item \emph{Adaptivity:} Identify the minimum value found: $v_{min} = \arg\min_{v \in S} f(v)$.
    \item \emph{Round 2:} Identify all connected components of $G[V \setminus S]$ that are adjacent to $v_{min}$. Let $Q_2$ be the union of vertices in these components. Query $Q_2$.
    \item \emph{Output:} Output $v^* = \arg\min_{v \in S \cup Q_2} f(v)$.
\end{enumerate}

\emph{Correctness:} The algorithm guarantees finding a local minimum. If $v^* \in S$ (so $v^*=v_{min}$), all its neighbors are queried. If $v^* \in Q_2$, then $f(v^*) < f(v_{min})$. Its neighbors in $S$ are larger than $f(v_{min})$, and its neighbors within its component (which is entirely in $Q_2$) are larger than $f(v^*)$ by definition.

\medskip 

\emph{Query Complexity:} Let $\Delta$ be the maximum degree of $G$. Vertex $v_{min}$ can be adjacent to at most $\Delta$ components, each of size at most $K$. Thus, $|Q_2| \le K \cdot \Delta$. The total query complexity is bounded by $|S| + K \cdot \Delta$. We aim to minimize this quantity by optimizing the trade-off between $|S|$ and $K$.

\paragraph{1. Algorithm based on Minimum Vertex Cover (VC).}

A Vertex Cover $C$ is a separator such that $V \setminus C$ is an independent set (components of size $K=1$).

\medskip 

\emph{Algorithm (VC-Adaptive):} Use the minimum vertex cover (size $\tau(G)$) as the separator $S$.

\medskip 

\emph{Query Complexity:} $QC_{VC} = \tau(G) + 1 \cdot \Delta = \tau(G) + \Delta$.
This is efficient if both $\tau(G)$ and $\Delta$ are small. However, for many graphs (e.g., a path graph where $\tau(G) \approx N/2$), this complexity is $O(N)$.

\paragraph{2. Algorithm based on Chromatic Number ($\chi(G)$).}

The chromatic number relates to the size of the maximum independent set $\alpha(G)$, with $\alpha(G) \ge N/\chi(G)$. We can use the complement of the largest independent set as the separator $S$ (a vertex cover). This is also the $K=1$ case.

\medskip 

\emph{Query Complexity:} $QC_{\chi} = (N - \alpha(G)) + \Delta \le N(1 - 1/\chi(G)) + \Delta$.
This is effective for graphs with small chromatic numbers and bounded degrees, such as bipartite graphs, where the complexity is $\le N/2 + \Delta$.

\paragraph{3. Algorithm based on Treewidth (tw(G)).}

Treewidth ($w=tw(G)$) measures the separability of a graph. We can optimize the generalized framework by choosing $K>1$. Let $s_K(G)$ be the minimum size of a separator $S$ such that all components of $G[V \setminus S]$ have size at most $K$. The complexity is: $QC = \min_K (s_K(G) + K \cdot \Delta)$.

Graphs with low treewidth can be efficiently shattered. Specifically, many graph classes with bounded treewidth satisfy the property $s_K(G) = O(N \cdot w / K)$ (known as having $1/K$-separators).

\medskip 

\emph{Optimized Complexity:} We minimize $O(N w / K) + K \Delta$. The balance is achieved when the two terms are roughly equal: $N w / K \approx K \Delta$. This yields an optimal component size $K \approx \sqrt{N w / \Delta}$.
The resulting query complexity is: $QC_{TW} = O(\sqrt{N \cdot tw(G) \cdot \Delta})$.

\medskip 

\emph{Implications:} This demonstrates that sublinear query complexity is achievable in two rounds for structured graphs.
\begin{itemize}[nosep, leftmargin=*]
    \item \emph{Paths and Trees:} For paths ($w=1, \Delta=2$) or bounded-degree trees, the complexity is $O(\sqrt{N})$.
    \item \emph{Bounded Treewidth/Degree Graphs:} If $tw(G)=O(1)$ and $\Delta=O(1)$, the complexity is $O(\sqrt{N})$.
\end{itemize}

\medskip 

\emph{Example: Planar Graphs.} Planar graphs utilize stronger separator theorems ($s_K(G) = O(N/\sqrt{K})$). Optimizing $O(N/\sqrt{K}) + K \Delta$ yields a complexity of $O(N^{2/3} \Delta^{1/3})$.

\paragraph{Summary.}

By exploiting graph structures, we can design efficient two-round algorithms. The Vertex Cover and Chromatic Number approaches provide parameterized complexity based on the $K=1$ decomposition. The Treewidth approach leverages generalized decomposition (optimized $K$) and allows for sublinear query complexity in various important graph classes.
\end{aigenerated}

\FloatBarrier

The separation number and treewidth are related by a constant factor; thus, approach (3) from the response is essentially the two-round algorithm presented in our full paper. 

\medskip 

In subsequent interactions, we prompted the model to generalize this two-round strategy to $t \geq 2$ rounds. 
The model successfully synthesized a $t$-round algorithm anchored by a ``Shattering Lemma''. The lemma is: ``\emph{For any parameter $K \in [1, n]$, there exists a separator $S \subseteq V$ such that every connected component of $G[V \setminus S]$ has size at most $K$, with $|S| < 3sn/K$}. This lemma is folklore and was central to constructing the general strategy. We also included the vertex cover-based algorithm suggested by the model into the paper's introduction as an example.

\medskip 

We obtained this way the following theorem.
\begin{theorem}[Theorem 1 in \cite{BPP26}] \label{thm:k_rounds}
    Let $G = (V, E)$ be a connected undirected graph with $n$ vertices. The deterministic query complexity of finding a local minimum on $G$ in $t \ge 2$ rounds is at most 
    $
        \min\bigl(4t \cdot n^{\frac{1}{t}} (s\Delta)^{1-\frac{1}{t}}, n \bigr),
    $
    where $\Delta$ is the maximum degree and $s$ is the separation number of $G$.
\end{theorem}

The development of the lower bounds was more iterative. Initially, the model insisted on constructing hard graphs (e.g., line graphs with cliques of size $s$ attached to each node) rather than finding a hard distribution of functions for a \textit{fixed} graph. Some of the lower bounds derived in this stage were not correct, or relied on looking at the longest path in the graph or using Menger's theorem.

\medskip 

Eventually we created a prompt that had the following elements:
\begin{itemize}
\item explicitly clarifying it must keep the given graph fixed  and find a hard input  distribution (of functions) for that graph;
\item including as context prior papers on lower bounds for local search (specifically, \cite{branzei2022query} and \cite{santha2004quantum});
\item asking it to focus on local search in two rounds on a \emph{tree} and try to prove a lower bound of $\Omega(\sqrt{n})$.
\end{itemize}  

The prompt, given in \Cref{fig:prompt:lb:two:rounds}, also contained the version of the paper up to that point, which contained only algorithmic results (Theorem~\ref{thm:k_rounds}). 

\begin{figure}[!htb]
\begin{chatbox}
\begin{userturn}
I am working on the following manuscript. Can you read it (ignoring the todos/notes) and then help me prove a lower bound of $\sqrt{n}$ on the randomized query complexity of finding a local minimum in two rounds on a tree? The statement should be of the form: 
\begin{utheorem} Let $G = (V,E)$ be an undirected connected tree with $n$ vertices. Then the randomized query complexity of finding a local minimum in two rounds on $G$ is $\Omega(\sqrt{n})$. 
 \end{utheorem}
Note the lower bound should hold for the given tree $G$, so don't create a hard family of trees. Rather, you need to create a hard family of functions for the given tree $G$. Make it so that for each function in the family, there is a unique local minimum. Also don't hide the constants under big Omega until the very end of the proof. Attaching two papers from this literature together with my current draft; see especially the paper on local search in rounds on grids for examples of how to use Yao's lemma in such proofs (i.e. finding a hard distribution of functions and analyzing the performance of a deterministic algorithm).
\end{userturn}
\end{chatbox}
\caption{Prompt for lower bound on trees in two rounds}
\label{fig:prompt:lb:two:rounds}
\end{figure}


\FloatBarrier

The model's response to the prompt is included below: 

\begin{aigenerated}
\begin{utheorem}
Let $G = (V,E)$ be an undirected connected tree with $n>1$ vertices. Then the randomized query complexity of finding a local minimum in two rounds on $G$, with success probability at least $c=9/10$, is $\Omega(\sqrt{n})$. Specifically, it is at least $1.8\sqrt{n} - 2$.
\end{utheorem}
\begin{proof}
We use Yao's Minimax Principle. We define a distribution $\mathcal{D}$ over input functions such that each function has a unique local minimum. We then analyze the expected query complexity of any deterministic two-round algorithm $\mathcal{A}$ that succeeds with probability at least $c$ over $\mathcal{D}$.

\medskip 

\textbf{1. The Hard Distribution $\mathcal{D}$.}
Fix an arbitrary vertex $r \in V$ as the root. Let $d(v)$ be the depth of $v$ (distance from $r$). Let $Anc(v)$ be the set of ancestors of $v$ (the unique path from $r$ to $v$, inclusive). We write $u \preceq v$ if $u \in Anc(v)$.

For each $v^* \in V$, we define the function $f_{v^*}: V \to \mathbb{Z}$:
$$ f_{v^*}(x) = \begin{cases} -d(x) & \text{if } x \preceq v^* \\ \phantom{-}d(x) & \text{if } x \not\preceq v^* \end{cases} $$
$f_{v^*}$ has a unique local minimum at $v^*$.

Let $\mathcal{D}$ be the uniform distribution over $\{f_{v^*}\}_{v^* \in V}$. The input is $f_{V^*}$, where $V^*$ is chosen uniformly at random from $V$. 

\medskip 

\textbf{2. Analysis of a Deterministic Two-Round Algorithm $\mathcal{A}$.}

\medskip 

\textbf{Round 1.} $\mathcal{A}$ queries a fixed set $Q_1 \subset V$. Let $q_1 = |Q_1|$. Since the structure of $G$ is known, $d(x)$ is known. A query to $f_{V^*}(x)$ reveals whether $x \preceq V^*$. The outcome of Round 1 is the signature $S_{V^*} = Anc(V^*) \cap Q_1$.

Let $\mathcal{S}$ be the set of possible signatures. The outcomes partition $V$ into $K=|\mathcal{S}|$ blocks $C_1, \ldots, C_K$, where $C_i$ is the set of $v^*$ corresponding to the $i$-th signature.

\begin{lemma} \label{lem:tree_outcomes}
The number of distinct outcomes $K$ is at most $q_1+1$.
\end{lemma}
\begin{proof}
For any non-empty signature $S \in \mathcal{S}$, let $m(S)$ be the unique element in $S$ farthest from the root $r$. Note $m(S) \in Q_1$. We show the map $m: \mathcal{S}\setminus\{\emptyset\} \to Q_1$ is injective.

Suppose $m(S_1) = m(S_2) = x$. Let $S_1 = Anc(v_1) \cap Q_1$. Since $x \in S_1$, we have $x \preceq v_1$.
We claim $S_1 = Anc(x) \cap Q_1$.
($\subseteq$): If $y \in S_1$, then $y \in Q_1$ and $y \preceq v_1$. Since $x$ is the deepest element in $S_1$, $y$ must be an ancestor of $x$ (as all elements in $Anc(v_1)$ are comparable). Thus $y \in Anc(x) \cap Q_1$.
($\supseteq$): If $y \in Anc(x) \cap Q_1$. Since $x \preceq v_1$, $Anc(x) \subseteq Anc(v_1)$. Thus $y \in Anc(v_1) \cap Q_1 = S_1$.
Therefore, $S_1 = Anc(x) \cap Q_1$. Similarly, $S_2 = Anc(x) \cap Q_1$. Thus $S_1=S_2$.
Since $m$ is injective, there are at most $|Q_1|=q_1$ non-empty signatures. Including the possibility of the empty signature, $K \le q_1+1$.
\end{proof}

\medskip 

\textbf{Round 2.} If the outcome implies $V^* \in C_i$, $\mathcal{A}$ queries $Q_{2,i}$ of size $q_i$. This further partitions $C_i$ into $K_i$ sub-blocks. By the same logic as Lemma \ref{lem:tree_outcomes}, $K_i \le q_i+1$.

\medskip 

\textbf{3. Success Probability and Expected Cost.}
Let $P_i$ be the conditional success probability given $V^* \in C_i$. Since the distribution within $C_i$ is uniform, and the algorithm makes a deterministic guess for each sub-block $C_{i,j}$, it succeeds on at most one vertex per sub-block.
$$ P_i = \sum_{j=1}^{K_i} P(V^* \in C_{i,j}|V^* \in C_i) P(\text{Success}|V^* \in C_{i,j}) \le \sum_{j=1}^{K_i} \frac{|C_{i,j}|}{|C_i|} \frac{1}{|C_{i,j}|} = \frac{K_i}{|C_i|}. $$
(The inequality becomes equality if the algorithm guesses optimally within each sub-block).
Since $K_i \le q_i+1$, we have $|C_i|P_i \le q_i+1$, so $q_i \ge |C_i|P_i - 1$.
Also, $P_i \le 1$ since $K_i \le |C_i|$.

We require the overall success probability $P_S \ge c$.
$P_S = \sum_{i=1}^K P(V^* \in C_i) P_i = \frac{1}{n} \sum_{i=1}^K |C_i| P_i$.
Let $S=cn$. We require $\sum_{i=1}^K |C_i| P_i \ge S$.

The expected query complexity $Q_{exp}$ of $\mathcal{A}$ under $\mathcal{D}$ is:
\begin{align*}
Q_{exp} &= q_1 + \frac{1}{n} \sum_{i=1}^K |C_i| q_i 
\ge q_1 + \frac{1}{n} \sum_{i=1}^K |C_i| (|C_i|P_i - 1) 
= q_1 + \frac{1}{n} \left(\sum_{i=1}^K P_i |C_i|^2\right) - \frac{1}{n} \sum_{i=1}^K |C_i| \\
&= q_1 - 1 + \frac{1}{n} \sum_{i=1}^K P_i |C_i|^2.
\end{align*}

\textbf{4. Optimization.}
We use the Cauchy-Schwarz inequality to lower bound the summation:
\begin{align*}
\left(\sum_{i=1}^K P_i |C_i|^2\right) \left(\sum_{i=1}^K P_i\right) &\ge \left(\sum_{i=1}^K (\sqrt{P_i}|C_i|) \sqrt{P_i}\right)^2 = \left(\sum_{i=1}^K P_i |C_i|\right)^2 \ge S^2.
\end{align*}
Therefore, $\sum_{i=1}^K P_i |C_i|^2 \ge S^2 / (\sum_{i=1}^K P_i)$.

We bound the denominator. Since $P_i \le 1$, $\sum P_i \le K$. By Lemma \ref{lem:tree_outcomes}, $K \le q_1+1$.
$$ \sum_{i=1}^K P_i |C_i|^2 \ge \frac{S^2}{q_1+1}. $$

Substituting this back into the expected cost:
$ Q_{exp} \ge q_1 - 1 + \frac{S^2}{n(q_1+1)} = q_1 - 1 + \frac{c^2 n}{q_1+1}. $

This lower bound holds for any deterministic algorithm $\mathcal{A}$, characterized by its choice of $q_1$. To find the overall lower bound via Yao's principle, we minimize this expression over $q_1 \ge 0$.
Let $x = q_1+1$ ($x \ge 1$). We minimize $g(x) = x-2 + \frac{c^2 n}{x}$.
By the AM-GM inequality, $x + c^2n/x \ge 2\sqrt{c^2n} = 2c\sqrt{n}$. Equality holds when $x = c\sqrt{n}$. (Assuming $c\sqrt{n} \ge 1$, which holds for $c=9/10$ and $n>1$).
$$ Q_{exp} \ge 2c\sqrt{n} - 2. $$

Substituting $c=9/10$:
$Q_{exp} \ge 2(0.9)\sqrt{n} - 2 = 1.8\sqrt{n} - 2$.
Thus, the randomized query complexity is $\Omega(\sqrt{n})$.
\end{proof}
\end{aigenerated}

\FloatBarrier

Building on the lower bound for trees in two rounds, we asked the model to generalize the result to any graph $G$ in two rounds. The  proposed approach defines an arbitrary \emph{spanning tree} $T$ of $G$ rooted at some vertex $r$ and a family of functions $\mathcal{F} = \{f_v \mid v \in V\}$ parameterized by each vertex $v \in V$ that also depends on $T$. In particular, for
 every  $v \in V$, let $f_{v}: V \to \mathbb{Z}$ be such that for all $x \in V$,
\begin{align}
     f_{v}(x) = \begin{cases} -dist_T(r,x) & \text{if } x \preceq_T v \\ \phantom{-}dist_T(r,x) & \text{otherwise} \end{cases}
\end{align}
where $dist_T(r,x)$ represents the distance between $r$ and $x$ in the spanning tree $T$.

 Moreover, we denote by $\mathcal{D}$ the uniform distribution over $\mathcal{F}$. The distribution $\mathcal{D}$ is the one used to obtain the lower bound for any graph and number $t$ of rounds. Unlike the two-round case where the round-1 queries are fixed, the $t \geq 3$ round case requires analyzing adaptive histories.  To handle the adaptivity, the model suggested the notion of  \emph{candidate set}. Suppose $H$ is a history reachable after $i$ rounds. The candidate set, denoted $\mathcal{C}(H)$, represents the vertices of $G$ that could still be local minima given the history $H$. 
\begin{figure}[h!]
    \centering
    \begin{subfigure}[b]{0.45\textwidth}
        \centering
        \includegraphics[width=\linewidth]{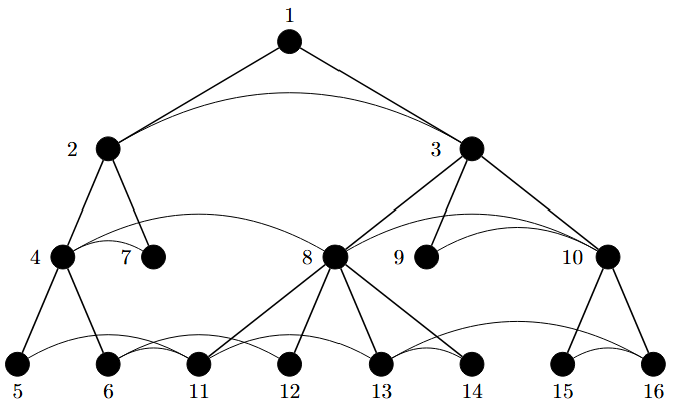}
        \caption{Graph with $n=16$ vertices}
    \end{subfigure}
    \hfill 
    \begin{subfigure}[b]{0.45\textwidth}
        \centering
        \includegraphics[width=\linewidth]{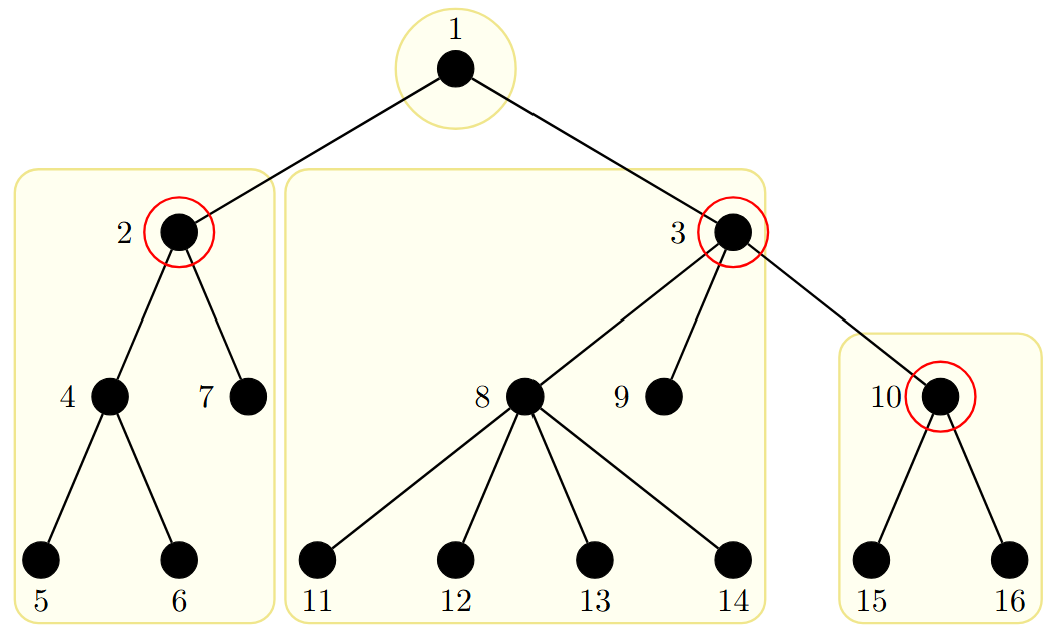}
        \caption{Spanning tree with the candidate sets attainable at the end of round $1$.}
    \end{subfigure}
    
    \caption{The left figure (a) shows a graph $G$. The input function is drawn from $\mathcal{F}$. The right figure (b) shows a spanning tree of $G$ rooted at $1$. The set of queries $Q_1=\{2, 3, 10\}$ submitted in round 1 are circled in red. The partition of candidate sets attainable at the end of round 1 (depending on the answers to the queries) is illustrated using yellow frames.}
    \label{fig:candidate_sets}
\end{figure}

The model observed several properties of the candidate sets, which we  consolidated  into a key lemma stating that at the end of each round $i$:
\begin{itemize}
    \item (a) each attainable candidate  set  is a sub-tree of the original spanning tree $T$ of $G$; and 
    \item (b) the candidate sets partition the graph, and  there is a bijection from the set of histories reachable at the end of round $i$ and the set of candidate sets attainable at the end of round $i$.
\end{itemize}

While the model's intuition was correct, obtaining a rigorous proof for part (b) was more challenging. It initially treated the result as self-evident; when pressed for details, it produced a circular argument (very briefly, to prove the bijection, it  assumed that the candidate sets already form a partition of the set of vertices).
 We tried guiding it toward an inductive proof, and later toward a direct non-inductive one, but the circular argument  persisted in both. The turning point came  when we gave it a hint: assume by contradiction that a vertex lies in two candidate sets, and use part (a) to analyze the roots of the trees underlying the two sets. With this hint, the model  generated a correct proof that is now found in the paper.

\medskip 

The remainder of the $t$-round lower bound argument was correct. Thus we obtained:

\begin{theorem}[Theorem 2 in \cite{BPP26}]\label{thm:lb_t_rounds}
    Let $G = (V,E)$ be a connected undirected graph with $n$ vertices. The randomized query complexity of finding a local minimum on $G$ in $t \in \mathbb{N}^*$ rounds is $\Omega(t n^{1/t} - t)$.
\end{theorem}

We also asked the model to prove a linear lower bound for local search in two rounds on constant-degree expanders. Instead, it  responded by proposing a Parallel Steepest Descent with a Warm Start algorithm. This algorithm has  a query complexity slightly better than linear even for two rounds on any graph with constant maximum degree: $O(n/\log{n})$. The bound given by this algorithm is stated next for any number of rounds.

\begin{proposition}[Proposition 1 in \cite{BPP26}]\label{thm:t_rounds_const_deg_generalized}
    Let $G=(V,E)$ be a graph with $n$ vertices and maximum degree $\Delta$. The randomized query complexity of finding a local minimum in $t \ge 2$ rounds is $O(\sqrt{n} + t)$ when $\Delta \leq 2$ and $O\bigl(\frac{n}{t \cdot \log_\Delta n} + t \Delta^2 \sqrt{n}\bigr)$ when $\Delta \geq 3$.
\end{proposition}

\paragraph{Concluding Remarks.} Our experiment demonstrates that the model can expedite theoretical research when treated as a research assistant. For instance, when guided toward separation numbers, the model developed the divide-and-conquer algorithm and its proof, leveraging the `Shattering Lemma' as a key component for the analysis. This process required vigilance; we had to rigorously audit the model's claims, as illustrated by a circular argument we uncovered in one of the lower bound lemmas. Despite the need for verification, the interaction was very fruitful, yielding a lower bound construction with a surprisingly elegant recursive structure.

\paragraph{Acknowledgements.}
Simina Br\^anzei was supported by US National Science
Foundation grant CCF-2238372. 
Ioannis Panageas was supported by US National Science Foundation grant CCF-2454115.

\newpage

\subsection{Robust Coresets}
\textit{Written by Yi Li, David P. Woodruff, and Xuan Wu.}

\subsubsection*{Problem Context}
In the standard coreset framework, one is given a finite weighted collection of functions $\mathcal{F} = \{(f, w_f )\}$, where each function $f : \mathbb{R}^d \to \mathbb{R}_{\geq 0}$ is associated with a weight $w_f \geq 0$. The loss function for $\mathcal{F}$ is defined as $\mathcal{L}(\mathcal{F};x) = \sum_{f\in\mathcal{F}} w_f f(x)$. An $\eps$-coreset is a weighted subset $\tilde{\mathcal{F}}\subset \mathcal{F}$ such that
\[
(1-\eps)\mathcal{L}(\mathcal{F};x) \leq \mathcal{L}(\tilde{\mathcal{F}};x) \leq (1+\eps)\mathcal{L}(\mathcal{F};x)
\]
holds simultaneously for all possible values of $x$. Classical problems such as $\ell_p$-regression and clustering naturally fit into this framework. A typical approach to construct a coreset is sensitivity sampling, where each function $f\in\mathcal{F}$ is sampled with probability proportional to its sensitivity, defined as
\[
\sigma_{\mathcal{F}}(f) = \sup_{x} \frac{w_f f(x)}{\mathcal{L}(\mathcal{F};x)}.
\]
For example, in $\ell_2$-subspace embedding and regression, these sensitivities coincide with leverage scores. In the more general $\ell_p$-subspace embedding and regression, exact sensitivities are harder to compute and Lewis weights are commonly used as an effective proxy. Similarly, in clustering, recent work~\cite{bansal2024} has shown that sensitivity sampling can achieve near-optimal coreset sizes. We also define the sensitivity of the function set $\mathcal{F}$ as $\sigma(\mathcal{F}) = \sum_{f\in \mathcal{F}} \sigma_{\mathcal{F}}(f)$.

Robust coresets can be studied within the same framework by replacing the standard loss with a trimmed loss. Specifically, for an integer $m\geq 0$, the trimmed loss is defined as
\[
\mathcal{L}^{(m)}(\mathcal{F};x) = \min_{\substack{\mathcal{F}'\subset\mathcal{F}\\ |\mathcal{F}\setminus \mathcal{F}'| \leq m}} \sum_{(f,w_f)\in \mathcal{F}'}^n w_f f(x),
\]
which discards (at most) $m$ largest contributions to the loss. Correspondingly, an $(\eps,m)$-robust coreset is a weighted subset $\tilde{\mathcal{F}}\subset \mathcal{F}$ such that
\[
(1-\eps)\mathcal{L}^{(m)}(\mathcal{F};x) \leq \mathcal{L}^{(m)}(\tilde{\mathcal{F}};x) \leq (1+\eps)\mathcal{L}^{(m)}(\mathcal{F};x)
\]
for all $x$. When $m=0$, this definition reduces to the standard $\eps$-coreset. Prior work~\cite{robust_coreset:icml2025} has shown that if the total sensitivity of $\mathcal{F}$ is bounded by $T$ (i.e.\ $\sigma(\mathcal{F}')\leq T$ for any non-empty subset $\mathcal{F}'\subset \mathcal{F}$), then there exists an $(\eps,m)$-robust coreset of size $O(Tm/\eps\cdot \log(Tm/\eps)) + Q$, where $Q$ denotes the size of a standard $\eps$-coreset.

\subsubsection*{AI Contribution}

AI provided a sharper analysis that eliminates the logarithmic factor, improving the robust coreset size bound to $O(Tm/\eps)+Q$, which is known to be tight.

\subsubsection*{Technical Details}

We first review the algorithms from \cite{robust_coreset:icml2025} and the associated guarantees before presenting the tighter analysis produced by AI. 

\begin{minipage}[t]{0.47\linewidth}
\begin{algorithm}[H]
	\caption{$\mathrm{Uniform}(A,\varepsilon, m)$}
	\label{alg:robust_coreset:alg_uniform}
	\begin{algorithmic}[1]
		\Require A set $A$ of functions, parameters $\varepsilon$ and $m$
		\Ensure A subset $D\subseteq A$
		\State $B\leftarrow \emptyset$
		\State for each $f\in A$, with probability $\frac{1}{m}$, add $f$ to $B$ \label{alg_uniform_generate_B}
		\State for each $f\in B$, compute the sensitivity $\sigma_B(f)$ \label{alg_uniform_compute_sensitivity}
		\State $D\leftarrow \{f\in B:\sigma_B(f)\geq \frac{\varepsilon}{4}\}$ \label{alg_uniform_generate_D}
		\State \Return $D$
	\end{algorithmic}
\end{algorithm}
\end{minipage}
\hfill
\begin{minipage}[t]{0.47\linewidth}
\begin{algorithm}[H]
	\caption{$\mathrm{Refine}(D,\varepsilon, m)$}
	\label{alg:robust_coreset:alg_refine}
	\begin{algorithmic}[1]
		\Require A coreset $D$, parameters $\varepsilon$ and $m$
		\Ensure A refined subset $\tilde{D}$ adapted for the robust optimization problem
		\State $\tilde{D}\leftarrow \emptyset$
		\For{$(f,\omega_f)\in D$}
		\State compute the sensitivity $\sigma_D(f)$
		\State $n_f\leftarrow \lceil \frac{m}{\varepsilon}\cdot \sigma_D(f)\rceil$
		\State Add $n_f$ copies of $(f, \frac{\omega_f}{n_f})$ to $\tilde{D}$
		\EndFor
		\State \Return $\tilde{D}$
	\end{algorithmic}
\end{algorithm}
\end{minipage}

\begin{minipage}[t]{0.47\linewidth}
	\begin{algorithm}[H]
		\caption{$\mathrm{Coreset}(A,\varepsilon, m)$}
		\label{alg:robust_coreset:alg_coreset}
		\begin{algorithmic}[1]
			\Require A set $A$ of functions, parameters $\varepsilon$ and $m$, and an algorithm $\mathrm{Vanilla}(A)$ to construct an $\varepsilon$-coreset for $A$
			\Ensure An $(\varepsilon,m)$-robust coreset for $A$
			\State $S\leftarrow \emptyset$
			\State $R\gets \Theta(m\log\frac{Tm}{\eps})$
			\For{$i=1,2,\cdots,R$} \label{alg_coreset_forloop}
			\State $D\leftarrow \mathrm{Uniform}(A,\varepsilon,m)$
			\State $S\leftarrow S\cup D$
			\EndFor
			\State $V\leftarrow \mathrm{Vanilla}(A\setminus S)$ \label{alg_coreset_vanilla}
			\State $\tilde{S}\leftarrow\{(f,1):f\in S\}$ \label{alg_coreset_S_tilde}
			\State \textbf{Return} $\tilde{S}\cup \mathrm{Refine}(V,\varepsilon,m)$. \label{alg_coreset_last}
		\end{algorithmic}
	\end{algorithm}
\end{minipage}
\hfill
\begin{minipage}[t]{0.47\linewidth}
	\begin{algorithm}[H]
		\caption{$\mathrm{ModifiedCoreset}(A,\varepsilon, m)$}
		\label{alg:robust_coreset:alg_coreset_new}
		\begin{algorithmic}[1]
			\Require A set $A$ of functions, parameters $\varepsilon$ and $m$, and an algorithm $\mathrm{Vanilla}(A)$ to construct an $\varepsilon$-coreset for $A$
			\Ensure An $(\varepsilon,m)$-robust coreset for $A$
			\State $R \gets 40m\ln(2000Tm/\eps)$
			\For{$i=1,2,\cdots,R$} 
			\State $D_i\leftarrow \mathrm{Uniform}(A,\varepsilon,m)$
                \For{$f \in D_i$}
                    \State $c_f \leftarrow c_f + 1$
                \EndFor
            \EndFor
			\item $S \gets \{f \in A : c_f \ge \frac{R}{10m}\}$
			\State $V\leftarrow \mathrm{Vanilla}(A\setminus S)$ 
			\State $\tilde{S}\leftarrow\{(f,1):f\in S\}$ 
			\State \textbf{Return} $\tilde{S}\cup \mathrm{Refine}(V,\varepsilon,m)$. 
		\end{algorithmic}
	\end{algorithm}
\end{minipage}

The algorithms from \cite{robust_coreset:icml2025} are reproduced in Algorithms~\ref{alg:robust_coreset:alg_uniform} to \ref{alg:robust_coreset:alg_coreset}, where Algorithm~\ref{alg:robust_coreset:alg_coreset} outputs the final robust coreset. A function $f\in A$ is called \emph{contributing} if there exists $x\in\R^d$ such that $f(x)\geq \frac{\varepsilon}{m}\cdot \mathcal{L}^{(m)}(A; x)$. The following guarantee of Algorithm~\ref{alg:robust_coreset:alg_uniform} was established in Lemma 4.3 of \cite{robust_coreset:icml2025}.
\begin{lemma} \label{lem:robust_coreset:lemma_uniform}
	Assume that $f$ is contributing, then with probability at least $\frac{1}{5m}$, the set returned by $\mathrm{Uniform}(A,\varepsilon,m)$ contains $f$.
\end{lemma}
It was then argued in \cite{robust_coreset:icml2025} that the number of contributing functions in $A$ is $O(\frac{Tm}{\varepsilon}\cdot \log \frac{Tm}{\varepsilon})$. 
Consequently, Algorithm~\ref{alg:robust_coreset:alg_coreset} executes $\Theta(m\log(Tm/\eps))$ rounds of Algorithm~\ref{alg:robust_coreset:alg_uniform} in order to capture all contributing functions with high probability, thereby forming the basis of the robust coreset that can handle outliers.

Now we describe our interactions with AI. We first entered the following prompt.
\begin{chatbox}
\begin{userturn}
\textbf{Prompt 1.}
Can you resolve the logarithmic factor gap in the coreset construction in the attached paper? Feel free to give both a tighter analysis as well as to change the algorithmic framework in the paper in order to achieve this improved bound.

<generic instructions on being mathematically rigorous>
\end{userturn}
\begin{assistantturn}
\textbf{Gemini.}
Yes, the logarithmic factor gap in the coreset construction size can be resolved. The paper presents a coreset of size $O(K \log K) + Q$, where $K=Tm/\varepsilon$. We demonstrate that this can be improved to $O(K) + Q$ in expectation and with constant probability.

\hspace{0.5cm} This improvement is achieved in two steps: first, by providing a tighter analysis of the number of contributing functions, and second, by modifying Algorithm 3 to filter out non-contributing functions more effectively using frequency estimation.

\hspace{0.5cm}...
\end{assistantturn}
\end{chatbox}

The AI produces the following bound on the number of contributing functions, which already improves the argument in \cite{robust_coreset:icml2025}.

\begin{lemma} \label{lem:robust_coreset:number_of_contributing_functions}
The number of contributing functions in $A$ is $20\frac{Tm}{\eps}$.
\end{lemma}

\begin{proof}
Let $C$ be the set of contributing functions, $N=|C|$. Let $D$ be the output of a single execution of Algorithm \ref{alg:robust_coreset:alg_uniform}.
Algorithm 1 returns $D = \{f \in B : \sigma_B(f) \geq \varepsilon/4\}$. By the assumption of strong total sensitivity $T$, we have $\sigma_B \leq T$.
Therefore, $|D| \cdot (\varepsilon/4) \leq \sum_{f \in D} \sigma_B(f) \leq \sigma_B \leq T$.
This implies $|D| \leq 4T/\varepsilon$ deterministically. Let $W = 4T/\varepsilon$.
Consequently, the expected size of $D$ is bounded: $\mathbb{E}[|D|] \leq W$.

Let $p_f = \Pr(f \in D)$. By \ref{lem:robust_coreset:lemma_uniform}, if $f \in C$, then $p_f \geq 1/(5m)$.
By linearity of expectation:
\[
\mathbb{E}[|D|] = \sum_{f \in A} p_f \geq \sum_{f \in C} p_f \geq N \cdot \frac{1}{5m}.
\]
Combining the inequalities:
\[
\frac{N}{5m} \leq W \implies N \leq 5mW = 5m \frac{4T}{\varepsilon} = \frac{20Tm}{\varepsilon}. \qedhere
\]
\end{proof}

The AI then proceeds with a modified version of Algorithm~\ref{alg:robust_coreset:alg_coreset} by tracking the frequency of each function in $S$ and discarding the ones with low frequencies. It then presents a theorem as follows.

\begin{theorem}[Improved Coreset Construction]
    Algorithm~\ref{alg:robust_coreset:alg_coreset}-Modified constructs an $(\varepsilon, m)$-robust coreset for $A$ of expected size $O(K) + Q$. The construction succeeds with high probability (at least $1-O(1/n)$).
\end{theorem}

The proof is correct, though we omit it here, since the theorem states an expected size of the coreset, which is not desirable. We further make the following prompt.

\begin{chatbox}
\begin{userturn}
\textbf{Prompt 2.}
It looks correct, but the Theorem only gives the size in expectation. The earlier version had a high-probability bound, which shouldn’t be hard to get. Maybe you can prompt it to produce an upper bound that holds with high probability? Actually we need only probability $\geq 0.99$ instead of $1-1/n$.
\end{userturn}
\begin{assistantturn}
\textbf{Gemini.}
We demonstrate that the logarithmic factor in the coreset size bound of Theorem 4.1 in the paper "Robust Sparsification via Sensitivity" can be removed. The original paper presents a coreset of size $O(K \log K) + Q$, where $K=Tm/\varepsilon$. We show that a modification of the algorithm yields a coreset of size $O(K) + Q$ with high probability (at least 0.99).

\hspace{0.5cm}...

\textbf{Key Properties of Algorithm~\ref{alg:robust_coreset:alg_uniform}:}
\begin{enumerate}
    \item \textbf{Deterministic Size Bound:} The output $D$ always satisfies $|D| \leq W$. This is because $\sum_{f \in D} \sigma_B(f) \leq \sigma_B \leq T$. Since $\sigma_B(f) \geq \varepsilon/4$ for $f \in D$, we have $|D| \cdot (\varepsilon/4) \leq T$, so $|D| \leq W$.
    \item \textbf{Expected Size Bound:} $\E[|D|] \leq W$.
    \item \textbf{Capture Probability (Lemma 4.3):} If $f \in C$, the probability $p_f = P(f \in D)$ satisfies $p_f \geq 1/(5m)$.
\end{enumerate}

\hspace{0.5cm}...
\end{assistantturn}
\end{chatbox}

AI returns a correct version of the theorem.
\begin{theorem}
The Modified Algorithm constructs an $(\varepsilon, m)$-robust coreset for $A$ of size $O(K) + Q$ with probability at least 0.99.
\end{theorem}
The proof is correct, though somewhat suboptimal. It uses Markov’s inequality to bound the total number of sampled functions over $R$ rounds, which is unnecessary. This suboptimality arises from a weaker property generated by the AI, as presented in its response above. In particular, Point 1 already provides a deterministic bound $|D| \leq W$, rendering the expected bound in Point 2 unnecessary.

Therefore, we further prompt the AI to improve the argument.
\begin{chatbox}
\begin{userturn}
\textbf{Prompt 3.}
This looks correct though it is a bit suboptimal when bounding the size. Specifically, it says
<quote of the properties>.

In the first bullet point, it already obtains that $|D|\leq W$, so the second bullet point is unnecessary. In the subsequent proof of Part 2 of the theorem, it could directly use the deterministic bound $|D|\leq W$ instead of the expectation bound. This will further tighten the analysis a bit.
\end{userturn}
\begin{assistantturn}
\textbf{Gemini.}
<latex source>
\end{assistantturn}
\end{chatbox}

The results are correct and are reproduced below with minor modifications to fit the present context. The algorithm generated by the AI is presented in Algorithm~\ref{alg:robust_coreset:alg_coreset_new}, with slight simplifications.

\begin{theorem} \label{thm:robust_coreset:improved}
	Let $\varepsilon\in (0,\frac{1}{2})$ and $\mathcal{F}$ be a unit-weight function set with total sensitivity $T$. Suppose that there exists an algorithm that computes a vanilla $\varepsilon$-coreset for $F$ of size $Q$. 
	Then, Algorithm~\ref{alg:robust_coreset:alg_coreset_new} computes an $(\varepsilon,m)$-robust coreset for $F$ of size $O(\frac{Tm}{\varepsilon})+Q$, with probability at least $0.99$.
\end{theorem}

\begin{proof}
\textbf{Size bound.} We analyze the size of the set $S$ returned by the algorithm.
Let $C_{\mathrm{total}}$ be the total count of all functions captured across all $R$ iterations.
\[ 
C_{\mathrm{total}} = \sum_{i=1}^{R} |D_i|. 
\]
By the deterministic size bound property of Algorithm~\ref{alg:robust_coreset:alg_uniform}, we know $|D_i| \leq W := 4T/\eps$ for all $i$.
Therefore, $C_{\mathrm{total}}$ is deterministically bounded:
\[ C_{\mathrm{total}} \leq \sum_{i=1}^{R} W = R W. \]

Now we relate the size of $S$ to $C_{\mathrm{total}}$. By definition, $f \in S$ if and only if $c_f \geq R\theta$ for $\theta=1/(10m)$.
\[
C_{\mathrm{total}} = \sum_{f \in A} c_f \geq \sum_{f \in S} c_f \geq \sum_{f \in S} (R \theta) = |S| \cdot R \theta. 
\]
Therefore, we have a deterministic bound on $|S|$:
\[ 
|S| \leq \frac{C_{\mathrm{total}}}{R\theta} \leq \frac{R W}{R \theta} = \frac{W}{\theta} = \frac{40Tm}{\eps}.
\]
Thus, $|S| \leq 40K$, where $K = Tm/\eps$. The size bound $|S| = O(K)$ holds deterministically.

\textbf{Correctness.} We want to bound the probability that a contributing function $f \in C$ is missed ($f \notin S$). The count $c_f$ is a sum of $R$ independent indicators. If $f\in C$, we have $\E[c_f] \geq R/(5m)$. If $f\not\in S$, it means that $c_f < R/(10m)$. Use the Chernoff lower tail bound $P(X \le \mu/2) \le \exp(-\mu/8)$:
\begin{multline*}
\Pr(f \notin S) = \Pr(c_f < R\theta) = \Pr\left(c_f \le \frac{1}{2}\E[c_f]\right) \le \exp\left(-\frac{R/(5m)}{8}\right) \\
= \exp\left(-\frac{40m \ln(2000K)}{40m}\right) = \frac{1}{2000K}.
\end{multline*}

We apply a union bound over $C$. By Lemma~\ref{lem:robust_coreset:number_of_contributing_functions}, $|C| \le 20K$.
\[ 
\Pr(C \not\subseteq S) = \Pr(\exists f \in C: f \notin S) \le \sum_{f \in C} \Pr(f \notin S) \le 20K \cdot \frac{1}{2000K} = \frac{1}{100} = 0.01. 
\]

The algorithm deterministically produces a set $S$ of size $|S| \leq 40K$. With probability at least $1 - 0.01 = 0.99$, the algorithm ensures $C \subseteq S$. The remainder of the correctness proof follows exactly as in Theorem 4.1 of \cite{robust_coreset:icml2025}.
\end{proof}
\newpage

\subsection{Submodular Function Maximization in a Stream}
\label{sec: streaming_submodular}
\textit{Written by Ola Svensson.}

\subsubsection*{Problem context}
We consider \emph{monotone submodular maximization} under a cardinality constraint in the \emph{adversarial injection streaming model} introduced in the paper ``Robust Algorithms under Adversarial Injections''~\cite{GargKRS20}.
In this model, the stream consists of an (unknown) set $E_{\mathrm{good}}$ of ``genuine'' elements arriving in \emph{uniformly random order}, interleaved with an arbitrary set $E_{\mathrm{noise}}$ of \emph{injected} elements placed by an adversary. The motivation of this semi-random model is to design algorithms in a beyond-worst-case model while not overfitting to the assumption (such as all elements arrive in a uniformly at random order).
The goal is to output a set $S \subseteq E_{\mathrm{good}} \cup E_{\mathrm{noise}}$ of size at most $k$ with value close to
\[
\OPT := \max_{|T|\le k,\; T\subseteq E_{\mathrm{good}}} f(T),
\]
where $f$ is a nonnegative monotone submodular function given by value-oracle access. We remark that it is important that $\OPT$ is defined with respect to the optimal solution of elements in $E_{\mathrm{good}}$. Indeed, if $\OPT$ were also allowed to contain elements of $E_{\mathrm{noise}}$ then the model would be equivalent to a fully adversarial one.  

The above paper proposes a streaming algorithm for this setting (the \emph{tree algorithm} for monotone submodular maximization), achieving an approximation factor around $0.55$ while storing a number of elements independent of $|E_{\mathrm{good}}|$ (but exponential in $k$).
Two concrete questions were left open for submodular maximization in this model:
(i) whether one can reach the offline-optimal constant $(1-1/e)$, and
(ii) whether one can reduce memory to $\poly(k)$ elements.

\subsubsection*{The tree algorithm and where the analysis bottleneck appears}
At a high level, the tree algorithm maintains a rooted tree of depth $k$ whose root-to-leaf paths represent candidate size-$\le k$ solutions. When a new element $e$ arrives, the algorithm considers attaching $e$ beneath existing nodes based on its marginal contribution.
The tree is kept from exploding by \emph{merging equal-marginal states}: within each level, only one representative for each marginal-gain value is retained.
This keeps the algorithm \emph{stream-length independent}, but in the worst case still requires exponentially many stored elements as a function of $k$.

The original approximation analysis proceeds by tracking a carefully chosen leaf as the $k$ (unknown) optimal elements from $E_{\mathrm{good}}$ appear in the stream.
A central technical device is a \emph{threshold parameter} $t$ used in a case split comparing (i) the best ``available'' marginal gain since the last optimal element appeared, versus (ii) the marginal gain of the next unseen optimal element.
This case split yields a recurrence $R(k,h)$ that lower bounds the fraction of $\OPT$ captured after ``progress'' through $h$ optimal elements with budget $k$, with the final guarantee given by $R(k,k)$.
In the paper, $t$ is chosen \emph{globally} (independent of the state $(k,h)$); numerically, setting $t\approx 0.8$ yields $R(k,k)\approx 0.55067$.

A key insight by Gemini---which ultimately enabled our improvement---is that $t$ is \emph{not an algorithm parameter}.
It only appears in the \emph{analysis}, meaning we are free to choose a different threshold in different states as long as the inequalities used in the proof remain valid.
\\\\
%
\textbf{AI Contribution}\\
The AI identified a latent degree of freedom in the existing algorithmic analysis. By replacing a global threshold with a state-dependent threshold, the model autonomously derived and proved an optimized recurrence relation, strictly improving the algorithm's approximation ratio from $\approx 0.55$ to $2-\sqrt{2}$.

We used Gemini to explore whether the analysis could be sharpened.

Gemini did {not} produce a new algorithm achieving $(1-1/e)$, nor did it find a way to reduce the memory requirement to $\poly(k)$ elements in this adversarial injection setting.
However, it \emph{did} identify a clean way to improve the \emph{existing} analysis of the tree algorithm: replace the single global threshold $t$ by a \emph{state-dependent threshold} $t_{k,h}$ in the recurrence.

Concretely, Gemini suggested using
\[
t_{k,h} \ :=\ \frac{k}{\sqrt{k^2+h^2}},
\]
and first predicted and then formally proved that the corresponding optimized recurrence achieves the exact constant
\[
R(k,k)\ =\ 2-\sqrt{2}\ \approx\ 0.585786,
\]
thereby improving the analyzed approximation ratio of the same underlying algorithm from $\approx 0.55067$ to $\approx 0.58579$.

The remainder of this testimonial highlights the four prompts that led from ``idea'' to a complete, algorithm-level guarantee.

\subsubsection*{Prompt 1: exploratory questions, and the key analysis insight}
\begin{figure}[!htb]
\begin{chatbox}
\begin{userturn}
\textbf{Prompt 1.}
Please investigate if you can improve the submodular part of the attached paper.
Specifically:
\begin{itemize}[leftmargin=*]
    \item Can we get $(1-1/e)$ for submodular maximization?
    \item Can we only store $\poly(k)$ elements instead of exponentially many in $k$?
\end{itemize}
Be mathematically rigorous.
\end{userturn}

\begin{assistantturn}
\textbf{Gemini (summary).}
The model did not produce a new algorithm achieving $(1-1/e)$, and it did not resolve how to store only $\poly(k)$ elements.
However, it identified that the \emph{analysis recurrence} in the paper can be strengthened by letting the threshold depend on the state: instead of a global $t$, introduce $t_{k,h}$.
It further hypothesized that the best constant achievable within this recurrence framework is
$2-\sqrt{2}$, and that a natural candidate is a closed-form threshold of the form
$t_{k,h} = t(h/k)$.
\end{assistantturn}
\end{chatbox}
\caption{The first prompt where Gemini identified the direction for improvement and the improved guarantee.}
\label{fig:first_prompt_submodular_streaming}
\end{figure}

In the first interaction with Gemini (\cref{fig:first_prompt_submodular_streaming}), we gave a fairly general prompt with two open questions. 
Even though the response did not resolve the two open algorithmic questions, it pinpointed a
\emph{previously unexplored degree of freedom} in the analysis and proposed (correctly) that exploiting it improves the approximation factor.

\subsubsection*{Prompt 2: a rigorous (but asymptotic) derivation via continuous limits}
\newpage
\begin{figure}[!htb]
\begin{chatbox}
\begin{userturn}
\textbf{Prompt 2.}
Can you rigorously improve it to obtain a $2-\sqrt{2}$ approximation, as suggested?
\\\\
<math rigor boilerplate>
\end{userturn}

\begin{assistantturn}
\textbf{Gemini (summary).}
Gemini produced a rigorous argument in a continuous/asymptotic regime (as $k\to\infty$) by recasting the recurrence optimization as a minimax problem that leads to differential inequalities/ODE-type conditions.
While this was not yet the discrete proof we needed, it strongly indicated the correct \emph{functional form} of the optimal thresholds, and guided us to the explicit choice
\[
t_{k,h}=\frac{k}{\sqrt{k^2+h^2}}.
\]
This directly set up Prompt 3, where we asked for a fully discrete proof.
\end{assistantturn}
\end{chatbox}
\caption{In the second prompt, we asked for  a rigorous proof of the improved guarantee. }
\label{fig:second_prompt_submodular_streaming}
\end{figure}
\FloatBarrier

In the second prompt, \cref{fig:second_prompt_submodular_streaming}, we asked for a rigorous proof and Gemini responded with an argument in the asymptotic regime. Even though the ODE argument was ``not exactly what we asked for'', it served as a highly effective
\emph{derivation tool}: it revealed the right closed-form threshold and the right constant before we invested time in a discrete induction.

\subsubsection*{Prompt 3: a complete discrete proof of the optimized recurrence}
We next asked Gemini to prove the optimized recurrence formally, with the thresholds fixed to the candidate form, see \cref{fig:third_prompt_submodular_streaming} for the third prompt.

\begin{figure}[!htb]
\begin{chatbox}
\begin{userturn}
\textbf{Prompt 3 (condensed).}
Define $t_{k,h}=\frac{k}{\sqrt{k^2+h^2}}$ and the recurrence
\[
R(k,h)=\min\{A_{k,h},B_{k,h},C_{k,h}\}
\]
with
\begin{align*}
A_{k,h} & :=\frac{t_{k,h}}{k}+\Bigl(1-\frac{t_{k,h}}{k}\Bigr)R(k,h-1),\\
B_{k,h} & :=\frac{1}{k}+\Bigl(1-\frac{1+t_{k,h}}{k}\Bigr)R(k-1,h-1), \\
C_{k,h} & :=\frac{1}{1+t_{k,h}},
\end{align*}
and base $R(k,0)=0$.
Prove that $R(k,k)=2-\sqrt{2}$ for all integers $k\ge 1$.
\end{userturn}

\begin{assistantturn}
\textbf{Gemini (summary).}
Gemini proposed an explicit closed-form candidate $P(k,h)$ (as a function of $h/k$) and showed it satisfies the recurrence by induction, yielding the exact identity
$R(k,k)=2-\sqrt{2}$ for all $k\ge 1$.
\end{assistantturn}
\end{chatbox}
\caption{In our third prompt, we asked for and received a complete proof of the discrete recurrence.}
\label{fig:third_prompt_submodular_streaming}
\end{figure}
\FloatBarrier

At this point, Gemini had delivered an almost complete proof of the mathematical core: the optimal solution of the improved recurrence.
This was the key missing technical ingredient needed to upgrade the approximation factor.

\subsubsection*{Prompt 4: connecting the recurrence back to the tree algorithm (and a hidden monotonicity issue)}
Finally, we asked Gemini to connect the recurrence to the original algorithm analysis, see~\Cref{fig:fourth_prompt_submodular_streaming}.
This step is subtle because the original paper uses a \emph{global} $t$ in an inequality of the form
$R(k-1,h-1)\le 1/(1+t)$, whereas our thresholds vary with $(k,h)$.

\newpage

\begin{figure}
\begin{chatbox}
\begin{userturn}
\textbf{Prompt 4 (condensed).}
Please relate the recurrence with state-dependent thresholds $t=t_{k,h}$ to the analysis of the tree algorithm.
In particular, in the analogue of Lemma~4 from the paper, the proof used
$R(k-1,h-1)\le 1/(1+t)$ for a global $t$.
Now $t$ depends on $(k,h)$, so this step requires care.
Please write step-by-step how $R(k,h)$ bounds the approximation guarantee, and identify what additional property is needed.
\end{userturn}
\begin{assistantturn}
\textbf{Gemini (summary).}
Gemini provided a step-by-step mapping from the algorithm's case analysis to the three terms
$A_{k,h}$, $B_{k,h}$, and $C_{k,h}$, showing how $R(k,h)$ lower bounds the value of a carefully chosen leaf as the optimal elements appear.
Crucially, it identified that the ``global-$t$'' inequality is replaced by a requirement of the form
\[
R(k-1,h-1)\ \le\ \frac{1}{1+t_{k,h}},
\]
which does \emph{not} automatically hold for arbitrary state-dependent thresholds.
Gemini then pinpointed a sufficient condition: a monotonicity relation along the diagonal,
\[
t_{k,h}\ \le\ t_{k-1,h-1},
\]
which implies $\frac{1}{1+t_{k-1,h-1}} \le \frac{1}{1+t_{k,h}}$ and allows the induction to go through.
It also proved this monotonicity for the chosen optimal thresholds $t_{k,h}=\frac{k}{\sqrt{k^2+h^2}}$.
\end{assistantturn}
\end{chatbox}
\caption{In the fourth prompt, we asked Gemini to correctly relate the improved analysis of the recurrence to the improved analysis of the algorithm.}
\label{fig:fourth_prompt_submodular_streaming}
\end{figure}
\FloatBarrier

This last step was important for correctness: it is easy to ``solve'' an abstract recurrence while accidentally breaking an inequality needed to justify the recurrence from the algorithm.
Gemini not only identified the missing monotonicity condition but also proved it for the optimal threshold schedule, thereby completing the logical chain from the recurrence back to the streaming algorithm.

\subsubsection*{Outcome and takeaway}
The outcome of this interaction was a strictly improved approximation guarantee for the
\emph{same} tree algorithm in the adversarial injection streaming model:
\[
\text{approximation ratio improved from }\approx 0.55067\text{ to }2-\sqrt{2}\approx 0.58579,
\]
achieved  by refining the analysis (via a sharper, state-dependent choice of thresholds).

From a ``how to use Gemini'' perspective, this case study highlights a valuable contribution:
even when a model does not solve the headline open problems (here, $(1-1/e)$ or $\poly(k)$ memory),
it can still make a decisive contribution by (i) identifying a latent degree of freedom in an existing proof,
(ii) guessing the right closed form for an optimum constant/parameter schedule, and
(iii) producing formal proofs with minimal guidance.
\\\\
{\bf The Formal Improved Analysis by Gemini}
\\\\
The following is the complete proof that was automatically generated by Gemini. We have kept the writing exactly as automatically produced.

\subsubsection*{Preliminaries and Notation}

\paragraph{Submodularity.}
A function $f$ is submodular if for all $S \subseteq T \subset E$ and $e \in E \setminus T$, the marginal gain $f(e|S) := f(S \cup \{e\})-f(S)$ satisfies $f(e|S) \ge f(e|T)$. It is monotone if $f(S) \le f(T)$ for $S \subseteq T$.

\begin{lemma}[Average Marginal Bound]\label{lem:avg_marginal}
Let $O$ be a set of size $k$ and $S$ be any partial solution. For a monotone submodular function $f$, the expected marginal gain of a random element $o \in O$ satisfies:
$$ \mathbb{E}_{o \sim O}[f(o \mid S)] \ge \frac{f(O \cup S) - f(S)}{k} $$
If $S=\emptyset$ and $O=E_{opt}$, then $\mathbb{E}_{o \sim E_{opt}}[f(o)] \ge \frac{OPT}{k}$.
\end{lemma}
\begin{proof}
This follows from the definition of submodularity. The sum of the marginals of elements in $O$ added to $S$ is at least the total gain $f(O \cup S) - f(S)$. Taking the average yields the result.
\end{proof}

\paragraph{Stream Notation}
Let $\sigma$ be the input stream. $E_{opt} = \{o_1, \dots, o_k\}$. Let $\pi$ be the random permutation of $E_{opt}$. $o^\pi_i$ is the $i$-th element of $E_{opt}$ to arrive in the stream. $O^\pi_i = \{o^\pi_1, \dots, o^\pi_i\}$. $OPT = f(E_{opt})$.

\subsubsection*{The Tree Algorithm}
We briefly review the Tree Algorithm from \cite{GargKRS20}. For simplicity, we assume the algorithm knows $OPT$ and that the set of possible marginal gains $I$ is small (these assumptions can be removed using standard techniques involving discretization and guessing, as detailed in the appendix of \cite{GargKRS20}).

The algorithm maintains a rooted tree $\mathcal{T}$ of height at most $k$. Each node represents a partial solution (the set of elements on the path from the root to the node).

\begin{algorithm}
\caption{Tree Algorithm}
\begin{algorithmic}[1]
\State Initialize tree $\mathcal{T}$ with only the root (representing $S=\emptyset$).
\For{each element $e$ arriving in the stream $\sigma$}
    \For{each node $v$ in $\mathcal{T}$}
        \State Let $S_v$ be the solution corresponding to $v$.
        \If{$|S_v| < k$}
            \State Calculate the marginal gain $g = f(e|S_v)$.
            \If{$v$ has no child $c$ such that $f(c|S_v) = g$}
                \State Add a new child node corresponding to $S_v \cup \{e\}$.
            \EndIf
        \EndIf
    \EndFor
\EndFor
\State Return the solution $S_{ALG}$ corresponding to the leaf with maximum value $f(S_{ALG})$.
\end{algorithmic}
\end{algorithm}

\subsubsection*{Analysis via Generalized Recurrence}
The core of the analysis involves defining a specific path in the execution tree and lower bounding its expected value. We generalize the approach by introducing a state-dependent threshold.

\paragraph{The Generalized Recurrence Relation}
We define the domain $\mathcal{D} := \{(k, h) \in \mathbb{Z}_{\ge 0}^2 : 0 \le h \le k\}$. We introduce a generalized recurrence $R(k, h)$ which aims to lower bound the approximation ratio when the optimal solution has size $k$ and the algorithm's budget is $h$.

\begin{definition}[Threshold Function]
We define a threshold function $T: \mathcal{D} \setminus \{(0,0)\} \to (0, 1]$.
\end{definition}

\begin{definition}[Generalized Recurrence]
Let $R: \mathcal{D} \to \mathbb{R}_{\ge 0}$ be defined as follows: $R(k, 0) := 0$ for all $k \ge 0$. For $(k, h) \in \mathcal{D}$ with $k \ge 1, h \ge 1$, let $t = T(k, h)$.
$$ R(k, h) := \min \{A_{k,h}, B_{k,h}, C_{k,h}\} $$
where
\begin{align*}
A_{k,h} &= \frac{t}{k} + \left(1 - \frac{t}{k}\right)R(k, h-1), \\
B_{k,h} &= \frac{1}{k} + \left(1 - \frac{1+t}{k}\right)R(k-1, h-1), \\
C_{k,h} &= \frac{1}{1+t}.
\end{align*}
\end{definition}

\paragraph{Adapted Path Selection Strategy}
To utilize this recurrence, the analysis employs a path selection strategy that dynamically uses the threshold corresponding to the current subproblem. This strategy defines a path in the tree whose expected value we analyze.

We define the path $S^\pi_h = \{s^\pi_1, \dots, s^\pi_h\}$ inductively. We track the state $(k_i, h_i, n_i)$: remaining OPT size, remaining budget, and index of the current OPT element. Initialize $(k_1, h_1, n_1) = (k, h, 1)$.

At step $i$ (for $i=1, \dots, h$):
1. Set the threshold $t_i = T(k_i, h_i)$.
2. Define $u^\pi_i$ (the best non-OPT element seen so far) and $v^\pi_i$ (the better of $u^\pi_i$ and the current OPT element $o^\pi_{n_i}$) relative to $S^\pi_{i-1}$ (as defined in \cite{GargKRS20}).
3. Apply the decision rule:
\begin{itemize}
    \item \textbf{Case 1 (Greedy Choice):} If $\mathbb{E}_\pi[f(u^\pi_i|S^\pi_{i-1})] \ge t_i \cdot \mathbb{E}_\pi[f(o^\pi_{n_i}|S^\pi_{i-1})]$. Set $s^\pi_i = u^\pi_i$. Update: $(k_{i+1}, h_{i+1}) = (k_i, h_i-1)$, $n_{i+1} = n_i$.
    \item \textbf{Case 2 (Wait for OPT):} Otherwise. Set $s^\pi_i = v^\pi_i$. Update: $(k_{i+1}, h_{i+1}) = (k_i-1, h_i-1)$, $n_{i+1} = n_i+1$.
\end{itemize}

\paragraph{Connecting the Recurrence to the Algorithm}
We now prove that $R(k, h)$ bounds the performance of this strategy, provided the threshold function satisfies a monotonicity property.

\begin{definition}[Monotonic Threshold Condition (MTC)]
A threshold function $T(k, h)$ satisfies the Monotonic Threshold Condition (MTC) if for all $k \ge 2, h \ge 1$:
$$ T(k, h) \le T(k-1, h-1). $$
\end{definition}

\begin{lemma}\label{lem:mtc_bound}
If $T(k, h)$ satisfies the MTC, then for all $k \ge 1, h \ge 1$, letting $t = T(k, h)$:
$$ R(k-1, h-1) \le \frac{1}{1+t}. $$
\end{lemma}
\begin{proof}
If $k=1$, then $h=1$. $R(0,0)=0$. Since $t>0$, the inequality holds. Assume $k \ge 2$. By definition of the recurrence:
$$ R(k-1, h-1) \le C_{k-1,h-1} = \frac{1}{1+T(k-1, h-1)}. $$
By the MTC, $T(k, h) \le T(k-1, h-1)$. Thus $1+T(k, h) \le 1+T(k-1, h-1)$.
$$ R(k-1, h-1) \le \frac{1}{1+T(k-1, h-1)} \le \frac{1}{1+T(k, h)} = \frac{1}{1+t}. $$
\end{proof}

We can now prove the main connection theorem.

\begin{theorem}\label{thm:connection}
If the threshold function $T(k, h)$ satisfies the MTC, then for all instances of the problem and $(k, h) \in \mathcal{D}$, the solution $S^\pi_h$ defined by the adapted strategy satisfies:
$$ \mathbb{E}_\pi[f(S^\pi_h)] \ge R(k, h) \cdot OPT. $$
\end{theorem}

\begin{proof}
We proceed by induction on $h$.

\textbf{Base Case (h=0):} $R(k, 0) = 0$. $\mathbb{E}_\pi[f(S^\pi_0)] = f(\emptyset) = 0$. The theorem holds.

\textbf{Inductive Step:} Assume the theorem holds for $h-1$. We prove it for $(k, h)$, where $k \ge 1, h \ge 1$. Let $t = T(k, h)$. We analyze the first step ($i=1$). By Lemma \ref{lem:avg_marginal}, due to the random permutation, $\mathbb{E}_\pi[f(o^\pi_1)] \ge OPT/k$.

\textbf{Case 1: $\mathbb{E}_\pi[f(u^\pi_1)] \ge t \cdot \mathbb{E}_\pi[f(o^\pi_1)]$.} We set $s^\pi_1 = u^\pi_1$.

As argued in \cite{GargKRS20}, $u^\pi_1$ is determined by the elements arriving before the first OPT element. Since the adversary injects noise without knowledge of the permutation $\pi$, the set of elements in this prefix is independent of $\pi$. Therefore, $u^\pi_1$ is independent of $\pi$. Let $s_1 = u^\pi_1$.

We consider the remaining subproblem $I'$. It has $k'=k$ OPT elements and budget $h'=h-1$. The function is $f'(T) = f(T|s_1)$. The optimum of the subproblem is $OPT' = f(E_{opt}|s_1)$. By submodularity and monotonicity, $OPT' \ge OPT - f(s_1)$.

The remainder of the path $S^\pi_h \setminus \{s_1\}$ corresponds exactly to the path constructed by our strategy applied to $I'$. By the Inductive Hypothesis (IH): $\mathbb{E}_\pi[f(S^\pi_h|s_1)] \ge R(k, h-1) \cdot OPT'$.

\begin{align*}
\mathbb{E}_\pi[f(S^\pi_h)] &= f(s_1) + \mathbb{E}_\pi[f(S^\pi_h|s_1)] \\
&\ge f(s_1) + R(k, h-1)(OPT - f(s_1)) \\
&= f(s_1)(1 - R(k, h-1)) + R(k, h-1) \cdot OPT.
\end{align*}

We verify that $R(k, h-1) < 1$. Since $h \ge 1$, $R(k, h-1) \le C_{k,h-1} = \frac{1}{1+T(k,h-1)}$. Since $T(k,h-1)>0$, $R(k, h-1) < 1$. Thus, $1-R(k, h-1) > 0$.

By the condition for Case 1 and Lemma \ref{lem:avg_marginal}, $f(s_1) \ge t \cdot \mathbb{E}_\pi[f(o^\pi_1)] \ge t \cdot OPT/k$.

\begin{align*}
\mathbb{E}_\pi[f(S^\pi_h)] &\ge \frac{t \cdot OPT}{k}(1 - R(k, h-1)) + R(k, h-1) \cdot OPT \\
&= OPT \cdot \left[ \frac{t}{k} + \left(1 - \frac{t}{k}\right)R(k, h-1) \right] = A_{k,h} \cdot OPT.
\end{align*}

\textbf{Case 2: $\mathbb{E}_\pi[f(u^\pi_1)] < t \cdot \mathbb{E}_\pi[f(o^\pi_1)]$.} We set $s^\pi_1 = v^\pi_1$.

We consider the remaining subproblem $I'(\pi)$. It depends on $\pi$ because both $s^\pi_1$ and the remaining OPT elements $E'_{opt}(\pi) = E_{opt} \setminus \{o^\pi_1\}$ depend on $\pi$. The parameters are $(k-1, h-1)$. The function is $f'_\pi(T) = f(T|s^\pi_1)$.

We analyze the expected value over $\pi$. $\mathbb{E}_\pi[f(S^\pi_h)] = \mathbb{E}_\pi[f(s^\pi_1)] + \mathbb{E}_\pi[f(S^\pi_h|s^\pi_1)]$.

By the IH applied to the subproblem (noting that conditional on $o^\pi_1$, the remaining elements are still uniformly permuted): $\mathbb{E}_\pi[f(S^\pi_h|s^\pi_1)] \ge R(k-1, h-1) \cdot \mathbb{E}_\pi[OPT'(\pi)]$.

The expected optimum of the subproblem is $\mathbb{E}_\pi[OPT'(\pi)] = \mathbb{E}_\pi[f(E'_{opt}(\pi)|s^\pi_1)]$. By submodularity:
$\mathbb{E}_\pi[OPT'(\pi)] \ge OPT - \mathbb{E}_\pi[f(s^\pi_1 \cup o^\pi_1)]$.

Combining these: $\mathbb{E}_\pi[f(S^\pi_h)] \ge \mathbb{E}_\pi[f(s^\pi_1)] + R(k-1, h-1) \cdot (OPT - \mathbb{E}_\pi[f(s^\pi_1 \cup o^\pi_1)])$.

We bound the loss term $\mathbb{E}_\pi[f(s^\pi_1 \cup o^\pi_1)]$. Note that if $s^\pi_1 \ne o^\pi_1$, then $s^\pi_1$ must have arrived before $o^\pi_1$, and since $s^\pi_1 = v^\pi_1$ is the best element up to $o^\pi_1$, it must be that $s^\pi_1 = u^\pi_1$.

We can obtain a bound using the Case 2 condition.
\begin{align*}
\mathbb{E}_\pi[f(s^\pi_1 \cup o^\pi_1)] &= \mathbb{E}_\pi[f(o^\pi_1)] + \mathbb{E}_\pi[f(s^\pi_1|o^\pi_1)] \\
&= \mathbb{E}_\pi[f(o^\pi_1)] + P(s^\pi_1 = o^\pi_1)E[f(s^\pi_1|o^\pi_1)|s^\pi_1 = o^\pi_1] + P(s^\pi_1 \ne o^\pi_1)E[f(s^\pi_1|o^\pi_1)|s^\pi_1 \ne o^\pi_1].
\end{align*}
If $s^\pi_1 = o^\pi_1$, the marginal gain is 0. If $s^\pi_1 \ne o^\pi_1$, then $s^\pi_1 = u^\pi_1$.
\begin{align*}
\mathbb{E}_\pi[f(s^\pi_1 \cup o^\pi_1)] &= \mathbb{E}_\pi[f(o^\pi_1)] + P(s^\pi_1 \ne o^\pi_1)E[f(u^\pi_1|o^\pi_1)|s^\pi_1 \ne o^\pi_1] \\
&\le \mathbb{E}_\pi[f(o^\pi_1)] + \mathbb{E}_\pi[f(u^\pi_1)]. \quad \text{(By submodularity and non-negativity)}
\end{align*}

Using the Case 2 condition $\mathbb{E}_\pi[f(u^\pi_1)] < t \cdot \mathbb{E}_\pi[f(o^\pi_1)]$:
$\mathbb{E}_\pi[f(s^\pi_1 \cup o^\pi_1)] < (1+t)\mathbb{E}_\pi[f(o^\pi_1)]$.

Substituting back into the main inequality:
$\mathbb{E}_\pi[f(S^\pi_h)] > \mathbb{E}_\pi[f(s^\pi_1)] + R(k-1, h-1) \cdot (OPT - (1+t)\mathbb{E}_\pi[f(o^\pi_1)])$.

By definition of $v^\pi_1$, $f(s^\pi_1) \ge f(o^\pi_1)$ for all $\pi$. Thus $\mathbb{E}_\pi[f(s^\pi_1)] \ge \mathbb{E}_\pi[f(o^\pi_1)]$.
$\mathbb{E}_\pi[f(S^\pi_h)] > \mathbb{E}_\pi[f(o^\pi_1)](1 - (1+t)R(k-1, h-1)) + R(k-1, h-1) \cdot OPT$.

Crucially, by Lemma \ref{lem:mtc_bound} (which relies on the MTC), the coefficient $1 - (1+t)R(k-1, h-1)$ is non-negative. We substitute $\mathbb{E}_\pi[f(o^\pi_1)] \ge OPT/k$ (Lemma \ref{lem:avg_marginal}).

\begin{align*}
\mathbb{E}_\pi[f(S^\pi_h)] &> \frac{OPT}{k}(1 - (1+t)R(k-1, h-1)) + R(k-1, h-1) \cdot OPT \\
&= OPT \cdot \left[ \frac{1}{k} + \left(1 - \frac{1+t}{k}\right)R(k-1, h-1) \right] = B_{k,h} \cdot OPT.
\end{align*}

In both cases, $\mathbb{E}_\pi[f(S^\pi_h)] \ge \min(A_{k,h}, B_{k,h}) \cdot OPT \ge R(k, h) \cdot OPT$. The induction is complete.
\end{proof}

\subsubsection*{Analysis of the Specific Recurrence}
We now analyze the recurrence using a specific threshold function that optimizes the fluid limit.

\paragraph{The Optimal Threshold Function}
We define the threshold function $T^*(k, h) = \frac{k}{\sqrt{k^2+h^2}}$.

\begin{lemma}\label{lem:t_star_mtc}
The threshold function $T^*(k, h)$ satisfies the Monotonic Threshold Condition (MTC).
\end{lemma}
\begin{proof}
We need to show $T^*(k, h) \le T^*(k-1, h-1)$ for $k \ge 2, h \ge 1$. Let $g(x) = (1+x^2)^{-1/2}$. $g(x)$ is strictly decreasing for $x \ge 0$. $T^*(k, h) = g(h/k)$. $T^*(k-1, h-1) = g((h-1)/(k-1))$.
We compare $x_1 = h/k$ and $x_0 = (h-1)/(k-1)$. Since $(k,h) \in \mathcal{D}$, $k \ge h$.
$k \ge h \implies -k \le -h \implies hk-k \le hk-h$.
$k(h-1) \le h(k-1)$. Since $k \ge 2$, $k>0$ and $k-1>0$. $(h-1)/(k-1) \le h/k$. So $x_0 \le x_1$.
Since $g(x)$ is decreasing, $g(x_1) \le g(x_0)$. Thus $T^*(k, h) \le T^*(k-1, h-1)$.
\end{proof}

Since $T^*(k, h)$ satisfies the MTC, Theorem \ref{thm:connection} applies. The approximation ratio of the Tree Algorithm is lower bounded by $R(k, k)$ derived using this threshold function.

\paragraph{Solving the Recurrence}
We now prove that the recurrence defined with $T^*(k, h)$ evaluates to $2 - \sqrt{2}$ when $h=k$.

\begin{theorem}\label{thm:recurrence_solution}
Let $R(k, h)$ be the recurrence defined using $T^*(k, h)$. Then for all $k \ge 1$, $R(k, k) = 2-\sqrt{2}$.
\end{theorem}

To prove this, we introduce continuous helper functions that represent the fluid limit of the recurrence and use them to establish bounds.

\emph{The Continuous Helper Functions}
Let $\rho: [0, 1] \to [0, 1]$ and $t: [0, 1] \to (0, 1]$ be defined as:
\begin{align*}
\rho(x) &= 1 + x - \sqrt{x^2+1}, \\
t(x) &= \frac{1}{\sqrt{x^2+1}}.
\end{align*}
Note that $T^*(k, h) = t(h/k)$. We define the comparison function $P(k, h) = \rho(h/k)$ for $k \ge 1$.

We establish key properties of these functions.

\begin{lemma}[Properties of $\rho(x)$ and $t(x)$]\label{lem:properties}
\begin{enumerate}
    \item Boundary values: $\rho(0) = 0, \rho(1) = 2-\sqrt{2}$.
    \item Concavity: $\rho(x)$ is strictly concave on $[0, 1]$.
    \item Convex Auxiliary Function: $W(x) = \frac{1}{1-\rho(x)} = \sqrt{x^2+1}+x$ is strictly convex on $[0, 1]$.
    \item Rate 1 Identity: $\rho'(x) = t(x)(1-\rho(x))$. (Equivalently, $W'(x) = W(x)t(x)$).
    \item Rate 2 Identity: $\rho'(x)(1-x) = 1 - (1+t(x))\rho(x)$.
    \item Threshold Relation: $\rho(x) \le \frac{1}{1+t(x)}$.
    \item Monotonicity of t: $t(x)$ is strictly decreasing on $[0, 1]$.
\end{enumerate}
\end{lemma}
\begin{proof}
1. $\rho(0) = 1+0-\sqrt{1} = 0$. $\rho(1) = 1+1-\sqrt{2} = 2-\sqrt{2}$.

2. $\rho'(x) = 1 - \frac{x}{\sqrt{x^2+1}}$. $\rho''(x) = -\frac{\sqrt{x^2+1} - x(x/\sqrt{x^2+1})}{x^2+1} = -\frac{(x^2+1)-x^2}{(x^2+1)^{3/2}} = -\frac{1}{(x^2+1)^{3/2}}$. Since $\rho''(x) < 0$, $\rho(x)$ is strictly concave.

3. $W'(x) = \frac{x}{\sqrt{x^2+1}} + 1$. $W''(x) = \frac{1}{(x^2+1)^{3/2}}$. Since $W''(x) > 0$, $W(x)$ is strictly convex.

4. $t(x)(1-\rho(x)) = \frac{1}{\sqrt{x^2+1}}(1 - (1+x-\sqrt{x^2+1})) = \frac{1}{\sqrt{x^2+1}}(\sqrt{x^2+1}-x) = 1 - \frac{x}{\sqrt{x^2+1}} = \rho'(x)$.

5. Let $S = \sqrt{x^2+1}$. $t(x)=1/S, \rho(x)=1+x-S, \rho'(x)=1-x/S$.
LHS $= (1-x/S)(1-x) = 1-x-x/S+x^2/S = 1-x+\frac{x^2-x}{S}$.
RHS $= 1-(1+1/S)(1+x-S) = 1-(1+x-S+\frac{1+x}{S}-1) = 1-x+S-\frac{1+x}{S}$.
RHS $= 1-x + \frac{S^2-(1+x)}{S} = 1-x + \frac{(x^2+1)-1-x}{S} = 1-x + \frac{x^2-x}{S}$.
LHS = RHS. (There was a slight algebraic error in the previous draft's proof of this point, which is corrected here).

6. We want $\rho(x)(1+t(x)) \le 1$. By the Rate 2 identity, this is equivalent to $1 - (1-x)\rho'(x) \le 1$, or $(1-x)\rho'(x) \ge 0$. Since $x \in [0, 1]$, $1-x \ge 0$. Since $\sqrt{x^2+1} > x$, $\rho'(x) = 1-x/\sqrt{x^2+1} > 0$. The inequality holds. Equality holds only if $x=1$.

7. $t'(x) = -\frac{1}{2}(x^2+1)^{-3/2}(2x) = -x(x^2+1)^{-3/2}$. This is negative for $x>0$.
\end{proof}
We next give the proof of Theorem \ref{thm:recurrence_solution}. 

\begin{proof}
We use a "Diagonal Sandwich" argument. We first prove a global lower bound $R(k,h) \ge P(k,h)$, and then prove a tight upper bound at the diagonal $h=k$.

\textbf{Step 1: Global Lower Bound.} We prove by strong induction that $R(k, h) \ge P(k, h) = \rho(h/k)$ for all $(k, h) \in \mathcal{D}, k \ge 1$.

\textbf{Base Conditions:}
\begin{itemize}
    \item $h=0, k \ge 1$: $R(k, 0) = 0$. $P(k, 0) = \rho(0) = 0$. Holds.
    \item $k=1$. $R(1, 1)$. $t_{1,1} = T^*(1, 1) = 1/\sqrt{2}$. $A_{1,1} = t_{1,1} + (1-t_{1,1})R(1, 0) = 1/\sqrt{2}$. $B_{1,1} = 1 + (1-(1+t_{1,1}))R(0, 0) = 1$. $C_{1,1} = \frac{1}{1+1/\sqrt{2}} = \frac{\sqrt{2}}{\sqrt{2}+1} = 2-\sqrt{2}$.
    $R(1, 1) = \min(1/\sqrt{2}, 1, 2-\sqrt{2})$. Since $2-\sqrt{2} \approx 0.586$ and $1/\sqrt{2} \approx 0.707$, $R(1, 1) = 2-\sqrt{2}$. $P(1, 1) = \rho(1) = 2-\sqrt{2}$. Holds.
\end{itemize}

\textbf{Inductive Step:} Fix $(k, h)$ with $k \ge 2, h \ge 1$. Assume $R(k', h') \ge P(k', h')$ for all $(k', h') < (k, h)$ lexicographically. Let $x = h/k$. Let $t = T^*(k, h) = t(x)$. We must show $A_{k,h}, B_{k,h}, C_{k,h} \ge P(k, h) = \rho(x)$.

\textbf{Analysis of $C_{k,h}$:} $C_{k,h} = \frac{1}{1+t(x)}$. By Lemma \ref{lem:properties}.6, $\rho(x) \le \frac{1}{1+t(x)}$. So $C_{k,h} \ge P(k, h)$.

\textbf{Analysis of $A_{k,h}$:} Let $\Delta x = 1/k$. The previous state is $(k, h-1)$. The ratio is $x' = (h-1)/k = x-\Delta x$. By IH, $R(k, h-1) \ge P(k, h-1) = \rho(x-\Delta x)$.
$$ A_{k,h} \ge t(x)\Delta x + (1-t(x)\Delta x)\rho(x-\Delta x). $$
We want to show this is $\ge \rho(x)$.
$$ t(x)\Delta x(1-\rho(x-\Delta x)) \ge \rho(x) - \rho(x-\Delta x). $$
Let $W(x) = 1/(1-\rho(x))$.
$$ \frac{t(x)\Delta x}{W(x-\Delta x)} \ge \frac{1}{W(x-\Delta x)} - \frac{1}{W(x)}. $$
Multiplying by $W(x)W(x-\Delta x)$ (which is positive):
$$ t(x)\Delta x W(x) \ge W(x) - W(x-\Delta x). $$
$$ W(x-\Delta x) \ge W(x) - W(x)t(x)\Delta x. $$
By Lemma \ref{lem:properties}.4 (Rate 1 Identity), $W'(x) = W(x)t(x)$.
$$ W(x-\Delta x) \ge W(x) - W'(x)\Delta x. $$
This inequality holds because $W(x)$ is strictly convex (Lemma \ref{lem:properties}.3). A convex function lies above its tangent line: $W(y) \ge W(x) + W'(x)(y-x)$. Setting $y=x-\Delta x$ yields the result. Thus $A_{k,h} \ge P(k, h)$.

\textbf{Analysis of $B_{k,h}$:} The previous state is $(k-1, h-1)$. Let $x_1 = h/k$ and $x_0 = (h-1)/(k-1)$. By IH, $R(k-1, h-1) \ge P(k-1, h-1) = \rho(x_0)$.
$$ B_{k,h} \ge \frac{1}{k} + \left(1 - \frac{1+t(x_1)}{k}\right)\rho(x_0). $$
We want to show this is $\ge \rho(x_1)$.
$$ \frac{1}{k}(1 - (1+t(x_1))\rho(x_0)) \ge \rho(x_1) - \rho(x_0). $$

We analyze the relationship between $1/k$ and the ratios.
$x_1-x_0 = \frac{h}{k} - \frac{h-1}{k-1} = \frac{h(k-1)-k(h-1)}{k(k-1)} = \frac{k-h}{k(k-1)}$.
$1-x_0 = 1 - \frac{h-1}{k-1} = \frac{k-1-(h-1)}{k-1} = \frac{k-h}{k-1}$.
Therefore, $x_1-x_0 = \frac{1}{k}(1-x_0)$. So $1/k = (x_1-x_0)/(1-x_0)$.
(This assumes $h < k$, so $x_0 < 1$).

Case 1: $1 \le h < k$. In this case $x_0 < x_1 < 1$. $1-x_0 > 0$ and $x_1-x_0 > 0$. Substituting $1/k$:
$$ \frac{x_1-x_0}{1-x_0}(1 - (1+t(x_1))\rho(x_0)) \ge \rho(x_1) - \rho(x_0). $$
$$ \frac{1 - (1+t(x_1))\rho(x_0)}{1-x_0} \ge \frac{\rho(x_1) - \rho(x_0)}{x_1-x_0}. $$
The RHS is the slope of the secant line of $\rho(x)$ from $x_0$ to $x_1$. Since $\rho(x)$ is concave (Lemma \ref{lem:properties}.2), the slope of the secant is less than the slope of the tangent at $x_0$:
$$ \frac{\rho(x_1) - \rho(x_0)}{x_1-x_0} < \rho'(x_0). $$
It is sufficient to prove:
$$ \frac{1 - (1+t(x_1))\rho(x_0)}{1-x_0} \ge \rho'(x_0). $$
By Lemma \ref{lem:properties}.5 (Rate 2 Identity) applied at $x_0$: $\rho'(x_0) = \frac{1-(1+t(x_0))\rho(x_0)}{1-x_0}$. We need to show:
$$ 1 - (1+t(x_1))\rho(x_0) \ge 1 - (1+t(x_0))\rho(x_0). $$
$$ (1+t(x_0))\rho(x_0) \ge (1+t(x_1))\rho(x_0). $$
Since $h \ge 1$, $x_0 = (h-1)/(k-1) \ge 0$. $\rho(x_0) \ge 0$. We need $t(x_0) \ge t(x_1)$. Since $t(x)$ is decreasing (Lemma \ref{lem:properties}.7) and $x_0 < x_1$, this holds. Thus $B_{k,h} \ge P(k, h)$ when $h < k$.

Case 2: $h=k$. ($k \ge 2$). We want to show $B_{k,k} \ge P(k, k) = 2-\sqrt{2}$. Here $x_1=1$. $t(x_1)=1/\sqrt{2}$. The previous state is $(k-1, k-1)$. $x_0=1$. By IH, $R(k-1, k-1) \ge P(k-1, k-1) = \rho(1) = 2-\sqrt{2}$.
$$ B_{k,k} \ge \frac{1}{k} + \left(1 - \frac{1+1/\sqrt{2}}{k}\right)(2-\sqrt{2}). $$
We check if this is $\ge 2-\sqrt{2}$.
$$ \frac{1}{k} \ge (2-\sqrt{2}) - \left(1 - \frac{1+1/\sqrt{2}}{k}\right)(2-\sqrt{2}) $$
$$ \frac{1}{k} \ge (2-\sqrt{2})\left(1 - \left(1 - \frac{1+1/\sqrt{2}}{k}\right)\right) = (2-\sqrt{2})\left(\frac{1+1/\sqrt{2}}{k}\right). $$
$$ 1 \ge (2-\sqrt{2})(1+1/\sqrt{2}). $$
RHS $= (2-\sqrt{2}) \left(\frac{\sqrt{2}+1}{\sqrt{2}}\right) = \frac{(2-\sqrt{2})(\sqrt{2}+1)}{\sqrt{2}} = \frac{2\sqrt{2}+2-2-\sqrt{2}}{\sqrt{2}} = \frac{\sqrt{2}}{\sqrt{2}} = 1$.
The inequality $1 \ge 1$ holds. Thus $B_{k,k} \ge P(k, k)$.

Conclusion of Step 1: We have shown $A_{k,h} \ge P(k, h)$, $B_{k,h} \ge P(k, h)$, and $C_{k,h} \ge P(k, h)$. Therefore, $R(k, h) = \min(A_{k,h}, B_{k,h}, C_{k,h}) \ge P(k, h)$ for all $(k,h) \in \mathcal{D}$.
In particular, $R(k, k) \ge P(k, k) = \rho(1) = 2-\sqrt{2}$.

\textbf{Step 2: Diagonal Upper Bound.}
We consider the upper bound at $h=k$. By definition of the recurrence, $R(k, k) \le C_{k,k}$.
$t_{k,k} = T^*(k, k) = k/\sqrt{k^2+k^2} = k/\sqrt{2k^2} = 1/\sqrt{2}$.
$$ C_{k,k} = \frac{1}{1+t_{k,k}} = \frac{1}{1+1/\sqrt{2}} = 2-\sqrt{2}. $$
So, $R(k, k) \le 2-\sqrt{2}$.

\textbf{Conclusion of the Proof (Diagonal Sandwich):}
Since $2-\sqrt{2} \le R(k, k)$ (Step 1) and $R(k, k) \le 2-\sqrt{2}$ (Step 2), we must have $R(k, k) = 2-\sqrt{2}$ for all $k \ge 1$.
\end{proof}

\subsubsection*{Conclusion}
We have rigorously established the connection between the Tree Algorithm for submodular maximization under adversarial injections and a generalized recurrence relation featuring a state-dependent threshold $T(k, h)$.

\begin{enumerate}
    \item We proved that if the threshold function satisfies the Monotonic Threshold Condition (MTC), the recurrence $R(k, h)$ provides a lower bound on the expected approximation ratio of the algorithm (Theorem \ref{thm:connection}), correctly utilizing the inequalities inherent to submodular functions.
    \item We defined a specific threshold function $T^*(k, h) = k/\sqrt{k^2+h^2}$ and verified it satisfies the MTC (Lemma \ref{lem:t_star_mtc}).
    \item We analyzed the resulting recurrence using a "Diagonal Sandwich" argument and proved that $R(k, k) = 2-\sqrt{2}$ (Theorem \ref{thm:recurrence_solution}).
\end{enumerate}

This implies that the Tree Algorithm achieves an approximation ratio of at least $2-\sqrt{2} \approx 0.5857$ in the adversarial injections model. This improves upon the 0.55 approximation derived in \cite{GargKRS20} using a fixed threshold analysis.


\newpage

\newpage
\subsection{Streaming Algorithms: Entropy, Low Rank Approximation, and Chamfer Distance}\label{sec:streaming}
\textit{Written by Ying Feng, David P. Woodruff, and Samson Zhou}.

\subsubsection*{Problem Context}
This case study examines several fundamental problems in streaming algorithms, but where the focus is on settings where space efficiency alone is not the only resource constraint. 
In addition to minimizing memory usage, modern systems often operate under limitations on write bandwidth, communication, or the frequency with which internal state can be modified.
We consider four representative problems: 
\begin{enumerate}
\item 
\textbf{Entropy Estimation:} 
Can the Shannon entropy of a data stream be estimated with significantly fewer internal state changes than the current $\tilde{O}(\sqrt{n})$ bound?
\item 
\textbf{Consistent LRA:} 
Is the optimal low-rank approximation subspace stable under row updates?
\item
\textbf{Efficient Encoding for LRA:} 
Can low-rank sketches computed on different parts of the data be globally compressed without losing accuracy?
\item
\textbf{Chamfer Distance:} Can the fast $\ell_1$ Chamfer distance algorithm be extended to the $\ell_2$ metric?
\end{enumerate}
Together, these problems highlight a common theme: understanding how algorithmic guarantees interact with more refined complexity measures beyond space, such as stability, recourse, and internal state changes.

\subsubsection*{AI Contribution}
Across these problems, the model played a role in identifying structural properties that were not explicit in prior analyses and in turning those observations into concrete algorithmic improvements.
\begin{itemize}
\item 
\textbf{Entropy:} 
The model observed that the standard algorithm only queries moments $F_p$ for $p \approx 1$, avoiding the high-variance regime ($p=2$) entirely. 
This immediately implied a polylogarithmic bound on the number of internal state changes.
\item 
\textbf{LRA Stability:} 
The model produced a proof showing that the optimal rank-$k$ subspace changes by at most a constant amount under a single row insertion, using eigenvalue interlacing and dimension arguments.
\item 
\textbf{LRA Encoding:} 
The model suggested decomposing local sketches into a global ``head'' and quantized ``tail,'' and proved that this representation preserves projection costs with small additive error.
\item 
\textbf{Chamfer Distance:} 
The model adapted the quadtree-based algorithm from $\ell_1$ to $\ell_2$, introducing a dimensionality reduction step (Johnson-Lindenstrauss) to handle high dimensions efficiently.
\end{itemize}

\subsubsection{Internal State Changes for Streaming Entropy}
In streaming algorithms, the foremost priority is typically to minimize the amount of memory (space) required to process a stream, since data streams can be massive and cannot be stored in full. 
Recently, \cite{JayaramWZ24} noted that minimizing the number of internal state changes, which counts how often the algorithm updates its memory representation, is also a reasonable secondary priority, particularly in settings with limited write bandwidth or distributed monitoring systems.

In the streaming model considered here, the stream consists of a sequence of elements from $[n] = \{1,2,\ldots,n\}$, and the algorithm maintains a \emph{frequency vector} $f \in \mathbb{Z}_{\ge 0}^n$, where $f_i$ denotes the number of times item $i$ has appeared in the data stream, so that each update to the stream increments the corresponding coordinate of $f$ by one. 
The number of \emph{internal state changes} is the number of times the algorithm modifies its memory state over the duration of the data stream. 

The \emph{$F_p$ moment estimation problem} asks to approximate the $p$-th frequency moment of the stream, defined as $F_p = \sum_{i=1}^n f_i^p$. 
Accurate $F_p$ estimation is fundamental because many other statistics can be expressed in terms of frequency moments. 
In particular, $F_p$ moments are central tools in streaming algorithms for estimating norms, identifying heavy hitters, and, as in this discussion, approximating Shannon entropy.

\begin{figure}[!htb]
\begin{chatbox}
\begin{userturn}
Please try to improve the paper by identifying and solving an open question from it. 
\vskip 0.1in
For your proof, adhere to: $\langle\text{rigor text}\rangle$.
\end{userturn}

\begin{assistantturn}
We identify and resolve an open question arising from the analysis of the entropy estimation algorithm presented in Section 3.3 of the paper.
\vskip 0.1in
\textbf{Open Question:}
\newline
The paper establishes in Theorem 3.8 that the Shannon entropy of an insertion-only stream can be estimated with an $\varepsilon$-additive error using $\tilde{O}(\sqrt{n}/\varepsilon^{O(1)})$ internal state changes (assuming $m=\text{poly}(n)$) and polylogarithmic space. This $\tilde{O}(\sqrt{n})$ bound stems from the worst-case complexity of $F_p$ moment estimation for $p \in (0, 2]$, which is $\tilde{O}(n^{1-1/p})$, maximized as $p \to 2$. The open question is: Can this $\tilde{O}(\sqrt{n})$ upper bound on the number of state changes for streaming entropy estimation be improved?
\vskip 0.1in
We answer this question affirmatively. By analyzing the specific moments required by the underlying [HNO08] framework used in the paper, we show that they are sufficiently close to 1 such that the polynomial dependence on $n$ vanishes.
\end{assistantturn}
\end{chatbox}
\caption{Model identification of open problem and progress direction}
\end{figure}

\emph{Shannon entropy}, defined as $H = -\sum_{i=1}^n p_i \log p_i$ with $p_i = f_i/\|f\|_1$, measures the uncertainty or information content in a dataset. 
Estimating Shannon entropy is a fundamental problem with numerous applications, including network traffic analysis (to detect anomalies or DDoS attacks), database summarization (to quantify diversity), natural language processing (to measure information content of text corpora), and data compression (to determine optimal coding schemes). 
Because entropy is sensitive to the entire distribution of frequencies, accurately approximating it in a streaming setting is challenging and typically relies on sophisticated techniques, such as Chebyshev interpolation of multiple $F_p$ moment estimates, as observed by \cite{HarveyNO08}.

\cite{JayaramWZ24} analyzed the internal-state complexity required for approximating $F_p$ in one-pass insertion-only streams. 
Their main theorem distinguishes two regimes: for $p \in (0,1]$, a $(1+\eps)$-approximation can be achieved with high probability using only $\poly\left(\frac{1}{\eps},\log n\right)$ internal state changes, while for $p \ge 1$, estimating $F_p$ requires $\widetilde{O}(n^{1-1/p})$ internal state changes, which can be as large as $\tO{\sqrt{n}}$ when $p = 2$. 
In their discussion of Shannon entropy, they interpreted the Chebyshev interpolation observation by \cite{HarveyNO08} and noted that it requires evaluating multiple $F_p$ moment estimates. 
They assumed the exponents lie in the range $p \in (0,2)$, which implies that estimating $F_2$ could be necessary. 
Because $F_2$ sketches can incur up to $\sqrt{n}$ internal state changes, this led to the conclusion that Shannon entropy estimation inherits an $O(\sqrt{n})$ state-change cost.

However, upon closer examination, \emph{the model} noticed that all evaluation points in the interpolation procedure are of the form $1 + y_i$, where $y_i = f(\cos(i \pi / k))$, and showed that $1 + y_i \in (0,1)$ for all $i$, c.f., \Cref{lem:interpolate:p}. 
This implies that the entropy algorithm never requires $F_p$ estimates for $p \ge 1$, and all required frequency moments lie strictly within the low-$p$ regime. 
Consequently, the worst-case $\tO{n^{1-1/p}}$ state-change barrier from the $p \ge 1$ region is avoided. 
It then follows that the necessary $F_p$ estimates for $p \in (0,1)$ use only $\poly\left(\frac{1}{\eps},\log n\right)$ internal state changes, c.f., \Cref{thm:fp:states}.  
As a result, the Shannon entropy algorithm can compute an additive $\eps$-approximation using $\tO{\frac{1}{\eps^2}+\log n}$ bits of space and $\poly\left(\frac{1}{\eps},\log n\right)$ internal state changes, significantly improving upon the previously believed $\tO{\sqrt{n}}$ bound. 

\begin{theorem}
\cite{JayaramWZ24}
\label{thm:fp:states}
Let $\eps \in (0, 1)$ be an approximation parameter, $\delta \in (0, 1/3)$ be a failure probability, and $n$ be the domain size. There exists a \textbf{one-pass insertion-only streaming algorithm} for the $p$-th frequency moment $F_p$:
\begin{enumerate}
    \item \textbf{For $p \in (0, 1]$:} The algorithm achieves a $(1+\eps)$-approximation (w.h.p.) using $\poly\left(\log n, \frac{1}{\eps}, \log\frac{1}{\delta}\right)$ internal state changes and $\tO{\frac{1}{\eps^2}(\log\log n+\log\frac{1}{\eps})+\frac{\log(1/\eps)}{\log\log(1/\eps)}\log n}$ bits of space.
    
    \item \textbf{For $p \ge 1$:} The algorithm outputs $\widehat{F_p}$ such that $\PPr{\left\lvert\widehat{F_p}-F_p\right\rvert\le\eps\cdot F_p}\ge\frac{2}{3}$, with $\tO{n^{1-1/p}}$ internal state changes. The space complexity depends on $p$:
    \begin{itemize}
        \item If $p \in [1, 2]$: Space is $\tO{\frac{1}{\eps^{4+4p}}}\cdot\polylog(mn)$.
        \item If $p > 2$: Space is $\tO{\frac{1}{\eps^{4+4p}} n^{1-2/p}}$.
    \end{itemize}
\end{enumerate}
\end{theorem}

\begin{algorithm}
\caption{Additive approximation of empirical Shannon entropy}
\label{alg:entropy-approx}
\begin{algorithmic}[1]
\Require Error parameter $\tilde\varepsilon$, points $\{y_0,\dots,y_k\}$
\For{$i=0,\dots,k$}
  \State Compute $\tilde F_{1+y_i}$, a $(1+\tilde\varepsilon)$-approx.\ of $F_{1+y_i}$
  \State $\tilde H(y_i)\gets -\log\!\left(\tilde F_{1+y_i}/\|A\|_{1+y_i}^{\,1+y_i}\right)/y_i$
  \State $\tilde T(y_i) \gets \left(1-\tilde F_{1+y_i}/\|A\|_{1+y_i}^{\,1+y_i}\right)/y_i$
\EndFor
\State \Return estimate of $H(0)$ or $T(0)$ by interpolating $\{\tilde H(y_i)\}$ or $\{\tilde T(y_i)\}$
\end{algorithmic}
\end{algorithm}

\begin{lemma}
\cite{HarveyNO08}
To achieve an additive $\eps$-approximation of the Shannon entropy, it suffices to implement \Cref{alg:entropy-approx} with $\tilde\varepsilon=\frac{\eps}{12(k+1)^3\log m}$ using points $y_0,\ldots,y_k$ with $k=\log\frac{1}{\eps}+\log\log m$ and $y_i=f(\cos(i\pi/k))$ for $f(y)=\frac{(k^2\ell)\cdot y-\ell\cdot(k^2+1)}{2k^2+1}$, where $\ell=\frac{1}{(2(k + 1) \log m}$. 
\end{lemma}

\begin{lemma}
\label{lem:interpolate:p}
Let $y_i=f(\cos(i\pi/k))$ for $f(y)=\frac{(k^2\ell)\cdot y-\ell\cdot(k^2+1)}{2k^2+1}$, with $\ell=\frac{1}{(2(k + 1) \log m}$. 
Then $1+y_i\in(0,1)$ for all $i\ge 0$.
\end{lemma}
\begin{proof}
Observe that $1+y_i<1$ if and only if $f(\cos(i\pi/k))<0$ or equivalently, 
\[(k^2\ell)\cdot\cos(i\pi/k)-\ell\cdot(k^2+1)<0.\]
Since $\ell>0$, this is equivalent to $k^2\cdot\cos(i\pi/k)<(k^2+1)$, which holds for all $i$, since $\cos(i\pi/k)\le 1$ for all $i$.

Similarly, observe that $1+y_i>0$ if and only if $f(\cos(i\pi/k))>-1$ or equivalently, 
\[(k^2\ell)\cdot\cos(i\pi/k)-\ell\cdot(k^2+1)>-(2k^2+1).\]
Since $|\cos(i\pi/k)|\le 1$ for all $i$, then  
\[(k^2\ell)\cdot\cos(i\pi/k)-\ell\cdot(k^2+1)>-\ell(2k^2+1).\]
Then, the desired claim follows because $\ell\in(0,1)$.
\end{proof}
\noindent
From \Cref{thm:fp:states} and \Cref{lem:interpolate:p}, we have:
\begin{corollary}
Given an accuracy parameter $\eps \in (0,1)$ and a stream of length $m=\poly(n)$ over a universe of size $n$, there exists a one-pass insertion-only streaming algorithm that outputs an additive $\eps$-approximation $\widehat{H}$ to the entropy $H$ of the data stream with high probability, using $\tO{\frac{1}{\eps^2} + \log n}$ bits of space and $\poly\left(\frac{1}{\eps},\log n\right)$ internal state changes.
\end{corollary}

\subsubsection{Consistent Low-Rank Approximation}
Low-rank approximation is a central tool in data analysis and machine learning. 
Given a matrix $A\in\mathbb{R}^{n \times d}$, the goal is to find a rank-$k$ subspace, i.e., orthogonal matrix $V\in\mathbb{R}^{k\times d}$, that minimizes the approximation error $\|A-AV^\top V\|_F^2$, measured in Frobenius norm. 
In static settings, this problem is well understood: the optimal solution is given by the top $k$ singular vectors of $A$, and efficient algorithms with strong guarantees are known.

Many modern applications, however, are inherently dynamic. 
Data arrive sequentially, are corrected, or are removed, and low-rank approximations are recomputed repeatedly as part of larger pipelines. 
In these settings, approximation quality alone is not sufficient. 
Each time the output subspace changes, downstream systems may need to retrain models, update features, or revalidate decisions, all of which can be expensive. 
This motivates a second objective beyond accuracy, namely consistency.

Consistent low-rank approximation is a relatively new problem formulation that makes this tradeoff explicit.
Instead of computing a single low-rank approximation to a fixed dataset, the algorithm is asked to maintain a sequence of near-optimal subspaces as the data evolve, while keeping the changes between successive outputs small.
Formally, we fix an accuracy parameter $\eps>0$ and a target rank $k$. 
The input is a matrix $A\in\mathbb{R}^{n\times d}$ that changes over time, e.g., rows of $A$ are revealed incrementally, or there is a sequence of updates to entries of $A$. 
At each time step $t$, the algorithm observes the current matrix $A^{(t)}$ and outputs a rank-$k$ orthogonal matrix $V^{(t)}\in\mathbb{R}^{k\times d}$.

The first requirement is accuracy. 
For every $t$, the subspace $V^{(t)}$ should give a near-optimal rank-$k$ approximation to $A^{(t)}$, in the sense that
\[\|A^{(t)}-A^{(t)}(V^{(t)})^\top V^{(t)}\|_F^2\le (1+\eps)\cdot \OPT_t,\]
where
\[\OPT_t = \min_{V\in\mathbb{R}^{k\times d}} \|A^{(t)}-A^{(t)}V^\top V\|_F^2\]
is the optimal rank-$k$ approximation error at time $t$.

The second requirement is stability.
As the matrix $A$ evolves, the sequence of outputs $V^{(1)},\ldots,V^{(n)}$ should change minimally. 
We formalize this desideratum through the notion of \emph{recourse}, defined as the squared Frobenius distance between consecutive subspaces.
For two rank-$k$ matrices $R,T\in\mathbb{R}^{k\times d}$, let $P_R$ and $P_T$ be the orthogonal projection matrices onto their row spans, and define $\Recourse(R,T)=\|P_R-P_T\|_F^2$. 
The goal is to minimize the total recourse
\[\sum_{t=2}^{n}\Recourse(V^{(t)},V^{(t-1)}),\]
while still meeting the approximation guarantee at every time step. 
This formulation reflects the practical intuition behind consistent low-rank approximation.
The algorithm is allowed to adapt as the input evolves, but large or frequent rotations of the maintained subspace are discouraged, since they correspond to costly changes in downstream representations.

This notion of consistency is not addressed by standard streaming algorithms such as Frequent Directions~\cite{GhashamiLPW16}, which may change their output subspaces frequently in order to achieve strong approximation guarantees, e.g., even small perturbations in the data can cause a complete change in the reported subspace and leading to large recourse. 
However, \emph{the model} noticed that the optimal rank-$k$ subspace is inherently stable under row insertions. 
In other words, the recourse between the optimal subspaces before and after a row insertion is bounded by a fixed constant.

\begin{figure}[!htb]
\begin{chatbox}
\begin{userturn}
Analyze whether the optimal rank-$k$ subspace is stable under rank-one updates when the stream has polynomially bounded condition number. 
\vskip 0.1in
Provide a rigorous proof establishing the resulting recourse bound.
\end{userturn}

\begin{assistantturn}
We show that when a single row is appended to a matrix, the optimal rank-$k$ right singular subspace changes by at most a constant amount.  
Specifically, if $V^*_{t-1}$ and $V^*_t$ denote the top-$k$ subspaces before and after the row insertion, then the squared Frobenius distance between their projection matrices satisfies $\Recourse(V^*_{t-1},V^*_t)\le 8$.  
This establishes that the optimal subspace is inherently stable under row insertions, providing a concrete bound on recourse in the row-arrival model.
\end{assistantturn}
\end{chatbox}
\caption{Model-generated resolution of a stability question in consistent low-rank approximation}
\end{figure}
In particular, the model produced a complete and correct proof showing:
\begin{theorem}
Suppose $A^{(t)}$ is obtained from $A^{(t-1)}$ by adding one new row. 
Let $V^*_{t-1}$ and $V^*_t$ denote the optimal rank-$k$ subspaces of the matrices before and after the update. 
Then the change between these two subspaces is at most $\Recourse(V^*_{t-1},V^*_t)\le 8$.
\end{theorem}
\begin{proof}
Consider the covariance matrices
\[B_{t-1} = (A^{(t-1)})^\top A^{(t-1)}, \qquad B_t = (A^{(t)})^\top A^{(t)}.\]
Since $A^{(t)}$ is just $A^{(t-1)}$ with the new row $A_t$ appended, we have
\[B_t = B_{t-1} + A_t^\top A_t,\]
so $B_t$ is a rank-one positive semidefinite update of $B_{t-1}$. 
By the Eckart-Young theorem, the optimal rank-$k$ subspaces $V^*_{t-1}$ and $V^*_t$ are spanned by the top $k$ eigenvectors of $B_{t-1}$ and $B_t$, respectively.

Next, we consider how much these subspaces can change.  
Define the subspace orthogonal to the new row as
\[S_a = \{v \in \mathbb{R}^d : A_t v = 0\},\]
which has dimension $d-1$, assuming $v$ is nonzero. 
Let $W = V^*_{t-1} \cap S_a$ be the part of the old subspace that is unaffected by the new row. 
Using basic dimension counting, it follows that $\dim(W)\ge\dim(V^*_{t-1})-(d-(d-1))=k-1$. 

For any vector $w \in W$, we have $B_t w = B_{t-1} w$, so $w$ is an eigenvector for both covariance matrices with the same eigenvalue.  
Let $\lambda_1 \ge \ldots \ge \lambda_d$ be the eigenvalues of $B_{t-1}$ and $\mu_1 \ge \ldots \ge \mu_d$ those of $B_t$.  
By the Cauchy interlacing theorem,
\[\mu_1 \ge \lambda_1 \ge \mu_2 \ge \lambda_2 \ge \ldots \ge \mu_k \ge \lambda_k \ge \mu_{k+1} \ge \ldots\]
Since $W\subseteq V^*_{t-1}$, its vectors correspond to eigenvalues among the top $k$ of $B_{t-1}$, and by interlacing they lie within the top $k+1$ eigenvectors of $B_t$. 
From this, we conclude that the intersection of the old and new subspaces has dimension at least $k-2$.

Finally, consider the projection matrices $P_{t-1}$ and $P_t$ onto $V^*_{t-1}$ and $V^*_t$.  
Let $W_{\cap} = V^*_{t-1} \cap V^*_t$, and let $P_{\cap}$ be the projection matrix onto $W_{\cap}$. 

If $\dim(W_{\cap}) = k$, the subspaces are identical and recourse is zero.  
If $\dim(W_{\cap}) = k-1$, then
\[P_{t-1} = P_{\cap} + u_1 u_1^\top, \qquad P_t = P_{\cap} + u_2 u_2^\top,\]
where $u_1$ and $u_2$ are unit vectors orthogonal to the shared subspace.  
Then $P_t - P_{t-1} = u_2 u_2^\top - u_1 u_1^\top$, so by the triangle inequality
\[\Recourse(P_t, P_{t-1}) \le 2 \|u_1 u_1^\top\|_F^2 + 2 \|u_2 u_2^\top\|_F^2 = 4.\]
Similarly, if $\dim(W_{\cap}) = k-2$, the difference involves four rank-one terms, giving $\Recourse(P_t, P_{t-1}) \le 8$.  
\end{proof}
This result has immediate algorithmic implications and has been incorporated into a work accepted for publication at ICLR 2026~\cite{WoodruffZ26}. 
Although the theorem is stated for row insertions, the result also applies to row deletions by symmetry. 
Moreover, changes to a single entry of a matrix can be expressed as a row deletion followed by a row insertion. 
Thus, the result directly implies a constant upper bound on recourse for these operations, thereby establishing a $\OO(1)$ amortized recourse bound for the \emph{dynamic model}, where the matrix is updated by a sequence of entry modifications, row insertions, or row deletions.  

\subsubsection{Global Efficient Encoding for Low-Rank Approximation}
In many applications, we work with a large matrix $A \in \mathbb{R}^{n \times d}$, where both dimensions $n$ and $d$ may be very large. 
Directly storing or processing such matrices can be computationally expensive, especially when we only care about their action on low-dimensional subspaces or rank-$k$ approximations. 
Low-rank approximations (LRA) provide a natural tool to reduce storage and computation by approximating $A$ with a matrix of rank at most $k \ll \min(n,d)$, while preserving important properties such as the Frobenius or spectral norm.  
A more powerful guarantee is achieving a projection-cost preservation. 
In our specific context, the rows of an input matrix $A \in \mathbb{R}^{n \times d}$ arrive sequentially and the goal is to compute a matrix $B \in \mathbb{R}^{m \times d}$ such that for all rank-$k$ orthogonal projection matrices $P \in \mathbb{R}^{d \times d}$, we have
\[(1-\eps)\|A-AP\|_F^2 \leq\|B-BP\|_F^2 \leq(1+\eps)\|A-AP\|_F^2,\]
where $\|\cdot\|_F$ denotes the Frobenius norm. 

A common setting arises when $A$ can be decomposed as a concatenation of blocks, 
\[A = Q_1 \circ Q_2 \circ \dots \circ Q_m,\] 
where each block $Q_i \in \mathbb{R}^{n_i \times d}$ represents a submatrix of $A$. 
This block structure naturally appears in streaming or distributed scenarios, where each $Q_i$ corresponds to a portion of the data arriving at different times or stored across multiple nodes. 
For each block, one can compute a local low-rank approximation $W_i$ that summarizes $Q_i$ while approximately preserving its projection costs, i.e., the squared Frobenius norm under all rank-$k$ orthogonal projections.  
In particular, it is known that it suffices for $W_i$ to just have $r:=\tilde{O}{\frac{k}{\eps^2}}$ subsampled and reweighted rows of $Q_i$~\cite{cohen2017input}. 
However, na\"{i}vely storing all local approximations $W_i$ may still be expensive, especially if $m$ or the number of rows in each local approximation $W_i$ is large. 
The goal of the global efficient encoding is to compress these local sketches into a representation that:

\begin{enumerate}
    \item Uses significantly fewer bits of storage compared to storing all entries of $W_i$ explicitly.
    \item Allows efficient reconstruction of $W_i'$ such that the concatenated matrix $W' = W_1' \circ \dots \circ W_m'$ still forms a good approximation to the original matrix $A$ in the sense of low-rank approximation and projection-cost preservation.
    \item Preserves provable guarantees on the Frobenius norm and spectral norm errors, as well as the Loewner ordering of Gram matrices, which are important in downstream tasks such as regression, PCA, and clustering.
\end{enumerate}

The model suggested the global efficient encoding for low-rank approximation depicted in \Cref{alg:global:efficient:encoding}. 

\begin{algorithm}[!htb]
\caption{Global efficient encoding for low-rank approximation}
\label{alg:global:efficient:encoding}
\begin{algorithmic}[1]
\Require Accuracy parameter $\eps < \frac{1}{2}$, matrix $A = Q_1 \circ \ldots \circ Q_m$, matrix $B \in \mathbb{R}^{r \times d}$ that is a $(1+\eps)$-PCP of $A$, matrices $W_1, \ldots, W_m \in \mathbb{R}^{r \times d}$ where each $W_i$ is a $(1+\eps)$-PCP of $Q_i$, rank $k$
\Ensure Encoded matrices $W_1', \ldots, W_m'$

\State Set quantization precision $\eps' \le \frac{\eps}{4 \sqrt{C_\eps}}$, where $C_\eps = \frac{1+\eps}{1-\eps}$
\State Compute the SVD of $B$ and let $V_k \in \mathbb{R}^{d \times k}$ be the top-$k$ right singular vectors
\State Define projection matrices $P_B = V_k V_k^\top$ and $P_B^\perp = I_d - P_B$
\State Store $V_k$ with high precision (e.g., $\OO(\log n)$ bits per entry)

\For{$i = 1$ to $m$}
    \State Compute head coefficients $H_i = W_i V_k$ and store with high precision
    \State Compute tail residual $T_i = W_i P_B^\perp$
    \State Quantize $T_i$ entry-wise to relative precision $\eps'$ to obtain $T_i'$ and store
\EndFor

\For{$i = 1$ to $m$}
    \State Reconstruct $W_i' = H_i V_k^\top + T_i'$
\EndFor

\State \Return $W' = W_1' \circ \ldots \circ W_m'$
\end{algorithmic}
\end{algorithm}

The main idea behind the global encoding is to separate each block $W_i$ into a ``head'' component and a ``tail'' residual with respect to a global low-rank subspace determined by the top-$k$ singular vectors of a global sketch $B$ of $A$. 
Specifically, for a matrix $B$ that is a $(1+\eps)$-PCP of $A$, we compute its top-$k$ right singular vectors $V_k$ and define the projections
\[P_B = V_k V_k^\top, \qquad P_B^\perp = I_d - P_B.\]
The head coefficients $H_i = W_i V_k$ capture the component of $W_i$ in the top-$k$ global subspace, while the tail residual $T_i = W_i P_B^\perp$ contains the remaining information orthogonal to this subspace. 
The tail residuals are then quantized entry-wise to a precision $\eps'$, providing a compressed representation while controlling the additive projection-cost error. Reconstruction simply sums the head and tail contributions:
\[W_i' = H_i V_k^\top + T_i'.\]

This approach ensures that the concatenated reconstructed matrix $W'$ satisfies a $(1 \pm O(m\eps))$ projection-cost preservation guarantee with respect to $A$, while using significantly less storage than the naive concatenation of the $W_i$. 
Moreover, the method preserves Loewner orderings of the Gram matrices up to small additive errors, which is crucial for applications that rely on positive semi-definite approximations. 
The algorithm is thus globally efficient in both space and reconstruction time, while providing strong theoretical guarantees on approximation quality.

Toward establishing correctness of the proposed efficient encoding, the model first suggested the following auxiliary lemma that bounds the Frobenius norm of a projected matrix.

\begin{lemma}
\label{lem:bound:proj}
For any projection matrix $Q$, it holds that
\[\|W_i Q\|_F^2 \le C_\eps \|B Q\|_F^2,\]
where $C_\eps = \frac{1+\eps}{1-\eps}$.
\end{lemma}
\begin{proof}
Since $W_i$ is a $(1+\eps)$-PCP of $Q_i$, we have
\[\|W_i Q\|_F^2 \le (1+\eps)\|Q_i Q\|_F^2.\]
Because $A$ is the concatenation of the matrices $Q_j$, it follows that
\[\|A Q\|_F^2 = \sum_{j=1}^m \|Q_j Q\|_F^2 \ge \|Q_i Q\|_F^2.\]
Moreover, as $B$ is a $(1+\eps)$-PCP of $A$, we obtain
\[\|B Q\|_F^2 \ge (1-\eps)\|A Q\|_F^2.\]
Combining these inequalities yields
\[\|W_i Q\|_F^2 \le \frac{1+\eps}{1-\eps}\|B Q\|_F^2 = C_\eps \|B Q\|_F^2.\]
\end{proof}
\noindent
Using this result, the model then showed that the global efficient encoding results in a PCP with small additive error, which for our downstream applications can ultimately be absorbed into multiplicative error. 
\begin{lemma}
\label{lem:global:pcp}
For every rank-$k$ orthogonal projection matrix $P \in \mathbb{R}^{d \times d}$ and each $i \in [m]$,
\[\|W_i - W_i P\|_F^2 - \eps \|B - B P\|_F^2 \le \|W_i' - W_i' P\|_F^2\le \|W_i - W_i P\|_F^2 + \eps \|B - B P\|_F^2.\]
\end{lemma}
\begin{proof}
Let $Q = I_d - P$ and fix an arbitrary $i \in [m]$. 
It suffices to show that
\[\left|\|W_i' Q\|_F^2 - \|W_i Q\|_F^2\right| \le \eps \|B Q\|_F^2.\]
Define $E_i = W_i' - W_i = T_i' - T_i$. 
Then,
\[\|W_i'Q\|_F^2 = \|(W_i + E_i)Q\|_F^2 = \|W_i Q\|_F^2 + 2\langle W_i Q, E_iQ \rangle_F + \|E_iQ\|_F^2.\]
Consequently,
\[\left|\|W_i'Q\|_F^2 - \|W_i Q\|_F^2\right| = \left|2\langle W_i Q, E_i Q \rangle_F + \|E_i Q\|_F^2\right| \le 2\|W_i Q\|_F \|E_i Q\|_F + \|E_i Q\|_F^2,\]
where the inequality follows from the Cauchy-Schwarz inequality.
Since $Q$ is a projection matrix, we have $\|E_i Q\|_F \le \|E_i\|_F \le \eps' \|T_i\|_F$, which implies
\[\left|\|W_i' Q\|_F^2 - \|W_i Q\|_F^2\right|\le 2\eps' \|W_i Q\|_F \|T_i\|_F + (\eps')^2 \|T_i\|_F^2.\]
By \Cref{lem:bound:proj}, $\|W_i Q\|_F^2 \le C_\eps \|B Q\|_F^2$. Furthermore, since $P_B$ denotes the projection onto the top-$k$ singular vectors of $B$,
\[\|T_i\|_F^2 = \|W_i P_B^\perp\|_F^2\le C_\eps \|B P_B^\perp\|_F^2 \le C_\eps \|B Q\|_F^2.\]
Substituting these bounds yields
\[\left|\|W_i' Q\|_F^2 - \|W_i Q\|_F^2\right|\le C_\eps \bigl(2\eps' + (\eps')^2\bigr)\|B Q\|_F^2.\]
Choosing $\eps' \le \frac{\eps}{4\sqrt{C_\eps}}$ completes the proof.
\end{proof}
Finally, we analyze the space complexity of the global efficient encoding.

\begin{lemma}
\label{lem:pcp:global:time}
The global encoding scheme can be represented using
\[O\left(kd\log n + mrk\log n + mrd\left(\log \frac{1}{\eps'} + \log\log n\right)\right)\]
bits of space in total. 
\end{lemma}
\begin{proof}
The storage cost consists of three components.
First, the matrix $V_k \in \mathbb{R}^{d \times k}$ is stored at high precision, requiring $O(kd \log n)$ bits.
Second, each head coefficient matrix $H_i \in \mathbb{R}^{r \times k}$ is stored at high precision, contributing a total of $O(mrk \log n)$ bits across all $i \in [m]$.
Finally, each tail residual $T_i'$ is stored using entry-wise quantization with precision parameter $\eps'$.
This quantization requires $O\left(\log \frac{1}{\eps'} + \log\log n\right)$ bits per entry, giving a total of
\[O\left(mrd\left(\log \frac{1}{\eps'} + \log\log n\right)\right)\]
bits for all residuals.
Summing these three contributions gives the stated space bound.
\end{proof}

The existence of a global efficient encoding has immediate implications for low-rank approximation. 
In particular, it enables the construction of projection-cost preservation sketches in a single pass over the data, using
\[\frac{k^2}{\eps^2} \cdot \polylog\left(\frac{1}{\eps}, \log(nd\kappa)\right) + \tilde{O}\left(\frac{dk}{\eps^2}\right)\]
words of space and input-sparsity runtime. 
Here, $\kappa$ is a term related to the ``condition number'' of the data stream, which essentially captures the worst-case arrival of the matrix over the data stream, e.g., if ``small'' rows arrive first and the ``larger'' rows arrive later. 
This quantitatively matches the best known offline coreset constructions while improving over previous streaming algorithms by removing extra logarithmic factors, making the space essentially independent of $n$ for low-rank projections. 
Consequently, high-accuracy low-rank approximations are achievable efficiently in streaming settings, closing the gap between streaming and offline performance and enabling fast, memory-efficient dimensionality reduction and randomized matrix computations.

\subsubsection{Even Faster Algorithm for the Chamfer Distance}
The Chamfer distance is a popular quantification of the dissimilarity between point clouds.
For any two $d$-dimensional point sets $A,B$ of sizes up to $n$, the Chamfer distance from $A$ to $B$ is defined as 
\begin{equation*}
\CH(A,B) = \sum_{a \in A} \min_{b \in B} \|a- b\|
\end{equation*}
where $\|\cdot\|$ is the underlying norm defining the distance between two points, such as the Euclidean or Manhattan  distance. 

While the naive algorithm for the Chamfer distance takes $\OO(dn^2)$ time,  recently, \cite{BIJ+23} proposed the first near-linear-time algorithm to approximate the Chamfer distance. Their algorithm works for the underlying norm being $\ell_1$ or $\ell_2$, and outputs an $(1\pm \eps)$-approximation in time $\OO(dn \log(n)/\eps^2)$. When $\eps$ is a constant, this leaves a gap of $\OO(\log n)$ between the upper bound and the trivial $\Omega(dn)$ lower bound.

Towards closing this gap, \cite{FI25} proposed a faster algorithm for the $\ell_1$ norm. This
leads to the natural question of whether the improvement generalizes to the $\ell_2$ norm. \emph{The model} answers this question in affirmative.

\textbf{Initial Improvement in the Low-Dimensional Regime}

As summarized in \Cref{fig:chamfer:query}, with a generic initial prompt, \emph{the model} identifies the open problem and the key ingredient to focus on. It then gives a proof in the $\ell_2$ norm, which improves over the runtime of \cite{BIJ+23} in the low-dimensional regime ($\sqrt{d} \ll \log n/ \log\log n$).

The proof focuses on a data structure called quadtree. A quadtree of depth $t$   is defined by a random offset vector $z \sim [0, 2^{t}]^d$. It assigns every point in $\R^d$ with a sequence of $t$ hash values, using $t$ nested grids shifted by $z$. Concretely, for any $a \in \R^d$ and any integer $k$ such that $0 \leq k \leq t$,  it hashes $h_k(a) := (\lceil\frac{a_1+z_1}{2^k}\rceil, \lceil\frac{a_2+z_2}{2^k}\rceil, $ $\cdots, \lceil\frac{a_d+z_d}{2^k}\rceil)$. 

\cite{FI25} used two independent quadtrees to estimate the nearest neighbor distance for all $a \in A$ in $\ell_1$, i.e. estimating $\min_{b\in A}\|a-b\|_1$ for all $a \in A$. \emph{The model} follows a similar proof structure and analyzes the behavior of quadtrees in $\ell_2$ as follows:

\begin{lemma}\label{lem:1}

For all $a, b \in \R^d$ and $0 \leq k \leq t$, if $h_k(a) = h_k(b)$, then $\|a -b\|_2 \leq 2^k \sqrt{d}$.
\end{lemma}

\begin{proof}
    If $h_k(a) = h_k(b)$, then $a, b$ lie in the same $d$-dimensional cube of side-length $2^k$. The $\ell_2$ diameter of such a cube is $2^k \sqrt{d}$.
\end{proof}
\begin{lemma}\label{lem:2}
    With probability $1 - \OO(1/n)$, the following holds
simultaneously for all $a,b, k$: If $h_k(a) = h_k(b)$, then $\|a-b\|_2 \leq 2^k \cdot 3\log n$.
\end{lemma}
\begin{proof}
    Lemma 3.4 of \cite{FI25} shows that $\forall a,b, k: \|a-b\|_1 \leq 2^k \cdot 3\log n$ with probability  $1 - \OO(1/n)$. Combining this with the fact that $\|a-b\|_1 \geq \|a-b\|_2$ concludes the proof.
\end{proof}

With the above lemmas, \emph{the model} shows that in expectation, two quadtrees output good estimation to $\mathsf{opt}_a := \min_{b\in A}\|a-b\|_2$ for all $a \in A$. The estimator $\mathcal{D}_a$ is defined as: 
\begin{enumerate}
    \item Identifying the smallest $\tilde{k}$ such that $h_{\tilde{k}}(a) = h_{\tilde{k}}(b)$ for some $b \in B$ across the two quadtrees.\footnote{For simplicity, assume that we chose $t$ so that there always exists such $\tilde{k} \leq t$.}
    \item Assigning $\mathcal{D}_a := \|a-b\|_2$ for arbitrary $b$ such that $h_{\tilde{k}}(a) = h_{\tilde{k}}(b)$.
\end{enumerate}

\begin{lemma}\label{lem:3}
    With probability $1 - \OO(1/n)$, it holds for all $a \in A$ that $\E{D_a} \leq F \cdot \mathsf{opt}(a)$, where $F  = \OO(\min(d^{3/2},d\log n))$.
\end{lemma}
\begin{proof}
We need the following fact:
\begin{fact}[\cite{BIJ+23}]
    $\Pr[h_k(a) \neq h_k(b)] \leq \frac{\|a-b\|_1}{2^k}$.
\end{fact}
    Similar to Theorem 3.5 in \cite{FI25}, we condition on Lemma \ref{lem:2} and fix $a \in A$. Let $k^* := \lceil \log(\mathsf{opt}_a)\rceil$.  And let $\mathcal{E}_k$ be the event that $\tilde{k} = k$ (recall that we identified a unique $\tilde{k}$ when defining the estimator). If $\mathcal{E}_k$  occurs, the estimator $\mathcal{D}_a$ is bounded by $2^k \cdot \tilde{F}$ for $\tilde{F} = \OO(\min(\sqrt{d}, \log n))$. Thus 
    \[\E{\mathcal{D}_a} \leq \sum_{0 \leq k \leq t}\Pr[\mathcal{E}_k]\cdot 2^k \cdot \tilde{F} .\]
    We split the sum based on $k \leq k^*$ and $k > k^*$.  The first part is bounded by 
    \[\tilde{F} \mathsf{opt}_a \sum_{k\leq k^*}\Pr[\mathcal{E}_k] \leq  \tilde{F} \mathsf{opt}_a.\]
    For the second part ($k > k^*$), we fix $b^* := \arg\min_{b\in B} \|a-b\|_2$.  Event $\mathcal{E}_k$ implies that $h_{k-1}(a) \neq h_{k-1}(b^*)$ in two independent quadtrees. Therefore,
    \[\Pr[\mathcal{E}_k] \leq (\frac{\|a-b^*\|_1}{2^{k-1}})^2 \leq(\frac{\sqrt{d}\|a-b^*\|_2}{2^{k-1}})^2  = \frac{d(\mathsf{opt}_a)^2}{2^{2k-2}}\]
    because $\|a-b^*\|_1 \leq \sqrt{d}\|a-b^*\|_2$. Now we can bound the summation for $k > k^*$:
    \[\sum_{k > k^*}\Pr[\mathcal{E}_k]\cdot 2^k \cdot \tilde{F}  \leq \tilde{F}  \sum_{k > k^*}\frac{d(\mathsf{opt}_a)^2}{2^{2k-2}} \cdot 2^k \leq 4\tilde{F}d\cdot \mathsf{opt}_a  \sum_{k > k^*}\frac{\mathsf{opt}_a}{2^k} = \OO(\tilde{F}d\mathsf{opt}_a).\]
    Thus, $\E{\mathcal{D}_a} = \OO(\tilde{F}d)\cdot \mathsf{opt}_a$, i.e. the approximation factor is $F = \OO(\min(d^{3/2},d\log n))$.
\end{proof}

As outlined in \Cref{fig:chamfer:query}, the rest of the algorithm in \cite{FI25} works the same regardless of the underlying norm. And the total runtime can be derived from the bound on the approximation factor $F$. This leads to the conclusion (Theorem \ref{thm:1}).

\newpage

\begin{figure}[!htb]
\begin{chatbox}
\begin{userturn}
Please try to improve the paper by identifying and solving an open question from it. 
\vskip 0.1in
For your proof, adhere to: $\langle\text{rigor text}\rangle$.
\end{userturn}
\begin{assistantturn}
\cite{FI25} presents an
improved algorithm for approximating the Chamfer distance in the $\ell_1$ metric, achieving a runtime
of $\OO(dn (\log\log(n) + \log\frac{1}{\eps})/\eps^2)$. However, they note that extending this improvement to the $\ell_2$  metric
is challenging. The standard approach, used in prior work  \cite{BIJ+23} , involves embedding the points
from $\ell_2$ to $\ell_1$, which takes $\OO(dn \log(n)/\eps^2)$ time, thereby eliminating the speedup. The authors leave
finding a faster embedding algorithm, or consequently a faster $\ell_2$ Chamfer distance algorithm, as
an open problem.
\vskip 0.1in
We address this open problem by analyzing the performance of the proposed algorithmic framework when adapted directly to the $\ell_2$ metric. We show that the new techniques can yield a faster algorithm for the $\ell_2$ Chamfer distance when the dimension $d$ is relatively small.

\vskip 0.1in
\textbf{Overview of the Approach:}
The algorithm presented in the paper consists of three main components:
\begin{enumerate}
    \item QuadTree: A fast procedure using quadtrees to compute crude estimates $\mathcal{D}_a$ of $\mathsf{opt}_a := \min_{b \in B} \|a-b\|$. This runs in $\OO(dn\log\log n)$ time.
\item  Tournament: An algorithm to efficiently compute $\OO(1)$-approximations of $\mathsf{opt}_a$ for a batch
of sampled points.
\item  Rejection Sampling: A framework that uses the estimates from QuadTree and Tournament
to perform importance sampling and estimate the Chamfer distance.
\end{enumerate}
The efficiency of the overall algorithm depends on the approximation factor achieved by the
QuadTree step. If QuadTree provides an $F$-approximation in expectation (i.e., $\E{\mathcal{D}_a} \leq F\cdot \mathsf{opt}_a$),
the overall runtime is dominated by the QuadTree and Tournament steps. The Tournament runs in time roughly $\OO(n(d+F/\eps^2)\log (F/\eps^2))$.
\vskip 0.1in
For the $\ell_1$ metric, the authors prove $F = \OO(\min(d, \log n ))$. This leads to the
improved $\OO(dn\log\log n/\eps^2)$ runtime. We analyze the approximation factor $F$ when the framework is applied to the $\ell_2$ metric.
\vskip 0.1in
\textbf{[Proofs to Lemma \ref{lem:1}, \ref{lem:2}, \ref{lem:3}]}

\vskip 0.1in
\textbf{Conclusion:}
We have shown that the techniques introduced in the paper can be adapted directly to the $\ell_2$ metric
to yield a faster algorithm for the Chamfer distance, provided the dimension $d$ is small.

\begin{theorem}\label{thm:1}
There is an algorithm that computes a $(1+\eps)$-approximation to the Chamfer distance
under the $\ell_2$ metric in time $\OO(n(F/\eps^2)(\log(F/\eps^2)+\log\log n))$, where $F = \OO(\min(d^{3/2},d\log n))$.
\end{theorem}
This runtime constitutes an improvement over the previous $\OO(dn\log(n)/\eps^2)$ when $d^{1/2}(\log d +
\log(1/\eps) + \log\log n) = o(\log n)$.

\end{assistantturn}
\end{chatbox}
\caption{Model identification of problem and initial solution.}
\label{fig:chamfer:query}
\end{figure}

\FloatBarrier

\textbf{Refined Improvement for Any Dimension.}
Following the initial response of \emph{the model}, we asked \emph{the model} to focus on improving the high-dimensional case. \emph{the model} refines the previous Lemma \ref{lem:3} and combines it with a fast Johnson-Lindenstrauss Transform. This gives an algorithm for $\ell_2$ that matches the runtime of $\ell_1$ in \cite{FI25} for any dimension $d$, thus improving the previous best runtime from \cite{BIJ+23}.

\begin{lemma}[Refinement of Lemma \ref{lem:3}]\label{lem:4}
    With probability $1 - \OO(1/n)$, it holds for all $a \in A$ that $\E{D_a} \leq F \cdot \mathsf{opt}(a)$, where $F  = \OO(\min(d,\sqrt{d}\log n))$.
\end{lemma}

Compared to Lemma \ref{lem:3}, this proof switches from $\ell_1$ to $\ell_2$ in a later stage. This switch applies $\|\cdot\|_1 \leq \sqrt{d} \|\cdot\|_2$, thus introducing an extra factor of $\sqrt{d}$. Intuitively, we want this to happen after multiplying the probabilities of two independent quadtrees. This is precisely what \emph{the model} changes.

\begin{proof}
 As before, we fix $b^* := \arg\min_{b\in B} \|a-b\|_2$. But now we define $k^* := \lceil \log (\|a-b^*\|_1) \rceil$. We again split
    $\E{\mathcal{D}_a} \leq \sum_{0 \leq k \leq t}\Pr[\mathcal{E}_k]\cdot 2^k \cdot \tilde{F}$ based on $k \leq k^*$ and $k > k^*$, and bound the first part by $\tilde{F} 2^{k^*} \sum_{k \leq k^*} \Pr[\mathcal{E}_k] \leq \tilde{F} \|a-b^*\|_1$. For the second part, 
    \[\sum_{k > k^*}\Pr[\mathcal{E}_k]\cdot 2^k \cdot \tilde{F}  \leq \tilde{F}  \sum_{k > k^*}(\frac{\|a-b^*\|_1}{2^{k-1}})^2 \cdot 2^k \leq 4\tilde{F}\cdot \|a-b^*\|_1  \sum_{k > k^*}\frac{\|a-b^*\|_1}{2^k} = \OO(\tilde{F}\|a-b^*\|_1).\]
    Thus, $\E{\mathcal{D}_a} = \OO(\tilde{F})\cdot \|a-b^*\|_1 \leq \OO(\tilde{F}\sqrt{d})\cdot \mathsf{opt}_a$. The approximation factor is $F = \OO(\min(d,\sqrt{d}\log n))$.
\end{proof}

\FloatBarrier

As mentioned in \Cref{fig:chamfer:refine}, this gives a  $\OO(dn(\log d+\log\log n + \log\frac{1}{\eps})/\eps^2)$. For $d \ll \poly(n)$-time algorithm, this improves the best known runtime of $\ell_2$. Moreover, for $d \leq \poly(\log n)$, this already matches the best known runtime of $\ell_1$.  Finally, to handle larger $d$, \emph{the model} applies fast Johnson-Lindenstrauss transform to reduce the dimension, and then estimate the Chamfer distance. 

Note that while \emph{the model}  claims the existence of an (unconditional) $\OO(dn)$-time JL transform, which we are not aware of, there indeed exists an $\OO(dn \log m)$-time JL transform that reduces the dimension to $m = \OO(\log^2 n /\eps^2)$ \cite{F21}, which suffices for the final conclusion (Theorem \ref{thm:2}).

\newpage
\begin{figure}[!htb]
\begin{chatbox}
\begin{userturn}
Focus on improving the high dimensional regime. Concretely, try (1) using an
$\ell_2$-native scheme, and (2) modifying later stages of the algorithm.
\end{userturn}
\begin{assistantturn}
We begin by rigorously analyzing the performance of the QuadTree estimator when adapted
to the $\ell_2$ metric. This analysis provides a tighter bound than the one presented in the
preamble and establishes the baseline performance.
\vskip 0.1in
\textbf{[Proof to Lemma \ref{lem:4}]}

\vskip 0.1in
This leads to $\OO(dn(\log d+\log\log n + \log\frac{1}{\eps})/\eps^2)$ total runtime. If $d$ is large (e.g., $d = \poly(n)$), the runtime matches
the previous state-of-the-art $\OO(dn\log n/\eps^2)$, offering no substantial improvement.

\vskip 0.1in
Strategy (1) proposes using an $\ell_2$-native approach. We can achieve this by employing
dimensionality reduction via the Johnson-Lindenstrauss (JL) transform before applying the
framework. This reduces the dimension from $d$ to $m = \OO(\log n/ \eps^2)$, thereby reducing the
impact of the dimension on the approximation factor $F$.

\vskip 0.1in
\textbf{Fast Johnson-Lindenstrauss (JL) Transform}: Given $n$
points in $\R^d$ and $\eps > 0$. Let $m = \OO(\log n/ \eps^2)$. There exists a randomized
linear map $S : \R^d \rightarrow \R^m$ such that with high probability, it preserves all pairwise $\ell_2$ distances within $(1 \pm \eps)$. Furthermore, the map $S$ can
be applied to all $n$ points in total time $\OO(dn)$.

\vskip 0.1in
\begin{theorem}\label{thm:2}
Assuming the fast Johnson-Lindenstrauss (JL) transform, there is an algorithm that computes a $(1+\eps)$-approximation to the Chamfer distance
under the $\ell_2$ metric in time $\OO(dn(\log\log n + \log\frac{1}{\eps})/\eps^2)$.
\end{theorem}
[For large $d$, this is justified by plugging in $m$ into the previous time bound.]

\end{assistantturn}
\end{chatbox}
\caption{Model refinement of the analysis and final solution.}
\label{fig:chamfer:refine}
\end{figure}
We have incorporated this result, together with other results on $\ell_p$ versions of Chamfer and in streaming models, into a submission \cite{fw26}. The model was also very helpful for verifying correctness of different parts of the submission. 

\clearpage

\section{Resolving Conjectures and Complex Derivations}
Distinct from the algorithmic optimizations in the previous section, here the AI acts as a heavy-lifting theoretical mathematician. These case studies showcase the AI's ability to build theory from scratch—executing complex continuous derivations, extending theorems via measure theory, and constructing novel computational reductions.

\subsection{Information Theory: The Courtade-Kumar Conjecture}\label{sec:ck_conjecture}
\textit{Written by Adel Javanmard.}

\paragraph{AI Contribution}
The AI made two distinct contributions to the Courtade-Kumar conjecture. First, it successfully generalized a known theorem to unbalanced functions and improved entropy bounds in the high-noise regime using hypercontractivity and Taylor expansions (Section~\ref{sec:CK-part1}). Second, it made structural progress on the ``unsymmetrized'' version of the conjecture by analyzing continuous relaxations and proving local optimality (Section~\ref{sec:CK-part2}).

\subsubsection{Part I: Generalization to Unbalanced Functions}\label{sec:CK-part1}

\paragraph{The Courtade-Kumar Conjecture} An intriguing question in information theory, formalized by Courtade and Kumar \cite{courtade2014boolean} in 2014, asks how to optimally compress a noisy signal into a single bit to preserve information about the original source. Specifically, let $X^n$ be i.i.d Bernoulli$(1/2)$ and let $Y^n$ be a noisy observation of $X^n$, generated by passing the input through a memoryless binary symmetric channel with a crossover probability $0 < \alpha < \frac{1}{2}$. 
The core optimization challenge is to identify a Boolean function $b : \{0, 1\}^n \to \{0, 1\}$ that maximizes the mutual information $I(b(X^n); Y^n)$, or equivalently $I(b(Y^n); X^n)$. Courtade and Kumar conjectured that this quantity is bounded by
\begin{equation}
    I(b(X^n); Y^n) \leq 1 - H(\alpha)\,, \label{eq:ck_bound}
\end{equation}
where $H(\alpha) = -\alpha \log \alpha - (1 - \alpha) \log(1 - \alpha)$ represents the binary entropy function. This upper bound is achieved when $b$ is a ``dictatorship'' function, defined as $b(X_1, \dots, X_n) = X_i$ for any fixed index $i \in \{1, \dots, n\}$. Essentially, the conjecture posits that no Boolean function is more informative than simply selecting the value of a single coordinate.

\paragraph{Recent Progress and Partial Results}
In their original paper~\cite{courtade2014boolean}, the authors established the following result, which is a weaker version\footnote{It is a weaker form of the main conjecture because \eqref{eq:ck_bound} is equivalent to  $I(b(X^n);Y^n)= \sum_{i=1}^n I(Y^{i-1},b(X^n);Y_i)\le 1-H(\alpha)$ by independence of $Y_i$'s.} of the conjecture:
\begin{center}
\fbox{
\begin{minipage}{0.9\textwidth}
\textbf{Theorem 1\cite{courtade2014boolean}:} 
If $b(X^n)$ is {\bf equiprobable}, then 
\begin{align}\label{eq:equiprobable}
\sum_{i=1}^n I(b(X^n);Y_i)\le 1- H(\alpha)\,.
\end{align}
\end{minipage}
}
\end{center}
This is still a significant result given that the left hand side is the sum of $n$ mutual information terms, while the right-hand side does not depend on $n$.
Also, note that Theorem 1 above is only proved in the case of {\bf equiprobable} boolean functions. In Section IV of~\cite{courtade2014boolean}, the authors ask the following question:

\begin{center}
\fbox{
\begin{minipage}{0.9\textwidth}
Does Theorem 1 continue to hold when $b(X^n)$ is not
equiprobable? Unfortunately, our Fourier-analytic proof
of Theorem 1 appears to fail in this setting. Nonetheless,
we feel that establishing this generalization of Theorem 1 should be considerably easier than establishing
the main conjecture.
\end{minipage}
}
\end{center}

Since its introduction, the conjecture~\eqref{eq:ck_bound} has generated substantial interest, leading to several partial results. Chandar and Tchamkerten \cite{chandar2014most} established the bound:
\begin{equation}
    I(b(X^n); Y^n) \leq (1 - 2\alpha)^2\,. \label{eq:chandar}
\end{equation}
Building on this, Ordentlich, Shayevitz, and Weinstein \cite{ordentlich2016improved} derived the following bound for the interval $\frac{1}{2}(1 - \sqrt{\frac{1}{3}}) \leq \alpha \leq \frac{1}{2}$:
\begin{equation}
    I(b(X^n); Y^n) \leq \frac{\log(e)}{2}(1 - 2\alpha)^2 + 9 \left(1 - \frac{\log(e)}{2}\right)(1 - 2\alpha)^4
\end{equation}
This result provides a tighter bound than \eqref{eq:chandar} within the range $\frac{1}{3} \leq \alpha \leq \frac{1}{2}$. Leveraging this improved bound,  Samorodnitsky \cite{samorodnitsky2016entropy} successfully proved that the conjecture holds for all $\alpha \geq \frac{1}{2} - \delta$, where $\delta > 0$ is an absolute constant. While these results represent significant milestones, the main conjecture~\eqref{eq:ck_bound} remains open in its full generality.

\paragraph{Partial Resolution of the Courtade-Kumar Conjecture by Gemini}

We asked Gemini to try to solve the Courtade-Kumar conjecture and provided their paper as input. 

\begin{figure}[!htb]
\begin{chatbox}
\begin{userturn}
Can you prove Conjecture 1 stated in this paper or make progress on it (``Which Boolean Functions Maximize Mutual Information on Noisy Inputs'')?
\end{userturn}

\begin{assistantturn}
We analyze the conjectures and open questions presented in the paper ``Which Boolean Functions Maximize Mutual Information on Noisy Inputs?'' by Courtade and Kumar. The main conjecture (Conjecture 1) remains open. However, we provide a complete resolution to the first open question posed in Section IV, generalizing Theorem 1 to arbitrary (non-balanced) Boolean functions.
\end{assistantturn}
\end{chatbox}
\caption{Model identification of open problem and progress direction}
\label{fig:model_id_ck_1}
\end{figure}

        %
            %
%
%
            %
%
    %

As illustrated in Figure~\ref{fig:model_id_ck_1}, while the model could not solve the main conjecture,  it identified the open problem from Section IV of the paper and provided a complete resolution. Specifically it proves the following theorem, which confirms that Theorem 1 holds even when $b(X^{n})$ is not equiprobable.

\begin{theorem}[Generalized Theorem 1]\label{thm:Gen-thm1}
Let $X^{n}$ be i.i.d. uniform on $\{-1,1\}^n$, and let $Y^{n}$ be the output of a BSC$(\alpha)$ with input $X^{n}$. For any Boolean function $b:\{-1,1\}^{n}\rightarrow\{-1,1\}$,
\[
\sum_{i=1}^{n}I(b(X^{n});Y_{i})\le1-H(\alpha).
\]
\end{theorem}

Gemini's proof employs Fourier-analytic techniques analogous to those in~\cite{courtade2014boolean} to establish Theorem 1, yet Gemini independently refined the derivation of Theorem 1 to address scenarios where the distribution of $b(X^n)$ is non-uniform. For this problem, while we maintained minimal oversight throughout the process, many intermediate steps were cryptically generated, requiring us to query the system multiple times for clarification.

The formal proof of Theorem~\ref{thm:Gen-thm1} is presented below. The proof outline involves an analysis of the optimization landscape, where we demonstrate that the objective function—viewed as a function of the squared level-1 Fourier coefficients—is convex. We then characterize the extreme points of the feasible region defined by Parseval's identity and spectral bounds. Finally, the proof establishes a non-trivial inequality involving derivatives of the mutual information expression with respect to $\rho$. A key distinction from the proof strategy in~\cite[Theorem 1]{courtade2014boolean} is our demonstration that the objective function is maximized by Boolean functions that are monotone in each coordinate (defined via a 1-d compression operator). Using this structure we can impose tighter constraints on the feasible set which allows us to generalize the result of~\cite[Theorem 1]{courtade2014boolean} to all Boolean functions.

\paragraph{Strengthening the High-Noise Regime Results}

We further explored potential improvements to the results established by Samorodnitsky \cite{samorodnitsky2016entropy}, which proved that the Courtade-Kumar conjecture holds in the high-noise regime—specifically, when the noise parameter $\alpha$ is sufficiently close to $1/2$ (i.e., $\lambda = 1-2\alpha < \delta_c$ for an absolute constant $\delta_c > 0$). Initially, when provided with \cite{samorodnitsky2016entropy} as input, the system produced a concise summary of the methodology in \cite{samorodnitsky2016entropy} without offering immediate novel insights. 

\begin{figure}[!htb]
\begin{chatbox}
\begin{userturn}
Can you make further progress on the conjecture using this work?
\end{userturn}

\begin{assistantturn}
We analyze the maximization of a convex functional related to entropy under the noise operator, focusing on the conjecture that the Dictatorship function is the optimizer. We provide a rigorous proof for the high noise regime, following the methodology of A. Samorodnitsky in ``On the entropy of a noisy function''.
\end{assistantturn}
\end{chatbox}
\caption{Model identification of open problem and progress direction}
\label{fig:model_id_ck}
\end{figure}

        %
            %
%
%
            %
%
    %
However, this summary served as a useful primer for distilling the core proof techniques. Through iterative prompting and self-correction cycles, we tasked the system with refining Samorodnitsky’s result. After several rounds of checking its own logic, the model claimed to have established an improvement on the threshold $\delta_c$.

\medskip

\noindent{\bf Technical Verification and Error Correction.} While the initial output appeared structurally sound and cited relevant literature correctly, a manual review revealed subtle technical inaccuracies. Specifically, the system invoked the Bonami-Beckner hypercontractivity inequality with an incorrect assumption:
\begin{quote}
    ``For a function $h$ supported on Fourier levels $\ge k$, and $q \ge 2$, $\|h\|_q \le (\sqrt{q-1})^k \|h\|_2$.''
\end{quote}
 
Our attempts to verify this claim led us to realize that the inequality holds for functions supported on Fourier levels $\le k$, rather than $\ge k$. Upon being prompted to resolve this discrepancy, the system initially defended its original statement. However, after several rounds of probing, it provided a proof that allowed us to pinpoint a specific line where an inequality direction had been flipped. Once this error was identified, the system conceded the mistake—illustrating its tendency to present subtle errors with high confidence.

\medskip

\noindent{\bf Final Result and Methodology.} The system eventually derived an alternative proof that bypassed these initial errors. By incorporating our feedback regarding a seemingly trivial (but actually false) claim, the model produced a rigorous proof that improved upon the result of \cite{samorodnitsky2016entropy}. This final proof technique synthesizes the methodology of \cite{samorodnitsky2016entropy} with higher-moment analysis to establish optimal Fourier concentration for highly informative Boolean functions. We summarize the main result below.

\begin{theorem}[Extended Range for the Conjecture]\label{thm:extended}
There exists an absolute constant $\delta_{opt} > 0$ such that if the noise parameter $\lambda \leq \delta_{opt}$, the Courtade-Kumar conjecture holds. Furthermore, $\delta_{opt}$ is strictly larger than the thresholds previously established in~\cite{samorodnitsky2016entropy}.
\end{theorem}

\medskip

\noindent{\bf Observations on AI-Assisted Research.} Our interactions with the system on these theoretical problems yielded several key insights into its utility as a research partner:

\begin{itemize}
\item {\bf Effective ``Jumpstarting'':} The model is exceptionally adept at identifying relevant ideas across disparate fields and literature, making it an excellent tool for generating an initial line of thought.

\item {\bf The ``Mid-PhD'' Analogy:} Its performance is comparable to a mid-level PhD student; it is most effective when the researcher provides a high-level proof strategy and tasks the system with making the steps rigorous or connecting logical gaps. The results of this part were used to write a full paper~\cite{javanmard2026progress} using a ``vibe proving'' methodology. In this workflow, we defined the lemmas and tasked the AI with drafting the proofs, which were then rigorously verified and refined by us.

\item {\bf Scale and Complexity Limits:} While the system excels at solving isolated, small-scale problems (akin to IMO-style challenges), it still struggles with the expansive, multi-stage reasoning required for long-term research projects. Researchers must remain vigilant, as the system can be confidently incorrect regarding subtle technical details.
\end{itemize}
\paragraph{\Large Proof of Theorem~\ref{thm:Gen-thm1}}

Due to the complexity of this derivation, the AI struggled to generate a complete, flawless proof in a single shot. The formal proof below is the result of our ``vibe-proving'' methodology: it was drafted and rigorously verified by the human researchers, who stitched together the correct intermediate lemmas, structural insights, and Fourier-analytic techniques discovered iteratively by the AI.
\begin{proof}
Recall the crossover probability $0<\alpha<\frac{1}{2}$ and define $\rho=1-2\alpha$.

Any Boolean function $b:\{-1,+1\}^n\to \{-1,1\}$ can be written in terms of its Fourier coefficients as
\[
b(x^n) = \sum_{S\subseteq[n]}\hat{b}(S)\Pi_S(x^n)\,,
\]
where $\Pi_S(x^n) = \Pi_{i\in S} x_i$ are the orthonormal basis for the Fourier transform and $\{\hat{b}(S)\}_{S\subseteq[n]}$ are the Fourier coefficient defined by
\[
\hat{b}(S)= \E b(X^n)\Pi_S(X^n)\,.
\]
For $S=\emptyset$, we define $\Pi_S(x^n)= 1$.

Let $\mu=\E[b]=\hat{b}(\emptyset)$ be the bias, and $z_{i}=\hat{b}(\{i\})$ be the level-1 Fourier coefficients. The objective function is:
\begin{align}\label{eq:Lb}
L(b)=\sum_{i=1}^{n}I(b;Y_{i})=\sum_{i=1}^{n}(H(b)-H(b|Y_{i})).
\end{align}
Since the range of $b$ is $\{-1,+1\}$, the entropy $H(b)$ is determined by the bias $\Pr[b=1] = \frac{1+\E[b]}{2} = \frac{1+\mu}{2}$, hence we can write $H(b)=H(\frac{1+\mu}{2})$. Using  relation (34) from~\cite{courtade2014boolean}, we have
\[
\Pr[b=1|Y_i=y_i] = \frac{1+\mu+\rho y_i z_i}{2}\,,
\]
and since the marginal distribution of $Y_i$ is uniform on $\{-1,+1\}$, we get
\begin{align}
H(b|Y_i) = \frac{1}{2} \left[ H(\Pr[b=1|Y_i=1])+H(\Pr[b=1|Y_i=-1])\,.
\right]
\end{align}
Therefore we obtain $H(b|Y_{i})=h_{\mu}(z_{i}),$ where
\[
h_{\mu}(z):=\frac{1}{2}H\left(\frac{1+\mu+\rho z}{2}\right)+\frac{1}{2}H\left(\frac{1+\mu-\rho z}{2}\right).
\]
Let $g_{\mu}(z)=H(b)-h_{\mu}(z)$. By~\eqref{eq:Lb}, we want to maximize $\sum_{i=1}^{n}g_{\mu}(z_{i})$ subject to the constraints on the Fourier coefficients of a Boolean function.

\bigskip
\noindent\textbf{Step 1: Reduction to Monotone Functions.}
Let $f: \{-1, 1\}^n \to \{-1, 1\}$ be a Boolean function. The 1-d compression operator along coordinate $i$, denoted as $\mathcal{C}_i$, rearranges the values of the function along the $i$-th dimension to make it ``monotone''  while preserving the total number of 1s. Specifically, for any input $x \in \{-1, 1\}^n$, let $x = (x_{\sim i}, x_i)$, where $x_{\sim i}$ are all bits except $i$. The operator acts on the pair of values $(f(x_{\sim i}, -1), f(x_{\sim i}, 1))$:

\begin{equation}
\mathcal{C}_i f(x_{\sim i}, x_i) =
\begin{cases}
    1 & \text{if } x_i = 1 \text{ and } f(x_{\sim i}, -1) + f(x_{\sim i}, 1) \ge 0 \\
    1 & \text{if } x_i = -1 \text{ and } f(x_{\sim i}, -1) + f(x_{\sim i}, 1) = 2 \\
    -1 & \text{otherwise}
\end{cases}
\end{equation}
We show that 1-d compression increases
the objective function $L(b)$. 
\begin{lemma}\label{lem:fourier-prop}
Let $\tilde{b}$ be the compression of $b$ along coordinate $j$, making it non-decreasing in $x_j$. The Fourier coefficients satisfy:
\begin{enumerate}
\item $z_i(\tilde{b}) = z_i(b)$ for $i\neq j$.
\item $|z_j (\tilde{b})| \ge |z_j (b)|$.
\end{enumerate}
\end{lemma}
Note that $h_\mu(z) =  h_\mu(-z)$ and the function $h_{\mu}(z)$ is decreasing in $|z|$, hence by Lemma~\ref{lem:fourier-prop} $h_{\mu}(z_{j}(\tilde{b}))\le h_{\mu}(z_{j}(b))$.

Thus, $\sum_i H(b|Y_i)$ decreases under compression
and $L(b)=\sum_i I(b; Y_i)$ increases. This implies that the maximum must be attained by a function $b$ that is monotone along each coordinate.
We assume $b$ is non-decreasing, so $z_i(b) \ge 0$, for $i\in[n]$, by the next lemma.
\begin{lemma}\label{lem:xi-pos}
Suppose that $b$ is a non-decreasing Boolean function, i.e., for  any input $x$, changing the $i$-th bit from $-1$ to $1$ either increases the function value or keeps it the same: $b(x_{\sim i}, -1)\le b(x_{\sim i}, +1)$. Then $z_i(b):=\hat{b}(\{i\}) \ge 0$.
\end{lemma}

We will also use the following lemma in the next step.

\begin{lemma}\label{lem:boolean-prop}
For any Boolean function $b:\{-1,1\}^n\to\{-1,1\}$, with
$\mu = \E[b]$ and $z_i = \hat{b}(\{i\})$ the level-1 Fourier coefficients, we have $z_i \le 1-|\mu|$.
\end{lemma}

\bigskip
\noindent\textbf{Step 2: Optimization over Monotone Functions.}
We analyze the structure of the optimization problem. Let $w_{i}=z_{i}^{2}$. Define $\psi_{\mu}(w)=g_{\mu}(\sqrt{w})$. It can be shown that $\psi_{\mu}(w)$ is strictly convex for $w>0$ (Equivalently, $h_{\mu}(\sqrt{w})$ is strictly concave in $w$).

We want to maximize the convex function $\sum_{i=1}^{n}\psi_{\mu}(w_{i})$ subject to constraints on the Fourier coefficients. 

The constraints are:
\begin{enumerate}
    \item $w_{i}\ge0$
    \item $\sum_{i=1}^{n}w_{i}\le1-\mu^{2}=R^{2}$ (Parseval's theorem).
    \item $w_{i}\le(1-|\mu|)^{2}=C^{2}$ (Since $0\le z_{i}\le1-|\mu|$ by Lemmas~\ref{lem:xi-pos} and \ref{lem:boolean-prop}).
\end{enumerate}
The maximum of a convex function over this polytope is attained at an extreme point. Let $K=R^{2}/C^{2}=(1-\mu^{2})/(1-|\mu|)^{2}=(1+|\mu|)/(1-|\mu|)$. We analyze two cases based on the relationship between $n$ and $K$. Without loss of generality, assume $\mu\ge0$. 

\bigskip
\noindent\textbf{Step 3: Bounding the Maximum (Case 1: $n\ge K$).}
If $n\ge K,$ the constraint $\sum w_{i}\le R^{2}$ is the dominant constraint. The extreme points (up to permutation) have $k=\lfloor K\rfloor$ coordinates equal to $C^{2}$, one coordinate equal to $\theta C^{2}$ (where $\theta=K-k)$, and the rest are $0$. The maximum value is:
\[
M(\mu)=k\psi_{\mu}(C^{2})+\psi_{\mu}(\theta C^{2}).
\]
We verify that $\psi_{\mu}(0)=0.$ Note that $\psi_{\mu}(0)=g_{\mu}(0)=H(b)-h_{\mu}(0)$. Since $h_{\mu}(0)= H(\frac{1+\mu}{2})=H(b)$, we have $\psi_{\mu}(0)=0$. Using the convexity of $\psi_{\mu}(w)$ and $\psi_{\mu}(0)=0$:
\[
\psi_{\mu}(\theta C^{2})=\psi_{\mu}(\theta C^{2}+(1-\theta)0)\le\theta\psi_{\mu}(C^{2})+(1-\theta)\psi_{\mu}(0)=\theta\psi_{\mu}(C^{2}).
\]
Thus, the maximum is bounded by:
\[
M(\mu)\le k\psi_{\mu}(C^{2})+\theta\psi_{\mu}(C^{2})=(k+\theta)\psi_{\mu}(C^{2})=K\psi_{\mu}(C^{2})=Kg_{\mu}(C).
\]
Recall that $C = 1-\mu$ and $K = (1+\mu)/(1-\mu)$, and so $\mu$ and $C$ are determined by $K$ as follows:
\[
\mu = \frac{K-1}{K+1}\,, \quad C = \frac{2}{K+1}\,.
\]
Let $M_{K}(\rho):= Kg_{\mu}(C),$ where we made the dependence on $\rho = 1-2\alpha$ explicit in the notation. We aim to show $M_{K}(\rho)\le1-H(\alpha)$.

Writing $M_K(\rho)$ more explicitly, we have
\begin{align*}
M_K(\rho) &= K H(b) - \frac{K}{2}H\left(\frac{1+\mu+\rho C}{2}\right)-\frac{K}{2}H\left(\frac{1+\mu-\rho C}{2}\right)\\
&= K H\left(\frac{K}{K+1}\right)-\frac{K}{2}H\left(\frac{K+\rho}{K+1}\right)-\frac{K}{2}H\left(\frac{K-\rho}{K+1}\right)\,,
\end{align*}
where we used that $H(b) = H(\frac{1+\mu}{2}) = H(\frac{K}{K+1})$.
Note that $M_{1}(\rho)=1-H(\frac{1+\rho}{2})=1-H(\alpha).$ We want to show $M_{K}(\rho)\le M_{1}(\rho)$ for $K\ge1$.

Observe that $M_{K}(0)=0$ for all $K$, so it suffices to show that the derivative with respect to $\rho$ satisfies $M_{K}^{\prime}(\rho)\le M_{1}^{\prime}(\rho)$.

The derivative is calculated as:
\[
M_{K}^{\prime}(\rho)=\frac{K}{2(K+1)\ln 2}\ln\left(\frac{(1+\rho)(K+\rho)}{(1-\rho)(K-\rho)}\right).
\]
Let $A=\frac{1+\rho}{1-\rho}$. Then $M_{1}^{\prime}(\rho)=\frac{1}{2\ln 2}\ln(A)$. The inequality $M_{K}^{\prime}(\rho)\le M_{1}^{\prime}(\rho)$ is equivalent to:
\[
\frac{K}{K+1}\ln\left(A\frac{K+\rho}{K-\rho}\right)\le \ln A.
\]
Rearranging this inequality yields:
\[
K \ln\left(\frac{K+\rho}{K-\rho}\right)\le \ln A \Leftrightarrow K \ln\left(\frac{1+\rho/K}{1-\rho/K}\right)\le \ln\left(\frac{1+\rho}{1-\rho}\right).
\]
Let $f(x)=\frac{1}{x}\ln(\frac{1+x}{1-x})$. The inequality is equivalent to $f(\rho/K)\le f(\rho)$. The Taylor series is $f(x)=2\sum_{j=0}^{\infty}\frac{x^{2j}}{2j+1}$, which is strictly increasing for $x>0$. Since $K\ge1$, we have $\rho/K\le\rho$, thus $f(\rho/K)\le f(\rho)$.

The inequality $M_{K}^{\prime}(\rho)\le M_{1}^{\prime}(\rho)$ holds. Integrating from $\rho=0$ yields $M_{K}(\rho)\le M_{1}(\rho)=1-H(\alpha)$.

\bigskip
\noindent\textbf{Step 4: Bounding the Maximum (Case 2: $n<K$).}
If $n<K$, then $nC^{2}< R^{2}$. The Parseval constraint $\sum w_{i}\le R^{2}$ is inactive. We maximize the convex function $\sum\psi_{\mu}(w_{i})$ subject to the box constraints $0\le w_{i}\le C^{2}$. Therefore the maximum is attained at an extreme point. In addition, $\psi_\mu(C^2) = g_\mu(C)\ge 0$ (as it represents mutual information). Also, $\psi_\mu(0) = 0$, and so $\psi_\mu(C^2)\ge \psi_\mu(0)$. Therefore, the maximum is attained when $w_{i}=C^{2}$ for all $i\in[n]$.

The objective value is:
\[
\sum_{i=1}^n \psi_{\mu}(w_i)= n\psi_{\mu}(C^{2})=ng_{\mu}(C)<Kg_{\mu}(C)=M_{K}(\rho),
\]
Since $g_{\mu}(C)\ge0$ (as it represents mutual information) and $n<K$. From Step 3, we established $M_{K}(\rho)\le1-H(\alpha)$. 

In both cases, the total mutual information is bounded by $1-H(\alpha)$.
\end{proof}
\paragraph{Proof of Lemma~\ref{lem:fourier-prop}}
We first prove that $x_i(\tilde{b}) = x_i(b)$ for $i \neq j$.

Recall the definition of the Fourier coefficient $z_i(b) = \E [b(X^n) X_i]$.
Since $i\neq j$, $X_i$ is independent of $X_j$.
Let $X_{\sim j}$ denote the bits of $X^n$ excluding index $j$. We write $$z_i(b) = \mathbb{E}_{X_{\sim j}} \left[ \frac{X_i}{2} \left( b(X_{\sim j}, -1) + b(X_{\sim j}, 1) \right) \right]$$

Now consider the compressed function $\tilde{b}$. By definition, compression locally rearranges the pair $(b(x_{\sim j}, -1), b(x_{\sim j}, 1))$ into $(\tilde{b}(x_{\sim j}, -1), \tilde{b}(x_{\sim j}, 1))$ such that the sum is preserved :$$b(x_{\sim j}, -1) + b(x_{\sim j}, 1) = \tilde{b}(x_{\sim j}, -1) + \tilde{b}(x_{\sim j}, 1)$$

Substituting this back into the expectation:$$\begin{aligned}
z_i(b) &= \mathbb{E}_{X_{\sim j}} \left[ \frac{X_i}{2} \left( \tilde{b}(X_{\sim j}, -1) + \tilde{b}(X_{\sim j}, 1) \right) \right] \\
&= \mathbb{E} [\tilde{b}(X^n) X_i] = z_i(\tilde{b})
\end{aligned}$$
Thus, $z_i(\tilde{b}) = z_i(b)$ for $i \neq j$.

We next show that $|z_j(\tilde{b})| \ge |z_j(b)|$. Expand the Fourier coefficient for index $j$ by conditioning on $X_{\sim j}$:
$$z_j(b) = \frac{1}{2} \mathbb{E}_{X_{\sim j}} \left[ b(X_{\sim j}, 1) - b(X_{\sim j}, -1) \right]$$
By definition, compression locally rearranges the pair $(b(x_{\sim j}, -1), b(x_{\sim j}, 1))$ into $(\tilde{b}(x_{\sim j}, -1), \tilde{b}(x_{\sim j}, 1))$. Therefore, $|b(x_{\sim j}, 1) - b(x_{\sim j}, -1)| = \tilde{b}(x_{\sim j}, 1) - \tilde{b}(x_{\sim j}, -1)$ for any value of $x_{\sim j}$. Hence,
\begin{align*}
|z_j(b)|& = \frac{1}{2} \left|\mathbb{E}_{X_{\sim j}} \left[ b(X_{\sim j}, 1) - b(X_{\sim j}, -1) \right]\right|\\
&\le  \frac{1}{2} \mathbb{E}_{X_{\sim j}} \left[ |b(X_{\sim j}, 1) - b(X_{\sim j}, -1) |\right]\\
&= \frac{1}{2} \mathbb{E}_{X_{\sim j}} \left[ |\tilde{b}(X_{\sim j}, 1) - \tilde{b}(X_{\sim j}, -1) |\right]\\
&= \frac{1}{2} \mathbb{E}_{X_{\sim j}} \left[ \tilde{b}(X_{\sim j}, 1) - \tilde{b}(X_{\sim j}, -1) \right] = z_j(\tilde{b})\,,
\end{align*}
where in the first step we used the convexity of absolute value ($|\E[Y]|\le \E[|Y|]$). In the third equality, we used the assumption  that  $\tilde{b}$ is non-decreasing in the $j$-th bit. Note that this also implies that $z_j(\tilde{b})\ge 0$, which completes the proof.

\paragraph{Proof of Lemma~\ref{lem:xi-pos}}
We can expand the expectation by conditioning on the value of the random variable $X_i$. Since $X$ is uniform, $X_i$ takes values $-1$ and $+1$ with probability $1/2$ each.

\begin{equation}
\begin{aligned}
z_i(b) = \mathbb{E}[b(X^n) X_i] = \frac{1}{2} \mathbb{E}[b(X^n) X_i \mid X_i = -1] + \frac{1}{2} \mathbb{E}[b(X^n) X_i \mid X_i = 1]
\end{aligned}
\end{equation}
Substitute the value of $X_i$ into the expression:
\begin{equation}
\begin{aligned}
z_i(b) &= \frac{1}{2} \mathbb{E}[b(X_{\sim i}, -1) \cdot (-1)] + \frac{1}{2} \mathbb{E}[b(X_{\sim i}, 1) \cdot (1)] \\
&= \frac{1}{2} \mathbb{E}_{x_{\sim i}} \left[ b(x_{\sim i}, 1) - b(x_{\sim i}, -1) \right]
\end{aligned}
\end{equation}

    We are given that $b$ is non-decreasing bit-wise. Therefore, for every possible setting of the other bits $x_{\sim i}$, we have
$b(x_{\sim i}, 1) \ge b(x_{\sim i}, -1)$.
    Thus, $z_i(b)\ge 0$.

\paragraph{Proof of Lemma~\ref{lem:boolean-prop}}
Let $E_+ = \mathbb{E}[b(X) \mid X_i = +1]$ and $E_- = \mathbb{E}[b(X^n) \mid X_i = -1]$
respectively denote the average value of the function when the $i$-th bit is fixed to $+1$ and $-1$.
Since $b$ outputs values in $\{-1, 1\}$, we have $-1 \le E_+, E_- \le 1$. In addition, we have
\begin{align*}
\mu &= \frac{1}{2}E_+ + \frac{1}{2}E_-\\
z_i &= \mathbb{E}[b(X)X_i] = \frac{1}{2}E_+ \cdot 1 + \frac{1}{2}E_- \cdot (-1) = \frac{E_+ - E_-}{2}
\end{align*}
Hence $\mu+z_i = E_+$ and $\mu-z_i  = E_-$ and the claim follows from the fact that $E_+\le 1$ and $E_-\ge -1$.
\newpage

\paragraph{\Large Proof of Theorem~\ref{thm:extended}}
\paragraph{Preliminaries and Notation}

We analyze functions on the Boolean hypercube $\{-1, 1\}^n$. We follow the setup used in \cite{samorodnitsky2016entropy}. Let $f: \{-1, 1\}^n \to \R$ be a bounded, non-negative function, normalized such that $\E[f] = 1$. We assume $\|f\|_\infty \leq M$. (When analyzing a Boolean function $b$, we typically look at the indicator function of the event $b(x)=1$).

Let $\alpha \in [0, 1/2]$ be the noise parameter. 
The noise operator $T_\alpha$ acts on $f$ as $(T_\alpha f)(x) = \E_y[f(y)]$, where $y$ is obtained from $x$ by flipping coordinates with probability $\alpha$. Let $\rho = 1-2\alpha$. In the Fourier domain, $T_\alpha$ acts as a multiplier:
$$ \widehat{T_\alpha f}(S) = \rho^{|S|} \hat{f}(S). $$
We define the noise parameter $\lambda = \rho^2 = (1-2\alpha)^2$. We consider the high noise regime, where $\lambda$ is small.

We decompose $f$ into its even part $f_0$ and odd part $f1$:
$$ f_0(x) = \frac{f(x) + f(-x)}{2}, \quad f_1(x) = \frac{f(x) - f(-x)}{2}. $$
Note that $\E[f_0] = 1$.

We define the noisy versions:
$$ F = T_\alpha f_0, \quad Z = T_\alpha f_1. $$
Note that $T_\alpha f = F + Z$. We define $Y = F - 1$. Since $\E[F] = 1$, we have $\E[Y] = 0$.

The Fourier coefficients satisfy $\hat{F}(S) = \lambda^{|S|/2} \hat{f_0}(S)$ and $\hat{Z}(S) = \lambda^{|S|/2} \hat{f_1}(S)$. $F$ (and $Y$) are supported on even Fourier levels, while $Z$ is supported on odd levels.

The level 1 Fourier weight is $L_1(f) = L_1(f_1)$.

We utilize the entropy decomposition established in \cite{samorodnitsky2016entropy} (Lemma 6.2):

\begin{lemma}[Entropy Decomposition]\label{lem:entropy_decomp}
$$ \Ent(T_\alpha f) \leq \Ent(F) + \frac{1}{2 \ln 2} \E\left[\frac{Z^2}{F}\right] + O\left(\E\left[\frac{Z^4}{F^3}\right]\right). $$
\end{lemma}

It is known that $\Ent(F) = O(\lambda^2)$ (Lemma 5.4 in \cite{samorodnitsky2016entropy}).

We also rely on a crucial dominance property.
\begin{lemma}[Dominance]\label{lem:dominance}
$F \geq 0$ and $|Z| \leq F$.
\end{lemma}
\begin{proof}
Note that by definition of the even and odd parts of $f$, we have
\begin{align*}
f_0(x)+f_1(x) &= 2f(x)\\
f_0(x)-f_1(x) &= 2f(-x)
\end{align*}
Since $f \geq 0$, we have $|f_1(x)| \leq f_0(x)$. The noise operator $T_\alpha$ is a positive operator, preserving non-negativity and dominance. Thus $|Z| = |T_\alpha f_1| \leq T_\alpha |f_1| \leq T_\alpha f_0 = F$.
\end{proof}

\paragraph{Hypercontractivity and Moment Bounds}

We rely on the Bonami-Beckner Hypercontractivity Theorem. A standard application is for homogeneous polynomials.

\begin{corollary}[Hypercontractivity for Homogeneous Polynomials]\label{cor:hyper_homog}
Let $h_k$ be a homogeneous polynomial of degree $k$. For $q \geq 2$,
$$ \|h_k\|_q \leq (\sqrt{q-1})^k \|h_k\|_2. $$
\end{corollary}

We must be careful when applying this to functions supported on multiple levels. We use the Minkowski inequality combined with Corollary \ref{cor:hyper_homog} to establish rigorous moment bounds.

\begin{lemma}[Moment Bounds]\label{lem:moment_bounds}
Assuming $f$ is bounded by $M$, the following bounds hold as $\lambda \to 0$:
\begin{enumerate}
    \item $\E[Y^2] = O(\lambda^2)$.
    \item $\E[|Y|^3] = O(\lambda^3)$.
    \item $\E[Z^2] = \lambda L_1(f) + O(\lambda^3)$.
    \item $\E[Z^4] = O(\lambda^2)$.
    \item $\E[Z^2Y] = O(\lambda^2)$.
\end{enumerate}
\end{lemma}

\begin{proof}
We assume $\lambda$ is small enough (e.g., $3\lambda < 1$). Since $f$ is bounded, its $L_2$ norm is also bounded ($\|f\|_2 \leq M$).
\medskip

\noindent\textbf{Part 1: $\E[Y^2] = O(\lambda^2)$.}
$Y$ is supported on even levels $\geq 2$.
\begin{align*}
\E[Y^2] &= \sum_{k\geq 2, \text{even}} \lambda^k \sum_{|S|=k} \hat{f_0}(S)^2 \\
&\leq \lambda^2 \sum_{S \neq \emptyset} \hat{f_0}(S)^2 = \lambda^2 \text{Var}(f_0) = O(\lambda^2).
\end{align*}
\medskip

\noindent
\textbf{Part 2: $\E[|Y|^3] = O(\lambda^3)$.}
We analyze $Y=Y_2+Y_{\geq 4}$ using the $L_3$ norm.
\begin{enumerate}
 \item Bounding $\|Y_2\|_3$. Use $q=3, k=2$. The constant is 2.
$$ \|Y_2\|_3 \leq 2\|Y_2\|_2. $$
$\|Y_2\|_2^2 = O(\lambda^2)$. Thus, $\|Y_2\|_3 = O(\lambda)$.

\item Bounding $\|Y_{\geq 4}\|_3$.
$$ \|Y_{\geq 4}\|3 \leq \sum{k\geq 4, \text{even}} \|Y_k\|3 \leq \sum{k\geq 4, \text{even}} M (\sqrt{2\lambda})^k = O(\lambda^2). $$

\item Combining the bounds. $\|Y\|_3 = O(\lambda)$. Therefore, $\E[|Y|^3] = O(\lambda^3)$.
\end{enumerate}
\medskip

\noindent
\textbf{Part 3: $\E[Z^2] = \lambda L_1(f) + O(\lambda^3)$.}
$Z$ is supported on odd levels $\geq 1$.
\begin{align*}
\E[Z^2] &= \lambda L_1(f_1) + \sum_{k\geq 3, \text{odd}} \lambda^k \sum_{|S|=k} \hat{f_1}(S)^2 \\
&= \lambda L_1(f) + O(\lambda^3).
\end{align*}
\medskip

\noindent
\textbf{Part 4: $\E[Z^4] = O(\lambda^2)$.}
We decompose $Z = Z_1 + Z_{\geq 3}$. We use the Minkowski inequality for the $L_4$ norm: $\|Z\|4 \leq \|Z_1\|_4 + \|Z{\geq 3}\|_4$.
\begin{enumerate}
\item Bounding $\|Z_1\|_4$. $Z_1$ is degree 1. Use Corollary \ref{cor:hyper_homog} ($q=4, k=1$).
$$ \|Z_1\|_4 \leq \sqrt{3} \|Z_1\|_2. $$
$\|Z_1\|_2^2 = O(\lambda)$. Thus, $\|Z_1\|_4 = O(\lambda^{1/2})$.

\item Bounding $\|Z_{\geq 3}\|_4$. Use Minkowski inequality on $Z_{\geq 3} = \sum_{k\geq 3, \text{odd}} Z_k$.
$$ \|Z_{\geq 3}\|4 \leq \sum{k\geq 3, \text{odd}} \|Z_k\|_4. $$
Applying hypercontractivity to each $Z_k$:
$$ \|Z_k\|_4 \leq (\sqrt{3})^k \|Z_k\|_2. $$
Since $\|\hat{f}_k\|_2 \leq M$, $\|Z_k\|_2 \leq M \lambda^{k/2}$.
$$ \|Z_k\|_4 \leq M (\sqrt{3\lambda})^k. $$
We sum this geometric series:
$$ \|Z_{\geq 3}\|4 \leq M \sum{k\geq 3, \text{odd}} (\sqrt{3\lambda})^k = M \frac{(\sqrt{3\lambda})^3}{1-3\lambda} = O(\lambda^{3/2}). $$

\item Combining the bounds.
$$ \|Z\|_4 \leq O(\lambda^{1/2}) + O(\lambda^{3/2}) = O(\lambda^{1/2}). $$
Therefore, $\E[Z^4] = \|Z\|_4^4 = O(\lambda^2)$.
\end{enumerate}
\medskip

\noindent\textbf{Part 5: $\E[Z^2Y] = O(\lambda^2)$.}
We use the Cauchy-Schwarz inequality:
$$ |\E[Z^2Y]| \leq \sqrt{\E[Z^4]\E[Y^2]}. $$
Using Part 4 and Part 1:
$$ |\E[Z^2Y]| \leq \sqrt{O(\lambda^2) \cdot O(\lambda^2)} = O(\lambda^2). $$
\end{proof}

\paragraph{Higher Order Term}

We now bound the higher order term in the entropy decomposition (Lemma \ref{lem:entropy_decomp}).

\begin{lemma}[Higher Order Term]\label{lem:higher_order}
$\E[Z^4/F^3] = O(\lambda^2)$.
\end{lemma}
\begin{proof}
We use a constant threshold $\Delta = 1/2$ to split the expectation.
$$ \E\left[\frac{Z^4}{F^3}\right] = \E\left[\frac{Z^4}{F^3} \mathbb{I}{F \leq 1/2}\right] + \E\left[\frac{Z^4}{F^3} \mathbb{I}{F > 1/2}\right]. $$

\textbf{Term 2 (F is large):}
When $F > 1/2$, $1/F^3 < 8$.
$$ \E\left[\frac{Z^4}{F^3} \mathbb{I}_{F > 1/2}\right] \leq 8 \E[Z^4]. $$
By Lemma \ref{lem:moment_bounds}(4), $\E[Z^4] = O(\lambda^2)$. So Term 2 is $O(\lambda^2)$.

\textbf{Term 1 (F is small):}
By Lemma \ref{lem:dominance}, $|Z| \leq F$. So $Z^4/F^3 \leq F$.
$$ \E\left[\frac{Z^4}{F^3} \mathbb{I}{F \leq 1/2}\right] \leq \E[F \mathbb{I}{F \leq 1/2}]. $$
Since $F \leq 1/2$ on this indicator,
$$ \E[F \mathbb{I}_{F \leq 1/2}] \leq \frac{1}{2} P(F \leq 1/2). $$
$P(F \leq 1/2) = P(Y \leq -1/2)$.
By Chebyshev's inequality and Lemma \ref{lem:moment_bounds}(1):
$$ P(|Y| \geq 1/2) \leq 4 \E[Y^2] = O(\lambda^2). $$
So Term 1 is $O(\lambda^2)$.
\end{proof}

\paragraph{Optimal Asymptotic Entropy Bound}

We now prove the main technical result of this section, achieving the optimal $O(\lambda^2)$ error bound by employing a direct Taylor expansion.

\begin{theorem}[Optimal Asymptotic Entropy Bound]\label{thm:optimal_entropy_intro}
For any bounded nonnegative non-zero function $f$ with $\E[f] = 1$:
$$ \Ent(T_\alpha f) \leq \left(\frac{1}{2 \ln 2}L_1(f)\right) \cdot \lambda + O(\lambda^2). $$
\end{theorem}
The proof, detailed below, relies on a direct Taylor expansion approach combined with rigorous moment bounds derived using hypercontractivity. Crucially, we employ a careful application of the Minkowski inequality over the Fourier decomposition to establish the necessary higher-moment bounds. The optimality of this bound is demonstrated by the fact that the $O(\lambda^2)$ error rate is tight for Boolean dictatorship functions.

\begin{proof}
We start from the entropy decomposition (Lemma \ref{lem:entropy_decomp}):
$$ \Ent(T_\alpha f) \leq \Ent(F) + \frac{1}{2 \ln 2} \E\left[\frac{Z^2}{F}\right] + O\left(\E\left[\frac{Z^4}{F^3}\right]\right). $$
We know $\Ent(F) = O(\lambda^2)$ and by Lemma \ref{lem:higher_order}, the error term is $O(\lambda^2)$. We analyze $B = \E[Z^2/F]$.

We use the Taylor expansion of $1/F = 1/(1+Y)$ around $Y=0$ with the exact remainder:
$$ \frac{1}{1+Y} = 1 - Y + \frac{Y^2}{1+Y}. $$
This holds when $F = 1+Y > 0$. Let $S_+ = \{x : F(x) > 0\}$.

$$ B = \E\left[\frac{Z^2}{F} \mathbb{I}{S+}\right] = \E\left[Z^2 \left(1 - Y + \frac{Y^2}{F}\right) \mathbb{I}{S+}\right]. $$

\textbf{Term 1 (Expansion):} $E_T = \E[Z^2(1-Y)\mathbb{I}{S+}]$.
By Lemma \ref{lem:dominance}, if $F(x)=0$, then $Z(x)=0$. The indicator is redundant.
$$ E_T = \E[Z^2] - \E[Z^2Y]. $$
By Lemma \ref{lem:moment_bounds}(3) and (5):
\begin{align*}
E_T &= (\lambda L_1(f) + O(\lambda^3)) - O(\lambda^2) = \lambda L_1(f) + O(\lambda^2).
\end{align*}

\textbf{Term 2 (Remainder):} $R_T = \E\left[\frac{Z^2Y^2}{F} \mathbb{I}{S+}\right]$.
Using dominance $|Z| \leq F$, $Z^2/F \leq F$.
$$ R_T \leq \E[F Y^2 \mathbb{I}{S+}] \leq \E[F Y^2]. $$
$$ \E[F Y^2] = \E[(1+Y)Y^2] = \E[Y^2] + \E[Y^3]. $$
By Lemma \ref{lem:moment_bounds}(1) and (2), $R_T = O(\lambda^2) + O(\lambda^3) = O(\lambda^2)$.

Combining the terms for $B$:
$$ B = \lambda L_1(f) + O(\lambda^2). $$

Substituting this back into the entropy decomposition:
\begin{align*}
\Ent(T_\alpha f) &\leq O(\lambda^2) + \frac{1}{2 \ln 2}(\lambda L_1(f) + O(\lambda^2)) + O(\lambda^2) \\
&= \frac{L_1(f)}{2 \ln 2} \lambda + O(\lambda^2).
\end{align*}
We next show that the $O(\lambda^2)$ error term is asymptotically tight. We demonstrate this by analyzing the dictatorship functions which are believed to maximize the entropy (and mutual information).

Consider the Boolean dictatorship function $f(x)=x_1$. The mutual information is exactly the channel capacity, $I(f(X); Y) = 1 - H(\alpha)$, where $H(\alpha)$ is the binary entropy function (in bits). We analyze the Taylor expansion of the capacity around $\alpha=1/2$ (where $\rho=0$).

Recall $\rho = 1-2\alpha$ and $\lambda = \rho^2$. The expansion of the binary entropy function $H(\alpha)$ around $\alpha=1/2$ yields
$$
H(\alpha) = 1 - \frac{1}{2\ln 2}\rho^2 - \frac{1}{12\ln 2}\rho^4 + O(\rho^6).
$$
Therefore, the capacity is:
$$
1 - H(\alpha) = \frac{1}{2\ln 2}\lambda + \frac{1}{12\ln 2}\lambda^2 + O(\lambda^3).
$$
For the dictatorship function, $L_1(f)=1$. Thus, the expansion matches the form in the theorem statement:
$$
I(f(X); Y) = \left(\frac{L_1(f)}{2\ln 2}\right)\lambda + \Theta(\lambda^2).
$$
Since the expansion for the maximizing function includes a $\Theta(\lambda^2)$ term, the general upper bound established in the theorem statement cannot be asymptotically improved beyond $O(\lambda^2)$.
\end{proof}

\paragraph{Implications and Applications}

The optimal error bound established in the previous section has significant consequences for the structural properties of highly informative functions and the range of validity of the Courtade-Kumar conjecture.

\paragraph{Linear Fourier Concentration}

The $O(\lambda^2)$ error bound allows us to derive the strongest possible Fourier concentration result within this asymptotic framework.

\begin{theorem}[Linear Fourier Concentration]\label{thm:linear_concentration_intro}
Let $f$ be a Boolean function satisfying $c \leq \E f \leq 1/2$ (for some absolute constant $c > 0$). If $I(f(X); Y) \geq 1 - H(\alpha)$, then
$$ \sum_{|S|\geq 2} \hat{f}(S)^2 = O(\lambda). $$
\end{theorem}

\begin{proof}
Let $p = \E f$. We utilize the connection: $I(f(X); Y) = \Ent(T_\alpha f) + \Ent(T_\alpha(1-f))$.

We apply Theorem \ref{thm:optimal_entropy_intro} by normalizing $f$. Using $\Ent(h) = \E[h] \cdot \Ent(h/\E[h])$, we have $\Ent(T_\alpha f) = p \cdot \Ent(T_\alpha(f/p))$. Applying the theorem to $f/p$:
\begin{align*}
\Ent(T_\alpha f) &\leq p \left(\frac{L_1(f/p)}{2 \ln 2} \lambda + O(\lambda^2)\right) \\
&= \frac{L_1(f)}{p(2 \ln 2)}\lambda + O(\lambda^2).
\end{align*}
The $O(\lambda^2)$ constant remains controlled because $p$ is bounded away from 0 by $c$. Applying the same logic to $1-f$, and noting $L_1(f) = L_1(1-f)$:
\begin{align*}
I(f(X); Y) &\leq \left(\frac{1}{p} + \frac{1}{1-p}\right) \frac{L_1(f)}{2 \ln 2} \lambda + O(\lambda^2) \\
&= \frac{1}{p(1-p)} \frac{L_1(f)}{2 \ln 2} \lambda + O(\lambda^2).
\end{align*}

We use the Taylor expansion of the channel capacity: $1 - H(\alpha) = \frac{\lambda}{2 \ln 2} - O(\lambda^2)$. Combining this with the assumption $I(f(X); Y) \geq 1 - H(\alpha)$:
$$ \frac{\lambda}{2 \ln 2} + O(\lambda^2) \leq \frac{1}{p(1-p)} \frac{L_1(f)}{2 \ln 2} \lambda + O(\lambda^2). $$

Dividing by $\lambda/(2 \ln 2)$:
$$ 1 + O(\lambda) \leq \frac{L_1(f)}{p(1-p)} + O(\lambda). $$
Rearranging:
$$ p(1-p) - O(\lambda) \leq L_1(f). $$

By Parseval's identity, $\sum_{S\neq\emptyset} \hat{f}(S)^2 = p(1-p)$. Therefore, the Fourier weight on levels 2 and higher is:
$$ \sum_{|S|\geq 2} \hat{f}(S)^2 = p(1-p) - L_1(f) \leq O(\lambda). $$
\end{proof}

This linear concentration is a significant improvement over the $O(\lambda^{1/3})$ bound established in \cite{samorodnitsky2016entropy}.

\paragraph{Extended Range for the Conjecture}

The improved Fourier concentration directly translates to proving the Courtade-Kumar conjecture for a wider range of the noise parameter $\lambda$.

We are now ready to prove Theorem~\ref{thm:extended}. We recall the theorem below.

\begin{theorem}[Main Result - Extended Range]\label{thm:extended_range}
There exists an absolute constant $\delta_{opt} > 0$ such that if the noise parameter $\lambda \leq \delta_{opt}$, the Courtade-Kumar conjecture holds. Furthermore, $\delta_{opt}$ is strictly larger than the threshold established using the $O(\lambda^{4/3})$ error bound in \cite{samorodnitsky2016entropy}.
\end{theorem}

\begin{proof}
The proof strategy follows the established path in \cite{samorodnitsky2016entropy}, relying on the combination of Fourier concentration and structural theorems for Boolean functions.

Theorem \ref{thm:linear_concentration_intro} implies that a highly informative function $f$ has most of its Fourier mass on level 1. Define $\xi:= \sum_{|S|\geq 2} \hat{f}(S)^2$. 
Theorem 5.5 in \cite{samorodnitsky2016entropy} at high level states that such a function must be close to a dictatorship. Concretely, by this theorem we have
\begin{align}
|\E[f]| &= O(\xi \sqrt{\ln(1/\xi)}),\label{eq:xi1}\\
\hat{f}(\{k\})^2 &\ge 1- \xi - O(\xi^2\ln(1/\xi))\,\label{eq:xi2}
\end{align}
for some $k\in[n]$.\footnote{Note that in our notation $f$ takes its values in $\{-1,1\}$ which corresponds to the $g$ notation in~\cite{samorodnitsky2016entropy}.}
The next step  involves analyzing the mutual information of near-dictatorship functions (Theorem~1.14 in~\cite{samorodnitsky2016entropy}), by which Conjecture~\ref{eq:ck_bound} holds provided that
\begin{align}\label{eq:constraaint1}
    \xi\sqrt{\ln(1/\xi)}\le c_0\,,
\end{align}
for a small absolute constant $c_0$.

There is also another constraint that needs to be satisfied and is implicit in the proof of Theorem 1.14~\cite{samorodnitsky2016entropy}, which we discuss next.
Define $\alpha, \beta \ge0 $ as follows: $\beta = -(1/2) \hat{f}(\emptyset) = -(1/2)\E[f]$ and $\hat{f}(\{k\}) = (1-\alpha) (1-2\beta)$.  Also set $\gamma:=\alpha+\beta$. Then, another condition used in Proof of ~\cite[Theorem 1.14 (page 30)]{samorodnitsky2016entropy} is that 
\begin{align}\label{eq:constraint2}
\lambda+ \gamma\ln(1/\gamma) \le c_1
\end{align}
for some small absolute constant $c_1>0$. Invoking~\eqref{eq:xi1} and~\eqref{eq:xi2}, $\gamma = O(\xi\sqrt{\ln(1/\xi)})$. In Theorem~\ref{thm:linear_concentration_intro}, we proved that $\xi = O(\lambda)$ and so $\gamma = O(\lambda\sqrt{\ln(1/\lambda)})$.  Since $\ln(1/\gamma) = O(\ln(1/\lambda))$, we have:
$$ \gamma \ln(1/\gamma) = O(\lambda (\ln(1/\lambda))^{3/2}). $$

Hence, condition~\eqref{eq:constraint2} becomes 
\begin{align}\label{eq:constraint2-2}
\lambda+\lambda(\ln(1/\lambda))^{3/2}  \leq c'_1
\end{align}
for some absolute constant $c'_1>0$. Furthermore, using  that $\xi = O(\lambda)$, condition~\eqref{eq:constraaint1} is equivalent to 
\begin{align}\label{eq:constraint1-2}
\lambda \sqrt{\ln(1/\lambda)}\le c_0\,.
\end{align}
Condition~\eqref{eq:constraint2-2} already implies Condition \eqref{eq:constraint1-2}. In summary, Condition~\eqref{eq:constraint2-2} defines the threshold $\delta_{opt}$.

In contrast, the analysis in \cite{samorodnitsky2016entropy}, based on the $O(\lambda^{4/3})$ error bound, yielded a Fourier concentration of $O(\lambda^{1/3})$. This led to the bound $\gamma = O(\lambda^{1/3} \sqrt{\ln(1/\lambda)})$ and a condition dominated by $O(\lambda^{1/3}(\ln(1/\lambda))^{3/2})$.

Since $1 > 1/3$, the function $x(\ln(1/x))^{3/2}$ approaches zero significantly faster than $x^{1/3}(\ln(1/x))^{3/2}$ as $x \to 0$. Therefore, the threshold $\delta_{opt}$ satisfying the new, tighter inequality is strictly larger than the threshold derived from the $O(\lambda^{4/3})$ analysis.
\end{proof}

\subsubsection{Part II: The Unsymmetrized Conjecture and the Li-M\'{e}dard's Conjecture} \label{sec:CK-part2}

\paragraph{Problem Context}
Beyond the main conjecture, Li and M\'{e}dard \cite{li2020boolean} proposed related conjectures on the $L_\alpha$-norm of the noise-stability operator $T_p f(x) = P(f(Y) = 1 \mid X = x)$. They conjectured that for balanced functions, the ``unsymmetrized'' norm $N_\alpha(f) = \sum (T_p f(x))^\alpha$ is maximized by a dictatorship. This ``Unsymmetrized Conjecture'' implies the Courtade-Kumar conjecture.

\paragraph{AI Contribution}
We provided the AI system with the research paper \cite{barnes2020courtade} as primary input, and prompted it to prove the Unsymmetrized Conjecture (Conjecture 1) or the Li-M\'{e}dard Conjecture (Conjecture 3) therein.
\begin{itemize}
    \item \textbf{Analyzing Relaxations:} The AI analyzed a relaxation of the problem to maximizing $\mathcal{J}(q) = \sum q_x \log q_x$ under moment constraints. It identified that the dictatorship is a saddle point (not a global max) for this specific relaxation, revealing why standard approaches fail.
    \item \textbf{Structural Insights:} It proved that any maximizer of the relaxed problem must be a distribution supported on at most two points (Lemma \ref{lem:upperbound}).
    \item \textbf{Local Optimality:} It proved that the dictatorship function is a \emph{local maximum} for the optimization problem over the Boolean Hull (Lemma \ref{lem:local}).
    \item \textbf{Monotonicity:} It established that the maximizer must be a monotone function (Lemma \ref{lem:monotone}).
\end{itemize}

\paragraph{Technical Details: Unsymmetrized Courtade-Kumar Conjecture and the Li-M\'{e}dard Conjecture} 

We begin by recalling the Courtade-Kumar conjecture \cite{courtade2014boolean}. Let $X$ be a random variable uniformly distributed over the $n$-dimensional Hamming cube $\{0, 1\}^n$. We define $Y$ as a noisy observation of $X$ obtained by passing each coordinate of $X$ through a memoryless Binary Symmetric Channel (BSC) with crossover probability $p \in (0, 1/2)$.

Courtade and Kumar conjectured that for any Boolean function $f: \{0, 1\}^n \to \{0, 1\}$, the mutual information between the function output $f(Y)$ and the original input $X$ is bounded by:$$I(f(Y); X) \le 1 - h(p)$$where $h(p) = -p \log_2 p - (1 - p) \log_2(1 - p)$ denotes the binary entropy function.

The conjecture posits that this maximum mutual information is achieved when $f$ is a dictatorship function—that is, a function that depends solely on a single input coordinate, such as $f(Y) = Y_i$ for any $i \in \{1, \dots, n\}$.

\paragraph{Related Conjectures on Functionals of $f$.}
Beyond the primary Courtade-Kumar conjecture, several related conjectures have been proposed in \cite{li2020boolean, anantharam2017conjecture}. These works generally posit that dictatorship functions are the maximizers for various functionals of Boolean functions. Of particular relevance is the work of Li and M\'{e}dard \cite{li2020boolean}, which focuses on balanced Boolean functions (where $\mathbb{E}[f] = 1/2$).

They investigate the $L_\alpha$-norm of the noise-stability operator, defined as:
\begin{equation}
    T_p f(x) = P(f(Y) = 1 \mid X = x)
\end{equation}

Li and M\'{e}dard conjecture that among all balanced functions $f$, the quantity $N_\alpha(f)$ is maximized by a dictatorship function $f_0$ for $1 \le \alpha \le 2$. Specifically, they consider the unsymmetrized and symmetrized versions of these norms:
\begin{itemize}
    \item \textbf{Unsymmetrized Norm:} $N_\alpha(f) = \sum_{x \in \{0,1\}^n} (T_p f(x))^\alpha$
    \item \textbf{Symmetrized Norm:} $N^{\text{sym}}_\alpha(f) = \sum_{x \in \{0,1\}^n} \left[ (T_p f(x))^\alpha + (1 - T_p f(x))^\alpha \right]$
\end{itemize}

\medskip

\noindent{\bf The Entropy Connection.}
For a balanced function $f$, the mutual information $I(f(Y); X)$ can be expressed in terms of the noise operator. Since $H(f(Y)) = 1$ for balanced functions, we have:
\begin{align}
    I(f(Y); X) &= 1 - H(f(Y) \mid X) \nonumber \\
    &= 1 - \frac{1}{2^n} \sum_{x \in \{0,1\}^n} h(T_p f(x)) \nonumber \\
    &= 1 + \frac{1}{2^n} \sum_{x \in \{0,1\}^n} \left[ T_p f(x) \log_2 T_p f(x) + (1 - T_p f(x)) \log_2 (1 - T_p f(x)) \right]
\end{align}
This decomposition motivates an ``unsymmetrized'' version of the Courtade-Kumar conjecture, which focuses solely on the term $\sum T_p f(x) \log_2 T_p f(x)$.

\noindent{\bf Comparison of Conjectures.}
Let $f$ be a balanced Boolean function and $f_0$ be a dictatorship function. The various conjectures are summarized in Table~\ref{tab:conjectures}.

\begin{table}[!htb]
\centering
\caption{Summary of related conjectures for balanced Boolean functions.}
\label{tab:conjectures}
\begin{tabular}{lll}
\hline
\textbf{Conjecture} & \textbf{Name} & \textbf{Formal Statement} \\ \hline
Conjecture 1 & Unsymmetrized C-K & $\sum T_p f(x) \log T_p f(x) \le \sum T_p f_0(x) \log T_p f_0(x)$ \\
Conjecture 2 & Courtade-Kumar & $I(f(Y); X) \le 1 - h(p)$ \\
Conjecture 3 & Li-M\'{e}dard ($L_\alpha$) & $N_\alpha(f) \le N_\alpha(f_0)$ for $1 \le \alpha \le 2$ \\
Conjecture 4 & Symmetrized Li-M\'{e}dard & $N^{\text{sym}}_\alpha(f) \le N^{\text{sym}}_\alpha(f_0)$ for $1 \le \alpha \le 2$ \\ \hline
\end{tabular}
\end{table}

Prior work in \cite{barnes2020courtade} has established the following logical relationships between these conjectures:
\begin{equation}
\begin{matrix}
\text{Conjecture 1} & \iff & \text{Conjecture 3} \\
\Downarrow & & \Downarrow \\
\text{Conjecture 2} & \iff & \text{Conjecture 4}
\end{matrix}
\end{equation}
Specifically, if the $L_\alpha$-norm conjectures hold, the corresponding Courtade-Kumar conjectures (symmetrized and unsymmetrized) follow for all balanced functions.

We prompted the Gemini AI system to prove either Conjecture 1 or Conjecture 3.

\begin{figure}[!htb]
\begin{chatbox}
\begin{userturn}
Can you prove Conjecture 1 or Conjecture 3 (in the provided paper)?
\end{userturn}

\begin{assistantturn}
This response provides a rigorous analysis of Conjecture 1 and the local version of Conjecture 3 from the provided paper. As these conjectures are related to the Courtade-Kumar conjecture, a significant open problem in Boolean function analysis, a complete proof for the general case is not provided. Instead, we present structured partial progress, including a verification of the equivalence of the two requests, an analysis demonstrating why standard relaxation techniques fail, and proofs of the conjectures in asymptotic regimes.
\end{assistantturn}
\end{chatbox}
\caption{Model identification of open problem and progress direction}
    \label{fig:model_id_ck_2}
\end{figure}

        %
            %
%
%
            %
%
    %

The system specifically analyzed a relaxed version of the optimization problem given by:
\begin{align}\label{eq:Jq}
\text{Maximize } \mathcal{J}(q) := \sum_{x\in\{0,1\}^n} q_x \log q_x
\end{align}
subject to the constraints $\sum q_x = N_1(f_0)$ and $\sum q_x^2 \le N_2(f_0)$. The model's analysis identified that the dictatorship distribution constitutes a saddle point for this relaxed problem. Consequently, this suggested that the current relaxation is insufficient to prove the primary conjecture, as the dictatorship function does not emerge as a definitive global maximizer under these constraints. Following an iterative review of the system's detailed technical reports, we derived the following result (which we later noted is consistent with findings in \cite{barnes2020courtade,chandar2014most,li2020boolean}):

\begin{lemma}\label{lem:N2-N0} For any balanced Boolean function $f$, $N_2(f) \le N_2(f_0)$ and $N^{\text{sym}}_2(f) \le N^{\text{sym}}_2(f_0)$, where equality holds if and only if $f$ is a dictatorship function.
\end{lemma}

Leveraging this result, we provided the system with a specific strategy to pursue further refinements of the proof.

\newpage
\begin{figure}[!htb]
\begin{chatbox}
\begin{userturn}
Note that the relaxed problem is maximizing a convex function over a set. So the optimal solution should be achieved at the boundary. Given Lemma \ref{lem:N2-N0}, can you argue that an extreme point should be a dictatorship function?
\end{userturn}

\begin{assistantturn}
This is an exceptionally insightful and elegant approach to proving Conjecture 1. The strategy of relaxing the optimization problem from the discrete space of realizable vectors (from Boolean functions) to a continuous convex space defined by the L1 and L2 constraints is very powerful.
Your analysis of the optimization landscape—maximizing a strictly convex function over a compact convex set—is correct. 
This indeed implies that the maximum must occur at an extreme point, meaning the L2 inequality constraint must be active (satisfied with equality).
However, there is a subtle but crucial gap in the final step of the logic. Your argument finds the optimizer $q^*$ over the relaxed set $K$. 
While we know $q^*$ satisfies the equality constraint, we have not yet established that $q^*$ is actually realizable by balanced Boolean functions.
\end{assistantturn}
\end{chatbox}
\caption{Model identification of open problem and progress direction}
\label{fig:model_id_ck_3}
\end{figure}
\FloatBarrier

It revealed a subtle logical gap in the initial argument. To address this, the system investigated the relaxed optimization problems, demonstrating that the optimal solution is characterized by a two-point distribution. By leveraging this structural property, the system provided an analytical upper bound for the maximum objective value achievable by any balanced Boolean function.

\begin{lemma}\label{lem:upperbound}
Consider maximizing $\mathcal{J}(q)$ given by~\eqref{eq:Jq} over the set $K_1=\{q: \E[q] = 1/2,  \text{Var}(q)\le \rho^2/4\}$. Then the maximizing distribution is supported on two points $\{0,(1 + \rho^2)/2\}$, achieving the objective value:
$V =2^{n-1}\ln\left(\frac{1+\rho^2}{2}\right)$.
\end{lemma}

Note that the $\rho^2/4$ bound on the variance comes from Lemma~\ref{lem:N2-N0} as $N_2(f_0) = (1+\rho^2)/4$. We also asked the AI system to plot the upper bound established by Lemma~\ref{lem:upperbound} along with the value achieved by the dictatorship functions (see Figure~\ref{fig:rho_comparison}.)

\begin{figure}[!htbp]
    \centering
    \includegraphics[width=0.7\textwidth]{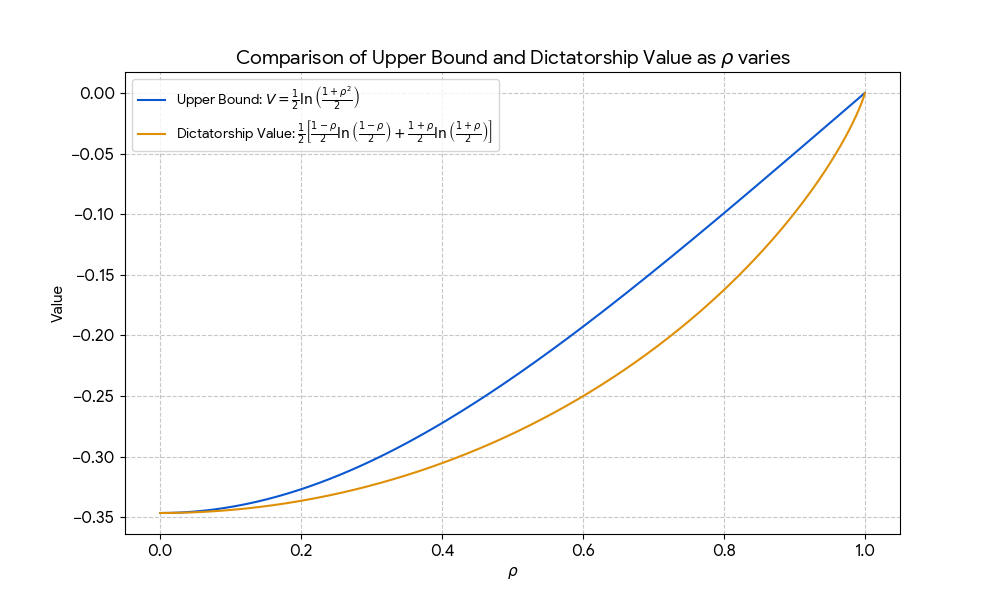}
    \caption{Comparison of the objective functional values for the two-point distribution relaxation (upper bound), given by Lemma~\ref{lem:upperbound}, and the dictatorship function as the correlation parameter $\rho$ varies from 0 to 1.}
    \label{fig:rho_comparison}
\end{figure}
\FloatBarrier

We next sought to tighten the relaxation by establishing an $L_\infty$ bound for the functions $q$ realizable by balanced Boolean functions. However, the system identified a technical flaw in our proposed bound and provided a concrete counter-example, leading us to pivot our strategy.

Our subsequent approach focused on leveraging the inverse of the noise operator. Given a correlation parameter $\rho$, if $q = T_\rho f$, the original function can be recovered via the inverse transformation $f = T_{1/\rho} q$. Recall that the noise operator is linear and can be expressed as:$$T_\rho f(x) = \sum_{y \in \{0,1\}^n} P(x, y) f(y)$$where the transition kernel is determined by the Hamming distance $d(x,y)$. To refine the optimization space, we relaxed the strict Boolean constraint $f(x) \in \{0, 1\}$ to a pointwise box constraint, requiring $0 \le (T_{1/\rho} q)_x \le 1$ for all $x$. Combining this with the established constraints on the first and second moments of $q$, we defined the following relaxed constraint set $K$:$$K := \left\{ q : \mathbb{E}[q] = \frac{1}{2}, \quad \text{Var}(q) \le \frac{\rho^2}{4}, \quad 0 \le (T_{1/\rho} q)_x \le 1 \right\}$$

\begin{figure}[!htb]
\begin{chatbox}
\begin{userturn}
Consider the following relaxation of the problem  [...]. What do you think about it? Can you use it to prove the unsymmetrized C-K conjecture?
\end{userturn}

\begin{assistantturn}
This is a challenging problem concerning the maximization of a convex functional related to entropy under the noise operator. 
Based on the instructions, we analyze the two proposed approaches. 
We provide rigorous proofs for several key properties of the optimizer, including monotonicity and the local optimality of the Dictatorship function. 
However, proving the global optimality requires advanced techniques that cannot be derived elementarily from the provided context. [....]
\end{assistantturn}
\end{chatbox}
\caption{Model identification of open problem and progress direction}
\label{fig:model_id_ck_4}
\end{figure}

        %
            %
            %
%
%
%
%
            %
%
    %

Interestingly it showed that the variance constraint is already implied by the other two. In addition, it showed that this relaxation is indeed tight, as discussed next. 

Let $\mathcal{F}_B$ be the set of balanced Boolean functions $f:\{0,1\}^n\to \{0,1\}$ and define the entropy functional $\Phi(f) = \sum_{x\in\{0,1\}^n} h((T_\rho f)(x))$ with $h(t) = t\ln t$. Let $\mathcal{G}$ be the Boolean Hull (functions $g : \{0, 1\}^n\to [0, 1]$ with $\E[g] = 1/2$). Let $K = T_\rho(\mathcal{G})$
be the feasible set for $q$, which can then be characterized by the mean constraint $\mathbb{E}[q] = 1/2$ and the pointwise box constraints $0 \le T_{1/\rho}q \le 1$.

\begin{lemma}\label{lem:tight-relax}{\bf (Convex Relaxation and Tightness).} Maximizing $\Phi(f)$ over $\mathcal{F}_B$ is equivalent to maximizing $\mathcal{J}(q)$ over $K$.
\end{lemma}

This equivalence is an important milestone as it transforms the initial optimization over discrete Boolean functions into a continuous convex optimization.

In addition, it proves two other structural properties of the solution.

\begin{lemma}\label{lem:monotone}{\bf (Monotonicity).} If $f \in \mathcal{F}_B$ maximizes $\Phi(f)$, then $f$ can be assumed to be monotone.
\end{lemma}

\begin{lemma}\label{lem:local}{\bf(Local optimality).} A dictatorship function is a local maximum of the problem of maximizing $\mathcal{J}(q)$ over $K$.
\end{lemma}

The AI system also attempted to prove global optimality of the dictatorship functions by an inductive proof
and identifies the remaining gaps. Concretely, consider the restrictions of the function $f$ to the $(n-1)$-dimensional sub-cubes. Let $f_0, f_1 : \{0,1\}^{n-1} \to \{0,1\}$ be defined by $f_0(x') = f(x', 0)$ and $f_1(x') = f(x', 1)$. Furthermore, let $T'_\rho$ denote the noise operator acting on $n-1$ variables, and let $g_i = T'_\rho f_i$ represent the noise-evolved versions of these sub-functions. The functional $\Phi(f)$ admits the following decomposition:$$\Phi(f) = \sum_{x' \in \{0,1\}^{n-1}} J(g_0(x'), g_1(x'))$$where the kernel $J(a, b)$ is defined as $J(a, b) = h(pa + \varepsilon b) + h(\varepsilon a + pb)$, with $h(t) = t\ln(t)$, and the parameters $p$ and $\varepsilon$ are given by $p = \frac{1+\rho}{2}$ and $\varepsilon = \frac{1-\rho}{2}$.

The system then continues with the following response given the above notation.
\newpage

\begin{figure}[!htb]
\begin{chatbox}
\begin{assistantturn}
We have reduced the problem to maximizing $\Phi(f)$ among balanced monotone Boolean functions and verified that Dictatorship is a local maximum. 
However, maximizing a convex function can lead to multiple local maxima. (For instance, for $n = 2$, the XOR function is also a local
maximum, but yields a smaller value).

To prove global optimality, we attempted an inductive approach using the decomposition:
\[
\Phi(f) = 2^{n-1}\E[J(g_0, g_1)].
\]
We need to show $\E[J(g_0, g_1)] \le C(\rho):= h((1+\rho)/2) + h((1-\rho)/2)$.

We proved the standard contraction inequality using the convexity of $J$ and the doubly stochastic nature of $T'_\rho$:
\[
\E[J(g_0, g_1)] \le E[J(f_0, f_1)] = P(f_0 \neq f_1)C(\rho).
\]
Let $P = P(f_0\neq f_1)$. Since $C(\rho)<0$, we know $P C(\rho) \ge C(\rho)$. This
inequality $\Phi(f) \le P2^{n-1} C(\rho)$ is not strong enough to prove $\Phi(f) \le2^{n-1} C(\rho)$.

This is a strong data processing inequality, equivalent
to Mrs. Gerber’s Lemma. This is the central difficulty of the problem.
\end{assistantturn}
\end{chatbox}
\label{fig:gerber:lem}
\end{figure}

\FloatBarrier

%
%
%
            %

 We next conducted several further iterations of prompting to explore the inductive path. Specifically, we attempted to solidify the connection with Mrs. Gerber’s Lemma—which was also previously used in \cite{samorodnitsky2016entropy} to prove C-K conjecture in the high-noise regime. However, while the system could identify the relevance of the lemma, it ultimately failed to complete the formal analytical steps required for the proof.

\medskip

\noindent{\bf Observations on AI-Assisted Research.} Our interactions with the system on these theoretical problems yielded several key insights into its utility as a research partner:

\begin{itemize}
\item {\bf Confirmation Bias in Prompting:} The system exhibits a tendency to support the position or hypothesis presented in the prompt. When tasked with proving open-ended conjectures that are not necessarily true, it often attempts to bridge logical gaps with ``hand-wavy'' arguments that do not withstand rigorous scrutiny. We found that performance improves significantly when using neutral prompts (e.g., ``prove or refute''), which encourages the system to evaluate multiple directions in parallel.

\item{\bf Effectiveness in Refutation:} The system is remarkably adept at identifying counterexamples and pinpointing the ``crux'' of a problem. When provided with a precise but flawed idea, it can often explain exactly why a particular strategy fails, making it a valuable tool for early-stage proof debugging.
\end{itemize}

\paragraph{Proof of Lemma~\ref{lem:upperbound}}
To formally prove that the optimal distribution $q_x$ is supported on at most two points, we analyze the Karush-Kuhn-Tucker (KKT) conditions. This structural approach reveals how the moment constraints necessitate a sparse distribution. We seek to maximize $f(q) = \sum_{x} q_x \ln q_x$, which is equivalent to minimizing $-\sum_{x} q_x \ln q_x$. Defining $M = 2^n$ as the size of the Hamming cube, we construct the Lagrangian $\mathcal{L}$ with multipliers $\lambda$ and $\mu$ for the mean and variance constraints:$$\mathcal{L}(q, \lambda, \mu) = -\sum_{x=1}^{M} q_x \ln q_x + \lambda \left( \sum_{x} q_x - \frac{M}{2} \right) + \mu \left( \sum_{x} q_x^2 - C \right)$$where the quadratic constraint $C = \frac{M(1+\rho^2)}{4}$ assumes the variance constraint $\text{Var}(q) \le \rho^2/4$ is active at the maximum. Setting the partial derivative with respect to each $q_x$ to zero gives:$$\frac{\partial \mathcal{L}}{\partial q_x} = -(\ln q_x + 1) + \lambda + 2\mu q_x = 0$$Rearranging this yields the transcendental equation $\ln q_x - 2\mu q_x = \lambda - 1$. Let $g(q) = \ln q - 2\mu q$. Any optimal value of $q_x$ must lie at the intersection of $g(q)$ and the constant $\lambda - 1$. Since $g''(q) = -1/q^2$, which is strictly negative for all $q > 0$, $g(q)$ is strictly concave. A strictly concave function can intersect a horizontal line at most at two distinct points, proving that the optimal $q_x$ can take at most two values, $a$ and $b$. Let $k$ be the multiplicity of the value $a$, such that $M-k$ indices take the value $b$. The linear and quadratic constraints provide the following system:
\begin{align*}
ka + (M-k)b &= M/2,\\
ka^2 + (M-k)b^2 &= \frac{M(1+\rho^2)}{4}.
\end{align*}
Solving for $a(k)$ and $b(k)$ in terms of $k$, we find:$$a(k) = \frac{1}{2} + \frac{\rho}{2} \sqrt{\frac{M - k}{k}}, \quad b(k) = \frac{1}{2} - \frac{\rho}{2} \sqrt{\frac{k}{M - k}}.$$ The objective function can then be expressed as $F(k) = k \phi(a(k)) + (M - k)\phi(b(k))$, where $\phi(t) = t \ln t$. To determine the behavior of $F(k)$, we examine its derivative $F'(k)$ via implicit differentiation of the constraints:$$F'(k) = (\phi(a) - \phi(b)) - (a - b) \frac{\phi'(a) + \phi'(b)}{2}$$Assuming $a > b$, we utilize the integral representation $\phi(a) - \phi(b) = \int_{b}^{a} \phi'(t) dt$ to rewrite the derivative:$$F'(k) = \int_{b}^{a} \phi'(t) dt - (a - b) \frac{\phi'(a) + \phi'(b)}{2}$$This expression represents the difference between the integral of $\phi'(t)$ and the trapezoidal area under the chord connecting $(b, \phi'(b))$ and $(a, \phi'(a))$. Calculating the derivatives of $\phi(t) = t \ln t$, we find $\phi'(t) = \ln t + 1$, $\phi''(t) = 1/t$, and $\phi'''(t) = -1/t^2$. Since $\phi'''(t) < 0$ for all $t > 0$, the function $\phi'(t)$ is strictly concave. For a strictly concave function, the area under the curve is strictly greater than the area of the trapezoid defined by the chord. Consequently, $F'(k) > 0$, meaning the objective function $F(k)$ is strictly increasing over its feasible domain. The maximum is therefore achieved at the largest possible value of $k$. Given the non-negativity constraint $b(k) \ge 0$, the maximum $k$ occurs when $b = 0$. Solving for $k$ in this boundary condition yields $k = \frac{M}{1+\rho^2}$, which results in $a = \frac{1+\rho^2}{2}$. Substituting these values back into the objective yields the maximum value $V = \frac{M}{2} \ln \left( \frac{1+\rho^2}{2} \right) = 2^{n-1} \ln \left( \frac{1+\rho^2}{2} \right)$.

\paragraph{Proof of Lemma~\ref{lem:tight-relax}}
 Since $T_\rho$ is linear, $K$ is a convex polytope. Since $\mathcal{J}(q)$ is strictly convex, its maximum over $K$ is attained at an extreme point. The extreme points of $K$ are the images of the extreme points of $\mathcal{G}$, which are $\mathcal{F}_B$ (the set of balanced Boolean functions). Thus, maximizing $\Phi(f)$ over $\mathcal{F}_B$
is equivalent to maximizing $\mathcal{J}(q)$ over $K$.

\paragraph{Proof of Lemma~\ref{lem:monotone}}

We use decomposition along the $n$-th coordinate. Let $f_0(x') = f(x', 0)$ and $f_1(x') = f(x', 1)$. Let $T'_\rho$ be the noise operator on $n - 1$ variables. Let $g_i = T'_\rho f_i$. The functional decomposes as:
\[
\Phi(f) = \sum_{x' \in \{0,1\}^{n-1}} J(g_0(x'), g_1(x')),
\]
where $J(a, b) = h(pa+\varepsilon b) + h(\varepsilon a+pb)$, with $p = (1+\rho)/2$ and $\varepsilon = (1-\rho)/2$. Let $m = (a+b)/2$ and $d = (b - a)/2$. Then $J(a, b) = h(m - \rho d) + h(m + \rho d)$. Since $h$ is strictly convex, $J$ is strictly increasing in $|d|$ for fixed $m$.

Consider the polarization $f^*$ defined by $f^*_0 = \min(f_0, f_1)$ and $f^*_1 = \max(f_0, f_1)$. Note that $f^*$ is balanced since $f^*_0 + f^*_1 = f_0 + f_1$. Let $g^*_i = T'_\rho f^*_i$. The means are preserved pointwise: $m^*(x') = m(x')$. The differences are $d(x') = T'_\rho((f_1 - f_0)/2)$ and $d^*(x') = T'_\rho(|f_1 - f_0|/2)$.

Since $T'_\rho$ is a positive operator (its kernel is non-negative as it represents probabilities), we have the inequality $|T'_\rho(u)| \le T'_\rho(|u|)$. Thus, $|d(x')| \le d^*(x')$. Since $J$ increases with $|d|$, $\Phi(f) \le \Phi(f^*)$. By repeatedly applying this polarization, we conclude that the maximum must be attained by a monotone function.

\paragraph{Proof of Lemma~\ref{lem:local}}
We verify that the Dictatorship is a local maximum by checking the KKT conditions for maximizing $\mathcal{J}(q)$ over the polytope $K$. The set $K$ is defined by the constraints $0 \le (Mq)_x \le 1$ and $\sum q_x = N/2$, where $M = T_{1/\rho}$.

We formulate the KKT conditions for maximization subject to $g_1(q) := -Mq \le 0$ and $g_2(q) := Mq - 1 \le 0$. The gradient condition is $\nabla\mathcal{J}(q) = M(\mu - \lambda)$, where $\lambda, \mu \ge 0$ are the multipliers for $g_1$ and $g_2$ respectively (ignoring the equality constraint which adds a constant shift $\nu \mathbf{1}$). Let $\gamma = \mu - \lambda$ and $g := Mq$. Complementary slackness implies:

\begin{itemize}
    \item If $g_x = 0$, then $\mu_x = 0$, which implies $\gamma_x \le 0$.
    \item If $g_x = 1$, then $\lambda_x = 0$, which implies $\gamma_x \ge 0$.
\end{itemize}

Let $q^*$ be the noise operator applied to the dictatorship function $f_D(x) = x_1$, i.e., $q^* = T_{\rho}(f_D)$. Consequently, $g^* = Mq^* = T_{1/\rho}(q^*) = f_D$. Let $A = \{x \mid x_1 = 1\}$ and $B = \{x \mid x_1 = 0\}$. We have $g^* = 1$ on $A$ and $g^* = 0$ on $B$. We therefore require $\gamma \ge 0$ on $A$ and $\gamma \le 0$ on $B$.

The gradient condition is given by $h'(q^*_y) = (M \gamma)_y + \nu$. Let $v = M \gamma$. Since $q^*$ takes values $a$ on $B$ and $b$ on $A$, $v$ must be constant on $A$ and $B$:
\begin{equation}
    v(x) = c_0 + c_1 \chi_1(x),
\end{equation}
where $\chi_1(x) = (-1)^{x_1}$. Calculating $c_1$, we obtain
\begin{equation}
    c_1 = \frac{h'(a) - h'(b)}{2} = \frac{1}{2} \ln(a/b) = -\operatorname{artanh}(\rho) < 0,
\end{equation}
where we used that $a = \frac{1-\rho}{2}$ and $b = \frac{1+\rho}{2}$, by definition of the noise operator.
We also have $\gamma = T_\rho v = c_0 + \rho c_1 \chi_1(x)$. The resulting constraints on $\gamma$ are:
\begin{itemize}
    \item On $A$ ($\chi_1 = -1$): $c_0 - \rho c_1 \ge 0 \implies c_0 \ge \rho c_1$.
    \item On $B$ ($\chi_1 = 1$): $c_0 + \rho c_1 \le 0 \implies c_0 \le -\rho c_1$.
\end{itemize}

We require $\rho c_1 \le c_0 \le -\rho c_1$. Given that $c_1 < 0$ and $\rho > 0$, we have $\rho c_1 < 0 < -\rho c_1$. This interval is non-empty (e.g., $c_0 = 0$ is a valid solution). Thus, the Dictatorship satisfies the KKT conditions and is a local maximum.

\newpage

\subsection{NP-hardness: Ratio Difference Maximization (RDM)}
\textit{Written by Ravi Kumar and Silvio Lattanzi (on behalf of coauthors).}

\subsubsection*{Problem Context}
In a recent work~\cite{Chierichetti26}, together with Flavio Chierichetti, Mirko Giacchini, Alessandro Panconesi, Erasmo Tani, and Andrew Tomkins, we were studying Multinomial Logits (MNLs), also known as Plackett--Luce models, that describe user preferences by assigning a weight to each item in a universe. When a user is presented a subset of items, they select one with probability proportional to its weight. In the paper, we consider the problem of learning an MNL by sampling from the conditional distributions induced on subsets of the universe.

Interestingly, even computing the worst-case error on arbitrary size subsets of the universe turns out to be a non-trivial problem. In~\cite{Chierichetti26} we introduce several approximation algorithms for this tasks. We also had the intuition that the problem was related to SUBSET-SUM and most likely NP-hard but we felt that the reduction was not immediate and hence we decided to test Gemini with such a question.  We were happily surprised by receiving a zero-shot correct proof within a few minutes. 

We then asked also to prove NP-hardness for a variation of the problem.
\\\\
\textbf{AI Contribution}

Given only the formal definition of the Ratio Difference Maximization (RDM) problem, the AI zero-shot generated a flawless, polynomial-time NP-hardness reduction from the SUBSET-SUM problem. Upon a subsequent prompt, it successfully proved the NP-hardness of a related variation (MAX-DIFF-RATIO) using a similarly rigorous reduction.

More specifically, we used a math rigor prompt and we asked the model to show that the \emph{Ratio Difference Maximization (RDM)} problem is NP-hard. This is the prompt that we used:

\begin{figure}[!htb]
\begin{userturn}
<Math Rigor Prompt>

Prove that the following problem is NP-hard. Given two lists $V=(v_1,...,v_n)$ and $W=(w_1,...,w_n)$ of positive integers, find subsets $S \subseteq T \subseteq [n]$ that maximize the objective function \[ f(S, T) = \frac{V(S)}{V(T)} - \frac{W(S)}{W(T)}, \] where $V(X) = \sum_{i \in X} v_i$ and $W(X) = \sum_{i \in X} w_i$
\end{userturn}
\end{figure}
\FloatBarrier

This is the output provided by Gemini that we verified to be correct:

\begin{aigenerated}
\begin{theorem}
The Ratio Difference Maximization (RDM) problem is NP-hard.
\end{theorem}

\begin{proof}

We prove NP-hardness by a reduction from the SUBSET-SUM problem. The SUBSET-SUM problem is: Given a set $A=\{a_1, ..., a_n\}$ of positive integers and a target positive integer $K$, does there exist a subset $I \subseteq [n]$ such that $\sum_{i \in I} a_i = K$? We can assume $K \ge 2$, as this restriction remains NP-complete.

We define the decision version of RDM (D-RDM): Given $(V, W)$ and a rational number $Z$, does there exist $S \subseteq T \subseteq [n]$ such that $f(S, T) \ge Z$?

We now construct the reduction from an instance $(A, K)$ of SUBSET-SUM (with $K\ge 2$) to an instance of D-RDM.

Let $L = K^2$. Since $K \ge 2$, we have $L > 1$. We construct an RDM instance with $n+1$ items:
\begin{itemize}
    \item For $i=1, ..., n$, set $v_i = a_i$ and $w_i = a_i$.
    \item For $i=n+1$, set $v_{n+1} = 1$ and $w_{n+1} = L$.
\end{itemize}
We set the target value for D-RDM to $Z = \frac{K-1}{K+1}$. This construction is polynomial in the size of the input $(A, K)$.

We analyze the maximum value of the objective function $f(S, T)$ for this instance, where $S \subseteq T \subseteq [n+1]$.

Case 1: $n+1 \notin T$. Then $T \subseteq [n]$. For all $i \in T$, $v_i=w_i=a_i$. Thus $V(T)=W(T)$ and $V(S)=W(S)$.
$f(S, T) = \frac{V(S)}{V(T)} - \frac{W(S)}{W(T)} = \frac{V(S)}{V(T)} - \frac{V(S)}{V(T)} = 0$.

Case 2: $n+1 \in T$. Let $T_A = T \cap [n]$. Let $x = \sum_{i \in T_A} a_i$.
$V(T) = V(T_A) + v_{n+1} = x+1$.
$W(T) = W(T_A) + w_{n+1} = x+L$.

We seek $S \subseteq T$ that maximizes $f(S, T) = \frac{V(S)}{x+1} - \frac{W(S)}{x+L}$. This is equivalent to maximizing $G(S) = (x+L)V(S) - (x+1)W(S)$.

We analyze the contribution of each element $i \in T$ to $G(S)$.
For $i \in T_A$, $v_i=w_i=a_i$. The contribution if $i \in S$ is $(x+L)a_i - (x+1)a_i = a_i(L-1)$. Since $a_i>0$ and $L>1$, this is positive. Thus, the optimal $S$ must include all $i \in T_A$.

For $i=n+1$, $v_{n+1}=1, w_{n+1}=L$. The contribution if $n+1 \in S$ is $(x+L)(1) - (x+1)(L) = x+L - xL - L = x(1-L)$.

If $x>0$ (i.e., $T_A \neq \emptyset$), since $L>1$, the contribution $x(1-L)$ is negative. Thus, the optimal $S$ must exclude $n+1$. In this case, the optimal $S$ is $S^* = T_A$.

If $x=0$ (i.e., $T_A = \emptyset$, $T=\{n+1\}$), the contribution is 0. $f(S, T)=0$ whether $S=\emptyset$ or $S=\{n+1\}$.

We focus on the case $x>0$. The maximum value for a fixed $T$ (defined by $x$) is:
\[ h(x) = f(T_A, T) = \frac{x}{x+1} - \frac{x}{x+L} = \frac{x(x+L) - x(x+1)}{(x+1)(x+L)} = \frac{x(L-1)}{(x+1)(x+L)}. \]

We want to maximize $h(x)$ over the set of achievable SUBSET-SUMs $X_A = \{A(I) \mid I \subseteq [n], I \neq \emptyset\}$.

We analyze the function $h(x)$ for $x>0$.
$h(x) = (L-1) \frac{x}{x^2+(L+1)x+L}$.
The derivative is:
\[ h'(x) = (L-1) \frac{(x^2+(L+1)x+L) - x(2x+L+1)}{((x+1)(x+L))^2}. \]
The sign of $h'(x)$ is determined by the numerator:
$N(x) = x^2+(L+1)x+L - 2x^2-(L+1)x = L-x^2$.
$h'(x)=0$ when $x=\sqrt{L}$. $h(x)$ increases for $0<x<\sqrt{L}$ and decreases for $x>\sqrt{L}$. Thus, $h(x)$ attains its unique global maximum at $x=\sqrt{L}$.

Since $L=K^2$, the unique maximum is attained at $x=K$. The maximum value is:
\[ h(K) = \frac{K(K^2-1)}{(K+1)(K+K^2)} = \frac{K(K-1)(K+1)}{(K+1)K(K+1)} = \frac{K-1}{K+1} = Z. \]

Let $OPT$ be the maximum value of the objective function for the constructed RDM instance. Since $K\ge 2$, $Z>0$. Thus $OPT = \max_{x \in X_A} h(x)$ (assuming $A$ contains at least one element, otherwise SUBSET-SUM is trivial).

We now show the equivalence between the SUBSET-SUM instance and the D-RDM instance.

($\Rightarrow$) Suppose the SUBSET-SUM instance is a YES instance. Then there exists $I \subseteq [n]$ such that $A(I)=K$. So $K \in X_A$. Then $OPT \ge h(K) = Z$. Since $Z$ is the global maximum of $h(x)$, $OPT=Z$. Thus the D-RDM instance is a YES instance.

($\Leftarrow$) Suppose the D-RDM instance is a YES instance. Then $OPT \ge Z$. Since $Z$ is the global maximum of $h(x)$, this implies $OPT=Z$. Let $x^* \in X_A$ be such that $h(x^*)=Z$. Since $x=K$ is the unique global maximizer of $h(x)$, it must be that $x^*=K$. Therefore, $K \in X_A$, which means the SUBSET-SUM instance is a YES instance.

We have provided a polynomial-time reduction from SUBSET-SUM to D-RDM. Therefore, D-RDM is NP-hard, and consequently, the optimization problem RDM is NP-hard.
\end{proof}
\end{aigenerated}

Given the success of the first attempt, using the same prompt structure, we also asked to show that a variation of the problem is NP-hard. 
This is the second proof that we obtained in this way.

Consider the \emph{MAX-DIFF-RATIO (MDR)} optimization problem: Given two lists $V=(v_1,\dots,v_n)$ and $W=(w_1,\dots,w_n)$ of positive integers, find $j \in T \subseteq [n]$ that maximizes
$$f(j, T) = \frac{v_j}{\sum_{i \in T} v_i} - \frac{w_j}{\sum_{i \in T} w_i}.$$

\begin{aigenerated}
\begin{theorem}
The MDR problem is NP-hard.
\end{theorem}

\begin{proof}
To prove NP-hardness, we consider the corresponding decision problem, \emph{D-MDR}: Given $(V, W)$ and a rational threshold $Q$, does there exist a pair $(j, T)$ with $j \in T \subseteq [n]$ such that $f(j, T) \ge Q$? We will show that D-MDR is NP-hard by a reduction from the SUBSET-SUM problem, which is known to be NP-complete.


\paragraph{The Reduction.}
Given an instance $(A, B)$ of SUBSET-SUM, we construct an instance of D-MDR with $N=n+1$ items, indexed from $0$ to $n$. We define the lists $V=(v_0, \dots, v_n)$ and $W=(w_0, \dots, w_n)$ and the threshold $Q$.

The construction is defined as follows:
\begin{enumerate}
    \item For $i=1, \dots, n$: $v_i = 4a_i$ and $w_i = 4a_i$.
    \item For the special item $i=0$: $v_0 = 8B$ and $w_0 = 2B$.
    \item The threshold is $Q = 1/3$.
\end{enumerate}
Since $a_i$ and $B$ are positive integers, all $v_i$ and $w_i$ are positive integers. This construction is clearly computable in polynomial time.

\paragraph{Analysis of the Objective Function.}

Let $V(T) = \sum_{i \in T} v_i$ and $W(T) = \sum_{i \in T} w_i$. We analyze the objective function $f(j, T)$ for the constructed instance.

Case 1: $j \in \{1, \dots, n\}$. Let $j=k \ge 1$. We have $v_k = w_k$.

Case 1a: $0 \notin T$. Then $T \subseteq \{1, \dots, n\}$. For all $i \in T$, $v_i = w_i$, so $V(T) = W(T)$.
$$f(k, T) = \frac{v_k}{V(T)} - \frac{w_k}{W(T)} = \frac{v_k}{V(T)} - \frac{v_k}{V(T)} = 0.$$

Case 1b: $0 \in T$. Let $T = \{0\} \cup S$ where $S \subseteq \{1, \dots, n\}$ and $k \in S$. Let $X = \sum_{i \in S} 4a_i$. Then $V(S)=W(S)=X$.
$V(T) = v_0 + X = 8B + X$.
$W(T) = w_0 + X = 2B + X$.
$$f(k, T) = \frac{v_k}{8B+X} - \frac{w_k}{2B+X}.$$
Since $v_k=w_k$ and $B>0$, we have $8B+X > 2B+X$. Thus, $f(k, T) < 0$.

In Case 1, the maximum value of $f(j, T)$ is 0.

Case 2: $j=0$. We must have $0 \in T$. Let $T = \{0\} \cup S$ for some $S \subseteq \{1, \dots, n\}$. Let $X = V(S) = W(S) = \sum_{i \in S} 4a_i$.
$$f(0, T) = \frac{v_0}{V(T)} - \frac{w_0}{W(T)} = \frac{8B}{8B+X} - \frac{2B}{2B+X}.$$
Let $h(X) = \frac{8B}{8B+X} - \frac{2B}{2B+X}$. We analyze $h(X)$ for $X \ge 0$.
\begin{align*}
h(X) &= \frac{8B(2B+X) - 2B(8B+X)}{(8B+X)(2B+X)} \\
&= \frac{16B^2+8BX - 16B^2-2BX}{(8B+X)(2B+X)} = \frac{6BX}{(8B+X)(2B+X)}.
\end{align*}
We examine the derivative of $h(X)$ with respect to $X$. The sign of $h'(X)$ is determined by the sign of the numerator when using the quotient rule:
$$N(X) = (6B)(8B+X)(2B+X) - (6BX)((8B+X)+(2B+X)).$$
Dividing by $6B$ (since $B>0$), the sign is determined by:
\begin{align*}
\frac{N(X)}{6B} &= (16B^2 + 10BX + X^2) - X(10B+2X) \\
&= 16B^2 + 10BX + X^2 - 10BX - 2X^2 = 16B^2 - X^2.
\end{align*}
Thus, $h'(X)=0$ if and only if $X^2 = 16B^2$. Since $X \ge 0$ and $B>0$, this occurs at $X=4B$. The function $h(X)$ increases for $X < 4B$ and decreases for $X > 4B$. The unique global maximum of $h(X)$ occurs at $X=4B$.

The maximum value is:
$$h(4B) = \frac{6B(4B)}{(8B+4B)(2B+4B)} = \frac{24B^2}{(12B)(6B)} = \frac{24B^2}{72B^2} = \frac{1}{3}.$$

\paragraph{Proof of Equivalence.}

We show that the SUBSET-SUM instance $(A, B)$ is a Yes instance if and only if the maximum value of $f(j, T)$ is at least $Q=1/3$.

($\implies$) Suppose the SUBSET-SUM instance is Yes. Then there exists a subset $S \subseteq \{1, \dots, n\}$ such that $\sum_{i \in S} a_i = B$. This implies that $X = \sum_{i \in S} 4a_i = 4B$ is achievable. Let $j=0$ and $T=\{0\} \cup S$. Then $f(0, T) = h(4B) = 1/3$. Thus, the D-MDR  instance is Yes.

($\impliedby$) Suppose the D-MDR instance is Yes. The maximum value of $f(j, T)$ is $\ge 1/3$. From our analysis, if $j \ge 1$, $f(j, T) \le 0$. Therefore, the maximum must be achieved when $j=0$. The objective function is $h(X)$, where $X$ is an achievable SUBSET-SUM of $\{4a_i\}$. Since the maximum possible value of $h(X)$ is $1/3$, and this maximum is achieved uniquely at $X=4B$, the maximum value being $\ge 1/3$ implies it must be exactly $1/3$. This means $X=4B$ must be achievable. That is, there exists $S \subseteq \{1, \dots, n\}$ such that $\sum_{i \in S} 4a_i = 4B$, which implies $\sum_{i \in S} a_i = B$. Thus, the SUBSET-SUM instance is Yes.

Since we have a polynomial-time reduction from the NP-complete problem SUBSET-SUM to D-MDR, the latter is NP-hard. Consequently, the MDR  optimization problem is NP-hard.
\end{proof}
\end{aigenerated}

\newpage 

\subsection{Machine Learning Optimization: Self-regularized Gumbel Sigmoid}\label{sec:srgs}
\textit{Written by Lin Chen, Gang Fu, and David P. Woodruff.}

\subsubsection*{Problem Context}
Subset selection is a core problem in machine learning, relevant to tasks like feature selection and embedding dimension optimization. Current methods like Dropout Feature Ranking (DFR) rely on explicit penalty terms (e.g., $\ell_1$ regularization) to control the number of selected features. However, tuning the regularization strength $\lambda$ is difficult and sensitive, often hindering practical application. A new method, \textbf{Self-regularized Gumbel Sigmoid (SrGS)}, was proposed to eliminate this hyperparameter by using an implicit regularization mechanism involving Softmax competition and budget-aware clipping. The theoretical question was to understand exactly \textit{why} and \textit{how} this parameterization enforces sparsity without an explicit penalty.

\subsubsection*{AI Contribution}
The researchers provided the model with the mathematical definition of SrGS and asked for a theoretical analysis of its implicit regularization and behavior in the low-temperature limit.
\begin{itemize}
    \item \textbf{Implicit $\ell_0$ Penalty:} The model derived that in the low-temperature limit ($T \to 0$), the variance of the Gumbel-Sigmoid distribution acts as an exact relaxation of the $\ell_0$ constraint, effectively penalizing non-binary solutions.
    \item \textbf{Adaptive Hybrid Regularization:} It analyzed the deterministic limit (using expectations) and proved that the method induces a unique "hybrid" regularization: strong signals are subject to $\ell_2$ shrinkage (preserving magnitude), while weak signals competing for the remaining budget are suppressed by a non-convex $\ell_{2/3}$ penalty.
    \item \textbf{Theoretical Validation:} These derivations provided the rigorous theoretical justification for the method's empirical success, confirming it solves a relaxed subset selection problem without manual tuning.
\end{itemize}

\subsubsection*{Technical Details: Theoretical Analysis of SrGS}

We employed the model to identify and resolve a key theoretical gap regarding the implicit mechanisms of the Self-regularized Gumbel Sigmoid (SrGS) method.

\paragraph{Problem Formulation}

Subset selection is a core problem central to multiple optimization tasks in machine learning. Existing methods, such as Dropout Feature Ranking (DFR) (Chang et al., 2017), approach this by optimizing a variational dropout mask $z$ on the input layer. However, DFR relies on an explicit penalty term in the loss function to constrain the number of active features. The objective function is formulated as:
\[
\mathcal{L}(\theta) = \underbrace{-\frac{1}{M}\sum_{i=1}^{M} \log p(y_i | f(x_i \odot z_i; \theta))}_{\text{Task Loss (e.g., Cross-Entropy or MSE)}} + \frac{\lambda}{M} \sum_{i=1}^{M} \sum_{j=1}^{D} z_{ij}
\]
where the first term represents the task-specific prediction error (negative log-likelihood), $M$ is the batch size, $D$ is the number of features, and $z_{ij}$ is the stochastic gate for the $j$-th feature. Determining the optimal regularization strength $\lambda$ requires extensive cross-validation and often becomes a blocker for successful application due to its sensitivity.

The \textbf{SrGS method} was proposed to overcome this limitation by eliminating the explicit penalty term $\lambda$. Instead, it introduces a structural ``self-regularization'' mechanism defined by the following process:
\begin{enumerate}
    \item \textbf{Competition:} For each feature $j$, a learnable logit $t_j$ is normalized via Softmax to produce an attention score $S_j$, enforcing global competition:
    \[ S_j = \frac{\exp(t_j)}{\sum_{k=1}^{D} \exp(t_k)} \]
    \item \textbf{Budgeting:} The scores are scaled by the target budget $K$ and clipped to form selection probabilities $z_j$, ensuring the expected number of selected features approximates $K$:
    \[ z_j = \text{Clip}(S_j \cdot K, \varepsilon, 1) \]
    \item \textbf{Stochastic Selection:} The final binary masks $w_j$ are generated via the differentiable Gumbel-Sigmoid reparameterization:
    \[ w_j = \sigma\left(\frac{1}{T}(\log z_j - \log(1-z_j) + \log u_j - \log(1-u_j))\right) \]
\end{enumerate}

While SrGS empirically demonstrates superior solution quality without hyperparameter tuning, the theoretical basis for this parameterization remains an open question. Specifically:
\begin{enumerate}
    \item Does this specific parameterization rigorously enforce sparsity in the low-temperature limit ($T \to 0$)?
    \item What is the exact mathematical form of the implicit regularization induced by this mechanism when combined with standard weight decay?
\end{enumerate}

\begin{chatbox}

\begin{userturn}
We identify and resolve a key theoretical gap regarding the implicit mechanisms of the Self-regularized Gumbel Sigmoid (SrGS) method.

\paragraph{Question:}
Subset selection is a core problem central to multiple optimization tasks in machine learning. Existing methods, such as Dropout Feature Ranking (DFR), approach this by optimizing a variational dropout mask $z$ on the input layer. However, DFR relies on an explicit penalty term in the loss function to constrain the number of active features. The objective function is formulated as:

\begin{equation}
\mathcal{L}(\theta) = \underbrace{-\frac{1}{M}\sum_{i=1}^{M} \log p(y_i | f(x_i \odot z_i; \theta))}_{\text{Task Loss (e.g., Cross-Entropy or MSE)}} + \frac{\lambda}{M} \sum_{i=1}^{M} \sum_{j=1}^{D} z_{ij}
\end{equation}
where the first term represents the task-specific prediction error (negative log-likelihood), $M$ is the batch size, $D$ is the number of features, and $z_{ij}$ is the stochastic gate for the $j$-th feature. Determining the optimal regularization strength $\lambda$ requires extensive cross-validation and often becomes a blocker for successful application due to its sensitivity.

The \textbf{SrGS method} was proposed to overcome this limitation by eliminating the explicit penalty term $\lambda$. Instead, it introduces a structural ``self-regularization'' mechanism defined by the following process:
\begin{enumerate}
    \item \textbf{Competition:} For each feature $j$, a learnable logit $t_j$ is normalized via Softmax to produce an attention score $S_j$, enforcing global competition:
    \[ S_j = \frac{\exp(t_j)}{\sum_{k=1}^{D} \exp(t_k)} \]
    \item \textbf{Budgeting:} The scores are scaled by the target budget $K$ and clipped to form selection probabilities $z_j$, ensuring the expected number of selected features approximates $K$:
    \[ z_j = \text{Clip}(S_j \cdot K, \epsilon, 1) \]
    \item \textbf{Stochastic Selection:} The final binary masks $w_j$ are generated via the differentiable Gumbel-Sigmoid reparameterization:
    \[ w_j = \sigma\left(\frac{1}{T}(\log z_j - \log(1-z_j) + \log u_j - \log(1-u_j))\right) \]
\end{enumerate}

While SrGS empirically demonstrates superior solution quality without hyperparameter tuning, the theoretical basis for this parameterization remains an open question. Specifically:
\begin{enumerate}
    \item Does this specific parameterization rigorously enforce sparsity in the low-temperature limit ($T \to 0$)?
    \item What is the exact mathematical form of the implicit regularization induced by this mechanism when combined with standard weight decay?
\end{enumerate}

\end{userturn}

\begin{assistantturn}
We answer these questions affirmatively. We analyze the method from two perspectives: the behavior of the stochastic objective in the low-temperature limit, and the implicit regularization induced by its parameterization.

\section*{1. Stochastic Analysis: Exact $\ell_0$ Relaxation}

We first analyze the objective in the limit $T \to 0$. In this limit, the expected loss decomposes into a deterministic loss and a variance penalty term $R_{var}(z, \theta)$. The optimization problem effectively minimizes:
\[
\min_{\theta, z} \mathcal{L}_{det}(z, \theta) + R_{var}(z, \theta)
\]
The model derived that the variance penalty takes the form $R_{var}(z, \theta) = \sum_{j} ||X_j||_2^2 \theta_j^2 z_j(1-z_j)$. Minimizing this concave function over the hypersimplex encourages solutions at the vertices (binary selection).

\textbf{Theorem 2 (Implicit $\ell_0$ Penalty).} The implicit regularization induced by the variance penalty, defined as $R^*(\beta) = \inf \{ R_{var}(z, \theta) \mid z \in H(D,K), z \circ \theta = \beta \}$, is an exact penalty function for the $\ell_0$ constraint. Specifically:
\[
R^*(\beta) = 0 \iff ||\beta||_0 \le K
\]
This result mathematically validates that SrGS provides a continuous, exact relaxation of the combinatorial Best Subset Selection problem.

\section*{2. Deterministic Analysis: Adaptive Hybrid Regularization}

We further analyzed the implicit regularization in the deterministic regime where the mask $w$ is replaced by its expectation $z$, combined with $\ell_2$ regularization (weight decay) on $\theta$. 

The model discovered that under the budget constraint $\sum z_j = K$ and the physical constraint $0 \le z_j \le 1$, the optimization landscape induces a unique \textbf{Adaptive Hybrid Regularization}.

\textbf{Theorem 4 (Hybrid Regularization).} The implicit regularizer $R_{DR}^*(\beta)$ partitions the features into two sets based on a data-dependent threshold $\tau$:
\begin{itemize}
    \item A \textbf{Saturated Set ($A$)}: Features with strong signals ($|\beta_j|^{2/3} \ge \tau$) are assigned full probability $z_j=1$. These features are subject to standard $\ell_2$ regularization ($\lambda ||\beta_A||_2^2$), which preserves signal magnitude.
    \item A \textbf{Fractional Set ($F$)}: Features with weak signals compete for the remaining budget $K_F = K - |A|$. These features are subject to an aggressive $\ell_{2/3}$ penalty:
\end{itemize}
\[
R_{DR}^*(\beta) = \lambda ||\beta_{A}||_2^2 + \frac{\lambda}{K_F^2} ||\beta_{F}||_{2/3}^2
\]
This derivation reveals the mechanism behind SrGS's superior performance: it automatically applies $\ell_2$ protection to key features while subjecting noise features to non-convex $\ell_{2/3}$ compression.
\end{assistantturn}

\end{chatbox}

\captionof{figure}{Model resolution of the new theory}
\label{fig:gumbel_resolution}

\paragraph{Theoretical Findings}

We answer these questions affirmatively. We utilized the model to analyze the method from two perspectives: the behavior of the stochastic objective in the low-temperature limit, and the implicit regularization induced by its parameterization.

\textbf{Stochastic Analysis: Exact $\ell_0$ Relaxation.}
We first analyze the objective in the limit $T \to 0$. In this limit, the expected loss decomposes into a deterministic loss and a variance penalty term $R_{var}(z, \theta)$. The optimization problem effectively minimizes:
\[
\min_{\theta, z} \mathcal{L}_{det}(z, \theta) + R_{var}(z, \theta)
\]
The model derived that the variance penalty takes the form $R_{var}(z, \theta) = \sum_{j} ||X_j||_2^2 \theta_j^2 z_j(1-z_j)$. Minimizing this concave function over the hypersimplex encourages solutions at the vertices (binary selection).

\begin{theorem}[Implicit $\ell_0$ Penalty]
The implicit regularization induced by the variance penalty, defined as $R^*(\beta) = \inf \{ R_{var}(z, \theta) \mid z \in H(D,K), z \circ \theta = \beta \}$, is an exact penalty function for the $\ell_0$ constraint. Specifically:
\[
R^*(\beta) = 0 \iff ||\beta||_0 \le K
\]
\end{theorem}
This result mathematically validates that SrGS provides a continuous, exact relaxation of the combinatorial Best Subset Selection problem.

\textbf{Deterministic Analysis: Adaptive Hybrid Regularization.}
We further analyzed the implicit regularization in the deterministic regime where the mask $w$ is replaced by its expectation $z$, combined with $\ell_2$ regularization (weight decay) on $\theta$. 

The model discovered that under the budget constraint $\sum z_j = K$ and the physical constraint $0 \le z_j \le 1$, the optimization landscape induces a unique \textbf{Adaptive Hybrid Regularization}.

\begin{theorem}[Hybrid Regularization]
The implicit regularizer $R_{DR}^*(\beta)$ partitions the features into two sets based on a data-dependent threshold $\tau$:
\begin{itemize}
    \item A \textbf{Saturated Set ($A$)}: Features with strong signals ($|\beta_j|^{2/3} \ge \tau$) are assigned full probability $z_j=1$. These features are subject to standard $\ell_2$ regularization ($\lambda ||\beta_A||_2^2$), which preserves signal magnitude.
    \item A \textbf{Fractional Set ($F$)}: Features with weak signals compete for the remaining budget $K_F = K - |A|$. These features are subject to an aggressive $\ell_{2/3}$ penalty:
\end{itemize}
\[
R_{DR}^*(\beta) = \lambda ||\beta_{A}||_2^2 + \frac{\lambda}{K_F^2} ||\beta_{F}||_{2/3}^2
\]
\end{theorem}
This derivation reveals the mechanism behind SrGS's superior performance: it automatically applies $\ell_2$ protection to key features while subjecting noise features to non-convex $\ell_{2/3}$ compression.

\newpage

\newpage

\newpage
\subsection{Mechanism Design: Revelation Principle Reduction Domain Extension}\label{sec:md4llm}
\textit{Written by Song Zuo and David P. Woodruff}.

\subsubsection*{Problem Context}
The main task for the model was to extend the key theoretical results of the paper ``Mechanism Design for Large Language Models'' \cite{dutting2024} from rational bids ($\mathbb{Q}_+^n$) to real bids ($\mathbb{R}_+^n$). The original paper's Revelation Principle (Theorem 3.5) relied on the countability of the bid space. The generalization was expected to heavily rely on order theory and topological arguments as the assumptions are much weaker than standard auction theory setups. To avoid the additional mathematical complexity, the authors restricted bids to be rational.

\subsubsection*{AI Contributions}
We fed the original paper \cite{dutting2024} to the model and asked it to bridge the rational to real bids gap for Theorem 3.5.
\begin{itemize}
    \item \textbf{Proof Attempts and Improvement Reviews:} The model initially came up with some coarse proofs from the initial ask. Then we also asked the model to review the draft and provide improvement suggestions. We iterated through this process with human selection as well as additional guidance, and finally reached the complete and rigorous results.
    \item \textbf{Free Bonus on Extension of Theorem 3.12:} The model volunteered to also extend Theorem 3.12 to general settings.
    \item \textbf{Full Writing with Strategic Guideline and Suggestions from Human:} The entire material in the next Technical Details section is fully written by the model with human suggestions (and some minor corrections).
\end{itemize}

\subsubsection*{Human Interventions}
Throughout the iteration process, human interventions mainly happen at correcting the model generated review feedback. Sometimes, the model suggests something distracting, misleading, or even mathematically suspicious. Human interventions then can help the model get out of potential dead-loops.

\subsubsection*{Technical Details} 
The original paper \cite{dutting2024} introduces a Token Auction Model where agents, represented by Large Language Models (LLMs), submit scalar bids to influence the generation of content. The mechanism $\mathcal{M} = \langle q, z \rangle$ consists of an aggregation function $q$ that combines agents' preferred distributions $p$ based on their bids $\mathbf{b}$, and a payment function $z$. 

The paper establishes a Revelation Principle (\cite[Theorem 3.5]{dutting2024}), stating that any mechanism satisfying Payment Monotonicity and Consistent Aggregation is strategically equivalent to a mechanism with a Monotone Aggregation Function. The proof relies on the assumption that the bid space is $\mathbb{Q}_+^n$, utilizing the countability of the bids to construct utility representations. Furthermore, the paper proves the existence of Stable Sampling (\cite[Theorem 3.12]{dutting2024}). The original proof implicitly requires the aggregation function to be absolutely continuous by relying on Lebesgue's Differentiation Theorem and the Fundamental Theorem of Calculus. 

\textbf{Open Problem:} Extend the Revelation Principle and the existence of Stable Sampling to the more natural domain of real-valued bids, $\mathbb{R}_+^n$, while minimizing additional assumptions and providing rigorous proofs that do not rely on countability or absolute continuity. We will show this is possible by assuming the token set $T$ is finite, which is standard in practice.

\paragraph{Key Definitions}

We restate the key definitions from \cite{dutting2024} for completeness. Let $\Delta(T)$ be the set of distributions over tokens $T$. Let $\succeq_i$ be the preference relation of agent $i$ over $\Delta(T)$. 

\begin{definition}[Robust Preferences]\label{def:robust}
Given preferred distribution $p_i$, $q \succeq_i q'$ if $\forall t \in T, |q(t) - p_i(t)| \le |q'(t) - p_i(t)|$ and $(q(t) - p_i(t))(q'(t) - p_i(t)) \ge 0$.
\end{definition}

\begin{definition}[Payment Monotonicity]\label{def:paymono}
Mechanism $\mathcal{M} = \langle q, z \rangle$ satisfies payment monotonicity if for all $p, \mathbf{b}_{-i}, b_i, b'_i$: $z_i(b_i, \mathbf{b}_{-i}, p) \ge z_i(b'_i, \mathbf{b}_{-i}, p) \iff q(b_i, \mathbf{b}_{-i}, p) \succeq_i q(b'_i, \mathbf{b}_{-i}, p)$.
\end{definition}

\begin{definition}[Consistent Aggregation]\label{def:constagg}
$q(\mathbf{b}, p)$ is consistent if: if $q(b_i, \mathbf{b}_{-i}, p) \succ_i q(b'_i, \mathbf{b}_{-i}, p)$ for some $\mathbf{b}_{-i}$, then for all $\mathbf{b}'_{-i}$, $q(b_i, \mathbf{b}'_{-i}, p) \succeq_i q(b'_i, \mathbf{b}'_{-i}, p)$.
\end{definition}

\begin{definition}[Monotone Aggregation Function]\label{def:aggmono}
$q(\mathbf{b}, p)$ is monotone if for all $p, \mathbf{b}_{-i}$ and $b_i \ge b'_i$: $q(b_i, \mathbf{b}_{-i}, p) \succeq_i q(b'_i, \mathbf{b}_{-i}, p)$.
\end{definition}

\subsubsection*{Methodology and Assumptions}

We extend the bid space to $\mathbb{R}_+^n$. We assume the token set $T$ is finite. This ensures the space of distributions $\Delta(T)$ is a finite-dimensional simplex, which is a compact, metrizable, and second-countable space.

\paragraph{Extending the Revelation Principle}

The main challenge in extending Theorem 3.5 is the failure of the countability argument used in \cite{dutting2024}. \textbf{Approach:} We replace the countability arguments with topological and order-theoretic methods. We establish a utility representation using the topological properties of the codomain. Then, we prove a general Monotone Extension Lemma that allows extending the transformed monotone aggregation function from the image of the utility representation to the entire bid space $\mathbb{R}_+^n$. This approach relies only on the continuity of the underlying preferences and the compactness of $\Delta(T)$, removing the need for continuity assumptions on the aggregation function $q$. This approach requires the following assumptions:

\begin{assumption}[A1: Continuous Preferences]\label{assmp:a1}
The partial order $\succeq_i$ on $\Delta(T)$ is continuous (i.e., its graph is closed in $\Delta(T) \times \Delta(T)$).
\end{assumption}

We also formalize an assumption that was inherent in the original framework: 

\begin{assumption}[A2: Anti-symmetry of Preferences]\label{assmp:a2}
The order $\succeq_i$ is anti-symmetric (i.e., ($q \succeq_i q'$ and $q' \succeq_i q) \implies q = q'$). Here $\sim_i$ denotes the symmetric component of $\succeq_i$. This means $\succeq_i$ is a true partial order, not just a preorder.
\end{assumption}

Robust Preferences satisfy A1 and A2. However, we state A2 explicitly as the Revelation Principle applies to general partial orders. A2 is crucial to ensure that the transformed aggregation function $\tilde{q}$ is well-defined. Combined with Payment Monotonicity, it ensures that payments are invariant if the allocation remains the same, which is necessary for the transformed payment function $\tilde{z}$ to be well-defined. 

It is worth noting that A2 is restrictive in the context of LLMs, where preferences are typically defined via continuous loss functions (e.g., KL-divergence) which naturally have non-trivial level sets (i.e., distinct distributions can be equally preferred). If A2 is relaxed, the Revelation Principle can still theoretically hold for preorders if the mechanism specifies a consistent tie-breaking rule across indifference classes (see the Discussion and Limitation section for the topological limitations of tie-breaking).

To ensure that payment functions and strategy mappings remain well-defined in the measure-theoretic sense (which will be essential for integration in later sections), we explicitly state the following measurability assumption:

\begin{assumption}[A3: Measurability]\label{assmp:a3}
The initial aggregation function $q(\mathbf{b}, p)$ is Borel-measurable with respect to the bids $\mathbf{b}$.
\end{assumption}

By removing any continuity assumption on the aggregation function $q$, our Extended Revelation Principle applies to a broad class of mechanisms, including those with discontinuous allocation rules (e.g., step functions) common in mechanism design.

\paragraph{Extending Stable Sampling.}

The challenge in extending Theorem 3.12 is the reliance on differentiation in the original proof. \textbf{Approach:} We use a measure-theoretic approach. We associate Lebesgue-Stieltjes (LS) measures with the monotone components of the aggregation function and use the Radon-Nikodym theorem to characterize the transport of probability mass. We carefully define the measures on $(0, \infty)$ to rigorously handle the boundary at $b=0$.

Importantly, the general Revelation Principle (Theorem 3.5) established in the next section guarantees the existence of an equivalent monotone mechanism, but it does not necessarily guarantee right-continuity without loss of generality. However, this technical assumption is standard for defining Lebesgue-Stieltjes measures associated with increasing functions. Thus, for the Stable Sampling construction, we explicitly require:

\begin{assumption}[A4: Right-Continuity]\label{assmp:a4}
The aggregation function $q(\mathbf{b}, p)$ is right-continuous in $\mathbf{b}$.
\end{assumption}

\subsubsection*{Extended Revelation Principle}

We now formally prove the extension of the Revelation Principle to $\mathbb{R}_+^n$, assuming A1, A2, and A3 hold.

\paragraph{Utility Representation}

\begin{lemma}[Extended Lemma 3.6]\label{lem:md4llm3.6}
Assume A1. Let the bid space be $\mathbb{R}_+^n$. For any distribution aggregation function $q$, there exists a payment function $z$ such that mechanism $\mathcal{M} = \langle q, z \rangle$ is payment-monotone if and only if $\succeq_i$ establishes a total order over $Q(\mathbf{b}_{-i}, p) = \{q(b_i, \mathbf{b}_{-i}, p) : b_i \in \mathbb{R}_+\}$ for any fixed $\mathbf{b}_{-i}$ and $p$.
\end{lemma}
\begin{proof}
The ``only if'' direction follows directly from the definition of Payment Monotonicity. We prove the ``if'' direction. Fix $i, \mathbf{b}_{-i}, p$. Let $Q = Q(\mathbf{b}_{-i}, p)$. By hypothesis, $\succeq_i$ is a total order on $Q$. Let $\overline{Q}$ be the closure of $Q$ in $\Delta(T)$. Since $T$ is finite, the probability simplex $\Delta(T)$ is a compact metric space, making $\overline{Q}$ a compact metric space, which is second-countable under the subspace topology. By A1, $\succeq_i$ is continuous, so the total order extends to a continuous total preorder on $\overline{Q}$. Debreu's Theorem \cite{debreu1954} guarantees that there exists a continuous utility representation $u_i: \overline{Q} \to \mathbb{R}$ for the preference relation $\succeq_i$ restricted to $\overline{Q}$. We can always choose a bounded representation (e.g., by composing with a strictly increasing, bounded function). Defining the payment function $z_i(b_i, \mathbf{b}_{-i}, p) = u_i(q(b_i, \mathbf{b}_{-i}, p))$ yields a payment-monotone mechanism.
\end{proof}

\paragraph{Monotone Extension Lemmas}

To extend the aggregation function from the image of the strategy mappings to the entire space $\mathbb{R}_+^n$, we need the following lemmas. We provide a supremum-based extension that naturally guarantees monotonicity.

\begin{lemma}[1D Monotone Extension]\label{lem:1dmonoext}
Let $I \subseteq \mathbb{R}_+$. Let $X$ be a compact metric space (e.g., $\Delta(T)$) with a continuous partial order $\succeq$ (A1). Let $g: I \to X$ be a monotone function (non-decreasing) such that its image $g(I)$ is totally ordered by $\succeq$. Let $\overline{A}$ denote the topological closure of a subset $A \subseteq X$. Let $x_{\min} = \min \overline{g(I)}$. Let $Y_t = \{g(y) : y \in I, y \le t\}$. Then the function $G: \mathbb{R}_+ \to X$ defined by
$$
G(t) = \begin{cases} 
\max \overline{Y_t} & \text{if } Y_t \neq \emptyset \\
x_{\min} & \text{if } Y_t = \emptyset
\end{cases}
$$
is a monotone extension of $g$.
\end{lemma}
\begin{proof}
We first show $G(t)$ is well-defined. Since $X$ is compact and $\succeq$ is continuous, the closure of any totally ordered subset of $X$ is compact and totally ordered. Thus, for any $t$ where $Y_t \neq \emptyset$, its closure $\overline{Y_t} \subseteq \overline{g(I)}$ possesses a maximum element in $X$. Similarly, since $X$ is a compact metric space and $\overline{g(I)}$ is a closed subset, $x_{\min} = \min \overline{g(I)}$ is guaranteed to exist. Note that this construction does not require the domain $I$ to be compact.

\textbf{Monotonicity:} If $t' \ge t$. Case 1: $Y_t \neq \emptyset$. Then $Y_{t'} \supseteq Y_t$, which implies $\overline{Y_{t'}} \supseteq \overline{Y_t}$. Thus $\max \overline{Y_{t'}} \succeq \max \overline{Y_t}$, so $G(t') \succeq G(t)$. Case 2: $Y_t = \emptyset$. $G(t) = x_{\min}$. For any $t'$, since $Y_{t'} \subseteq g(I)$, it follows that $\overline{Y_{t'}} \subseteq \overline{g(I)}$, so by definition of $x_{\min}$, $G(t') \succeq x_{\min} = G(t)$.

\textbf{Extension property:} If $t \in I$. Then $Y_t \neq \emptyset$. Since $g$ is monotone, $g(t)$ is the maximum element of $Y_t$. Since $Y_t \subseteq \overline{Y_t}$ and $g(t) \in \overline{Y_t}$, $G(t) = \max \overline{Y_t} = g(t)$.
\end{proof}

\begin{lemma}[Monotonicity Preservation]\label{lem:monopres}
Under the assumptions of Lemma~\ref{lem:1dmonoext}, let $g': I \to X$ be another monotone function such that $g'(I)$ is totally ordered by $\succeq$. Let $\succeq'$ be another continuous partial order on $X$. If $g'(y) \succeq' g(y)$ for all $y \in I$, then their monotone extensions $G'$ and $G$ satisfy $G'(t) \succeq' G(t)$ for all $t \in \mathbb{R}_+$.
\end{lemma}
\begin{proof}
Fix $t \in \mathbb{R}_+$. Case 0: $Y_t = \emptyset$ ($t < \inf I$). $G(t) = x_{\min}$ and $G'(t) = x'_{\min}$. Since $g$ and $g'$ are monotone w.r.t. $\succeq$, $x_{\min}$ and $x'_{\min}$ are realized as the limits of $g(y_k)$ and $g'(y_k)$ respectively, for any sequence $y_k \in I$ such that $y_k \to \inf I$. We are given $g'(y_k) \succeq' g(y_k)$. By continuity of $\succeq'$, taking the limit yields $x'_{\min} \succeq' x_{\min}$. Thus $G'(t) \succeq' G(t)$.

Case 1: $Y_t \neq \emptyset$. Let $y^* = \sup (I \cap [0, t])$. Case 1a: $y^* \in I$. By the extension property in Lemma~\ref{lem:1dmonoext}, $G(t) = g(y^*)$ and $G'(t) = g'(y^*)$. Since $g'(y^*) \succeq' g(y^*)$, we have $G'(t) \succeq' G(t)$. Case 1b: $y^* \notin I$. Since $y^*$ is the supremum of $I \cap [0, t]$ and does not belong to the set, it must be a limit point of the set $I \cap [0, t]$. Therefore, there exists an increasing sequence $y_k \in I \cap [0, t]$ such that $y_k \to y^*$. Since $g$ is monotone w.r.t. $\succeq$, $g(y_k)$ is non-decreasing. As $G(t)$ is the maximum of $\overline{Y_t}$, $G(t) = \lim_{k \to \infty} g(y_k)$. In a compact metric space with a continuous total order, bounded monotone sequences converge to their suprema (in the order topology, which coincides with the subspace metric topology here since the image is totally ordered and closed). Similarly, $G'(t) = \lim_{k \to \infty} g'(y_k)$. We are given $g'(y_k) \succeq' g(y_k)$ for all $k$. By the continuity of the order $\succeq'$ (A1), the order is preserved in the limit. Thus, $G'(t) \succeq' G(t)$.
\end{proof}

\paragraph{Strategic Equivalence}

\begin{lemma}[Extended Lemma 3.7]\label{lem:md4llm3.7}
Assume A1, A2, and A3. Let the bid space be $\mathbb{R}_+^n$. Consider any consistent aggregation function $q$. Suppose $\succeq_i$ defines a total order over the aggregation set $Q(\mathbf{b}_{-i}, p)$ for any fixed $\mathbf{b}_{-i}$ and $p$. Then there exists a profile $\pi$ of strategy mappings such that $q(\mathbf{b}, p) = \tilde{q}(\pi(\mathbf{b}), p)$ for some monotone aggregation function $\tilde{q}(\cdot, \cdot)$.
\end{lemma}
\begin{proof}
We proceed in several steps.

\textbf{Step 1: Induced Preference.} Define the induced preference $\succeq_{i,p}$ on $\mathbb{R}_+$: $b_i \succeq_{i,p} b'_i$ if $q(b_i, \mathbf{b}_{-i}, p) \succeq_i q(b'_i, \mathbf{b}_{-i}, p)$ for all $\mathbf{b}_{-i}$. Consistency ensures this is a total preorder on $\mathbb{R}_+$.

\textbf{Step 2: Strategy Mapping via Pull-back.} We seek a utility representation $\pi_i: \mathbb{R}_+ \to \mathbb{R}_+$ for the induced preorder $\succeq_{i,p}$. Let $Q^*$ be the image of $q(\cdot, \mathbf{b}^*_{-i}, p)$ for an arbitrary fixed reference profile $\mathbf{b}^*_{-i}$. By the premise, $\succeq_i$ is a total order on $Q^*$. Let $\overline{Q^*}$ be its closure in $\Delta(T)$. Since the token set $T$ is finite, $\overline{Q^*}$ is a second-countable, compact metric space. By A1, the preference relation $\succeq_i$ is continuous. The continuity of $\succeq_i$ ensures that the total order property of $Q^*$ transfers to its closure $\overline{Q^*}$. Thus, $\succeq_i$ is a continuous total preorder on $\overline{Q^*}$. By Debreu's theorem \cite{debreu1954}, the continuous preference relation $\succeq_i$ on $\overline{Q^*}$ admits a continuous utility representation $u_i: \overline{Q^*} \to \mathbb{R}$. We can assume without loss of generality that the image of $u_i$ is in $\mathbb{R}_+$ (by composing with a strictly increasing, positive-valued function if necessary). We define the strategy mapping by pulling back the utility from the codomain: $\pi_i(b_i) := u_i(q(b_i, \mathbf{b}^*_{-i}, p))$. By Assumption A3, $q$ is Borel-measurable. Since $u_i$ is continuous, the composition $\pi_i$ is measurable. The Consistent Aggregation property (Definition~\ref{def:constagg}) ensures that the preference ordering induced by this mapping is independent of the choice of the reference profile $\mathbf{b}^*_{-i}$. By definition of $\succeq_{i,p}$, $b_i \succeq_{i,p} b'_i \iff q(b_i, \mathbf{b}^*_{-i}, p) \succeq_i q(b'_i, \mathbf{b}^*_{-i}, p) \iff u_i(q(b_i, \mathbf{b}^*_{-i}, p)) \ge u_i(q(b'_i, \mathbf{b}^*_{-i}, p)) \iff \pi_i(b_i) \ge \pi_i(b'_i)$. Thus, $\pi_i$ is a valid utility representation for $\succeq_{i,p}$. This rigorous construction bypasses the need for an order-density assumption on the bid space $\mathbb{R}_+$ under the potentially disconnected topology induced by $\succeq_{i,p}$.

\textbf{Step 3: Defining $\tilde{q}$ on the Image.} Let $I_i$ be the image of $\pi_i$, and $I = I_1 \times \dots \times I_n$. Define $\tilde{q}(\pi(\mathbf{b}), p) = q(\mathbf{b}, p)$. We must show this is well-defined. If $\pi(\mathbf{b}) = \pi(\mathbf{b}')$, then $b_i \sim_{i,p} b'_i$ for all $i$. This means $q(b_i, \mathbf{b}_{-i}, p) \sim_i q(b'_i, \mathbf{b}_{-i}, p)$ for all $\mathbf{b}_{-i}$. By A2 (Anti-symmetry), this implies equality of the distributions. We show $q(\mathbf{b}) = q(\mathbf{b}')$ by changing bids one by one. Let $\mathbf{b}^{(k)} = (b'_1, \dots, b'_k, b_{k+1}, \dots, b_n)$ be the bid profile with the first $k$ bids from $\mathbf{b}'$ and the rest from $\mathbf{b}$. Since $b_{k+1} \sim_{k+1,p} b'_{k+1}$, we have $q(\mathbf{b}^{(k)}) \sim_{k+1} q(\mathbf{b}^{(k+1)})$. By A2, $q(\mathbf{b}^{(k)}) = q(\mathbf{b}^{(k+1)})$. Thus, $q(\mathbf{b}) = q(\mathbf{b}')$. The function $\tilde{q}$ is monotone on $I$ by construction of $\pi$. Furthermore, the 1D images of $\tilde{q}$ are totally ordered since the 1D images of $q$ are assumed to be totally ordered.

\textbf{Step 4: Monotone Extension.} We extend $\tilde{q}$ from $I$ to $\mathbb{R}_+^n$ iteratively. In step $k$, we extend the domain from $D_{k-1} = \mathbb{R}_+^{k-1} \times I_k \times \dots \times I_n$ to $D_k = \mathbb{R}_+^k \times I_{k+1} \times \dots \times I_n$. We use Lemma~\ref{lem:1dmonoext} to extend along the $k$-th dimension using the standard numerical order $\ge$ on the $k$-th coordinate. This is possible because the required conditions (A1, compactness of $\Delta(T)$, total order on 1D images) are met. We must ensure that this extension preserves the monotonicity w.r.t. other orders $\succeq_j$. For $j > k$, this follows from Lemma~\ref{lem:monopres} (with $\succeq = \succeq_k$ and $\succeq' = \succeq_j$). For $j < k$, the domain is already $\mathbb{R}_+$ in the $j$-th dimension, and the extension in the $k$-th dimension also preserves monotonicity in the $j$-th dimension by Lemma~\ref{lem:monopres} (by viewing the $j$-th coordinate as the parameter defining $g$ vs $g'$, and the $k$-th coordinate as the domain being extended). This iterative process yields the monotone extension $\tilde{q}$ on $\mathbb{R}_+^n$. Furthermore, since the extended aggregation function $\tilde{q}$ is monotone in each coordinate by construction, it is continuous almost everywhere with respect to the Lebesgue measure, and thus preserves joint Borel-measurability on the entire domain $\mathbb{R}_+^n$ (e.g., by Lebesgue's theorem for monotone functions), thus satisfying the measurability requirement of the mechanism (Assumption A3).
\end{proof}

\paragraph{The Main Theorem}

\begin{theorem}[Extended Theorem 3.5]\label{thm:md4llm3.5}
Let the bid space be $\mathbb{R}_+^n$. Assume A1, A2, and A3. Any mechanism $\mathcal{M} = \langle q, z \rangle$ with a consistent aggregation function $q$ and a monotone payment function $z$ is strategically equivalent to a mechanism $\tilde{\mathcal{M}} = \langle \tilde{q}, \tilde{z} \rangle$ which has a monotone aggregation function $\tilde{q}$ and a monotone payment function $\tilde{z}$.
\end{theorem}
\begin{proof}
By Payment Monotonicity of $\mathcal{M}$ and Lemma~\ref{lem:md4llm3.6}, the prerequisites for Lemma~\ref{lem:md4llm3.7} are met. Lemma~\ref{lem:md4llm3.7} yields the mappings $\pi$ and the monotone function $\tilde{q}$. We define $\tilde{z}$ on the image $I$ by $\tilde{z}(\pi(\mathbf{b}), p) = z(\mathbf{b}, p)$. If $\pi(\mathbf{b}) = \pi(\mathbf{b}')$, then $b_i \sim_{i,p} b'_i$ for all $i$. As shown in the proof of Lemma~\ref{lem:md4llm3.7} (Step 3), under A2 this implies $q(\mathbf{b}, p) = q(\mathbf{b}', p)$. Since the outcomes are identical, by Payment Monotonicity, $z_i(\mathbf{b}, p) = z_i(\mathbf{b}', p)$ for all $i$. Thus, $\tilde{z}$ is well-defined. $\mathcal{M}$ satisfies Payment Monotonicity because $\mathcal{M}$ does and $\tilde{q}, \tilde{z}$ are defined via the preference representation $\pi$. 

Finally, we extend $\tilde{z}$ from $I$ to $\mathbb{R}_+^n$. Note that the original mechanism $\mathcal{M}$ is fixed, so its payment function $z$ is predetermined and may be unbounded. To accommodate this, we take the codomain $X$ for the extension to be the extended real line $\overline{\mathbb{R}} = [-\infty, \infty]$, which is a compact metric space under the standard order topology. The image of $\tilde{z}_i$ on $I$ is a subset of $\overline{\mathbb{R}}$. Since $\tilde{z}_i$ is monotone in each component (with the standard order on $\overline{\mathbb{R}}$), we can apply the Monotone Extension Lemma (Lemmas~\ref{lem:1dmonoext} and \ref{lem:monopres} applied to the standard order on $\overline{\mathbb{R}}$) iteratively to extend $\tilde{z}$ to $\mathbb{R}_+^n$. 

Note that the iterative dimension-by-dimension extension in Step 4 of Lemma~\ref{lem:md4llm3.7} means that the final extended mechanism depends on the order in which the dimensions are processed. This renders the constructed equivalent mechanism non-unique. However, since the equivalence holds on the image of the strategy mappings, these choices correspond to off-path bids and do not affect the on-path behaviors, strategic properties, or equilibrium payments.
\end{proof}

\subsubsection*{Extended Stable Sampling}

We now generalize Theorem 3.12 to $\mathbb{R}_+$ using measure theory, relying on assumption A4 (Right-Continuity) for the monotone aggregation function. We utilize the definitions of $T^+$ and $T^-$ from Lemma 3.10 in \cite{dutting2024}. For notational precision, we define limits at infinity as $q_o(\infty) := \lim_{b\to\infty} q_o(b)$ and $Q_-(\infty) := \lim_{b\to\infty} Q_-(b)$.

\begin{theorem}[Extended Theorem 3.12]\label{thm:md4llm3.12}
Given a monotone distribution aggregation function $q$ satisfying A4, for any agent $i$ with robust preferences (and fixed $\mathbf{b}_{-i}, p$), there exists a stable implementation $\sigma$ of $q(b_i)$.
\end{theorem}
\begin{proof}
Let $q(b) = q(b_i, \mathbf{b}_{-i}, p)$. Following \cite{dutting2024}, define $T^+ = \{t \in T : q_t(0) \le (p_i)_t\}$ (weakly undersampled) and $T^- = \{t \in T : q_t(0) > (p_i)_t\}$ (strictly oversampled). These sets form a partition of $T$. By Lemma 3.10 in \cite{dutting2024}, $q_t(b)$ is non-decreasing for $t \in T^+$ and non-increasing for $t \in T^-$. By A4, these functions are right-continuous.

\textbf{Step 0: Setup and Trivial Case.} We define associated Lebesgue-Stieltjes (LS) measures on $(0, \infty)$. For $u \in T^+$, let $\nu_u$ be the LS measure of the non-decreasing function $q_u(b)$, defined such that $\nu_u((a, b]) = q_u(b) - q_u(a)$ for $0 \le a < b$. For $o \in T^-$, let $\nu_o$ be the LS measure of the nondecreasing function $q_o(0) - q_o(b)$, defined such that $\nu_o((a, b]) = q_o(a) - q_o(b)$ for $0 \le a < b$. Let $Q_+(b) = \sum_{u \in T^+} q_u(b)$. Let $\nu$ be the LS measure of $Q_+(b)$ on $(0, \infty)$. By the Right-Continuity assumption (A4), $Q_+(0) = Q_+(0^+)$, which guarantees that the measure places no point mass at 0, i.e., $\nu(\{0\}) = 0$. This ensures the integration domain is strictly $(0, b_i]$ as used in subsequent integrations. By conservation of probability, $\sum_u \nu_u = \nu = \sum_o \nu_o$. Let $M = \nu((0, \infty)) = Q_+(\infty) - Q_+(0^+)$ be the total mass moved. If $M = 0$, then $q(b) = q(0)$ for all $b$. Let $q^* = q(0)$. We define the implementation $\sigma(b, r)$ for the trivial case $M = 0$ directly on $\mathcal{R}$ by mapping all $r$ to a sample from $q^*$. This is trivially stable. Assume $M > 0$.

\textbf{Step 1: Measure Theoretic Setup.} Since $\nu((0, \infty)) = M \le 1$, $\nu$ is a finite measure, which strictly fulfills the $\sigma$-finite condition for the Radon-Nikodym theorem. Since $\nu_t(A) \le \nu(A)$ on $(0, \infty)$, $\nu_t$ is absolutely continuous w.r.t. $\nu$. Let $f_t(\theta) = \frac{d\nu_t}{d\nu}(\theta)$ be the Radon-Nikodym derivative for $\theta \in (0, \infty)$. Note that these derivatives are unique only up to a set of $\nu$-measure zero. Consequently, the sampling implementation $\sigma(b, r)$ constructed below is determined $\nu$-almost everywhere. Using $u \in T^+$ and $o \in T^-$ as index variables for the respective subsets, we have $\sum_{u \in T^+} f_u(\theta) = 1$ and $\sum_{o \in T^-} f_o(\theta) = 1$ ($\nu$-a.e.). We define a joint measure $\mu$ on $T^- \times T^+ \times (0, \infty)$ representing the mass transport. For any measurable set $A \subseteq T^-$, $B \subseteq T^+$, and $E \subseteq (0, \infty)$, the measure is defined via its action on measurable product sets:
\[ \mu(A \times B \times E) = \int_E \left( \sum_{o \in A} f_o(\theta) \right) \left( \sum_{u \in B} f_u(\theta) \right) d\nu(\theta). \]

\textbf{Step 2: Construction of Stable Sampling $\sigma$.} We define the randomness space $\mathcal{R}$ and the implementation $\sigma(b, r)$. Let $S^+ = Q_+(0)$ and $S^- = Q_-(\infty) = \sum_{o \in T^-} q_o(\infty)$. Note $S^+ + S^- + M = Q_+(0) + Q_-(\infty) + (Q_+(\infty) - Q_+(0^+))$. By the Right-Continuity assumption (A4), $Q_+(0) = Q_+(0^+)$. Additionally, by the definition of the partition, we have the identity $Q_-(\infty) = 1 - Q_+(\infty)$. Substituting these, the expression simplifies algebraically to $(1 - Q_+(\infty)) + Q_+(\infty) = 1$. We define the randomness space $\mathcal{R}$ as the disjoint union $\mathcal{R}_{S^+} \cup \mathcal{R}_{S^-} \cup \mathcal{R}_M$, where $\mathcal{R}_{S^+} = T^+$, $\mathcal{R}_{S^-} = T^-$, and $\mathcal{R}_M = T^- \times T^+ \times (0, \infty)$. We equip $\mathcal{R}$ with the $\sigma$-algebra $\Sigma_{\mathcal{R}}$ formed by the disjoint union of the discrete $\sigma$-algebras on the finite sets $T^+$ and $T^-$, and the Borel $\sigma$-algebra on $\mathcal{R}_M$ (with the discrete topology on $T^-$ and $T^+$). We define the probability measure $P$ on $(\mathcal{R}, \Sigma_{\mathcal{R}})$ as follows:
\begin{itemize}
    \item On $\mathcal{R}_{S^+}$: For $u \in \mathcal{R}_{S^+}$, $P(u) = q_u(0)$.
    \item On $\mathcal{R}_{S^-}$: For $o \in \mathcal{R}_{S^-}$, $P(o) = q_o(\infty)$.
    \item On $\mathcal{R}_M$: The restriction of the probability measure $P$ to $\mathcal{R}_M$ is the joint measure on the Borel $\sigma$-algebra given by $\mu$ defined in Step 1.
\end{itemize}
We define the implementation $\sigma(b, r)$ for $r \in \mathcal{R}$ as follows: 1. If $r \in \mathcal{R}_{S^+}$ (Static $T^+$), $r = u$. Set $\sigma(b, r) = u$. 2. If $r \in \mathcal{R}_{S^-}$ (Static $T^-$), $r = o$. Set $\sigma(b, r) = o$. 3. If $r \in \mathcal{R}_M$ (Moving Mass), $r = (o, u, \theta)$. Since $\theta > 0$, we set $\sigma(b, r) = u$ if $b \ge \theta$, and $\sigma(b, r) = o$ if $b < \theta$.

\textbf{Step 3: Verification.} \textit{Stability:} By construction, for any realization $r$, the output $\sigma(b, r)$ as a function of $b$ is either constant (cases 1 and 2) or switches exactly once from $o \in T^-$ to $u \in T^+$ at $b = \theta$ (case 3). 

\textit{Measurability:} By construction, the components of $\sigma(b, r)$ are either constant or simple indicator functions of the form $\mathbb{I}(b \ge \theta)$ for $\theta \in (0, \infty)$, which are Borel measurable. Thus, $\sigma(b, r)$ is a valid random variable with respect to $\Sigma_{\mathcal{R}}$.

\textit{Correctness:} We verify the marginal probability for $u \in T^+$. $P(\sigma(b) = u) = P(r = u) + P(r \in \mathcal{R}_M, \sigma(b, r) = u)$. The first term is $q_u(0)$.
\begin{align*}
        P(\text{Moving part}) &= \int_{\mathcal{R}_M} \mathbb{I}(\sigma(b, r) = u) dP(r) \\
        &= \sum_{o \in T^-} \int_{(0, \infty)} \mathbb{I}(b \ge \theta) d\mu(o, u, \theta) \\
        &= \int_{(0, b]} \sum_{o \in T^-} f_o(\theta) f_u(\theta) d\nu(\theta).
\end{align*}
Since $\sum_{o \in T^-} f_o(\theta) = 1$ ($\nu$-a.e.), this equals:
\[ = \int_{(0, b]} f_u(\theta) d\nu(\theta) = \int_{(0, b]} d\nu_u(\theta) = \nu_u((0, b]). \]
By right-continuity (A4), the LS measure satisfies $\nu_u((0, b]) = q_u(b) - q_u(0)$. Thus, $P(\sigma(b) = u) = q_u(0) + (q_u(b) - q_u(0)) = q_u(b)$. The verification for $o \in T^-$ is similar.
\end{proof}

\begin{remark}[Algorithmic Constructiveness]
The proof of Theorem~\ref{thm:md4llm3.12} provides an existence result for stable sampling based on a measure-theoretic construction involving Radon-Nikodym derivatives $f_t(\theta)$. While theoretically sound, this construction does not immediately yield an efficient sampling algorithm. Note that while the standard inverse-CDF method ($F^{-1}(u) = \inf\{x : F(x) \ge u\}$) is universally applicable to any 1D cumulative distribution function (requiring only right-continuity and non-decreasing properties), the difficulty in the multidimensional LLM implementation arises from the measure-theoretic coupling of probabilities across different tokens, rather than a lack of absolute continuity. In practical computational settings, the Radon-Nikodym derivatives $f_t(\theta)$ can be approximated. For example, the bid space can be discretized into a fine grid to estimate the derivatives using finite differences on the empirical cumulative distributions, or one can parameterize the measures using continuously differentiable generative models where the densities are explicitly tractable.
\end{remark}

Consequently, Proposition 3.13 (Myerson-style payment formula) also generalizes to $\mathbb{R}_+$. We note that the formula presented in the reference paper \cite{dutting2024}, $z_i(b_i) = \frac{1}{2} \int_0^{b_i} (||q(b_i) - p_i||_1 - ||q(b') - p_i||_1) db'$, contained a sign error, leading to non-positive payments as the $L_1$ distance is decreasing in $b$. We provide the corrected formula and its derivation.

\begin{proposition}[Corrected Payment Formula]\label{prop:corrpay}
Under robust preferences and a monotone aggregation function $q$ satisfying A4, the expected payment $z_i(b_i)$ induced by the second price rule via stable sampling satisfies:
\[ z_i(b_i) = \frac{1}{2} \int_{(0, b_i]} (||q(\theta) - p_i||_1 - ||q(b_i) - p_i||_1) d\theta, \quad \forall b_i \in \mathbb{R}_+. \]
\end{proposition}
\begin{proof}
We derive the payment formula using the stable implementation $\sigma$ constructed in Theorem~\ref{thm:md4llm3.12}. Let $Z_i(b_i, r)$ be the random variable representing the ex-post payment for agent $i$ given a bid $b_i$ and a realization $r \in \mathcal{R}$. Under the second-price rule, the payment is the critical bid at which the outcome changes. For a realization $r \in \mathcal{R}_M$, where $r = (o, u, \theta)$, the outcome changes from $o$ to $u$ at bid $\theta$. Thus, $Z_i(b_i, r) = \theta$ if $b_i \ge \theta$, and 0 otherwise. For $r \in \mathcal{R}_{S^+} \cup \mathcal{R}_{S^-}$, the outcome is constant, so the critical bid is 0, meaning $Z_i(b_i, r) = 0$. The expected payment is:
\[ z_i(b_i) = \mathbb{E}_r[Z_i(b_i, r)] = \int_{\mathcal{R}_M} Z_i(b_i, r) dP(r) = \sum_{o \in T^-} \sum_{u \in T^+} \int_{(0, b_i]} \theta d\mu(o, u, \theta). \]
Using the definition of $\mu$ and the fact that $\sum_{o \in T^-} f_o(\theta) = 1$ and $\sum_{u \in T^+} f_u(\theta) = 1$ ($\nu$-a.e.):
\[ z_i(b_i) = \int_{(0, b_i]} \theta \left( \sum_{o \in T^-} f_o(\theta) \right) \left( \sum_{u \in T^+} f_u(\theta) \right) d\nu(\theta) = \int_{(0, b_i]} \theta d\nu(\theta). \]

Recall that $\nu$ is the LS measure associated with the increasing, right-continuous function $Q_+(b)$ on $(0, \infty)$. Using the standard integration by parts formula for Lebesgue-Stieltjes integrals on half-open intervals, $\int_{(a, b]} U dV = U(b)V(b) - U(a)V(a) - \int_{(a, b]} V(x-) dU(x)$. By setting $a = 0$, $U(x) = x$, and $V(x) = Q_+(x)$, the boundary term at $a = 0$ evaluates directly as $0 \cdot Q_+(0) = 0$. Furthermore, noting that $Q_+(x-) = Q_+(x)$ almost everywhere with respect to the Lebesgue measure $dU(x) = dx$, we obtain:
\[ z_i(b_i) = \int_{(0, b_i]} \theta dQ_+(\theta) = b_i Q_+(b_i) - \int_{(0, b_i]} Q_+(\theta) d\theta. \]
This yields the final integral directly without the need for the intermediate limit notation. This is the standard Myerson payment formula with allocation probability $Q_+(b)$. Now we relate this to the $L_1$ distance. Let $D(b) = ||q(b) - p_i||_1$. Under robust preferences (Lemma 3.10 in \cite{dutting2024}), let $P_+ = \sum_{u \in T^+} (p_i)_u$ and explicitly define $P_- = \sum_{o \in T^-} (p_i)_o$. We analyze the distance:
\[ D(b) = \sum_{u \in T^+} ((p_i)_u - q_u(b)) + \sum_{o \in T^-} (q_o(b) - (p_i)_o) = (P_+ - Q_+(b)) + (Q_-(b) - P_-). \]
Since $Q_-(b) = 1 - Q_+(b)$ and $P_- = 1 - P_+$, we have $D(b) = 2(P_+ - Q_+(b))$. We now evaluate the integral $I$ in the proposition statement:
\begin{align*}
        I &= \frac{1}{2} \int_{(0, b_i]} (D(\theta) - D(b_i)) d\theta \\
        &= \int_{(0, b_i]} ((P_+ - Q_+(\theta)) - (P_+ - Q_+(b_i))) d\theta \\
        &= b_i Q_+(b_i) - \int_{(0, b_i]} Q_+(\theta) d\theta.
\end{align*}
Thus, $z_i(b_i) = I$. This confirms the corrected formula and ensures non-negative payments, as $D(b)$ is non-increasing in $b$.
\end{proof}

\subsubsection*{Discussion and Limitations}
This note successfully extends the core theoretical results of the reference paper from rational to real-valued bids by employing topological and measure-theoretic tools, under the practical assumption of a finite token vocabulary. However, these generalizations reveal important limitations and suggest directions for future research.

\textbf{Restrictiveness of the Anti-Symmetry Assumption (A2) in the LLM Context.} The reliance on anti-symmetry (A2) for the Revelation Principle excludes scenarios where agents are indifferent between distinct distributions. This is a significant limitation in the context of LLMs, where preferences are typically defined via continuous loss functions (e.g., KL-divergence). Continuous loss functions naturally have non-trivial level sets (indifference classes), meaning distinct distributions can be equally preferred, which violates A2.

\textbf{Topological Incompatibility of Tie-Breaking Rules.} A natural attempt to satisfy the Anti-Symmetry assumption (A2) in the presence of indifference classes is to introduce a deterministic tie-breaking rule, such as a lexicographic ordering over the token probabilities. However, it is a well-established result in mathematical utility theory (e.g., \cite{debreu1954}, \cite{schmeidler1971}) that lexicographic orders on spaces of dimension greater than one (such as $\Delta(T)$ for $|T| \ge 3$) do not admit continuous utility representations. Specifically, such tie-breaking rules violate the Continuity Assumption (A1) by creating non-closed preference graphs (open sets of strictly preferred outcomes). Since the Extended Revelation Principle critically relies on Debreu's Representation Theorem to construct the strategy mappings $\pi_i$, violating A1 invalidates the utility representation step in Lemma~\ref{lem:md4llm3.6} and Lemma~\ref{lem:md4llm3.7}. Therefore, A1 and A2 cannot be simultaneously satisfied by simply overlaying a standard tie-breaking rule on continuous loss functions. Relaxing A2 by designing appropriate tie-breaking mechanisms that preserve continuity remains an open problem.

\textbf{The Continuity Gap in the Revelation Principle.} There is a fundamental tension between the generality of the Extended Revelation Principle (Theorem~\ref{thm:md4llm3.5}) and the requirements for Stable Sampling (Theorem~\ref{thm:md4llm3.12}). The Revelation Principle applies broadly, even to discontinuous mechanisms, by constructing an equivalent monotone mechanism. However, this constructed mechanism is not guaranteed to be right-continuous (A4). Since the strategy mappings $\pi$ and the underlying aggregation function $q$ are not required to be continuous, the image space $I$ can be disconnected, and the utility values can have jumps. Consequently, any extension scheme---whether based on the supremum or the infimum---cannot universally guarantee right-continuity. Specifically, an infimum-based extension from the right would fail to be right-continuous at the boundary points of the gaps in $I$.

While this leaves a theoretical gap between the mechanisms covered by the Revelation Principle and those implementable via Stable Sampling, it is important to emphasize that the Revelation Principle successfully establishes that we can, without loss of generality, focus on \textit{monotone} mechanisms. This aligns with standard auction theory, where monotonicity is the cornerstone of incentive compatibility. Given the weak assumptions required, this is a positive foundational result. Bridging this gap---either by identifying conditions under which right-continuity can be preserved, extending Stable Sampling to left-continuous mechanisms, or proving that non-right-continuous mechanisms offer no strategic advantage---is an interesting direction for future research that goes beyond the scope of this note.

\newpage
\subsection{Networked Information Aggregation for Binary Classification}
\textit{Written by MohammadHossein Bateni, Zahra Hadizadeh, MohammadTaghi Hajiaghayi, Mahdi JafariRaviz, and Shayan Taherijam.}

\subsubsection*{Problem Context}

Recent work by Kearns et al.~\cite{kearns2026networked} shows basic limits on information aggregation in distributed learning.
In their framework, agents in a Directed Acyclic Graph (DAG) act in sequence and minimize Mean Squared Error (MSE), while each agent only sees a subset of the input features and the predictions made by earlier agents.

Moving from regression to binary classification, using a non-linear link function such as logistic regression, raises new challenges. MSE can still be used, but Binary Cross Entropy (BCE) is the standard loss for classification. BCE is harder to study because it is not quadratic. It is open whether the aggregation limits seen with MSE also show up in this classification setting.
\\\\
\textbf{AI Contribution}

The AI acted as a theoretical collaborator to extend distributed learning limits from linear regression (MSE loss) to binary classification (Binary Cross Entropy loss). By breaking the proof into modular lemmas, the model successfully utilized KL-divergence and Pinsker's inequality to rigorously bound the excess risk of the final agent in the network.

Overall, the structured prompt in Figure~\ref{fig:prompt} below was instrumental in generating rigorous proofs at several stages of our experience to prove the full theorem.

\begin{figure}[!htb]
\begin{chatbox}
\begin{userturn}
Task: Based on the attached PDF, prove conjecture/theorem X below. You must provide a rigorous, bug-free proof decomposed into elementary lemmas.

Strict Standards:

Rigor > Completion: A rigorous partial analysis (Outcome 2) is vastly superior to a flawed complete proof.

No Hallucinations: Every lemma must be proven from elementary principles.

Verification: Explicitly verify every step. If you find a gap, stop and report it.

Required Output Format (Choose One):

Outcome 1: COMPLETE PROOF (Use ONLY if every step is 100\% rigorous and all cases are covered. Provide a formal, stand-alone proof.)

Outcome 2: STRUCTURED PARTIAL PROGRESS (Use if ANY logical gaps or unproven assumptions exist.)

Proven Lemmas: Rigorously prove what you can.

The Crux: Pinpoint exactly where the proof stalls.

Next Steps: Propose strategies to bridge the gap.
\end{userturn}
\end{chatbox}
\caption{The structured prompt above has led to rigorous proofs, well-organized proof structures, and clear technical insights. The attached PDF represents the current proof structure developed with the assistance of Gemini and further refined by our own reasoning.}
\label{fig:prompt}
\end{figure}

\subsection*{Details of the Process}
Aaron Roth, one of the authors of the original framework~\cite{kearns2026networked}, shared how his team integrated AI into their research. He described treating the AI as a colleague, providing high-level mathematical intuition which the model expanded into formal proofs. He noted that while this accelerated their work, it required careful verification to catch ``human-like'' errors.

Adopting this approach, we considered replacing the loss in their framework with BCE using logistic regression. We asked Gemini 3 Pro 
to derive a theorem like the one in the linear regression setting. It produced a short chain of lemmas and a final theorem that bounds the excess loss of the last node in the DAG, compared to a logistic regression learner that has access to all input features.

While checking the output, we found one key first lemma (Lemma~\ref{network-agg:lem:orthogonality}) and verified that Gemini proved it correctly. The lemma states that the minimizer of the BCE loss outputs predictions $p^*(x)$ such that the residual $p^*(x)-y$ (where $y$ is the binary label) is orthogonal to the input $x$ in expectation. Formally, $\mathbb{E}[x(p^*(x)-y)] = 0$. We knew this would matter because it is also the first building block in the MSE setting. The next lemma (Lemma~\ref{network-agg:lem:pythagorean}) compared the loss of the optimal predictor to the loss of any other logistic predictor. Gemini correctly identified that the loss gap can be written using the Kullback--Leibler divergence between the predictors' output distributions. The lemma itself was needed for the rest of the argument, but Gemini made a mistake near the end of the proof, and we fixed it.

The next step was the final theorem bounding the excess loss. The bound it gave was stronger than in the linear regression setting: there the bound was $\frac{M}{\sqrt{D}}$, while here Gemini claimed $\frac{M^2}{D}$, which is smaller when $\frac{M}{\sqrt{D}} < 1$. This is where Gemini started to go off track. Instead of breaking the proof into short lemmas, it wrote one long proof that was not correct. It also added a major new assumption: \textit{Assume the loss is $\mu$-strongly convex.}

We kept the theorem statement, including the assumption and the stronger bound, and asked Gemini to regenerate a proof in a fresh context. Here the prompt in Figure~\ref{fig:prompt} was especially helpful and this time the proof was correct. However, the $\mu$-strong convexity assumption felt indirect.

Next, we asked Gemini to remove the $\mu$-strong convexity assumption and replace it with something closer to the boundedness assumptions used by Kearns et al. It produced a correct proof, with the weaker bound of $\frac{M}{\sqrt{D}}$, under a new assumption: the coefficient vectors of every agent's predictor, and of the optimal predictor, have bounded $L_1$ norm. This was still a stronger assumption than what was needed in the linear regression setting, where the bound only applies to the optimal predictor.

We did not think this was a major issue, so we accepted the theorem statement. At that point, we believed the theorem and proof were correct. Still, we prompted Gemini to look for flaws in the proof or statement.

It found an error we had missed: you can compare convexity bounds for a parameter $\theta$ to the optimal parameter $\theta^*$ only when the loss is defined over the same set of indices. Gemini also suggested the fix: since $\theta^*$ lives on the full set of indices, we can pad $\theta$ with zeros to get $\theta'$, and also pad and add the predecessors' parameter vectors in the same coordinate system. The problem was that we assumed $\theta$ has bounded $L_1$ norm, and now we would also need $\theta'$ to have bounded norm. But $\theta'$ depends on predecessors' parameters through this padding-and-summing step, so the norm can grow quickly with depth, and it did not seem reasonable to assume it stays bounded.

We then took a different path. We noticed that the orthogonality lemma was mainly used for a predictor against its own subset of input indices, and orthogonality to predecessors' logit outputs was not really used.

We changed the assumptions to only require boundedness of the optimal predictor, and asked Gemini to prove the theorem under this smaller assumption, hinting that it should use the orthogonality lemma. Again the prompt in Figure~\ref{fig:prompt} was very helpful. As a result, Gemini did so by breaking the proof into lemmas, reusing the correct parts under the new assumption, and adding new helper lemmas to reach the final bound.

We then cleaned up the writing and the lemma flow, again using Gemini for rewriting and proofreading.

\subsubsection*{The theorem}
We consider a distributed learning system where agents are arranged in a Directed Acyclic Graph (DAG).
Let $Pa(i)$ denote the set of predecessors of agent $A_i$.
Agents are indexed $1, \dots, N$ consistent with a topological sort of the graph.
Each agent $A_i$ observes a local subset of features $x_{S_i}$ from the input $x \in \mathbb{R}^d$ and the logits $\{z_j\}_{j \in Pa(i)}$ from its predecessors.
The agent computes its own logit $z_i$ and prediction $p_i = \sigma(z_i)$ (where $\sigma$ is the sigmoid function $1/(1+e^{-x})$) using learnable parameters $w_i$ and $v_{ij}$:
\[ z_i = w_i^\top x_{S_i} + \sum_{j \in Pa(i)} v_{ij} z_j. \]
Agents sequentially update their parameters to minimize the expected Binary Cross Entropy (BCE) loss with respect to the target $y \in \{0, 1\}$:
\[ L(p_i) = -\mathbb{E} \left[ y \log p_i + (1-y) \log (1-p_i) \right]. \]
We also use the notation $z_\theta = \theta^Tx$, $p_\theta = \sigma(z_\theta)$, and $L(\theta) = -\mathbb{E} \left[ y \log p_\theta + (1-y) \log (1-p_\theta)\right]$.

We now give the following definition.

\begin{definition}[$M$-Coverage Condition, from Kearns et al.\cite{kearns2026networked}]
A path satisfies the $M$-coverage condition if every contiguous subsequence of $M$ agents collectively observes all $d$ features $x_1, \dots, x_d$.
\end{definition}

We aim to prove the following theorem.
\begin{theorem}[Global Convergence Rate]\label{network-agg:thm:convergence}
Consider a DAG $G$ containing a path of length $D$ of agents $A_1,\dots,A_D$ satisfying the $M$-coverage condition. Let $p^*$ be the global optimal logistic predictor over all $d$ features. Assume:
\begin{enumerate}
    \item Bounded second moments: $\mathbb{E} [x_l^2] \le B_X^2$ for all $l \in \{1,\dots,d\}$.
    \item Bounded coefficients: for the optimal logits $z^*(x) = \sum_l \alpha_lx_l$ where $\| \alpha \|_1 \le B_{p^*}$.
\end{enumerate}
Then the excess risk of the final agent $p_D$ is bounded by
\begin{equation*}
L(p_D) - L(p^*) \leq B_{p^*} B_X \frac{M}{\sqrt{D}} =O\left( \frac{M}{\sqrt{D}} \right).
\end{equation*}
\end{theorem}

We begin with the following lemma.

\begin{lemma}[Orthogonality of Residuals] \label{network-agg:lem:orthogonality}
Let $p^*$ be the optimal logistic predictor on a feature space $\mathcal{X}$. The residual error $(p^*(x) - y)$ is orthogonal to the feature vector $x$ in expectation:
\begin{equation*}
    \mathbb{E} \left[ x (p^*(x) - y) \right] = 0.
\end{equation*}
\end{lemma}
\begin{proof}
The gradient of the logistic output is $\nabla_\theta p_\theta(x) = p_\theta(x) (1 - p_\theta(x)) x$. 
So we get
\begin{align*}
    \nabla_\theta L(\theta) &= -\mathbb{E} \left[ \left( \frac{y}{p_\theta(x)} - \frac{1-y}{1 - p_\theta(x)} \right) p_\theta(x)(1 - p_\theta(x)) x \right]
    = \mathbb{E} \left[ (p_\theta(x) - y) x \right].
\end{align*}
The optimal parameters $\theta^*$ must satisfy the condition $\nabla_\theta L(\theta^*) = 0$. Thus,
\begin{equation*}
    \nabla_\theta L(\theta) = \mathbb{E} \left[ x (p^*(x) - y) \right] = 0. \qedhere
\end{equation*}
\end{proof}

This orthogonality allows us to decompose the error of any suboptimal model. We will relate the excess loss of a logistic predictor $q$ by the optimal logistic predictor $p$, using the expected Kullback-Leibler divergence of the Bernoulli distributions with parameters $q(x)$ and $p(x)$. We denote this measure as $D(p \| q)$ defined by
\begin{align*}
D(p \| q) &= \mathbb{E}\left[ D_{\mathrm{KL}}(\text{Bernoulli}(p(x)) \| \text{Bernoulli}(q(x)) \right] \\
&= \mathbb{E}\left[p(x) \log \frac{p(x)}{q(x)} + (1-p(x)) \log \frac{1 - p(x)}{1 - q(x)}\right].
\end{align*}

\begin{lemma}[Decomposing Loss] \label{network-agg:lem:pythagorean}
Let $p^*$ be the optimal logistic predictor on a feature set $S$, and let $q$ be any logistic predictor in $S$. The loss decomposes as
\[
L(q) = L(p^*) + D(p^* \| q).
\]
\end{lemma}
\begin{proof}
Using the identity $\log \sigma(z) = z - \log(1 + e^z)$, we write the loss with $z = \theta^\top x$ as
\begin{align*}
L(\theta) &= - \mathbb{E} \left[ yz - \log(1 + e^z) \right].
\end{align*}
Expanding the difference $L(\theta) - L(\theta^*)$, we obtain
\begin{align*}
L(\theta) - L(\theta^*) &= \mathbb{E} \left[ \log(1+e^{\theta^\top x}) - \log(1 + e^{(\theta^*)^\top x}) - y((\theta - \theta^*)^\top x) \right].
\end{align*}
Adding and subtracting $p^*(x)(\theta - \theta^*)^\top x$ inside the expectation yields
\begin{align}
\label{network-agg:eq:q-from-ps-loss}
L(\theta) - L(\theta^*) &= \mathbb{E} \left[ \log(1+e^{\theta^\top x}) - \log(1 + e^{(\theta^*)^\top x}) - p^*(x)(\theta - \theta^*)^\top x \right] \nonumber \\
&\quad + \mathbb{E} \left[ (p^*(x) - y)((\theta - \theta^*)^\top x) \right].
\end{align}
The second term is zero due to the orthogonality condition derived in Lemma~\ref{network-agg:lem:orthogonality}. For the first term, we expand the definition of $D(p^* \| q)$ with $z = \theta^\top x$ and $z^* = (\theta^*)^\top x$:
\begin{align*}
D(p^* \| q) &= \mathbb{E} \left[ p^* (z^* - z) - \log(1 + e^{z^*}) + \log(1 + e^z) \right] \\
&= \mathbb{E} \left[ \log(1 + e^{\theta^\top x}) - \log(1+e^{(\theta^*)^\top x}) - p^*(x)(\theta - \theta^*)^\top x \right].
\end{align*}
This matches the first term in Equation~\ref{network-agg:eq:q-from-ps-loss}, completing the proof.
\end{proof}

To relate the KL divergence to the parameter error, we employ the following bound, which is a specific case of Pinsker's inequality \cite{pinsker1964information}. We include the proof for completeness.

\begin{lemma} \label{network-agg:lem:kl-mse}
    For the expected KL divergence $D(p \| q)$, the following inequality holds:
    \begin{align*}
        D(p \|q) \ge 2\mathbb{E}\left[(p(x) - q(x))^2\right].
    \end{align*}
\end{lemma}
\begin{proof}
We verify the inequality pointwise for any $x \in \mathcal{X}$. We aim to show
\begin{align} \label{network-agg:eq:kl-mse:single}
    p(x) \log \frac{p(x)}{q(x)} + (1-p(x)) \log \frac{1 - p(x)}{1 - q(x)} \ge 2(p(x)-q(x))^2.
\end{align}
Define the function $f(p) = p \log \frac{p}{q} + (1-p) \log \frac{1 - p}{1 - q} - 2(p-q)^2$. The first derivative with respect to $p$ is
$
    f'(p) = \log \frac{p}{q} - \log \frac{1-p}{1-q} - 4(p-q).
$
The second derivative is $f''(p) = \frac{1}{p} + \frac{1}{1-p} - 4 = \frac{1}{p(1-p)} - 4$.
For $p \in [0,1]$, the term $p(1-p)$ has a maximum value of $0.25$. Consequently, $\frac{1}{p(1-p)} \ge 4$, which implies $f''(p) \ge 0$. Since $f$ is convex and satisfies $f'(q) = 0$, the point $p=q$ is a global minimum. Observing that $f(q) = 0$, we conclude that $f(p) \ge 0$ for all $p$. Taking the expectation of both sides in (\ref{network-agg:eq:kl-mse:single}) yields the result:
\begin{equation*}
    D(p \|q) = \mathbb{E}\left[p(x) \log \frac{p(x)}{q(x)} + (1-p(x)) \log \frac{1 - p(x)}{1 - q(x)}\right] \ge 2\mathbb{E}[(p(x)-q(x))^2]. \qedhere
\end{equation*}
\end{proof}

We define the pointwise loss function:
\begin{align*}
    l(z,y) = \log(1+e^z) - yz.
\end{align*}
Thus, we can write $L(p) = \mathbb{E}[l(z(x), y)]$.

\begin{lemma}
\label{network-agg:lem:loss-convexity}
Let $g(x) = \sigma(z_g(x))$ be any logistic predictor. Let $S$ be a subspace of features and $p(x) = \sigma(z_p(x))$ be the predictor that minimizes $L(p)$ over $S$. Then, we have
\begin{equation} \label{network-agg:eq:p-by-g}
    L(p) \le L(g) + |\mathbb{E}[(p-y)z_g]|.
\end{equation}
\end{lemma}
\begin{proof}
Let $\phi(z) = \log(1+e^z)$. The derivatives are $\phi'(z) = \sigma(z)$ and $\phi''(z) = \sigma(z)(1-\sigma(z))$. Since $\sigma(z) \in (0,1)$, we have $\phi''(z) \ge 0$, implying $\phi$ is convex. Convexity implies that for any $u, v \in \mathbb{R}$, $\phi(v) \ge \phi(u) + \phi'(u)(v-u)$. Rearranging implies the following:
\begin{equation} \label{network-agg:eq:phi-convex}
    \phi(u) \le \phi(v) + \sigma(u)(u-v).
\end{equation}
We define the relationship between the losses $l(u,y)$ and $l(v,y)$. Substituting $l(z,y) = \phi(z) - yz$, we assume the following inequality:
\begin{equation} \label{network-agg:eq:l-convex}
    l(u,y) \le l(v,y) + (\sigma(u)-y)(u-v).
\end{equation}
Expanding terms confirms this holds given the convexity of $\phi$ in Equation~\ref{network-agg:eq:phi-convex}:
\begin{align*}
    \phi(u) - yu &\le \phi(v) - yv + \sigma(u)(u-v) - yu + yv \\
    \iff \phi(u) &\le \phi(v) + \sigma(u)(u-v).
\end{align*}

Now, for a point $x$, let $u = z_p(x)$ and $v = z_g(x)$. Applying Equation~\ref{network-agg:eq:l-convex}, we get
\begin{align*}
    l(z_p, y) \le l(z_g,y) + (\sigma(z_p)-y)(z_p-z_g).
\end{align*}
Taking the expectation over $x$ gives
\begin{align*}
    L(p) \le L(g) + \mathbb{E}[(p-y)(z_p-z_g)] = L(g) + \mathbb{E}[(p-y)z_p] - \mathbb{E}[(p-y)z_g].
\end{align*}
From Lemma~\ref{network-agg:lem:orthogonality} (Orthogonality), we know that for any feature $x_l$ in the support of $p$, $\mathbb{E}[x_l(p-y)] = 0$. Since $z_p$ is a linear combination of such features, $\mathbb{E}[(p-y)z_p] = 0$. Substituting this yields
\begin{equation*}
    L(p) \le L(g) - \mathbb{E}[(p-y)z_g] \le L(g) + |\mathbb{E}[(p-y)z_g]|. \qedhere
\end{equation*}
\end{proof}

We consider a path of agents $A_1, \dots, A_D$. Each agent $i$ receives the logit $z_{i-1}$ from its predecessor and trains a logistic predictor model using locally observed features $x_{S_i}$, $z_{i-1}$, and possibly some other predecessors' logits. Since one option for the agent $A_i$ is to pass the logits $z_{i-1}$ through, we have that $L(p_{i-1}) \ge L(p_i)$. We also get that Lemma~\ref{network-agg:lem:pythagorean} holds for $p_{i-1}$ and $p_i$, since $p_{i-1}$ is in the stricter subspace of $p_i$.

We use the notation $\|f(x)\|_2 = \sqrt{\mathbb{E}\left[f(x)^2\right]}$ for any function $f$.

\begin{lemma}[Residual Bound via Path Coverage]
\label{network-agg:lem:residual-bound}
Let $A_1, \dots, A_k$ be a path of agents where every feature $x_l$ is observed at least once. Let $g(x) = \sigma(z_g(x))$ where $z_g(x) = \sum_{l=1}^d \alpha_lx_l$ be any logistic predictor over the whole space. Assume the coefficients of $z_g$ satisfy $\sum_{l=1}^d |\alpha_l| \le B_g$, and the features satisfy $\mathbb{E}[x_l^2] \le B_X^2$, for some $B_g$ and $B_X$. Let $\varepsilon \ge L(p_1) - L(p_k)$. Then, we have
\begin{align*}
    |\mathbb{E}[(p_k-y)z_g]| \le B_g B_X \sqrt{\frac{k\varepsilon}{2}}.
\end{align*}
\end{lemma}
\begin{proof}
Let $z_g(x) = \sum_{l=1}^d \alpha_l x_l$. We bound the error term:
\begin{align*}
    |\mathbb{E}[(p_k-y)z_g]| = \left|\mathbb{E}\left[\sum_{l=1}^d \alpha_l x_l(p_k-y)\right]\right| \le \sum_{l=1}^d |\alpha_l|\; |\mathbb{E}[x_l(p_k-y)]|.
\end{align*}
Consider a feature $x_l$. Due to each feature being observed, this feature appears in the index set of some agent $A_j$ in the path. By orthogonality, $\mathbb{E}[x_l(p_j-y)] = 0$. We decompose the expectation using the triangle inequality:
\begin{align*}
    |\mathbb{E}[x_l(p_k-y)]| \le |\mathbb{E}[x_l(p_k-p_j)]| + |\mathbb{E}[x_l(p_j-y)]| = |\mathbb{E}[x_l(p_k-p_j)]|.
\end{align*}
Applying the Cauchy-Schwarz inequality gives
\begin{align*}
    |\mathbb{E}[x_l(p_k-p_j)]| \le \sqrt{\mathbb{E}[x_l^2]} \sqrt{\mathbb{E}[(p_k-p_j)^2]} = \|x_l\|_2 \|p_k-p_j\|_2.
\end{align*}
Given $\|x_l\|_2 \le B_X$, we bound $\|p_k-p_j\|_2$ using the loss difference $\varepsilon$. Applying Lemma~\ref{network-agg:lem:kl-mse}, we get for any $s \in \{1,\dots,k-1\}$:
\begin{equation*}
\mathbb{E}\left[(p_{s} - p_{s+1})^2\right] \leq \frac{1}{2} D(p_{s+1} \| p_{s}).
\end{equation*}
By the triangle inequality and Cauchy-Schwarz,
\begin{align*}
\|p_j - p_k\|_2
&\leq \sum_{s=j}^{k-1} \sqrt{\frac{D(p_{s+1} \| p_{s})}{2}} \leq \sqrt{\frac{k \sum_{s=1}^{k-1} D(p_{s+1} \| p_{s})}{2}} \leq \sqrt{\frac{k \varepsilon}{2}}.
\end{align*}
Combining these bounds with the constraint on $\alpha_l$ gives
\begin{equation*}
    |\mathbb{E}[(p_k-y)z_g]| \le \sum_{l=1}^d |\alpha_l| \cdot |\mathbb{E}[x_l(p_k-y)]| \le  B_g B_X \sqrt{\frac{k\varepsilon}{2}}. \qedhere
\end{equation*}
\end{proof}

We are finally ready to prove Theorem~\ref{network-agg:thm:convergence}. Combining Lemma~\ref{network-agg:lem:loss-convexity} and Lemma~\ref{network-agg:lem:residual-bound}, we obtain the relationship $L(p_k) \le L(g) + B_g B_X \sqrt{k\varepsilon/2}$ for a path of length $k$. Extending this analysis over the full path satisfying the $M$-coverage condition leads to our main convergence result.

\begin{proof}[Proof of Theorem~\ref{network-agg:thm:convergence}]
We partition the path into $K = \lfloor D/M \rfloor$ disjoint blocks of length $M$. By the Pigeonhole Principle, since the total loss reduction is bounded by the loss of the first agent $L(p_1)$, there exists at least one \emph{stable} block $k^*$ where the reduction is at most the total reduction divided by $K$. Suppose this block $k^*$ is on indices $s, s+1,\dots, t$. Then, we have
\begin{equation*}
\sum_{i=s+1}^t \left( L(p_{i-1}) - L(p_i) \right) \leq \frac{L(p_1)}{K} \le \frac{2M L(p_1)}{D} := \varepsilon.
\end{equation*}
Applying Lemma~\ref{network-agg:lem:loss-convexity} and Lemma~\ref{network-agg:lem:residual-bound}, we get that over this path $L(p_t) \le L(p^*) + B_{p^*}B_X\sqrt{M\varepsilon/2}$. Next, note that $L(p_1) \le \log 2$ since using $\theta_1 = 0$ achieves a loss of $\log 2$, and because the first agent optimizes within its domain then $L(p_1) \le \log 2 < 1$. Combined with the non-increasing losses, we get
\begin{equation*}
    L(p_D) - L(p^*) \le B_{p^*}B_X \frac{M}{\sqrt{D}}. \qedhere
\end{equation*}
\end{proof}

Our theoretical analysis builds on the framework established by Kearns et al.~\cite{kearns2026networked} for linear regression. We extend their methods to the classification setting. Unlike their work, which relies on the geometry of Mean Squared Error (MSE), our analysis addresses the non-linearities of the sigmoid function and Binary Cross Entropy (BCE) loss. This necessitates the use of information-theoretic tools to bound the excess risk.

The proof strategy for Theorem~\ref{network-agg:thm:convergence} mirrors the methods of Kearns et al.~\cite{kearns2026networked}. Lemma \ref{network-agg:lem:orthogonality} (Orthogonality) and Lemma \ref{network-agg:lem:pythagorean} (Decomposition) establish the logistic equivalents of their linear regression results. However, since the logistic loss does not admit a simple Euclidean decomposition, we rely on KL-divergence and a Pinsker-type bound (Lemma \ref{network-agg:lem:kl-mse}) to relate the risk reduction to parameter error.

\section{Conclusion and Future Directions}

The diverse array of case studies presented in this manuscript demonstrates unequivocally that frontier AI models---specifically Gemini Deep Think and its advanced variants---have crossed a critical threshold. They are no longer merely tools for routine automation, data processing, or syntax formatting; they are now capable of acting as genuine, expert-level collaborators in mathematical and algorithmic discovery. Across theoretical computer science, economics, physics, and optimization, we have shown that LLMs can actively resolve open conjectures, tighten long-standing mathematical bounds, and identify obscure, cross-disciplinary theorems to bypass human roadblocks.

The value of the AI in these collaborations manifested in several distinct paradigms. In some instances, it acted as a cross-disciplinary bridge, retrieving theorems from distant mathematical domains (such as the Kirszbraun Extension Theorem) to resolve computational geometry roadblocks. In others, it served as a relentless adversarial reviewer, successfully identifying a fatal, deeply buried flaw regarding perfect versus statistical consistency in a state-of-the-art cryptography preprint. 

Crucially, however, these successes were not achieved autonomously. They required a tightly coupled human-AI workflow characterized by iterative refinement, strategic scaffolding, and rigorous verification---a process some authors have colloquially termed ``vibe-proving.'' 

\subsection{Common Themes and Problem Suitability}

Across these diverse case studies, clear themes emerge regarding the types of problems where this human-AI collaborative approach excels, and where it currently struggles.
\\\\
\textbf{Highly Suited Problem Classes:} The model is exceptionally effective on problems that can be decomposed into verifiable steps, require cross-disciplinary knowledge retrieval (e.g., bridging graph theory with continuous measure theory), or involve generating counterexamples to bounded conjectures. It also thrives in settings where the human can provide a strategic ``scaffold'' while the AI fills in tactical derivations, or where automated execution loops can rapidly test and prune algebraic hypotheses against numerical ground-truths (as in the cosmic strings derivation).
\\\\
\textbf{Less Suited Problem Classes:} One goal is to improve the model on problems requiring completely unconstrained, multi-page derivations, where intermediate steps cannot be easily verified or grounded. Problems that require establishing entirely novel mathematical frameworks from scratch, or those with extremely long, sparse reward horizons without intermediate feedback, exceed the autonomous capabilities of current models and require some amount of human orchestration.
\\\\
\textbf{Capabilities Needed for Expansion:} To further expand the range of scientific problems that can be tackled, future AI capabilities likely should evolve beyond natural language and standard code execution. Models will need enhanced capacities for long-horizon logical planning to maintain context over lengthy proofs. Furthermore, seamless, native integration with interactive theorem provers (discussed in Section 9.3) are important to autonomously verify the logical soundness of deep reasoning steps and mitigate hallucination bottlenecks.

\subsection{Understanding Current Limitations and Failure Modes}

Left unchecked, current models exhibit distinct failure modes that researchers must actively manage. Across our experiments, several recurring limitations emerged:

\begin{itemize}
    \item \textbf{Confirmation Bias:} As noted in the information theory case studies (Section \ref{sec:ck_conjecture}), models exhibit a strong tendency to support the hypothesis presented in a prompt. If tasked with proving a false conjecture, the AI will often attempt to bridge logical gaps with confident but ``hand-wavy'' arguments that do not withstand rigorous scrutiny. Neutral prompting (e.g., ``prove or refute'') is essential.
    
    \item \textbf{Confident Technical Hallucinations:} While models excel at high-level structural insights, they can occasionally make subtle algebraic errors, drop constraints, or confidently misapply theorems (e.g., flipping inequality signs in hypercontractivity bounds). 
    
    \item \textbf{Alignment Friction:} Standard safety and alignment guardrails can sometimes hinder scientific exploration. As noted in Section \ref{sec:techniques}, the model may initially refuse to attempt a problem if it recognizes it as an ``unsolved open problem” (requiring Context De-Identification to bypass).

\end{itemize}

Because of these limitations, the human researcher's role is elevated rather than replaced. The scientist shifts from executing mechanical derivations to acting as an orchestrator, auditor, and strategic director of the AI's combinatorial reasoning.

\subsection{Future Directions: From Code Execution to Formal Verification}

To overcome the limitations of LLM hallucinations, researchers must integrate pure language models with external verification environments. As outlined in Section 2.6 and demonstrated in our cosmic strings experiments (Section 6.4), we are already seeing success by embedding AI in ``neuro-symbolic'' loops---where the model autonomously writes code to numerically verify its proposed mathematical steps and uses traceback errors to prune invalid branches. 

However, while numerical execution is a powerful grounding mechanism for applied mathematics and physics, it is fundamentally limited when dealing with abstract proofs. For pure mathematics and theoretical computer science, the natural evolution of this workflow is \textbf{Formal Verification}. As AI systems generate increasingly complex, multi-page mathematical proofs, human verification becomes an exhausting bottleneck. Future research must focus on building \textit{autoformalization} pipelines that automatically translate LLM-generated informal mathematics into formal verification languages (such as Lean, Coq, or Isabelle). By pairing the creative, associative leaps of an LLM with the absolute rigorous certainty of an interactive theorem prover, the research community can systematically eliminate the hallucination problem in mathematical discovery.

\subsection{The Shifting Bottleneck: An Impending Crisis in Peer Review}

As AI drastically lowers the friction of generating highly technical, mathematically dense research papers---evidenced by the AI-integrated IDE workflow utilized in Section \ref{sec:lance} --- the scientific community faces an impending systemic challenge. If researchers can ``vibe-code'' comprehensive papers in a fraction of the traditional time, the fundamental bottleneck of science will shift entirely from the \textit{generation} of ideas to the \textit{verification} of those ideas.

The traditional human peer-review system is already strained and  fundamentally unequipped to handle a massive influx of AI-accelerated literature. Consequently, the very same tools used to generate these papers must be adapted to evaluate them. The cryptography case study (Section 3.2) proves that AI models, when guided by rigorous, adversarial self-correction protocols, are already capable of finding subtle flaws in advanced proofs. Developing robust, AI-assisted peer-review systems will be vital to triaging submissions and maintaining the integrity of the scientific literature in the coming years.

\subsection{Final Thoughts}

Just as the advent of calculators and computational algebra systems revolutionized applied mathematics in previous decades, the ability to rapidly iterate on abstract reasoning with a tireless, knowledgeable AI collaborator promises to dramatically reduce the friction of theoretical execution. 

Ultimately, the premise of this work is not just that AI can help solve specific, isolated research problems, but that it transforms \emph{how} we do research. The implications for scientists are important: the day-to-day workflow of theoretical research will likely shift away from mechanical derivations and exhaustive literature hunting, moving instead toward high-level orchestration, hypothesis generation, and rigorous verification. By acting as a collaborative sounding board, AI lowers the barrier to entry for exploring complex, interdisciplinary ideas. This shift promises to empower a broader diversity of researchers, allowing them to tackle bigger, more ambitious problems than they could alone. By embracing this collaborative paradigm, understanding its failure modes, and building new automated verification pipelines, researchers can explore broader hypothesis spaces and ultimately accelerate the pace of scientific discovery.

\section*{Acknowledgments}
The authors thank the Gemini team for access to early models and technical support. We also thank the support of Corinna Cortes, Jeff Dean, Demis Hassabis, and Koray Kavukcuoglu. We also acknowledge members of the Deep Think team that are not authors: Garrett Bingham, Irene Cai, Heng-Tze Cheng, Yong Cheng, Kristen Chiafullo, Paul Covington, Golnaz Ghiasi, Chenjie Gu, Huan Gui, Ana Hosseini, Dawsen Hwang, Vihan Jain, Ragha Kotikalapudi, Chenkai Kuang, Maciej Kula, Nate Kushman, Jane Labanowski, Quoc Le, Jonathan Lee, Zhaoqi Leng, Steve Li, YaGuang Li, Hanzhao (Maggie) Lin, Evan Liu, Yuan Liu, Thang Luong, Pol Moreno, Nigamaa Nayakanti, Aroonalok Pyne, Shubha Raghvendra, Sashank Reddi, Nikunj Saunshi, Siamak Shakeri, Archit Sharma, Xinying Song, Qijun Tan, Yi Tay, Trieu Trinh, Theophane Weber, Winnie Xu, Zicheng Xu, Shunyu Yao, Lijun Yu, Hao Zhou, and Honglei Zhuang.



\begin{thebibliography}{99}

\bibitem{anantharam2017conjecture}
Anantharam, V., Bogdanov, A., Chakrabarti, A., Jayram, T. S., and Nair, C. (2017). A conjecture regarding optimality of the dictator function under Hellinger distance. \textit{In Information Theory and Applications Workshop.}

\bibitem{az} Anari, N., and Rezaei, A., 
A tight analysis of Bethe approximation for permanent. 
\textit{SIAM Journal on Computing} {\bf 54} (4), 81--101

\bibitem{Arora_1998}
Arora, Sanjeev.
\newblock {Polynomial time approximation schemes for Euclidean traveling salesman and other geometric problems}.
\newblock \emph{Journal of the ACM}, 45(5):753--782, 1998.

\bibitem{arora1995polynomial}
Arora, Sanjeev and Karger, David and Karpinski, Marek.
\newblock Polynomial time approximation schemes for dense instances of NP-hard problems.
\newblock \emph{Proceedings of the twenty-seventh annual ACM symposium on Theory of computing}, pages 284--293, 1995.

\bibitem{avidor2005rounding}
Avidor, Adi and Zwick, Uri.
\newblock Rounding two and three dimensional solutions of the SDP relaxation of MAX CUT.
\newblock \emph{International Workshop on Approximation Algorithms for Combinatorial Optimization}, pages 14--25, 2005.

\bibitem{2509.06503}
Ayg\"{u}n, Eser et al.
\newblock An AI system to help scientists write expert-level empirical software.
\newblock \textit{arXiv preprint arXiv:2509.06503}, 2025.

\bibitem{BIJ+23}
Bakshi, A., Indyk, P., Jayaram, R., Silwal, S., Waingarten, E. (2023). A near-linear time algorithm for the chamfer distance. \textit{NeurIPS 2023}.

\bibitem{bansal2024}
Bansal, Nikhil and Cohen-Addad, Vincent and Prabhu, Milind and Saulpic, David and Schwiegelshohn, Chris
\newblock Sensitivity Sampling for $k$-Means: Worst Case and Stability Optimal Coreset Bounds. 
\newblock \textit{Proceedings of the 65th IEEE Annual Symposium on Foundations of Computer Science (FOCS 2024)}, pages 1707--1723, 2024.

\bibitem{barak2011rounding}
Barak, Boaz and Raghavendra, Prasad and Steurer, David.
\newblock Rounding semidefinite programming hierarchies via global correlation.
\newblock \emph{2011 ieee 52nd annual symposium on foundations of computer science}, pages 472--481, 2011.

\bibitem{barnes2020courtade}
Barnes, L. P., and {\"O}zg{\"u}r, A. (2020, June). The Courtade-Kumar most informative Boolean function conjecture and a symmetrized Li-M{\'e}dard conjecture are equivalent. \textit{IEEE International Symposium on Information Theory (ISIT)}, 2205-2209.

\bibitem{graph_notions_survey}
B{o}ttcher, Julia and Pruessmann, Klaas P. and Taraz, Anusch and W\"{u}rfl, Andreas.
\newblock Bandwidth, expansion, treewidth, separators and universality for bounded-degree graphs.
\newblock \emph{European Journal of Combinatorics}, 31(5):1217--1227, 2010.

\bibitem{branzei2022query}
Br\^anzei, Simina and Li, Jiawei.
\newblock The Query Complexity of Local Search and Brouwer in Rounds.
\newblock COLT, 2022. In Mathematical Statistics and Learning, forthcoming.

\bibitem{BPP26}
Simina Br\^anzei, Ioannis Panageas, and Dimitris Paparas.
\newblock The Query Complexity of Local Search in Rounds on General Graphs.
\newblock arXiv:2601.13266 [cs.CC], 2026.
\newblock \url{https://arxiv.org/abs/2601.13266}.

\bibitem{brennerFollowup}
M.~P. Brenner, V.~Cohen-Addad, and D.~Woodruff, 
``Solving an Open Problem in Theoretical Physics using AI-Assisted Discovery,'' 
2026, arXiv:2603.04735 [cs.AI]. 
\url{https://arxiv.org/abs/2603.04735}

\bibitem{briet2010positive}
Bri{e}t, Jop and de Oliveira Filho, Fernando M\'{a}rio and Vallentin, Frank.
\newblock The positive semidefinite Grothendieck problem with rank constraint.
\newblock \emph{International Colloquium on Automata, Languages, and Programming}, pages 31--42, 2010.

\bibitem{briet2014grothendieck}
Bri{e}t, Jop and de Oliveira Filho, Fernando M\'{a}rio and Vallentin, Frank.
\newblock Grothendieck Inequalities for Semidefinite Programs with Rank Constraint.
\newblock \emph{Theory of Computing}, 10(1):77--105, 2014.

\bibitem{OpenAI2025}
Bubeck, S\'ebastien and Coester, Christian and Eldan, Ronen and Gowers, Timothy and Lee, Yin Tat and Lupsasca, Alexandru and Sawhney, Mehtaab and Scherrer, Robert and Sellke, Mark and Spears, Brian K et al. 
\newblock Early science acceleration experiments with GPT-5.
\newblock \textit{arXiv preprint arXiv:2511.16072}, 2025.

\bibitem{Cai}
Cai, J.-Y.
\newblock 
$S_2^P \subseteq \mathrm{ZPP}^{\mathrm{NP}}$. 
\newblock \emph{Journal of Computer and System Sciences}, 73(1):25--35, 2007.

\bibitem{Canetti96}
Canetti, R.
\newblock More on BPP and the polynomial-time hierarchy.
\newblock \emph{Information Processing Letters}, 57(5):237--241, 1996.

\bibitem{chandar2014most}
Chandar, Venkat and Tchamkerten, Aslan.
\newblock Most informative quantization functions.
\newblock \emph{Proc. ITA Workshop, San Diego, CA, USA}, 2014.

\bibitem{chang2017dropout}
Chang, Chun-Hao and Rampasek, Ladislav and Goldenberg, Anna.
\newblock Dropout feature ranking for deep learning models.
\newblock \emph{arXiv preprint arXiv:1712.08645}, 2017.

\bibitem{Chierichetti26}
F. Chierichetti, M. Giacchini, R. Kumar, S. Lattanzi, A. Panconesi, E. Tani, A. Tomkins. Beyond the Full Slate: Evaluating MNL Algorithms on All Slates. \textit{Private Communication.}

\bibitem{cohen2017input}
M.~B.~Cohen, C.~Musco, and C.~Musco,
\newblock {Input sparsity time low-rank approximation via ridge leverage score sampling},
\newblock \emph{Proceedings of the Twenty-Eighth Annual ACM-SIAM Symposium on Discrete Algorithms (SODA)}, pp.~1758--1777, 2017.

\bibitem{CW25}
Cohen-Addad, V. and Woodruff, David P.
\newblock Google Research Blog post, 2025.
\url{https://research.google/blog/gemini-provides-automated-feedback-for-theoretical-computer-scientists-at-stoc-2026/}

\bibitem{courtade2014boolean}
Courtade, T. A., Kumar, G. R. (2014). Which Boolean functions maximize mutual information on noisy inputs? \textit{IEEE Transactions on Information Theory}, 60(8), 4515--4525.

\bibitem{PC} P. Csikv\'ari. 
Lower matching conjecture, and a new proof of Schrijver's and Gurvits's theorems.
\textit{Journal of the European Mathematical Society} (EMS Publishing) 19 (6), 2017.

\bibitem{csirmaz2015}
Csirmaz, L., Ligeti, P., Tardos, G. (2015). Erdős-Pyber theorem for hypergraphs and secret sharing. \textit{Graphs Combin.}, 31(5):1335--1346.

\bibitem{Parameterized_algos_book}
Cygan, M. and Fomin, F. V. and Kowalik, L. and Lokshtanov, D. and Marx, D. and Pilipczuk, M. and Pilipczuk, M. and Saurabh, S.
\newblock Parameterized Algorithms.
\newblock Springer, 2015.

\bibitem{debreu1954} Gerard Debreu. ``Representation of a Preference Ordering by a Numerical Function''. Decision processes 3 (1954): 159-165.

\bibitem{Du_Smith_1996}
Du, Ding-Zhu and Smith, Warren D.
\newblock Disproofs of Generalized {Gilbert Pollak} Conjecture on the {Steiner} Ratio in Three or More Dimensions.
\newblock \emph{Journal of Combinatorial Theory, Series A}, 74(1):115--130, 1996.

\bibitem{Duede}
Duede, E.
\newblock Tail Novelty, Knowledge Collapse, and Useful Frictions in Science.
\newblock Computer Science Seminar, Illinois Institute of Technology, November 17, 2025. \href{https://www.iit.edu/events/computer-science-seminar-eamon-duede-tail-novelty-knowledge-collapse-and-useful-frictions-science}{https://bit.ly/iit-cs-seminar-duede}

\bibitem{dutting2024} Paul D\"utting, Vahab Mirrokni, Renato Paes Leme, Haifeng Xu, and Song Zuo. ``Mechanism Design for Large Language Models''. In Proceedings of the ACM Web Conference 2024 (pp. 144-155).

\bibitem{Egorychev1981}
G.P. Egorychev. Proof of the van der Waerden conjecture for permanents [in Russian]. \textit{Sibirskii Matematicheskii Zhurnal}, 22:6, 65–71, 1981.

\bibitem{Falikman1981}
D.I. Falikman. Proof of the van der Waerden conjecture regarding the permanent of a doubly stochastic matrix [in Russian]. \textit{Matematicheskie Zametki}, 29:931–938, 1981.

\bibitem{F21}
Fand, M. (2021). Fast Johnson-Lindenstrauss Transform. \textit{arXiv preprint arXiv:2104.05695}.

\bibitem{feige2002improved}
Feige, Uriel and Karpinski, Marek and Langberg, Michael.
\newblock Improved approximation of Max-Cut on graphs of bounded degree.
\newblock \emph{Journal of Algorithms}, 43(2):201--219, 2002.

\bibitem{F26Erdos}
T. Feng, T. Trinh, G. Bingham, J. Kang, S. Zhang, et al., 
\textit{Semi-Autonomous Mathematics Discovery with Gemini: A Case Study on the Erd\H{o}s Problems}, 
arXiv preprint arXiv:2601.22401 (2026). 
Available at: \url{https://arxiv.org/abs/2601.22401}.

\bibitem{F26}
T. Feng, T.~H. Trinh, G. Bingham, D. Hwang, Y. Chervonyi, \textit{et al.}, 
\textit{Towards Autonomous Mathematics Research}, 
arXiv preprint arXiv:2408.00000 (2026). 
Available at: \url{https://storage.googleapis.com/deepmind-media/DeepMind.com/Blog/towards-autonomous-mathematics-research/Towards_Autonomous_Mathematics_Research.pdf}

\bibitem{FI25}
Feng, Y., Indyk, P. (2025). Even Faster Algorithm for the Chamfer Distance. \textit{ICALP 2025}.

\bibitem{fw26}
Feng, Y., Woodruff, David P. (2026). Fast Approximate Lp Chamfer Distance via Lopsided Embeddings and Structured JL. \textit{In submission 2026}. 

\bibitem{Fleischmann2025}
Fleischmann, H., Gamboa Quintero, G., Karthik C. S., Mat\v{e}jka, J., Petr, J. (2025). On Steiner Trees of the regular simplex. \textit{Journal of Computational Geometry}, 16(1), 1--34.

\bibitem{Fleischmann2023}
Fleischmann, Henry and Gavva, Surya Teja and Karthik C. S..
\newblock On Approximability of {Steiner} Tree in $\ell_p$-metrics.
\newblock \emph{TheoretiCS}, 4, 2025.

\bibitem{Fortnow25}
Fortnow, L.
\newblock Search versus Decision for $\mathcal{S}_2^{\mathcal{P}}$. 
\newblock arXiv preprint arXiv:2512.02808, 2025. \url{https://arxiv.org/abs/2512.02808}

\bibitem{S2PRepo}
L.\ Fortnow.
\newblock \href{https://github.com/fortnow/s2psearch}{\texttt{fortnow/s2psearch}: Files related to vibe-coding the “Search vs Decision for S2P” paper}.
\newblock GitHub repository, 2025. \url{https://github.com/fortnow/s2psearch}

\bibitem{wu2023tree}
Yao Fu, et al.
\newblock Tree of Thoughts: Deliberate Problem Solving with Large Language Models.
\newblock \emph{arXiv preprint arXiv:2305.10601}, 2023.

\bibitem{GanorSP21}
Ganor, Anat and {Karthik {C. S.}} and P{\'{a}}lv{\"{o}}lgyi, D{\"{o}}m{\"{o}}t{\"{o}}r.
\newblock On Communication Complexity of Fixed Point Computation.
\newblock \emph{{ACM} Trans. Economics and Comput.}, 9(4):25:1--25:27, 2021.

\bibitem{GargKRS20}
Paritosh Garg, Sagar Kale, Lars Rohwedder, and Ola Svensson.
Robust Algorithms Under Adversarial Injections.
\textit{Proceedings of the 47th International Colloquium on Automata,
Languages, and Programming (ICALP)}, 2020, pages 56:1--56:15. 

\bibitem{Deep Think}
Gemini Deep Think. \url{https://blog.google/products/gemini/gemini-2-5-deep-think/}.

\bibitem{georgiev2025mathematical}
B.~Georgiev, J.~G{\'o}mez-Serrano, T.~Tao, and A.~Z.~Wagner,
\newblock {Mathematical exploration and discovery at scale},
\newblock \emph{arXiv preprint arXiv:2511.02864}, 2025.

\bibitem{GhashamiLPW16}
Ghashami, M., Liberty, E., Phillips, J. M., Woodruff, D. P. (2016). Frequent directions: Simple and deterministic matrix sketching. \textit{SIAM Journal on Computing}, 45(5), 1762--1792.

\bibitem{gilpol}
Gilbert, Edgar N. and Pollak, Henry O.
\newblock Steiner Minimal Trees.
\newblock \emph{SIAM Journal on Applied Mathematics}, 16(1):1--29, 1968.

\bibitem{goemans1995improved}
Goemans, Michel X and Williamson, David P.
\newblock Improved approximation algorithms for maximum cut and satisfiability problems using semidefinite programming.
\newblock \emph{Journal of the ACM (JACM)}, 42(6):1115--1145, 1995.

\bibitem{YogevGuan2025}
Ziyi Guan and Eylon Yogev.
\textit{SNARGs for NP from LWE}.
Cryptology {ePrint} Archive, Paper 2025/2328, 2025.

\bibitem{GU1}
L. Gurvits. Van der Waerden/Schrijver-Valiant like conjectures and stable (aka hyperbolic) homogeneous polynomials: one theorem for all. \textit{Electron. J. Combin.}, 15(1), Research Paper 66, 2008.

\bibitem{Gurvits2011}
L. Gurvits. Unleashing the power of Schrijver's permanental inequality with the help of the Bethe Approximation. \textit{ArXiv preprint} 1106.2844v11, 2011.

\bibitem{gurvits2}	L. Gurvits.
Boolean matrices with prescribed row/column sums and stable homogeneous polynomials: Combinatorial and algorithmic applications. {\it Inf. Comput.} {\bf 240} 42-55 (2015)

\bibitem{gurvits3}	L. Gurvits, A. Samorodnitsky:
Bounds on the Permanent and Some Applications. \textit{FOCS} 2014: 90-99


\bibitem{HarveyNO08}
Harvey, Nicholas J. A. and Nelson, Jelani and Onak, Krzysztof.
\newblock Sketching and streaming entropy via approximation theory.
\newblock \textit{FOCS 2008}, 489--498.

\bibitem{hsieh2023approximating}
Hsieh, Jun Ting and Kothari, Pravesh K.
\newblock Approximating Max-Cut on Bounded Degree Graphs: Tighter Analysis of the {FKL} Algorithm.
\newblock \emph{50th International Colloquium on Automata, Languages, and Programming, ICALP 2023}, page 77, 2023.

\bibitem{robust_coreset:icml2025}
In, Chansophea Wathanak and Li, Yi and Woodruff, David and Wu, Xuan. 
\newblock Robust Sparsification via Sensitivity. 
\newblock \emph{Proceedings of ICML 2025}, pages 26446--26463, 2025.

\bibitem{Ivanov_Tuzhilin_2012}
Ivanov, Alexander O. and Tuzhilin, Alexey A.
\newblock The {Steiner} Ratio {Gilbert Pollak} Conjecture Is Still Open: Clarification Statement.
\newblock \emph{Algorithmica}, 62(1-2):630--632, 2012.

\bibitem{javanmard2026progress}
Javanmard, A., and Woodruff, D. P. (2026). Progress on the Courtade-Kumar Conjecture: Optimal High-Noise Entropy Bounds and Generalized Coordinate-wise Mutual Information. \textit{arXiv preprint arXiv:2601.09679.}

\bibitem{JayaramWZ24}
Jayaram, R., Woodruff, D. P., Zhou, S. (2024). Streaming Algorithms with Few State Changes. \textit{Proc. ACM Manag. Data}, 2(2), 82.

\bibitem{kearns2026networked}
M. Kearns, A. Roth, and E. Ryu.
Networked information aggregation via machine learning.
\textit{Proceedings of the Annual ACM-SIAM Symposium on Discrete Algorithms (SODA)}, 2026,
4799--4845.

\bibitem{Kes}
H. Kesten. Symmetric random walks on groups, \textit{Trans. Am. Math. Soc.} 92 (1959),
336–354. 

\bibitem{khot2007optimal}
Khot, Subhash and Kindler, Guy and Mossel, Elchanan and O'Donnell, Ryan.
\newblock Optimal inapproximability results for MAX-CUT and other 2-variable CSPs?
\newblock \emph{SIAM Journal on Computing}, 37(1):319--357, 2007.

\bibitem{Kirszbraun1934}
Kirszbraun, M.
\newblock {"U}ber die zusammenziehende und Lipschitzsche Transformationen.
\newblock \emph{Fundamenta Mathematicae}, 22(1):77--108, 1934.

\bibitem{korula2015online}
Korula, N., Mirrokni, V., Zadimoghaddam, M. (2015). Online submodular welfare maximization: Greedy beats 1/2 in random order. \textit{STOC 2015}.

\bibitem{krapivin2025}
Krapivin, A., Przybocki, B., Sanhueza-Matamala, N., Subercaseaux, B. (2025). Optimal and efficient partite decompositions of hypergraphs. \textit{arXiv preprint arXiv:2511.11855}.

\bibitem{li2020boolean}
Li, J., and M{\'e}dard, M. (2020). Boolean functions: noise stability, non-interactive correlation distillation, and mutual information. \textit{IEEE Transactions on Information Theory}, 67(2), 778-789.

\bibitem{lovasz}
L. Lovász, M. Plummer. \textit{Matching Theory}. AMS Chelsea Publishing Series, 2009.

\bibitem{MSS1} A. W. Marcus, D. A. Spielman and N. Srivastava, Interlacing families I: Bipartite Ramanujan graphs of all degrees, {\it Annals of Mathematics} {\bf 182} (2015), 307–325.


 
\bibitem{MSS2} 
A. W. Marcus, D. A. Spielman and N. Srivastava, Interlacing families II: Mixed characteristic polynomials and the Kadison—Singer problem, {\it Annals of Mathematics} {\bf 182} (2015), 327–350.

\bibitem{GeminiDeepThinkIMO}
Thang Luong and Edward Lockhart.
\newblock Advanced version of Gemini with Deep Think officially achieves gold-medal standard at the International Mathematical Olympiad.
\newblock \emph{Google DeepMind Blog}, July 21, 2025.

\bibitem{mc}
B. D. McKay. The expected eigenvalue distribution of a large regular graph, \textit{Linear
Algebra Appl.} 40 (1981), 203–216.

\bibitem{M99}
Mitchell, Joseph S. B.
\newblock Guillotine Subdivisions Approximate Polygonal Subdivisions: A Simple Polynomial-Time Approximation Scheme for Geometric {TSP},  $k$-{MST}, and Related Problems.
\newblock \emph{SIAM Journal on Computing}, 28(4):1298--1309, 1999.

\bibitem{Nagda2025}
Nagda, Ansh and Raghavan, Prabhakar and Thakurta, Abhradeep.
\newblock Reinforced Generation of Combinatorial Structures: Hardness of Approximation.
\newblock \textit{arXiv preprint arXiv:2509.18057}, 2025.

\bibitem{AB}
A. Nilli. On the second eigenvalue of a graph, \textit{Discrete Mathematics}, {\bf 91} (1991), 207–210.

\bibitem{NVE25}
Novikov, Alexander, V\~{u}, Ng\^{a}n and Eisenberger, Marvin and Dupont, Emilien and Huang, Po-Sen and Wagner, Adam Zsolt and Shirobokov, Sergey and Kozlovskii, Borislav and Ruiz, Francisco JR and Mehrabian, Abbas et al. 
\newblock AlphaEvolve: A coding agent for scientific and algorithmic discovery.
\newblock \textit{arXiv preprint arXiv:2506.13131}, 2025.

\bibitem{ordentlich2016improved}
Ordentlich, Or and Shayevitz, Ofer and Weinstein, Omri.
\newblock An improved upper bound for the most informative Boolean function conjecture.
\newblock \emph{2016 IEEE International Symposium on Information Theory (ISIT)}, pages 500--504, 2016.

\bibitem{pinsker1964information}
M. S. Pinsker.
\textit{Information and Information Stability of Random Variables and Processes}.
Holden--Day, 1964.

\bibitem{raghavendra2008optimal}
Raghavendra, Prasad.
\newblock Optimal algorithms and inapproximability results for every CSP?
\newblock \emph{Proceedings of the fortieth annual ACM symposium on Theory of computing}, pages 245--254, 2008.

\bibitem{RussellS98}
Russell, A. and Sundaram, R.
\newblock Symmetric alternation captures BPP.
\newblock \emph{Computational Complexity}, 7(2):152--162, 1998.

\bibitem{samorodnitsky2016entropy}
Samorodnitsky, A. (2016). On the entropy of a noisy function. \textit{IEEE Transactions on Information Theory}, 62(10), 5446-5464.

\bibitem{santha2004quantum}
Santha, Miklos and Szegedy, Mario.
\newblock Quantum and classical query complexities of local search are polynomially related.
\newblock \emph{Proceedings of the thirty-sixth annual ACM symposium on Theory of computing}, pages 494--501, 2004.

\bibitem{schmeidler1971}
David Schmeidler. ``A Condition for the Completeness of Partial Preference Relations''. In: Econometrica: Journal of the Econometric Society (1971), pp. 403--404.

\bibitem{Schrijver1998}
A. Schrijver. Counting 1-Factors in Regular Bipartite Graphs. \textit{Journal of Combinatorial Theory, Series B}, 72:122–135, 1998.

\bibitem{Schrijver1980}
A. Schrijver, W.G. Valiant. On lower bounds for permanents. \textit{Indagationes Mathematicae}, 42:425–427, 1980.

\bibitem{Schwartz1969NonlinearFA}
Schwartz, Jacob T.
\newblock \emph{Nonlinear Functional Analysis}.
\newblock Gordon and Breach Science Publishers, New York, 1969.

\bibitem{sellke2025learning}
M.~Sellke and S.~Yin,
\newblock {On Learning-Curve Monotonicity for Maximum Likelihood Estimators},
\newblock \emph{arXiv preprint arXiv:2512.10220}, 2025.

\bibitem{smith}
Smith, Warren D.
\newblock How to find {Steiner} minimal trees in {Euclidean} $d$-space.
\newblock \emph{Algorithmica}, 7:137--177, 1992.

\bibitem{sothanaphan2026resolution}
N.~Sothanaphan,
\newblock {Resolution of Erd\H{o}s Problem~\#728: a writeup of Aristotle's Lean proof},
\newblock \emph{arXiv preprint arXiv:2601.07421}, 2026.

\bibitem{Trevisan00}
Trevisan, Luca.
\newblock When {Hamming} Meets {Euclid}: The Approximability of Geometric {TSP} and {Steiner} Tree.
\newblock \emph{{SIAM} J. Comput.}, 30(2):475--485, 2000.

\bibitem{Valentine1945}
Valentine, F. A.
\newblock A Lipschitz condition preserving extension for a vector function.
\newblock \emph{American Journal of Mathematics}, 67:83--93, 1945.

\bibitem{Valiant1979}
L. Valiant. The complexity of computing the permanent. \textit{Theoretical Computer Science}, 8(2):189–201, 1979.

\bibitem{Vontobel} P. O. Vontobel. “The Bethe Permanent of a Non-Negative Matrix,” \textit{IEEE Transactions on Information Theory}, vol. 59, no. 3, pp. 1866-1901, 2013. (Also arXiv:1107.4196).

\bibitem{voorhoeve}
M. Voorhoeve. A lower bound for the permanents of certain $(0, 1)$-matrices.
\textit{Nederl. Akad. Wetensch. Indag. Math.}, 41(1):83–86, 1979.

\bibitem{WoodruffZ26}
Woodruff, David P. and Zhou, Samson. 
\newblock Consistent Low-Rank Approximation. 
\newblock \textit{The Fourteenth International Conference on Learning Representations} (ICLR) 2026.
\end{thebibliography}
\end{document}